\documentclass[runningheads,a4paper]{llncs}

\usepackage[firstpage]{draftwatermark}
\usepackage{wrapfig}
\usepackage{graphicx}
\usepackage{paralist,enumitem,xspace}
\usepackage{amssymb,amsmath}
\usepackage{amsopn}
\usepackage{pgfcore}
\usepackage{mathrsfs}
\usepackage{stmaryrd}
\usepackage{braket}
\usepackage{booktabs}
\usepackage{svg}
\usepackage{floatrow}
\usepackage[margin=0pt,font=footnotesize,skip=0pt]{caption}

\usepackage[linesnumbered,ruled,vlined]{algorithm2e}
\usepackage[export]{adjustbox}

\usepackage[english]{babel}

\usepackage[%
rm={oldstyle=false,proportional=true},%
sf={oldstyle=false,proportional=true},%
tt={oldstyle=false,proportional=true,variable=true},%
qt=false%
]{cfr-lm}
\usepackage{graphicx}

\usepackage{paralist}

\usepackage{cite}

\usepackage[T1]{fontenc}
\usepackage{lmodern}

\usepackage[math]{blindtext}

\usepackage{csquotes}

\usepackage{microtype}

\usepackage{url}


\usepackage{xcolor}

\usepackage[
bookmarks=false,
breaklinks=true,
colorlinks=true,
linkcolor=black,
citecolor=black,
urlcolor=black,
pdfpagelayout=SinglePage,
pdfstartview=Fit
]{hyperref}
\usepackage[all]{hypcap}

\usepackage{pdfcomment}

\usepackage[capitalise,nameinlink]{cleveref}
\crefname{section}{Sect.}{Sect.}
\Crefname{section}{Section}{Sections}

\usepackage{xspace}
\newcommand{\eg}{e.\,g.,\xspace}
\newcommand{\ie}{i.\,e.,\xspace}

\DeclareFontFamily{U}{MnSymbolC}{}
\DeclareSymbolFont{MnSyC}{U}{MnSymbolC}{m}{n}
\DeclareFontShape{U}{MnSymbolC}{m}{n}
{<-6>  MnSymbolC5
  <6-7>  MnSymbolC6
  <7-8>  MnSymbolC7
  <8-9>  MnSymbolC8
  <9-10> MnSymbolC9
  <10-12> MnSymbolC10
  <12->   MnSymbolC12%
}{}
\DeclareMathSymbol{\powerset}{\mathord}{MnSyC}{180}
\usepackage{tikz}
\usepackage{pgfplots}
\usetikzlibrary{arrows,backgrounds,decorations,decorations.pathmorphing,positioning,fit,automata,shapes,snakes,patterns,plotmarks,calc}
\setsvg{inkscape = inkscape -z -D}
\usepackage{subfig}

\newcommand{\f}{\varphi}

\newcommand{\MTL}{MTL\xspace}

\newcommand{\Intvl}{I}
\newcommand{\sPara}{c}
\newcommand{\tPara}{\tau}
\newcommand{\val}{\nu}
\newcommand{\x}{\mathbf{x}}
\newcommand{\G}{\mathbf{G}}

\newcommand{\F}{\mathbf{F}}
\newcommand{\ev}{\mathbf{F}}
\newcommand{\U}{\mathbf{U}}
\newcommand{\true}{\mathit{true}}
\newcommand{\nmodels}{\not\models}
\newcommand{\aand}{\,\wedge\,}

\newcommand{\Reals}{\mathbb{R}}

\newcommand{\PosReals}{\mathbb{R}_{\ge0}}
\newcommand{\domain}{\mathcal{D}}
\newcommand{\setof}[1]{\left\{#1\right\}}

\newcommand{\params}{\mathcal{P}}
\newcommand{\timedomain}{T}
\newcommand{\valuedomain}{V}
\newcommand{\paramspace}{{\domain_\params}}
\newcommand{\timeparams}{\params^\timedomain}
\newcommand{\valueparams}{\params^\valuedomain}
\newcommand{\p}{\mathbf{p}}
\newcommand{\Traces}{X}
\newcommand{\TracesSubset}{Y}
\newcommand{\validitydomain}[1]{\mathcal{V} (#1)}
\newcommand{\validitydomainboundary}[1]{\partial\mathcal{V} (#1)}

\newcommand{\mypara}[1]{\vspace{0.3em} \noindent {\bf #1.\ }}
\newcommand{\myipara}[1]{\vspace{0.3em} \noindent {\em #1.\/ }}

\renewcommand{\qed}{$\blacksquare$}

\DeclareMathOperator*{\argmin}{argmin}

\newcommand{\pOrder}{\trianglelefteq}
\newcommand{\tOrder}{\preceq}
\newcommand{\lexOrder}{\preceq_\mathrm{lex}}
\newcommand{\scalarOrder}{\preceq_\mathrm{scalar}}
\newcommand{\polarity}{\mathrm{sgn}}
\newcommand{\eqdef}{\mathrel{\stackrel{\makebox[0pt]{\mbox{\normalfont\tiny def}}}{=}}}

\newcommand{\stepStl}{\text{step}}
\newcommand{\overshootSTL}{\f_{\text{overshoot}}}
\newcommand{\cost}{J}

\newcommand{\proj}{\pi}
\newcommand{\projscalar}{\proj_{\mathrm{scalar}}}
\newcommand{\projlex}{\proj_{\mathrm{lex}}}
\newcommand{\inflex}{\inf_{\lexOrder}}

\newcommand{\assign}{\leftarrow}

\newcommand{\valup}{\val^{u}}
\newcommand{\valdn}{\val^{\ell}}
\newcommand{\ceil}[1]{\left\lceil#1 \right\rceil}

\newcommand{\lanei} {L_i}

\newcommand{\aaboxes}{\mathcal{B}}

\newcommand{\downward}{D}

\newcommand{\model}[1]{\llbracket\f({#1})\rrbracket}

\newcommand{\labelingfn}{\ell}
\newcommand{\labelingfnbox}{\ell_\mathrm{box}}
\newcommand{\stepinput}{\mathtt{lane\_change}}
\newcommand{\powerSet}[1]{\mathbf{2}^{#1}}

\makeatletter
\def\blfootnote{\gdef\@thefnmark{}\@footnotetext}
\makeatother

\begin{document}

\pdfglyphtounicode{A}{0041}
\pdfglyphtounicode{AE}{00C6}
\pdfglyphtounicode{AEacute}{01FC}
\pdfglyphtounicode{AEmacron}{01E2}
\pdfglyphtounicode{AEsmall}{00E6}
\pdfglyphtounicode{Aacute}{00C1}
\pdfglyphtounicode{Aacutesmall}{00E1}
\pdfglyphtounicode{Abreve}{0102}
\pdfglyphtounicode{Abreveacute}{1EAE}
\pdfglyphtounicode{Abrevecyrillic}{04D0}
\pdfglyphtounicode{Abrevedotbelow}{1EB6}
\pdfglyphtounicode{Abrevegrave}{1EB0}
\pdfglyphtounicode{Abrevehookabove}{1EB2}
\pdfglyphtounicode{Abrevetilde}{1EB4}
\pdfglyphtounicode{Acaron}{01CD}
\pdfglyphtounicode{Acircle}{24B6}
\pdfglyphtounicode{Acircumflex}{00C2}
\pdfglyphtounicode{Acircumflexacute}{1EA4}
\pdfglyphtounicode{Acircumflexdotbelow}{1EAC}
\pdfglyphtounicode{Acircumflexgrave}{1EA6}
\pdfglyphtounicode{Acircumflexhookabove}{1EA8}
\pdfglyphtounicode{Acircumflexsmall}{00E2}
\pdfglyphtounicode{Acircumflextilde}{1EAA}
\pdfglyphtounicode{Acute}{00B4}
\pdfglyphtounicode{Acutesmall}{00B4}
\pdfglyphtounicode{Acyrillic}{0410}
\pdfglyphtounicode{Adblgrave}{0200}
\pdfglyphtounicode{Adieresis}{00C4}
\pdfglyphtounicode{Adieresiscyrillic}{04D2}
\pdfglyphtounicode{Adieresismacron}{01DE}
\pdfglyphtounicode{Adieresissmall}{00E4}
\pdfglyphtounicode{Adotbelow}{1EA0}
\pdfglyphtounicode{Adotmacron}{01E0}
\pdfglyphtounicode{Agrave}{00C0}
\pdfglyphtounicode{Agravesmall}{00E0}
\pdfglyphtounicode{Ahookabove}{1EA2}
\pdfglyphtounicode{Aiecyrillic}{04D4}
\pdfglyphtounicode{Ainvertedbreve}{0202}
\pdfglyphtounicode{Alpha}{0391}
\pdfglyphtounicode{Alphatonos}{0386}
\pdfglyphtounicode{Amacron}{0100}
\pdfglyphtounicode{Amonospace}{FF21}
\pdfglyphtounicode{Aogonek}{0104}
\pdfglyphtounicode{Aring}{00C5}
\pdfglyphtounicode{Aringacute}{01FA}
\pdfglyphtounicode{Aringbelow}{1E00}
\pdfglyphtounicode{Aringsmall}{00E5}
\pdfglyphtounicode{Asmall}{0061}
\pdfglyphtounicode{Atilde}{00C3}
\pdfglyphtounicode{Atildesmall}{00E3}
\pdfglyphtounicode{Aybarmenian}{0531}
\pdfglyphtounicode{B}{0042}
\pdfglyphtounicode{Bcircle}{24B7}
\pdfglyphtounicode{Bdotaccent}{1E02}
\pdfglyphtounicode{Bdotbelow}{1E04}
\pdfglyphtounicode{Becyrillic}{0411}
\pdfglyphtounicode{Benarmenian}{0532}
\pdfglyphtounicode{Beta}{0392}
\pdfglyphtounicode{Bhook}{0181}
\pdfglyphtounicode{Blinebelow}{1E06}
\pdfglyphtounicode{Bmonospace}{FF22}
\pdfglyphtounicode{Brevesmall}{02D8}
\pdfglyphtounicode{Bsmall}{0062}
\pdfglyphtounicode{Btopbar}{0182}
\pdfglyphtounicode{C}{0043}
\pdfglyphtounicode{Caarmenian}{053E}
\pdfglyphtounicode{Cacute}{0106}
\pdfglyphtounicode{Caron}{02C7}
\pdfglyphtounicode{Caronsmall}{02C7}
\pdfglyphtounicode{Ccaron}{010C}
\pdfglyphtounicode{Ccedilla}{00C7}
\pdfglyphtounicode{Ccedillaacute}{1E08}
\pdfglyphtounicode{Ccedillasmall}{00E7}
\pdfglyphtounicode{Ccircle}{24B8}
\pdfglyphtounicode{Ccircumflex}{0108}
\pdfglyphtounicode{Cdot}{010A}
\pdfglyphtounicode{Cdotaccent}{010A}
\pdfglyphtounicode{Cedillasmall}{00B8}
\pdfglyphtounicode{Chaarmenian}{0549}
\pdfglyphtounicode{Cheabkhasiancyrillic}{04BC}
\pdfglyphtounicode{Checyrillic}{0427}
\pdfglyphtounicode{Chedescenderabkhasiancyrillic}{04BE}
\pdfglyphtounicode{Chedescendercyrillic}{04B6}
\pdfglyphtounicode{Chedieresiscyrillic}{04F4}
\pdfglyphtounicode{Cheharmenian}{0543}
\pdfglyphtounicode{Chekhakassiancyrillic}{04CB}
\pdfglyphtounicode{Cheverticalstrokecyrillic}{04B8}
\pdfglyphtounicode{Chi}{03A7}
\pdfglyphtounicode{Chook}{0187}
\pdfglyphtounicode{Circumflexsmall}{02C6}
\pdfglyphtounicode{Cmonospace}{FF23}
\pdfglyphtounicode{Coarmenian}{0551}
\pdfglyphtounicode{Csmall}{0063}
\pdfglyphtounicode{D}{0044}
\pdfglyphtounicode{DZ}{01F1}
\pdfglyphtounicode{DZcaron}{01C4}
\pdfglyphtounicode{Daarmenian}{0534}
\pdfglyphtounicode{Dafrican}{0189}
\pdfglyphtounicode{Dbar}{0110}
\pdfglyphtounicode{Dcaron}{010E}
\pdfglyphtounicode{Dcedilla}{1E10}
\pdfglyphtounicode{Dcircle}{24B9}
\pdfglyphtounicode{Dcircumflexbelow}{1E12}
\pdfglyphtounicode{Dcroat}{0110}
\pdfglyphtounicode{Ddotaccent}{1E0A}
\pdfglyphtounicode{Ddotbelow}{1E0C}
\pdfglyphtounicode{Decyrillic}{0414}
\pdfglyphtounicode{Deicoptic}{03EE}
\pdfglyphtounicode{Delta}{2206}
\pdfglyphtounicode{Deltagreek}{0394}
\pdfglyphtounicode{Dhook}{018A}
\pdfglyphtounicode{Dieresis}{00A8}
\pdfglyphtounicode{DieresisAcute}{F6CC}
\pdfglyphtounicode{DieresisGrave}{F6CD}
\pdfglyphtounicode{Dieresissmall}{00A8}
\pdfglyphtounicode{Digamma}{D875 DFCB}
\pdfglyphtounicode{Digammagreek}{03DC}
\pdfglyphtounicode{Djecyrillic}{0402}
\pdfglyphtounicode{Dlinebelow}{1E0E}
\pdfglyphtounicode{Dmonospace}{FF24}
\pdfglyphtounicode{Dotaccentsmall}{02D9}
\pdfglyphtounicode{Dslash}{0110}
\pdfglyphtounicode{Dsmall}{0064}
\pdfglyphtounicode{Dtopbar}{018B}
\pdfglyphtounicode{Dz}{01F2}
\pdfglyphtounicode{Dzcaron}{01C5}
\pdfglyphtounicode{Dzeabkhasiancyrillic}{04E0}
\pdfglyphtounicode{Dzecyrillic}{0405}
\pdfglyphtounicode{Dzhecyrillic}{040F}
\pdfglyphtounicode{E}{0045}
\pdfglyphtounicode{Eacute}{00C9}
\pdfglyphtounicode{Eacutesmall}{00E9}
\pdfglyphtounicode{Ebreve}{0114}
\pdfglyphtounicode{Ecaron}{011A}
\pdfglyphtounicode{Ecedillabreve}{1E1C}
\pdfglyphtounicode{Echarmenian}{0535}
\pdfglyphtounicode{Ecircle}{24BA}
\pdfglyphtounicode{Ecircumflex}{00CA}
\pdfglyphtounicode{Ecircumflexacute}{1EBE}
\pdfglyphtounicode{Ecircumflexbelow}{1E18}
\pdfglyphtounicode{Ecircumflexdotbelow}{1EC6}
\pdfglyphtounicode{Ecircumflexgrave}{1EC0}
\pdfglyphtounicode{Ecircumflexhookabove}{1EC2}
\pdfglyphtounicode{Ecircumflexsmall}{00EA}
\pdfglyphtounicode{Ecircumflextilde}{1EC4}
\pdfglyphtounicode{Ecyrillic}{0404}
\pdfglyphtounicode{Edblgrave}{0204}
\pdfglyphtounicode{Edieresis}{00CB}
\pdfglyphtounicode{Edieresissmall}{00EB}
\pdfglyphtounicode{Edot}{0116}
\pdfglyphtounicode{Edotaccent}{0116}
\pdfglyphtounicode{Edotbelow}{1EB8}
\pdfglyphtounicode{Efcyrillic}{0424}
\pdfglyphtounicode{Egrave}{00C8}
\pdfglyphtounicode{Egravesmall}{00E8}
\pdfglyphtounicode{Eharmenian}{0537}
\pdfglyphtounicode{Ehookabove}{1EBA}
\pdfglyphtounicode{Eightroman}{2167}
\pdfglyphtounicode{Einvertedbreve}{0206}
\pdfglyphtounicode{Eiotifiedcyrillic}{0464}
\pdfglyphtounicode{Elcyrillic}{041B}
\pdfglyphtounicode{Elevenroman}{216A}
\pdfglyphtounicode{Emacron}{0112}
\pdfglyphtounicode{Emacronacute}{1E16}
\pdfglyphtounicode{Emacrongrave}{1E14}
\pdfglyphtounicode{Emcyrillic}{041C}
\pdfglyphtounicode{Emonospace}{FF25}
\pdfglyphtounicode{Encyrillic}{041D}
\pdfglyphtounicode{Endescendercyrillic}{04A2}
\pdfglyphtounicode{Eng}{014A}
\pdfglyphtounicode{Enghecyrillic}{04A4}
\pdfglyphtounicode{Enhookcyrillic}{04C7}
\pdfglyphtounicode{Eogonek}{0118}
\pdfglyphtounicode{Eopen}{0190}
\pdfglyphtounicode{Epsilon}{0395}
\pdfglyphtounicode{Epsilontonos}{0388}
\pdfglyphtounicode{Ercyrillic}{0420}
\pdfglyphtounicode{Ereversed}{018E}
\pdfglyphtounicode{Ereversedcyrillic}{042D}
\pdfglyphtounicode{Escyrillic}{0421}
\pdfglyphtounicode{Esdescendercyrillic}{04AA}
\pdfglyphtounicode{Esh}{01A9}
\pdfglyphtounicode{Esmall}{0065}
\pdfglyphtounicode{Eta}{0397}
\pdfglyphtounicode{Etarmenian}{0538}
\pdfglyphtounicode{Etatonos}{0389}
\pdfglyphtounicode{Eth}{00D0}
\pdfglyphtounicode{Ethsmall}{00F0}
\pdfglyphtounicode{Etilde}{1EBC}
\pdfglyphtounicode{Etildebelow}{1E1A}
\pdfglyphtounicode{Euro}{20AC}
\pdfglyphtounicode{Ezh}{01B7}
\pdfglyphtounicode{Ezhcaron}{01EE}
\pdfglyphtounicode{Ezhreversed}{01B8}
\pdfglyphtounicode{F}{0046}
\pdfglyphtounicode{FFIsmall}{0066 0066 0069}
\pdfglyphtounicode{FFLsmall}{0066 0066 006C}
\pdfglyphtounicode{FFsmall}{0066 0066}
\pdfglyphtounicode{FIsmall}{0066 0069}
\pdfglyphtounicode{FLsmall}{0066 006C}
\pdfglyphtounicode{Fcircle}{24BB}
\pdfglyphtounicode{Fdotaccent}{1E1E}
\pdfglyphtounicode{Feharmenian}{0556}
\pdfglyphtounicode{Feicoptic}{03E4}
\pdfglyphtounicode{Fhook}{0191}
\pdfglyphtounicode{Finv}{2132}
\pdfglyphtounicode{Fitacyrillic}{0472}
\pdfglyphtounicode{Fiveroman}{2164}
\pdfglyphtounicode{Fmonospace}{FF26}
\pdfglyphtounicode{Fourroman}{2163}
\pdfglyphtounicode{Fsmall}{0066}
\pdfglyphtounicode{G}{0047}
\pdfglyphtounicode{GBsquare}{3387}
\pdfglyphtounicode{Gacute}{01F4}
\pdfglyphtounicode{Gamma}{0393}
\pdfglyphtounicode{Gammaafrican}{0194}
\pdfglyphtounicode{Gangiacoptic}{03EA}
\pdfglyphtounicode{Gbreve}{011E}
\pdfglyphtounicode{Gcaron}{01E6}
\pdfglyphtounicode{Gcedilla}{0122}
\pdfglyphtounicode{Gcircle}{24BC}
\pdfglyphtounicode{Gcircumflex}{011C}
\pdfglyphtounicode{Gcommaaccent}{0122}
\pdfglyphtounicode{Gdot}{0120}
\pdfglyphtounicode{Gdotaccent}{0120}
\pdfglyphtounicode{Gecyrillic}{0413}
\pdfglyphtounicode{Germandbls}{0053 0053}
\pdfglyphtounicode{Germandblssmall}{0073 0073}
\pdfglyphtounicode{Ghadarmenian}{0542}
\pdfglyphtounicode{Ghemiddlehookcyrillic}{0494}
\pdfglyphtounicode{Ghestrokecyrillic}{0492}
\pdfglyphtounicode{Gheupturncyrillic}{0490}
\pdfglyphtounicode{Ghook}{0193}
\pdfglyphtounicode{Gimarmenian}{0533}
\pdfglyphtounicode{Gjecyrillic}{0403}
\pdfglyphtounicode{Gmacron}{1E20}
\pdfglyphtounicode{Gmir}{2141}
\pdfglyphtounicode{Gmonospace}{FF27}
\pdfglyphtounicode{Grave}{0060}
\pdfglyphtounicode{Gravesmall}{0060}
\pdfglyphtounicode{Gsmall}{0067}
\pdfglyphtounicode{Gsmallhook}{029B}
\pdfglyphtounicode{Gstroke}{01E4}
\pdfglyphtounicode{H}{0048}
\pdfglyphtounicode{H18533}{25CF}
\pdfglyphtounicode{H18543}{25AA}
\pdfglyphtounicode{H18551}{25AB}
\pdfglyphtounicode{H22073}{25A1}
\pdfglyphtounicode{HPsquare}{33CB}
\pdfglyphtounicode{Haabkhasiancyrillic}{04A8}
\pdfglyphtounicode{Hadescendercyrillic}{04B2}
\pdfglyphtounicode{Hardsigncyrillic}{042A}
\pdfglyphtounicode{Hbar}{0126}
\pdfglyphtounicode{Hbrevebelow}{1E2A}
\pdfglyphtounicode{Hcedilla}{1E28}
\pdfglyphtounicode{Hcircle}{24BD}
\pdfglyphtounicode{Hcircumflex}{0124}
\pdfglyphtounicode{Hdieresis}{1E26}
\pdfglyphtounicode{Hdotaccent}{1E22}
\pdfglyphtounicode{Hdotbelow}{1E24}
\pdfglyphtounicode{Hmonospace}{FF28}
\pdfglyphtounicode{Hoarmenian}{0540}
\pdfglyphtounicode{Horicoptic}{03E8}
\pdfglyphtounicode{Hsmall}{0068}
\pdfglyphtounicode{Hungarumlaut}{02DD}
\pdfglyphtounicode{Hungarumlautsmall}{02DD}
\pdfglyphtounicode{Hzsquare}{3390}
\pdfglyphtounicode{I}{0049}
\pdfglyphtounicode{IAcyrillic}{042F}
\pdfglyphtounicode{IJ}{0132}
\pdfglyphtounicode{IUcyrillic}{042E}
\pdfglyphtounicode{Iacute}{00CD}
\pdfglyphtounicode{Iacutesmall}{00ED}
\pdfglyphtounicode{Ibreve}{012C}
\pdfglyphtounicode{Icaron}{01CF}
\pdfglyphtounicode{Icircle}{24BE}
\pdfglyphtounicode{Icircumflex}{00CE}
\pdfglyphtounicode{Icircumflexsmall}{00EE}
\pdfglyphtounicode{Icyrillic}{0406}
\pdfglyphtounicode{Idblgrave}{0208}
\pdfglyphtounicode{Idieresis}{00CF}
\pdfglyphtounicode{Idieresisacute}{1E2E}
\pdfglyphtounicode{Idieresiscyrillic}{04E4}
\pdfglyphtounicode{Idieresissmall}{00EF}
\pdfglyphtounicode{Idot}{0130}
\pdfglyphtounicode{Idotaccent}{0130}
\pdfglyphtounicode{Idotbelow}{1ECA}
\pdfglyphtounicode{Iebrevecyrillic}{04D6}
\pdfglyphtounicode{Iecyrillic}{0415}
\pdfglyphtounicode{Ifractur}{2111}
\pdfglyphtounicode{Ifraktur}{2111}
\pdfglyphtounicode{Igrave}{00CC}
\pdfglyphtounicode{Igravesmall}{00EC}
\pdfglyphtounicode{Ihookabove}{1EC8}
\pdfglyphtounicode{Iicyrillic}{0418}
\pdfglyphtounicode{Iinvertedbreve}{020A}
\pdfglyphtounicode{Iishortcyrillic}{0419}
\pdfglyphtounicode{Imacron}{012A}
\pdfglyphtounicode{Imacroncyrillic}{04E2}
\pdfglyphtounicode{Imonospace}{FF29}
\pdfglyphtounicode{Iniarmenian}{053B}
\pdfglyphtounicode{Iocyrillic}{0401}
\pdfglyphtounicode{Iogonek}{012E}
\pdfglyphtounicode{Iota}{0399}
\pdfglyphtounicode{Iotaafrican}{0196}
\pdfglyphtounicode{Iotadieresis}{03AA}
\pdfglyphtounicode{Iotatonos}{038A}
\pdfglyphtounicode{Ismall}{0069}
\pdfglyphtounicode{Istroke}{0197}
\pdfglyphtounicode{Itilde}{0128}
\pdfglyphtounicode{Itildebelow}{1E2C}
\pdfglyphtounicode{Izhitsacyrillic}{0474}
\pdfglyphtounicode{Izhitsadblgravecyrillic}{0476}
\pdfglyphtounicode{J}{004A}
\pdfglyphtounicode{Jaarmenian}{0541}
\pdfglyphtounicode{Jcircle}{24BF}
\pdfglyphtounicode{Jcircumflex}{0134}
\pdfglyphtounicode{Jecyrillic}{0408}
\pdfglyphtounicode{Jheharmenian}{054B}
\pdfglyphtounicode{Jmonospace}{FF2A}
\pdfglyphtounicode{Jsmall}{006A}
\pdfglyphtounicode{K}{004B}
\pdfglyphtounicode{KBsquare}{3385}
\pdfglyphtounicode{KKsquare}{33CD}
\pdfglyphtounicode{Kabashkircyrillic}{04A0}
\pdfglyphtounicode{Kacute}{1E30}
\pdfglyphtounicode{Kacyrillic}{041A}
\pdfglyphtounicode{Kadescendercyrillic}{049A}
\pdfglyphtounicode{Kahookcyrillic}{04C3}
\pdfglyphtounicode{Kappa}{039A}
\pdfglyphtounicode{Kastrokecyrillic}{049E}
\pdfglyphtounicode{Kaverticalstrokecyrillic}{049C}
\pdfglyphtounicode{Kcaron}{01E8}
\pdfglyphtounicode{Kcedilla}{0136}
\pdfglyphtounicode{Kcircle}{24C0}
\pdfglyphtounicode{Kcommaaccent}{0136}
\pdfglyphtounicode{Kdotbelow}{1E32}
\pdfglyphtounicode{Keharmenian}{0554}
\pdfglyphtounicode{Kenarmenian}{053F}
\pdfglyphtounicode{Khacyrillic}{0425}
\pdfglyphtounicode{Kheicoptic}{03E6}
\pdfglyphtounicode{Khook}{0198}
\pdfglyphtounicode{Kjecyrillic}{040C}
\pdfglyphtounicode{Klinebelow}{1E34}
\pdfglyphtounicode{Kmonospace}{FF2B}
\pdfglyphtounicode{Koppacyrillic}{0480}
\pdfglyphtounicode{Koppagreek}{03DE}
\pdfglyphtounicode{Ksicyrillic}{046E}
\pdfglyphtounicode{Ksmall}{006B}
\pdfglyphtounicode{L}{004C}
\pdfglyphtounicode{LJ}{01C7}
\pdfglyphtounicode{LL}{004C 004C}
\pdfglyphtounicode{Lacute}{0139}
\pdfglyphtounicode{Lambda}{039B}
\pdfglyphtounicode{Lcaron}{013D}
\pdfglyphtounicode{Lcedilla}{013B}
\pdfglyphtounicode{Lcircle}{24C1}
\pdfglyphtounicode{Lcircumflexbelow}{1E3C}
\pdfglyphtounicode{Lcommaaccent}{013B}
\pdfglyphtounicode{Ldot}{013F}
\pdfglyphtounicode{Ldotaccent}{013F}
\pdfglyphtounicode{Ldotbelow}{1E36}
\pdfglyphtounicode{Ldotbelowmacron}{1E38}
\pdfglyphtounicode{Liwnarmenian}{053C}
\pdfglyphtounicode{Lj}{01C8}
\pdfglyphtounicode{Ljecyrillic}{0409}
\pdfglyphtounicode{Llinebelow}{1E3A}
\pdfglyphtounicode{Lmonospace}{FF2C}
\pdfglyphtounicode{Lslash}{0141}
\pdfglyphtounicode{Lslashsmall}{0142}
\pdfglyphtounicode{Lsmall}{006C}
\pdfglyphtounicode{M}{004D}
\pdfglyphtounicode{MBsquare}{3386}
\pdfglyphtounicode{Macron}{00AF}
\pdfglyphtounicode{Macronsmall}{00AF}
\pdfglyphtounicode{Macute}{1E3E}
\pdfglyphtounicode{Mcircle}{24C2}
\pdfglyphtounicode{Mdotaccent}{1E40}
\pdfglyphtounicode{Mdotbelow}{1E42}
\pdfglyphtounicode{Menarmenian}{0544}
\pdfglyphtounicode{Mmonospace}{FF2D}
\pdfglyphtounicode{Msmall}{006D}
\pdfglyphtounicode{Mturned}{019C}
\pdfglyphtounicode{Mu}{039C}
\pdfglyphtounicode{N}{004E}
\pdfglyphtounicode{NJ}{01CA}
\pdfglyphtounicode{Nacute}{0143}
\pdfglyphtounicode{Ncaron}{0147}
\pdfglyphtounicode{Ncedilla}{0145}
\pdfglyphtounicode{Ncircle}{24C3}
\pdfglyphtounicode{Ncircumflexbelow}{1E4A}
\pdfglyphtounicode{Ncommaaccent}{0145}
\pdfglyphtounicode{Ndotaccent}{1E44}
\pdfglyphtounicode{Ndotbelow}{1E46}
\pdfglyphtounicode{Ng}{014A}
\pdfglyphtounicode{Nhookleft}{019D}
\pdfglyphtounicode{Nineroman}{2168}
\pdfglyphtounicode{Nj}{01CB}
\pdfglyphtounicode{Njecyrillic}{040A}
\pdfglyphtounicode{Nlinebelow}{1E48}
\pdfglyphtounicode{Nmonospace}{FF2E}
\pdfglyphtounicode{Nowarmenian}{0546}
\pdfglyphtounicode{Nsmall}{006E}
\pdfglyphtounicode{Ntilde}{00D1}
\pdfglyphtounicode{Ntildesmall}{00F1}
\pdfglyphtounicode{Nu}{039D}
\pdfglyphtounicode{O}{004F}
\pdfglyphtounicode{OE}{0152}
\pdfglyphtounicode{OEsmall}{0153}
\pdfglyphtounicode{Oacute}{00D3}
\pdfglyphtounicode{Oacutesmall}{00F3}
\pdfglyphtounicode{Obarredcyrillic}{04E8}
\pdfglyphtounicode{Obarreddieresiscyrillic}{04EA}
\pdfglyphtounicode{Obreve}{014E}
\pdfglyphtounicode{Ocaron}{01D1}
\pdfglyphtounicode{Ocenteredtilde}{019F}
\pdfglyphtounicode{Ocircle}{24C4}
\pdfglyphtounicode{Ocircumflex}{00D4}
\pdfglyphtounicode{Ocircumflexacute}{1ED0}
\pdfglyphtounicode{Ocircumflexdotbelow}{1ED8}
\pdfglyphtounicode{Ocircumflexgrave}{1ED2}
\pdfglyphtounicode{Ocircumflexhookabove}{1ED4}
\pdfglyphtounicode{Ocircumflexsmall}{00F4}
\pdfglyphtounicode{Ocircumflextilde}{1ED6}
\pdfglyphtounicode{Ocyrillic}{041E}
\pdfglyphtounicode{Odblacute}{0150}
\pdfglyphtounicode{Odblgrave}{020C}
\pdfglyphtounicode{Odieresis}{00D6}
\pdfglyphtounicode{Odieresiscyrillic}{04E6}
\pdfglyphtounicode{Odieresissmall}{00F6}
\pdfglyphtounicode{Odotbelow}{1ECC}
\pdfglyphtounicode{Ogoneksmall}{02DB}
\pdfglyphtounicode{Ograve}{00D2}
\pdfglyphtounicode{Ogravesmall}{00F2}
\pdfglyphtounicode{Oharmenian}{0555}
\pdfglyphtounicode{Ohm}{2126}
\pdfglyphtounicode{Ohookabove}{1ECE}
\pdfglyphtounicode{Ohorn}{01A0}
\pdfglyphtounicode{Ohornacute}{1EDA}
\pdfglyphtounicode{Ohorndotbelow}{1EE2}
\pdfglyphtounicode{Ohorngrave}{1EDC}
\pdfglyphtounicode{Ohornhookabove}{1EDE}
\pdfglyphtounicode{Ohorntilde}{1EE0}
\pdfglyphtounicode{Ohungarumlaut}{0150}
\pdfglyphtounicode{Oi}{01A2}
\pdfglyphtounicode{Oinvertedbreve}{020E}
\pdfglyphtounicode{Omacron}{014C}
\pdfglyphtounicode{Omacronacute}{1E52}
\pdfglyphtounicode{Omacrongrave}{1E50}
\pdfglyphtounicode{Omega}{2126}
\pdfglyphtounicode{Omegacyrillic}{0460}
\pdfglyphtounicode{Omegagreek}{03A9}
\pdfglyphtounicode{Omegainv}{2127}
\pdfglyphtounicode{Omegaroundcyrillic}{047A}
\pdfglyphtounicode{Omegatitlocyrillic}{047C}
\pdfglyphtounicode{Omegatonos}{038F}
\pdfglyphtounicode{Omicron}{039F}
\pdfglyphtounicode{Omicrontonos}{038C}
\pdfglyphtounicode{Omonospace}{FF2F}
\pdfglyphtounicode{Oneroman}{2160}
\pdfglyphtounicode{Oogonek}{01EA}
\pdfglyphtounicode{Oogonekmacron}{01EC}
\pdfglyphtounicode{Oopen}{0186}
\pdfglyphtounicode{Oslash}{00D8}
\pdfglyphtounicode{Oslashacute}{01FE}
\pdfglyphtounicode{Oslashsmall}{00F8}
\pdfglyphtounicode{Osmall}{006F}
\pdfglyphtounicode{Ostrokeacute}{01FE}
\pdfglyphtounicode{Otcyrillic}{047E}
\pdfglyphtounicode{Otilde}{00D5}
\pdfglyphtounicode{Otildeacute}{1E4C}
\pdfglyphtounicode{Otildedieresis}{1E4E}
\pdfglyphtounicode{Otildesmall}{00F5}
\pdfglyphtounicode{P}{0050}
\pdfglyphtounicode{Pacute}{1E54}
\pdfglyphtounicode{Pcircle}{24C5}
\pdfglyphtounicode{Pdotaccent}{1E56}
\pdfglyphtounicode{Pecyrillic}{041F}
\pdfglyphtounicode{Peharmenian}{054A}
\pdfglyphtounicode{Pemiddlehookcyrillic}{04A6}
\pdfglyphtounicode{Phi}{03A6}
\pdfglyphtounicode{Phook}{01A4}
\pdfglyphtounicode{Pi}{03A0}
\pdfglyphtounicode{Piwrarmenian}{0553}
\pdfglyphtounicode{Pmonospace}{FF30}
\pdfglyphtounicode{Psi}{03A8}
\pdfglyphtounicode{Psicyrillic}{0470}
\pdfglyphtounicode{Psmall}{0070}
\pdfglyphtounicode{Q}{0051}
\pdfglyphtounicode{Qcircle}{24C6}
\pdfglyphtounicode{Qmonospace}{FF31}
\pdfglyphtounicode{Qsmall}{0071}
\pdfglyphtounicode{R}{0052}
\pdfglyphtounicode{Raarmenian}{054C}
\pdfglyphtounicode{Racute}{0154}
\pdfglyphtounicode{Rcaron}{0158}
\pdfglyphtounicode{Rcedilla}{0156}
\pdfglyphtounicode{Rcircle}{24C7}
\pdfglyphtounicode{Rcommaaccent}{0156}
\pdfglyphtounicode{Rdblgrave}{0210}
\pdfglyphtounicode{Rdotaccent}{1E58}
\pdfglyphtounicode{Rdotbelow}{1E5A}
\pdfglyphtounicode{Rdotbelowmacron}{1E5C}
\pdfglyphtounicode{Reharmenian}{0550}
\pdfglyphtounicode{Rfractur}{211C}
\pdfglyphtounicode{Rfraktur}{211C}
\pdfglyphtounicode{Rho}{03A1}
\pdfglyphtounicode{Ringsmall}{02DA}
\pdfglyphtounicode{Rinvertedbreve}{0212}
\pdfglyphtounicode{Rlinebelow}{1E5E}
\pdfglyphtounicode{Rmonospace}{FF32}
\pdfglyphtounicode{Rsmall}{0072}
\pdfglyphtounicode{Rsmallinverted}{0281}
\pdfglyphtounicode{Rsmallinvertedsuperior}{02B6}
\pdfglyphtounicode{S}{0053}
\pdfglyphtounicode{SF010000}{250C}
\pdfglyphtounicode{SF020000}{2514}
\pdfglyphtounicode{SF030000}{2510}
\pdfglyphtounicode{SF040000}{2518}
\pdfglyphtounicode{SF050000}{253C}
\pdfglyphtounicode{SF060000}{252C}
\pdfglyphtounicode{SF070000}{2534}
\pdfglyphtounicode{SF080000}{251C}
\pdfglyphtounicode{SF090000}{2524}
\pdfglyphtounicode{SF100000}{2500}
\pdfglyphtounicode{SF110000}{2502}
\pdfglyphtounicode{SF190000}{2561}
\pdfglyphtounicode{SF200000}{2562}
\pdfglyphtounicode{SF210000}{2556}
\pdfglyphtounicode{SF220000}{2555}
\pdfglyphtounicode{SF230000}{2563}
\pdfglyphtounicode{SF240000}{2551}
\pdfglyphtounicode{SF250000}{2557}
\pdfglyphtounicode{SF260000}{255D}
\pdfglyphtounicode{SF270000}{255C}
\pdfglyphtounicode{SF280000}{255B}
\pdfglyphtounicode{SF360000}{255E}
\pdfglyphtounicode{SF370000}{255F}
\pdfglyphtounicode{SF380000}{255A}
\pdfglyphtounicode{SF390000}{2554}
\pdfglyphtounicode{SF400000}{2569}
\pdfglyphtounicode{SF410000}{2566}
\pdfglyphtounicode{SF420000}{2560}
\pdfglyphtounicode{SF430000}{2550}
\pdfglyphtounicode{SF440000}{256C}
\pdfglyphtounicode{SF450000}{2567}
\pdfglyphtounicode{SF460000}{2568}
\pdfglyphtounicode{SF470000}{2564}
\pdfglyphtounicode{SF480000}{2565}
\pdfglyphtounicode{SF490000}{2559}
\pdfglyphtounicode{SF500000}{2558}
\pdfglyphtounicode{SF510000}{2552}
\pdfglyphtounicode{SF520000}{2553}
\pdfglyphtounicode{SF530000}{256B}
\pdfglyphtounicode{SF540000}{256A}
\pdfglyphtounicode{SS}{0053 0053}
\pdfglyphtounicode{SSsmall}{0073 0073}
\pdfglyphtounicode{Sacute}{015A}
\pdfglyphtounicode{Sacutedotaccent}{1E64}
\pdfglyphtounicode{Sampigreek}{03E0}
\pdfglyphtounicode{Scaron}{0160}
\pdfglyphtounicode{Scarondotaccent}{1E66}
\pdfglyphtounicode{Scaronsmall}{0161}
\pdfglyphtounicode{Scedilla}{015E}
\pdfglyphtounicode{Schwa}{018F}
\pdfglyphtounicode{Schwacyrillic}{04D8}
\pdfglyphtounicode{Schwadieresiscyrillic}{04DA}
\pdfglyphtounicode{Scircle}{24C8}
\pdfglyphtounicode{Scircumflex}{015C}
\pdfglyphtounicode{Scommaaccent}{0218}
\pdfglyphtounicode{Sdotaccent}{1E60}
\pdfglyphtounicode{Sdotbelow}{1E62}
\pdfglyphtounicode{Sdotbelowdotaccent}{1E68}
\pdfglyphtounicode{Seharmenian}{054D}
\pdfglyphtounicode{Sevenroman}{2166}
\pdfglyphtounicode{Shaarmenian}{0547}
\pdfglyphtounicode{Shacyrillic}{0428}
\pdfglyphtounicode{Shchacyrillic}{0429}
\pdfglyphtounicode{Sheicoptic}{03E2}
\pdfglyphtounicode{Shhacyrillic}{04BA}
\pdfglyphtounicode{Shimacoptic}{03EC}
\pdfglyphtounicode{Sigma}{03A3}
\pdfglyphtounicode{Sixroman}{2165}
\pdfglyphtounicode{Smonospace}{FF33}
\pdfglyphtounicode{Softsigncyrillic}{042C}
\pdfglyphtounicode{Ssmall}{0073}
\pdfglyphtounicode{Stigmagreek}{03DA}
\pdfglyphtounicode{T}{0054}
\pdfglyphtounicode{Tau}{03A4}
\pdfglyphtounicode{Tbar}{0166}
\pdfglyphtounicode{Tcaron}{0164}
\pdfglyphtounicode{Tcedilla}{0162}
\pdfglyphtounicode{Tcircle}{24C9}
\pdfglyphtounicode{Tcircumflexbelow}{1E70}
\pdfglyphtounicode{Tcommaaccent}{0162}
\pdfglyphtounicode{Tdotaccent}{1E6A}
\pdfglyphtounicode{Tdotbelow}{1E6C}
\pdfglyphtounicode{Tecyrillic}{0422}
\pdfglyphtounicode{Tedescendercyrillic}{04AC}
\pdfglyphtounicode{Tenroman}{2169}
\pdfglyphtounicode{Tetsecyrillic}{04B4}
\pdfglyphtounicode{Theta}{0398}
\pdfglyphtounicode{Thook}{01AC}
\pdfglyphtounicode{Thorn}{00DE}
\pdfglyphtounicode{Thornsmall}{00FE}
\pdfglyphtounicode{Threeroman}{2162}
\pdfglyphtounicode{Tildesmall}{02DC}
\pdfglyphtounicode{Tiwnarmenian}{054F}
\pdfglyphtounicode{Tlinebelow}{1E6E}
\pdfglyphtounicode{Tmonospace}{FF34}
\pdfglyphtounicode{Toarmenian}{0539}
\pdfglyphtounicode{Tonefive}{01BC}
\pdfglyphtounicode{Tonesix}{0184}
\pdfglyphtounicode{Tonetwo}{01A7}
\pdfglyphtounicode{Tretroflexhook}{01AE}
\pdfglyphtounicode{Tsecyrillic}{0426}
\pdfglyphtounicode{Tshecyrillic}{040B}
\pdfglyphtounicode{Tsmall}{0074}
\pdfglyphtounicode{Twelveroman}{216B}
\pdfglyphtounicode{Tworoman}{2161}
\pdfglyphtounicode{U}{0055}
\pdfglyphtounicode{Uacute}{00DA}
\pdfglyphtounicode{Uacutesmall}{00FA}
\pdfglyphtounicode{Ubreve}{016C}
\pdfglyphtounicode{Ucaron}{01D3}
\pdfglyphtounicode{Ucircle}{24CA}
\pdfglyphtounicode{Ucircumflex}{00DB}
\pdfglyphtounicode{Ucircumflexbelow}{1E76}
\pdfglyphtounicode{Ucircumflexsmall}{00FB}
\pdfglyphtounicode{Ucyrillic}{0423}
\pdfglyphtounicode{Udblacute}{0170}
\pdfglyphtounicode{Udblgrave}{0214}
\pdfglyphtounicode{Udieresis}{00DC}
\pdfglyphtounicode{Udieresisacute}{01D7}
\pdfglyphtounicode{Udieresisbelow}{1E72}
\pdfglyphtounicode{Udieresiscaron}{01D9}
\pdfglyphtounicode{Udieresiscyrillic}{04F0}
\pdfglyphtounicode{Udieresisgrave}{01DB}
\pdfglyphtounicode{Udieresismacron}{01D5}
\pdfglyphtounicode{Udieresissmall}{00FC}
\pdfglyphtounicode{Udotbelow}{1EE4}
\pdfglyphtounicode{Ugrave}{00D9}
\pdfglyphtounicode{Ugravesmall}{00F9}
\pdfglyphtounicode{Uhookabove}{1EE6}
\pdfglyphtounicode{Uhorn}{01AF}
\pdfglyphtounicode{Uhornacute}{1EE8}
\pdfglyphtounicode{Uhorndotbelow}{1EF0}
\pdfglyphtounicode{Uhorngrave}{1EEA}
\pdfglyphtounicode{Uhornhookabove}{1EEC}
\pdfglyphtounicode{Uhorntilde}{1EEE}
\pdfglyphtounicode{Uhungarumlaut}{0170}
\pdfglyphtounicode{Uhungarumlautcyrillic}{04F2}
\pdfglyphtounicode{Uinvertedbreve}{0216}
\pdfglyphtounicode{Ukcyrillic}{0478}
\pdfglyphtounicode{Umacron}{016A}
\pdfglyphtounicode{Umacroncyrillic}{04EE}
\pdfglyphtounicode{Umacrondieresis}{1E7A}
\pdfglyphtounicode{Umonospace}{FF35}
\pdfglyphtounicode{Uogonek}{0172}
\pdfglyphtounicode{Upsilon}{03A5}
\pdfglyphtounicode{Upsilon1}{03D2}
\pdfglyphtounicode{Upsilonacutehooksymbolgreek}{03D3}
\pdfglyphtounicode{Upsilonafrican}{01B1}
\pdfglyphtounicode{Upsilondieresis}{03AB}
\pdfglyphtounicode{Upsilondieresishooksymbolgreek}{03D4}
\pdfglyphtounicode{Upsilonhooksymbol}{03D2}
\pdfglyphtounicode{Upsilontonos}{038E}
\pdfglyphtounicode{Uring}{016E}
\pdfglyphtounicode{Ushortcyrillic}{040E}
\pdfglyphtounicode{Usmall}{0075}
\pdfglyphtounicode{Ustraightcyrillic}{04AE}
\pdfglyphtounicode{Ustraightstrokecyrillic}{04B0}
\pdfglyphtounicode{Utilde}{0168}
\pdfglyphtounicode{Utildeacute}{1E78}
\pdfglyphtounicode{Utildebelow}{1E74}
\pdfglyphtounicode{V}{0056}
\pdfglyphtounicode{Vcircle}{24CB}
\pdfglyphtounicode{Vdotbelow}{1E7E}
\pdfglyphtounicode{Vecyrillic}{0412}
\pdfglyphtounicode{Vewarmenian}{054E}
\pdfglyphtounicode{Vhook}{01B2}
\pdfglyphtounicode{Vmonospace}{FF36}
\pdfglyphtounicode{Voarmenian}{0548}
\pdfglyphtounicode{Vsmall}{0076}
\pdfglyphtounicode{Vtilde}{1E7C}
\pdfglyphtounicode{W}{0057}
\pdfglyphtounicode{Wacute}{1E82}
\pdfglyphtounicode{Wcircle}{24CC}
\pdfglyphtounicode{Wcircumflex}{0174}
\pdfglyphtounicode{Wdieresis}{1E84}
\pdfglyphtounicode{Wdotaccent}{1E86}
\pdfglyphtounicode{Wdotbelow}{1E88}
\pdfglyphtounicode{Wgrave}{1E80}
\pdfglyphtounicode{Wmonospace}{FF37}
\pdfglyphtounicode{Wsmall}{0077}
\pdfglyphtounicode{X}{0058}
\pdfglyphtounicode{Xcircle}{24CD}
\pdfglyphtounicode{Xdieresis}{1E8C}
\pdfglyphtounicode{Xdotaccent}{1E8A}
\pdfglyphtounicode{Xeharmenian}{053D}
\pdfglyphtounicode{Xi}{039E}
\pdfglyphtounicode{Xmonospace}{FF38}
\pdfglyphtounicode{Xsmall}{0078}
\pdfglyphtounicode{Y}{0059}
\pdfglyphtounicode{Yacute}{00DD}
\pdfglyphtounicode{Yacutesmall}{00FD}
\pdfglyphtounicode{Yatcyrillic}{0462}
\pdfglyphtounicode{Ycircle}{24CE}
\pdfglyphtounicode{Ycircumflex}{0176}
\pdfglyphtounicode{Ydieresis}{0178}
\pdfglyphtounicode{Ydieresissmall}{00FF}
\pdfglyphtounicode{Ydotaccent}{1E8E}
\pdfglyphtounicode{Ydotbelow}{1EF4}
\pdfglyphtounicode{Yen}{00A5}
\pdfglyphtounicode{Yericyrillic}{042B}
\pdfglyphtounicode{Yerudieresiscyrillic}{04F8}
\pdfglyphtounicode{Ygrave}{1EF2}
\pdfglyphtounicode{Yhook}{01B3}
\pdfglyphtounicode{Yhookabove}{1EF6}
\pdfglyphtounicode{Yiarmenian}{0545}
\pdfglyphtounicode{Yicyrillic}{0407}
\pdfglyphtounicode{Yiwnarmenian}{0552}
\pdfglyphtounicode{Ymonospace}{FF39}
\pdfglyphtounicode{Ysmall}{0079}
\pdfglyphtounicode{Ytilde}{1EF8}
\pdfglyphtounicode{Yusbigcyrillic}{046A}
\pdfglyphtounicode{Yusbigiotifiedcyrillic}{046C}
\pdfglyphtounicode{Yuslittlecyrillic}{0466}
\pdfglyphtounicode{Yuslittleiotifiedcyrillic}{0468}
\pdfglyphtounicode{Z}{005A}
\pdfglyphtounicode{Zaarmenian}{0536}
\pdfglyphtounicode{Zacute}{0179}
\pdfglyphtounicode{Zcaron}{017D}
\pdfglyphtounicode{Zcaronsmall}{017E}
\pdfglyphtounicode{Zcircle}{24CF}
\pdfglyphtounicode{Zcircumflex}{1E90}
\pdfglyphtounicode{Zdot}{017B}
\pdfglyphtounicode{Zdotaccent}{017B}
\pdfglyphtounicode{Zdotbelow}{1E92}
\pdfglyphtounicode{Zecyrillic}{0417}
\pdfglyphtounicode{Zedescendercyrillic}{0498}
\pdfglyphtounicode{Zedieresiscyrillic}{04DE}
\pdfglyphtounicode{Zeta}{0396}
\pdfglyphtounicode{Zhearmenian}{053A}
\pdfglyphtounicode{Zhebrevecyrillic}{04C1}
\pdfglyphtounicode{Zhecyrillic}{0416}
\pdfglyphtounicode{Zhedescendercyrillic}{0496}
\pdfglyphtounicode{Zhedieresiscyrillic}{04DC}
\pdfglyphtounicode{Zlinebelow}{1E94}
\pdfglyphtounicode{Zmonospace}{FF3A}
\pdfglyphtounicode{Zsmall}{007A}
\pdfglyphtounicode{Zstroke}{01B5}
\pdfglyphtounicode{a}{0061}
\pdfglyphtounicode{aabengali}{0986}
\pdfglyphtounicode{aacute}{00E1}
\pdfglyphtounicode{aadeva}{0906}
\pdfglyphtounicode{aagujarati}{0A86}
\pdfglyphtounicode{aagurmukhi}{0A06}
\pdfglyphtounicode{aamatragurmukhi}{0A3E}
\pdfglyphtounicode{aarusquare}{3303}
\pdfglyphtounicode{aavowelsignbengali}{09BE}
\pdfglyphtounicode{aavowelsigndeva}{093E}
\pdfglyphtounicode{aavowelsigngujarati}{0ABE}
\pdfglyphtounicode{abbreviationmarkarmenian}{055F}
\pdfglyphtounicode{abbreviationsigndeva}{0970}
\pdfglyphtounicode{abengali}{0985}
\pdfglyphtounicode{abopomofo}{311A}
\pdfglyphtounicode{abreve}{0103}
\pdfglyphtounicode{abreveacute}{1EAF}
\pdfglyphtounicode{abrevecyrillic}{04D1}
\pdfglyphtounicode{abrevedotbelow}{1EB7}
\pdfglyphtounicode{abrevegrave}{1EB1}
\pdfglyphtounicode{abrevehookabove}{1EB3}
\pdfglyphtounicode{abrevetilde}{1EB5}
\pdfglyphtounicode{acaron}{01CE}
\pdfglyphtounicode{acircle}{24D0}
\pdfglyphtounicode{acircumflex}{00E2}
\pdfglyphtounicode{acircumflexacute}{1EA5}
\pdfglyphtounicode{acircumflexdotbelow}{1EAD}
\pdfglyphtounicode{acircumflexgrave}{1EA7}
\pdfglyphtounicode{acircumflexhookabove}{1EA9}
\pdfglyphtounicode{acircumflextilde}{1EAB}
\pdfglyphtounicode{acute}{00B4}
\pdfglyphtounicode{acutebelowcmb}{0317}
\pdfglyphtounicode{acutecmb}{0301}
\pdfglyphtounicode{acutecomb}{0301}
\pdfglyphtounicode{acutedeva}{0954}
\pdfglyphtounicode{acutelowmod}{02CF}
\pdfglyphtounicode{acutetonecmb}{0341}
\pdfglyphtounicode{acyrillic}{0430}
\pdfglyphtounicode{adblgrave}{0201}
\pdfglyphtounicode{addakgurmukhi}{0A71}
\pdfglyphtounicode{adeva}{0905}
\pdfglyphtounicode{adieresis}{00E4}
\pdfglyphtounicode{adieresiscyrillic}{04D3}
\pdfglyphtounicode{adieresismacron}{01DF}
\pdfglyphtounicode{adotbelow}{1EA1}
\pdfglyphtounicode{adotmacron}{01E1}
\pdfglyphtounicode{ae}{00E6}
\pdfglyphtounicode{aeacute}{01FD}
\pdfglyphtounicode{aekorean}{3150}
\pdfglyphtounicode{aemacron}{01E3}
\pdfglyphtounicode{afii00208}{2015}
\pdfglyphtounicode{afii08941}{20A4}
\pdfglyphtounicode{afii10017}{0410}
\pdfglyphtounicode{afii10018}{0411}
\pdfglyphtounicode{afii10019}{0412}
\pdfglyphtounicode{afii10020}{0413}
\pdfglyphtounicode{afii10021}{0414}
\pdfglyphtounicode{afii10022}{0415}
\pdfglyphtounicode{afii10023}{0401}
\pdfglyphtounicode{afii10024}{0416}
\pdfglyphtounicode{afii10025}{0417}
\pdfglyphtounicode{afii10026}{0418}
\pdfglyphtounicode{afii10027}{0419}
\pdfglyphtounicode{afii10028}{041A}
\pdfglyphtounicode{afii10029}{041B}
\pdfglyphtounicode{afii10030}{041C}
\pdfglyphtounicode{afii10031}{041D}
\pdfglyphtounicode{afii10032}{041E}
\pdfglyphtounicode{afii10033}{041F}
\pdfglyphtounicode{afii10034}{0420}
\pdfglyphtounicode{afii10035}{0421}
\pdfglyphtounicode{afii10036}{0422}
\pdfglyphtounicode{afii10037}{0423}
\pdfglyphtounicode{afii10038}{0424}
\pdfglyphtounicode{afii10039}{0425}
\pdfglyphtounicode{afii10040}{0426}
\pdfglyphtounicode{afii10041}{0427}
\pdfglyphtounicode{afii10042}{0428}
\pdfglyphtounicode{afii10043}{0429}
\pdfglyphtounicode{afii10044}{042A}
\pdfglyphtounicode{afii10045}{042B}
\pdfglyphtounicode{afii10046}{042C}
\pdfglyphtounicode{afii10047}{042D}
\pdfglyphtounicode{afii10048}{042E}
\pdfglyphtounicode{afii10049}{042F}
\pdfglyphtounicode{afii10050}{0490}
\pdfglyphtounicode{afii10051}{0402}
\pdfglyphtounicode{afii10052}{0403}
\pdfglyphtounicode{afii10053}{0404}
\pdfglyphtounicode{afii10054}{0405}
\pdfglyphtounicode{afii10055}{0406}
\pdfglyphtounicode{afii10056}{0407}
\pdfglyphtounicode{afii10057}{0408}
\pdfglyphtounicode{afii10058}{0409}
\pdfglyphtounicode{afii10059}{040A}
\pdfglyphtounicode{afii10060}{040B}
\pdfglyphtounicode{afii10061}{040C}
\pdfglyphtounicode{afii10062}{040E}
\pdfglyphtounicode{afii10063}{F6C4}
\pdfglyphtounicode{afii10064}{F6C5}
\pdfglyphtounicode{afii10065}{0430}
\pdfglyphtounicode{afii10066}{0431}
\pdfglyphtounicode{afii10067}{0432}
\pdfglyphtounicode{afii10068}{0433}
\pdfglyphtounicode{afii10069}{0434}
\pdfglyphtounicode{afii10070}{0435}
\pdfglyphtounicode{afii10071}{0451}
\pdfglyphtounicode{afii10072}{0436}
\pdfglyphtounicode{afii10073}{0437}
\pdfglyphtounicode{afii10074}{0438}
\pdfglyphtounicode{afii10075}{0439}
\pdfglyphtounicode{afii10076}{043A}
\pdfglyphtounicode{afii10077}{043B}
\pdfglyphtounicode{afii10078}{043C}
\pdfglyphtounicode{afii10079}{043D}
\pdfglyphtounicode{afii10080}{043E}
\pdfglyphtounicode{afii10081}{043F}
\pdfglyphtounicode{afii10082}{0440}
\pdfglyphtounicode{afii10083}{0441}
\pdfglyphtounicode{afii10084}{0442}
\pdfglyphtounicode{afii10085}{0443}
\pdfglyphtounicode{afii10086}{0444}
\pdfglyphtounicode{afii10087}{0445}
\pdfglyphtounicode{afii10088}{0446}
\pdfglyphtounicode{afii10089}{0447}
\pdfglyphtounicode{afii10090}{0448}
\pdfglyphtounicode{afii10091}{0449}
\pdfglyphtounicode{afii10092}{044A}
\pdfglyphtounicode{afii10093}{044B}
\pdfglyphtounicode{afii10094}{044C}
\pdfglyphtounicode{afii10095}{044D}
\pdfglyphtounicode{afii10096}{044E}
\pdfglyphtounicode{afii10097}{044F}
\pdfglyphtounicode{afii10098}{0491}
\pdfglyphtounicode{afii10099}{0452}
\pdfglyphtounicode{afii10100}{0453}
\pdfglyphtounicode{afii10101}{0454}
\pdfglyphtounicode{afii10102}{0455}
\pdfglyphtounicode{afii10103}{0456}
\pdfglyphtounicode{afii10104}{0457}
\pdfglyphtounicode{afii10105}{0458}
\pdfglyphtounicode{afii10106}{0459}
\pdfglyphtounicode{afii10107}{045A}
\pdfglyphtounicode{afii10108}{045B}
\pdfglyphtounicode{afii10109}{045C}
\pdfglyphtounicode{afii10110}{045E}
\pdfglyphtounicode{afii10145}{040F}
\pdfglyphtounicode{afii10146}{0462}
\pdfglyphtounicode{afii10147}{0472}
\pdfglyphtounicode{afii10148}{0474}
\pdfglyphtounicode{afii10192}{F6C6}
\pdfglyphtounicode{afii10193}{045F}
\pdfglyphtounicode{afii10194}{0463}
\pdfglyphtounicode{afii10195}{0473}
\pdfglyphtounicode{afii10196}{0475}
\pdfglyphtounicode{afii10831}{F6C7}
\pdfglyphtounicode{afii10832}{F6C8}
\pdfglyphtounicode{afii10846}{04D9}
\pdfglyphtounicode{afii299}{200E}
\pdfglyphtounicode{afii300}{200F}
\pdfglyphtounicode{afii301}{200D}
\pdfglyphtounicode{afii57381}{066A}
\pdfglyphtounicode{afii57388}{060C}
\pdfglyphtounicode{afii57392}{0660}
\pdfglyphtounicode{afii57393}{0661}
\pdfglyphtounicode{afii57394}{0662}
\pdfglyphtounicode{afii57395}{0663}
\pdfglyphtounicode{afii57396}{0664}
\pdfglyphtounicode{afii57397}{0665}
\pdfglyphtounicode{afii57398}{0666}
\pdfglyphtounicode{afii57399}{0667}
\pdfglyphtounicode{afii57400}{0668}
\pdfglyphtounicode{afii57401}{0669}
\pdfglyphtounicode{afii57403}{061B}
\pdfglyphtounicode{afii57407}{061F}
\pdfglyphtounicode{afii57409}{0621}
\pdfglyphtounicode{afii57410}{0622}
\pdfglyphtounicode{afii57411}{0623}
\pdfglyphtounicode{afii57412}{0624}
\pdfglyphtounicode{afii57413}{0625}
\pdfglyphtounicode{afii57414}{0626}
\pdfglyphtounicode{afii57415}{0627}
\pdfglyphtounicode{afii57416}{0628}
\pdfglyphtounicode{afii57417}{0629}
\pdfglyphtounicode{afii57418}{062A}
\pdfglyphtounicode{afii57419}{062B}
\pdfglyphtounicode{afii57420}{062C}
\pdfglyphtounicode{afii57421}{062D}
\pdfglyphtounicode{afii57422}{062E}
\pdfglyphtounicode{afii57423}{062F}
\pdfglyphtounicode{afii57424}{0630}
\pdfglyphtounicode{afii57425}{0631}
\pdfglyphtounicode{afii57426}{0632}
\pdfglyphtounicode{afii57427}{0633}
\pdfglyphtounicode{afii57428}{0634}
\pdfglyphtounicode{afii57429}{0635}
\pdfglyphtounicode{afii57430}{0636}
\pdfglyphtounicode{afii57431}{0637}
\pdfglyphtounicode{afii57432}{0638}
\pdfglyphtounicode{afii57433}{0639}
\pdfglyphtounicode{afii57434}{063A}
\pdfglyphtounicode{afii57440}{0640}
\pdfglyphtounicode{afii57441}{0641}
\pdfglyphtounicode{afii57442}{0642}
\pdfglyphtounicode{afii57443}{0643}
\pdfglyphtounicode{afii57444}{0644}
\pdfglyphtounicode{afii57445}{0645}
\pdfglyphtounicode{afii57446}{0646}
\pdfglyphtounicode{afii57448}{0648}
\pdfglyphtounicode{afii57449}{0649}
\pdfglyphtounicode{afii57450}{064A}
\pdfglyphtounicode{afii57451}{064B}
\pdfglyphtounicode{afii57452}{064C}
\pdfglyphtounicode{afii57453}{064D}
\pdfglyphtounicode{afii57454}{064E}
\pdfglyphtounicode{afii57455}{064F}
\pdfglyphtounicode{afii57456}{0650}
\pdfglyphtounicode{afii57457}{0651}
\pdfglyphtounicode{afii57458}{0652}
\pdfglyphtounicode{afii57470}{0647}
\pdfglyphtounicode{afii57505}{06A4}
\pdfglyphtounicode{afii57506}{067E}
\pdfglyphtounicode{afii57507}{0686}
\pdfglyphtounicode{afii57508}{0698}
\pdfglyphtounicode{afii57509}{06AF}
\pdfglyphtounicode{afii57511}{0679}
\pdfglyphtounicode{afii57512}{0688}
\pdfglyphtounicode{afii57513}{0691}
\pdfglyphtounicode{afii57514}{06BA}
\pdfglyphtounicode{afii57519}{06D2}
\pdfglyphtounicode{afii57534}{06D5}
\pdfglyphtounicode{afii57636}{20AA}
\pdfglyphtounicode{afii57645}{05BE}
\pdfglyphtounicode{afii57658}{05C3}
\pdfglyphtounicode{afii57664}{05D0}
\pdfglyphtounicode{afii57665}{05D1}
\pdfglyphtounicode{afii57666}{05D2}
\pdfglyphtounicode{afii57667}{05D3}
\pdfglyphtounicode{afii57668}{05D4}
\pdfglyphtounicode{afii57669}{05D5}
\pdfglyphtounicode{afii57670}{05D6}
\pdfglyphtounicode{afii57671}{05D7}
\pdfglyphtounicode{afii57672}{05D8}
\pdfglyphtounicode{afii57673}{05D9}
\pdfglyphtounicode{afii57674}{05DA}
\pdfglyphtounicode{afii57675}{05DB}
\pdfglyphtounicode{afii57676}{05DC}
\pdfglyphtounicode{afii57677}{05DD}
\pdfglyphtounicode{afii57678}{05DE}
\pdfglyphtounicode{afii57679}{05DF}
\pdfglyphtounicode{afii57680}{05E0}
\pdfglyphtounicode{afii57681}{05E1}
\pdfglyphtounicode{afii57682}{05E2}
\pdfglyphtounicode{afii57683}{05E3}
\pdfglyphtounicode{afii57684}{05E4}
\pdfglyphtounicode{afii57685}{05E5}
\pdfglyphtounicode{afii57686}{05E6}
\pdfglyphtounicode{afii57687}{05E7}
\pdfglyphtounicode{afii57688}{05E8}
\pdfglyphtounicode{afii57689}{05E9}
\pdfglyphtounicode{afii57690}{05EA}
\pdfglyphtounicode{afii57694}{FB2A}
\pdfglyphtounicode{afii57695}{FB2B}
\pdfglyphtounicode{afii57700}{FB4B}
\pdfglyphtounicode{afii57705}{FB1F}
\pdfglyphtounicode{afii57716}{05F0}
\pdfglyphtounicode{afii57717}{05F1}
\pdfglyphtounicode{afii57718}{05F2}
\pdfglyphtounicode{afii57723}{FB35}
\pdfglyphtounicode{afii57793}{05B4}
\pdfglyphtounicode{afii57794}{05B5}
\pdfglyphtounicode{afii57795}{05B6}
\pdfglyphtounicode{afii57796}{05BB}
\pdfglyphtounicode{afii57797}{05B8}
\pdfglyphtounicode{afii57798}{05B7}
\pdfglyphtounicode{afii57799}{05B0}
\pdfglyphtounicode{afii57800}{05B2}
\pdfglyphtounicode{afii57801}{05B1}
\pdfglyphtounicode{afii57802}{05B3}
\pdfglyphtounicode{afii57803}{05C2}
\pdfglyphtounicode{afii57804}{05C1}
\pdfglyphtounicode{afii57806}{05B9}
\pdfglyphtounicode{afii57807}{05BC}
\pdfglyphtounicode{afii57839}{05BD}
\pdfglyphtounicode{afii57841}{05BF}
\pdfglyphtounicode{afii57842}{05C0}
\pdfglyphtounicode{afii57929}{02BC}
\pdfglyphtounicode{afii61248}{2105}
\pdfglyphtounicode{afii61289}{2113}
\pdfglyphtounicode{afii61352}{2116}
\pdfglyphtounicode{afii61573}{202C}
\pdfglyphtounicode{afii61574}{202D}
\pdfglyphtounicode{afii61575}{202E}
\pdfglyphtounicode{afii61664}{200C}
\pdfglyphtounicode{afii63167}{066D}
\pdfglyphtounicode{afii64937}{02BD}
\pdfglyphtounicode{agrave}{00E0}
\pdfglyphtounicode{agujarati}{0A85}
\pdfglyphtounicode{agurmukhi}{0A05}
\pdfglyphtounicode{ahiragana}{3042}
\pdfglyphtounicode{ahookabove}{1EA3}
\pdfglyphtounicode{aibengali}{0990}
\pdfglyphtounicode{aibopomofo}{311E}
\pdfglyphtounicode{aideva}{0910}
\pdfglyphtounicode{aiecyrillic}{04D5}
\pdfglyphtounicode{aigujarati}{0A90}
\pdfglyphtounicode{aigurmukhi}{0A10}
\pdfglyphtounicode{aimatragurmukhi}{0A48}
\pdfglyphtounicode{ainarabic}{0639}
\pdfglyphtounicode{ainfinalarabic}{FECA}
\pdfglyphtounicode{aininitialarabic}{FECB}
\pdfglyphtounicode{ainmedialarabic}{FECC}
\pdfglyphtounicode{ainvertedbreve}{0203}
\pdfglyphtounicode{aivowelsignbengali}{09C8}
\pdfglyphtounicode{aivowelsigndeva}{0948}
\pdfglyphtounicode{aivowelsigngujarati}{0AC8}
\pdfglyphtounicode{akatakana}{30A2}
\pdfglyphtounicode{akatakanahalfwidth}{FF71}
\pdfglyphtounicode{akorean}{314F}
\pdfglyphtounicode{alef}{05D0}
\pdfglyphtounicode{alefarabic}{0627}
\pdfglyphtounicode{alefdageshhebrew}{FB30}
\pdfglyphtounicode{aleffinalarabic}{FE8E}
\pdfglyphtounicode{alefhamzaabovearabic}{0623}
\pdfglyphtounicode{alefhamzaabovefinalarabic}{FE84}
\pdfglyphtounicode{alefhamzabelowarabic}{0625}
\pdfglyphtounicode{alefhamzabelowfinalarabic}{FE88}
\pdfglyphtounicode{alefhebrew}{05D0}
\pdfglyphtounicode{aleflamedhebrew}{FB4F}
\pdfglyphtounicode{alefmaddaabovearabic}{0622}
\pdfglyphtounicode{alefmaddaabovefinalarabic}{FE82}
\pdfglyphtounicode{alefmaksuraarabic}{0649}
\pdfglyphtounicode{alefmaksurafinalarabic}{FEF0}
\pdfglyphtounicode{alefmaksurainitialarabic}{FEF3}
\pdfglyphtounicode{alefmaksuramedialarabic}{FEF4}
\pdfglyphtounicode{alefpatahhebrew}{FB2E}
\pdfglyphtounicode{alefqamatshebrew}{FB2F}
\pdfglyphtounicode{aleph}{2135}
\pdfglyphtounicode{allequal}{224C}
\pdfglyphtounicode{alpha}{03B1}
\pdfglyphtounicode{alphatonos}{03AC}
\pdfglyphtounicode{amacron}{0101}
\pdfglyphtounicode{amonospace}{FF41}
\pdfglyphtounicode{ampersand}{0026}
\pdfglyphtounicode{ampersandmonospace}{FF06}
\pdfglyphtounicode{ampersandsmall}{0026}
\pdfglyphtounicode{amsquare}{33C2}
\pdfglyphtounicode{anbopomofo}{3122}
\pdfglyphtounicode{angbopomofo}{3124}
\pdfglyphtounicode{angbracketleft}{27E8}
\pdfglyphtounicode{angbracketright}{27E9}
\pdfglyphtounicode{angkhankhuthai}{0E5A}
\pdfglyphtounicode{angle}{2220}
\pdfglyphtounicode{anglebracketleft}{3008}
\pdfglyphtounicode{anglebracketleftvertical}{FE3F}
\pdfglyphtounicode{anglebracketright}{3009}
\pdfglyphtounicode{anglebracketrightvertical}{FE40}
\pdfglyphtounicode{angleleft}{2329}
\pdfglyphtounicode{angleright}{232A}
\pdfglyphtounicode{angstrom}{212B}
\pdfglyphtounicode{anoteleia}{0387}
\pdfglyphtounicode{anticlockwise}{27F2}
\pdfglyphtounicode{anudattadeva}{0952}
\pdfglyphtounicode{anusvarabengali}{0982}
\pdfglyphtounicode{anusvaradeva}{0902}
\pdfglyphtounicode{anusvaragujarati}{0A82}
\pdfglyphtounicode{aogonek}{0105}
\pdfglyphtounicode{apaatosquare}{3300}
\pdfglyphtounicode{aparen}{249C}
\pdfglyphtounicode{apostrophearmenian}{055A}
\pdfglyphtounicode{apostrophemod}{02BC}
\pdfglyphtounicode{apple}{F8FF}
\pdfglyphtounicode{approaches}{2250}
\pdfglyphtounicode{approxequal}{2248}
\pdfglyphtounicode{approxequalorimage}{2252}
\pdfglyphtounicode{approximatelyequal}{2245}
\pdfglyphtounicode{approxorequal}{224A}
\pdfglyphtounicode{araeaekorean}{318E}
\pdfglyphtounicode{araeakorean}{318D}
\pdfglyphtounicode{arc}{2312}
\pdfglyphtounicode{archleftdown}{21B6}
\pdfglyphtounicode{archrightdown}{21B7}
\pdfglyphtounicode{arighthalfring}{1E9A}
\pdfglyphtounicode{aring}{00E5}
\pdfglyphtounicode{aringacute}{01FB}
\pdfglyphtounicode{aringbelow}{1E01}
\pdfglyphtounicode{arrowboth}{2194}
\pdfglyphtounicode{arrowbothv}{2195}
\pdfglyphtounicode{arrowdashdown}{21E3}
\pdfglyphtounicode{arrowdashleft}{21E0}
\pdfglyphtounicode{arrowdashright}{21E2}
\pdfglyphtounicode{arrowdashup}{21E1}
\pdfglyphtounicode{arrowdblboth}{21D4}
\pdfglyphtounicode{arrowdblbothv}{21D5}
\pdfglyphtounicode{arrowdbldown}{21D3}
\pdfglyphtounicode{arrowdblleft}{21D0}
\pdfglyphtounicode{arrowdblright}{21D2}
\pdfglyphtounicode{arrowdblup}{21D1}
\pdfglyphtounicode{arrowdown}{2193}
\pdfglyphtounicode{arrowdownleft}{2199}
\pdfglyphtounicode{arrowdownright}{2198}
\pdfglyphtounicode{arrowdownwhite}{21E9}
\pdfglyphtounicode{arrowheaddownmod}{02C5}
\pdfglyphtounicode{arrowheadleftmod}{02C2}
\pdfglyphtounicode{arrowheadrightmod}{02C3}
\pdfglyphtounicode{arrowheadupmod}{02C4}
\pdfglyphtounicode{arrowhorizex}{F8E7}
\pdfglyphtounicode{arrowleft}{2190}
\pdfglyphtounicode{arrowleftbothalf}{21BD}
\pdfglyphtounicode{arrowleftdbl}{21D0}
\pdfglyphtounicode{arrowleftdblstroke}{21CD}
\pdfglyphtounicode{arrowleftoverright}{21C6}
\pdfglyphtounicode{arrowlefttophalf}{21BC}
\pdfglyphtounicode{arrowleftwhite}{21E6}
\pdfglyphtounicode{arrownortheast}{2197}
\pdfglyphtounicode{arrownorthwest}{2196}
\pdfglyphtounicode{arrowparrleftright}{21C6}
\pdfglyphtounicode{arrowparrrightleft}{21C4}
\pdfglyphtounicode{arrowright}{2192}
\pdfglyphtounicode{arrowrightbothalf}{21C1}
\pdfglyphtounicode{arrowrightdblstroke}{21CF}
\pdfglyphtounicode{arrowrightheavy}{279E}
\pdfglyphtounicode{arrowrightoverleft}{21C4}
\pdfglyphtounicode{arrowrighttophalf}{21C0}
\pdfglyphtounicode{arrowrightwhite}{21E8}
\pdfglyphtounicode{arrowsoutheast}{2198}
\pdfglyphtounicode{arrowsouthwest}{2199}
\pdfglyphtounicode{arrowtableft}{21E4}
\pdfglyphtounicode{arrowtabright}{21E5}
\pdfglyphtounicode{arrowtailleft}{21A2}
\pdfglyphtounicode{arrowtailright}{21A3}
\pdfglyphtounicode{arrowtripleleft}{21DA}
\pdfglyphtounicode{arrowtripleright}{21DB}
\pdfglyphtounicode{arrowup}{2191}
\pdfglyphtounicode{arrowupdn}{2195}
\pdfglyphtounicode{arrowupdnbse}{21A8}
\pdfglyphtounicode{arrowupdownbase}{21A8}
\pdfglyphtounicode{arrowupleft}{2196}
\pdfglyphtounicode{arrowupleftofdown}{21C5}
\pdfglyphtounicode{arrowupright}{2197}
\pdfglyphtounicode{arrowupwhite}{21E7}
\pdfglyphtounicode{arrowvertex}{F8E6}
\pdfglyphtounicode{asciicircum}{005E}
\pdfglyphtounicode{asciicircummonospace}{FF3E}
\pdfglyphtounicode{asciitilde}{007E}
\pdfglyphtounicode{asciitildemonospace}{FF5E}
\pdfglyphtounicode{ascript}{0251}
\pdfglyphtounicode{ascriptturned}{0252}
\pdfglyphtounicode{asmallhiragana}{3041}
\pdfglyphtounicode{asmallkatakana}{30A1}
\pdfglyphtounicode{asmallkatakanahalfwidth}{FF67}
\pdfglyphtounicode{asterisk}{002A}
\pdfglyphtounicode{asteriskaltonearabic}{066D}
\pdfglyphtounicode{asteriskarabic}{066D}
\pdfglyphtounicode{asteriskcentered}{2217}
\pdfglyphtounicode{asteriskmath}{2217}
\pdfglyphtounicode{asteriskmonospace}{FF0A}
\pdfglyphtounicode{asterisksmall}{FE61}
\pdfglyphtounicode{asterism}{2042}
\pdfglyphtounicode{asuperior}{0061}
\pdfglyphtounicode{asymptoticallyequal}{2243}
\pdfglyphtounicode{at}{0040}
\pdfglyphtounicode{atilde}{00E3}
\pdfglyphtounicode{atmonospace}{FF20}
\pdfglyphtounicode{atsmall}{FE6B}
\pdfglyphtounicode{aturned}{0250}
\pdfglyphtounicode{aubengali}{0994}
\pdfglyphtounicode{aubopomofo}{3120}
\pdfglyphtounicode{audeva}{0914}
\pdfglyphtounicode{augujarati}{0A94}
\pdfglyphtounicode{augurmukhi}{0A14}
\pdfglyphtounicode{aulengthmarkbengali}{09D7}
\pdfglyphtounicode{aumatragurmukhi}{0A4C}
\pdfglyphtounicode{auvowelsignbengali}{09CC}
\pdfglyphtounicode{auvowelsigndeva}{094C}
\pdfglyphtounicode{auvowelsigngujarati}{0ACC}
\pdfglyphtounicode{avagrahadeva}{093D}
\pdfglyphtounicode{aybarmenian}{0561}
\pdfglyphtounicode{ayin}{05E2}
\pdfglyphtounicode{ayinaltonehebrew}{FB20}
\pdfglyphtounicode{ayinhebrew}{05E2}
\pdfglyphtounicode{b}{0062}
\pdfglyphtounicode{babengali}{09AC}
\pdfglyphtounicode{backslash}{005C}
\pdfglyphtounicode{backslashmonospace}{FF3C}
\pdfglyphtounicode{badeva}{092C}
\pdfglyphtounicode{bagujarati}{0AAC}
\pdfglyphtounicode{bagurmukhi}{0A2C}
\pdfglyphtounicode{bahiragana}{3070}
\pdfglyphtounicode{bahtthai}{0E3F}
\pdfglyphtounicode{bakatakana}{30D0}
\pdfglyphtounicode{bar}{007C}
\pdfglyphtounicode{bardbl}{2225}
\pdfglyphtounicode{barmonospace}{FF5C}
\pdfglyphtounicode{bbopomofo}{3105}
\pdfglyphtounicode{bcircle}{24D1}
\pdfglyphtounicode{bdotaccent}{1E03}
\pdfglyphtounicode{bdotbelow}{1E05}
\pdfglyphtounicode{beamedsixteenthnotes}{266C}
\pdfglyphtounicode{because}{2235}
\pdfglyphtounicode{becyrillic}{0431}
\pdfglyphtounicode{beharabic}{0628}
\pdfglyphtounicode{behfinalarabic}{FE90}
\pdfglyphtounicode{behinitialarabic}{FE91}
\pdfglyphtounicode{behiragana}{3079}
\pdfglyphtounicode{behmedialarabic}{FE92}
\pdfglyphtounicode{behmeeminitialarabic}{FC9F}
\pdfglyphtounicode{behmeemisolatedarabic}{FC08}
\pdfglyphtounicode{behnoonfinalarabic}{FC6D}
\pdfglyphtounicode{bekatakana}{30D9}
\pdfglyphtounicode{benarmenian}{0562}
\pdfglyphtounicode{bet}{05D1}
\pdfglyphtounicode{beta}{03B2}
\pdfglyphtounicode{betasymbolgreek}{03D0}
\pdfglyphtounicode{betdagesh}{FB31}
\pdfglyphtounicode{betdageshhebrew}{FB31}
\pdfglyphtounicode{beth}{2136}
\pdfglyphtounicode{bethebrew}{05D1}
\pdfglyphtounicode{betrafehebrew}{FB4C}
\pdfglyphtounicode{between}{226C}
\pdfglyphtounicode{bhabengali}{09AD}
\pdfglyphtounicode{bhadeva}{092D}
\pdfglyphtounicode{bhagujarati}{0AAD}
\pdfglyphtounicode{bhagurmukhi}{0A2D}
\pdfglyphtounicode{bhook}{0253}
\pdfglyphtounicode{bihiragana}{3073}
\pdfglyphtounicode{bikatakana}{30D3}
\pdfglyphtounicode{bilabialclick}{0298}
\pdfglyphtounicode{bindigurmukhi}{0A02}
\pdfglyphtounicode{birusquare}{3331}
\pdfglyphtounicode{blackcircle}{25CF}
\pdfglyphtounicode{blackdiamond}{25C6}
\pdfglyphtounicode{blackdownpointingtriangle}{25BC}
\pdfglyphtounicode{blackleftpointingpointer}{25C4}
\pdfglyphtounicode{blackleftpointingtriangle}{25C0}
\pdfglyphtounicode{blacklenticularbracketleft}{3010}
\pdfglyphtounicode{blacklenticularbracketleftvertical}{FE3B}
\pdfglyphtounicode{blacklenticularbracketright}{3011}
\pdfglyphtounicode{blacklenticularbracketrightvertical}{FE3C}
\pdfglyphtounicode{blacklowerlefttriangle}{25E3}
\pdfglyphtounicode{blacklowerrighttriangle}{25E2}
\pdfglyphtounicode{blackrectangle}{25AC}
\pdfglyphtounicode{blackrightpointingpointer}{25BA}
\pdfglyphtounicode{blackrightpointingtriangle}{25B6}
\pdfglyphtounicode{blacksmallsquare}{25AA}
\pdfglyphtounicode{blacksmilingface}{263B}
\pdfglyphtounicode{blacksquare}{25A0}
\pdfglyphtounicode{blackstar}{2605}
\pdfglyphtounicode{blackupperlefttriangle}{25E4}
\pdfglyphtounicode{blackupperrighttriangle}{25E5}
\pdfglyphtounicode{blackuppointingsmalltriangle}{25B4}
\pdfglyphtounicode{blackuppointingtriangle}{25B2}
\pdfglyphtounicode{blank}{2423}
\pdfglyphtounicode{blinebelow}{1E07}
\pdfglyphtounicode{block}{2588}
\pdfglyphtounicode{bmonospace}{FF42}
\pdfglyphtounicode{bobaimaithai}{0E1A}
\pdfglyphtounicode{bohiragana}{307C}
\pdfglyphtounicode{bokatakana}{30DC}
\pdfglyphtounicode{bparen}{249D}
\pdfglyphtounicode{bqsquare}{33C3}
\pdfglyphtounicode{braceex}{F8F4}
\pdfglyphtounicode{braceleft}{007B}
\pdfglyphtounicode{braceleftbt}{F8F3}
\pdfglyphtounicode{braceleftmid}{F8F2}
\pdfglyphtounicode{braceleftmonospace}{FF5B}
\pdfglyphtounicode{braceleftsmall}{FE5B}
\pdfglyphtounicode{bracelefttp}{F8F1}
\pdfglyphtounicode{braceleftvertical}{FE37}
\pdfglyphtounicode{braceright}{007D}
\pdfglyphtounicode{bracerightbt}{F8FE}
\pdfglyphtounicode{bracerightmid}{F8FD}
\pdfglyphtounicode{bracerightmonospace}{FF5D}
\pdfglyphtounicode{bracerightsmall}{FE5C}
\pdfglyphtounicode{bracerighttp}{F8FC}
\pdfglyphtounicode{bracerightvertical}{FE38}
\pdfglyphtounicode{bracketleft}{005B}
\pdfglyphtounicode{bracketleftbt}{F8F0}
\pdfglyphtounicode{bracketleftex}{F8EF}
\pdfglyphtounicode{bracketleftmonospace}{FF3B}
\pdfglyphtounicode{bracketlefttp}{F8EE}
\pdfglyphtounicode{bracketright}{005D}
\pdfglyphtounicode{bracketrightbt}{F8FB}
\pdfglyphtounicode{bracketrightex}{F8FA}
\pdfglyphtounicode{bracketrightmonospace}{FF3D}
\pdfglyphtounicode{bracketrighttp}{F8F9}
\pdfglyphtounicode{breve}{02D8}
\pdfglyphtounicode{brevebelowcmb}{032E}
\pdfglyphtounicode{brevecmb}{0306}
\pdfglyphtounicode{breveinvertedbelowcmb}{032F}
\pdfglyphtounicode{breveinvertedcmb}{0311}
\pdfglyphtounicode{breveinverteddoublecmb}{0361}
\pdfglyphtounicode{bridgebelowcmb}{032A}
\pdfglyphtounicode{bridgeinvertedbelowcmb}{033A}
\pdfglyphtounicode{brokenbar}{00A6}
\pdfglyphtounicode{bstroke}{0180}
\pdfglyphtounicode{bsuperior}{0062}
\pdfglyphtounicode{btopbar}{0183}
\pdfglyphtounicode{buhiragana}{3076}
\pdfglyphtounicode{bukatakana}{30D6}
\pdfglyphtounicode{bullet}{2022}
\pdfglyphtounicode{bulletinverse}{25D8}
\pdfglyphtounicode{bulletoperator}{2219}
\pdfglyphtounicode{bullseye}{25CE}
\pdfglyphtounicode{c}{0063}
\pdfglyphtounicode{caarmenian}{056E}
\pdfglyphtounicode{cabengali}{099A}
\pdfglyphtounicode{cacute}{0107}
\pdfglyphtounicode{cadeva}{091A}
\pdfglyphtounicode{cagujarati}{0A9A}
\pdfglyphtounicode{cagurmukhi}{0A1A}
\pdfglyphtounicode{calsquare}{3388}
\pdfglyphtounicode{candrabindubengali}{0981}
\pdfglyphtounicode{candrabinducmb}{0310}
\pdfglyphtounicode{candrabindudeva}{0901}
\pdfglyphtounicode{candrabindugujarati}{0A81}
\pdfglyphtounicode{capslock}{21EA}
\pdfglyphtounicode{careof}{2105}
\pdfglyphtounicode{caron}{02C7}
\pdfglyphtounicode{caronbelowcmb}{032C}
\pdfglyphtounicode{caroncmb}{030C}
\pdfglyphtounicode{carriagereturn}{21B5}
\pdfglyphtounicode{cbopomofo}{3118}
\pdfglyphtounicode{ccaron}{010D}
\pdfglyphtounicode{ccedilla}{00E7}
\pdfglyphtounicode{ccedillaacute}{1E09}
\pdfglyphtounicode{ccircle}{24D2}
\pdfglyphtounicode{ccircumflex}{0109}
\pdfglyphtounicode{ccurl}{0255}
\pdfglyphtounicode{cdot}{010B}
\pdfglyphtounicode{cdotaccent}{010B}
\pdfglyphtounicode{cdsquare}{33C5}
\pdfglyphtounicode{cedilla}{00B8}
\pdfglyphtounicode{cedillacmb}{0327}
\pdfglyphtounicode{ceilingleft}{2308}
\pdfglyphtounicode{ceilingright}{2309}
\pdfglyphtounicode{cent}{00A2}
\pdfglyphtounicode{centigrade}{2103}
\pdfglyphtounicode{centinferior}{00A2}
\pdfglyphtounicode{centmonospace}{FFE0}
\pdfglyphtounicode{centoldstyle}{00A2}
\pdfglyphtounicode{centsuperior}{00A2}
\pdfglyphtounicode{chaarmenian}{0579}
\pdfglyphtounicode{chabengali}{099B}
\pdfglyphtounicode{chadeva}{091B}
\pdfglyphtounicode{chagujarati}{0A9B}
\pdfglyphtounicode{chagurmukhi}{0A1B}
\pdfglyphtounicode{chbopomofo}{3114}
\pdfglyphtounicode{cheabkhasiancyrillic}{04BD}
\pdfglyphtounicode{check}{2713}
\pdfglyphtounicode{checkmark}{2713}
\pdfglyphtounicode{checyrillic}{0447}
\pdfglyphtounicode{chedescenderabkhasiancyrillic}{04BF}
\pdfglyphtounicode{chedescendercyrillic}{04B7}
\pdfglyphtounicode{chedieresiscyrillic}{04F5}
\pdfglyphtounicode{cheharmenian}{0573}
\pdfglyphtounicode{chekhakassiancyrillic}{04CC}
\pdfglyphtounicode{cheverticalstrokecyrillic}{04B9}
\pdfglyphtounicode{chi}{03C7}
\pdfglyphtounicode{chieuchacirclekorean}{3277}
\pdfglyphtounicode{chieuchaparenkorean}{3217}
\pdfglyphtounicode{chieuchcirclekorean}{3269}
\pdfglyphtounicode{chieuchkorean}{314A}
\pdfglyphtounicode{chieuchparenkorean}{3209}
\pdfglyphtounicode{chochangthai}{0E0A}
\pdfglyphtounicode{chochanthai}{0E08}
\pdfglyphtounicode{chochingthai}{0E09}
\pdfglyphtounicode{chochoethai}{0E0C}
\pdfglyphtounicode{chook}{0188}
\pdfglyphtounicode{cieucacirclekorean}{3276}
\pdfglyphtounicode{cieucaparenkorean}{3216}
\pdfglyphtounicode{cieuccirclekorean}{3268}
\pdfglyphtounicode{cieuckorean}{3148}
\pdfglyphtounicode{cieucparenkorean}{3208}
\pdfglyphtounicode{cieucuparenkorean}{321C}
\pdfglyphtounicode{circle}{25CB}
\pdfglyphtounicode{circleR}{00AE}
\pdfglyphtounicode{circleS}{24C8}
\pdfglyphtounicode{circleasterisk}{229B}
\pdfglyphtounicode{circlecopyrt}{20DD}
\pdfglyphtounicode{circledivide}{2298}
\pdfglyphtounicode{circledot}{2299}
\pdfglyphtounicode{circleequal}{229C}
\pdfglyphtounicode{circleminus}{2296}
\pdfglyphtounicode{circlemultiply}{2297}
\pdfglyphtounicode{circleot}{2299}
\pdfglyphtounicode{circleplus}{2295}
\pdfglyphtounicode{circlepostalmark}{3036}
\pdfglyphtounicode{circlering}{229A}
\pdfglyphtounicode{circlewithlefthalfblack}{25D0}
\pdfglyphtounicode{circlewithrighthalfblack}{25D1}
\pdfglyphtounicode{circumflex}{02C6}
\pdfglyphtounicode{circumflexbelowcmb}{032D}
\pdfglyphtounicode{circumflexcmb}{0302}
\pdfglyphtounicode{clear}{2327}
\pdfglyphtounicode{clickalveolar}{01C2}
\pdfglyphtounicode{clickdental}{01C0}
\pdfglyphtounicode{clicklateral}{01C1}
\pdfglyphtounicode{clickretroflex}{01C3}
\pdfglyphtounicode{clockwise}{27F3}
\pdfglyphtounicode{club}{2663}
\pdfglyphtounicode{clubsuitblack}{2663}
\pdfglyphtounicode{clubsuitwhite}{2667}
\pdfglyphtounicode{cmcubedsquare}{33A4}
\pdfglyphtounicode{cmonospace}{FF43}
\pdfglyphtounicode{cmsquaredsquare}{33A0}
\pdfglyphtounicode{coarmenian}{0581}
\pdfglyphtounicode{colon}{003A}
\pdfglyphtounicode{colonmonetary}{20A1}
\pdfglyphtounicode{colonmonospace}{FF1A}
\pdfglyphtounicode{colonsign}{20A1}
\pdfglyphtounicode{colonsmall}{FE55}
\pdfglyphtounicode{colontriangularhalfmod}{02D1}
\pdfglyphtounicode{colontriangularmod}{02D0}
\pdfglyphtounicode{comma}{002C}
\pdfglyphtounicode{commaabovecmb}{0313}
\pdfglyphtounicode{commaaboverightcmb}{0315}
\pdfglyphtounicode{commaaccent}{F6C3}
\pdfglyphtounicode{commaarabic}{060C}
\pdfglyphtounicode{commaarmenian}{055D}
\pdfglyphtounicode{commainferior}{002C}
\pdfglyphtounicode{commamonospace}{FF0C}
\pdfglyphtounicode{commareversedabovecmb}{0314}
\pdfglyphtounicode{commareversedmod}{02BD}
\pdfglyphtounicode{commasmall}{FE50}
\pdfglyphtounicode{commasuperior}{002C}
\pdfglyphtounicode{commaturnedabovecmb}{0312}
\pdfglyphtounicode{commaturnedmod}{02BB}
\pdfglyphtounicode{compass}{263C}
\pdfglyphtounicode{complement}{2201}
\pdfglyphtounicode{compwordmark}{200C}
\pdfglyphtounicode{congruent}{2245}
\pdfglyphtounicode{contourintegral}{222E}
\pdfglyphtounicode{control}{2303}
\pdfglyphtounicode{controlACK}{0006}
\pdfglyphtounicode{controlBEL}{0007}
\pdfglyphtounicode{controlBS}{0008}
\pdfglyphtounicode{controlCAN}{0018}
\pdfglyphtounicode{controlCR}{000D}
\pdfglyphtounicode{controlDC1}{0011}
\pdfglyphtounicode{controlDC2}{0012}
\pdfglyphtounicode{controlDC3}{0013}
\pdfglyphtounicode{controlDC4}{0014}
\pdfglyphtounicode{controlDEL}{007F}
\pdfglyphtounicode{controlDLE}{0010}
\pdfglyphtounicode{controlEM}{0019}
\pdfglyphtounicode{controlENQ}{0005}
\pdfglyphtounicode{controlEOT}{0004}
\pdfglyphtounicode{controlESC}{001B}
\pdfglyphtounicode{controlETB}{0017}
\pdfglyphtounicode{controlETX}{0003}
\pdfglyphtounicode{controlFF}{000C}
\pdfglyphtounicode{controlFS}{001C}
\pdfglyphtounicode{controlGS}{001D}
\pdfglyphtounicode{controlHT}{0009}
\pdfglyphtounicode{controlLF}{000A}
\pdfglyphtounicode{controlNAK}{0015}
\pdfglyphtounicode{controlRS}{001E}
\pdfglyphtounicode{controlSI}{000F}
\pdfglyphtounicode{controlSO}{000E}
\pdfglyphtounicode{controlSOT}{0002}
\pdfglyphtounicode{controlSTX}{0001}
\pdfglyphtounicode{controlSUB}{001A}
\pdfglyphtounicode{controlSYN}{0016}
\pdfglyphtounicode{controlUS}{001F}
\pdfglyphtounicode{controlVT}{000B}
\pdfglyphtounicode{coproduct}{2A3F}
\pdfglyphtounicode{copyright}{00A9}
\pdfglyphtounicode{copyrightsans}{00A9}
\pdfglyphtounicode{copyrightserif}{00A9}
\pdfglyphtounicode{cornerbracketleft}{300C}
\pdfglyphtounicode{cornerbracketlefthalfwidth}{FF62}
\pdfglyphtounicode{cornerbracketleftvertical}{FE41}
\pdfglyphtounicode{cornerbracketright}{300D}
\pdfglyphtounicode{cornerbracketrighthalfwidth}{FF63}
\pdfglyphtounicode{cornerbracketrightvertical}{FE42}
\pdfglyphtounicode{corporationsquare}{337F}
\pdfglyphtounicode{cosquare}{33C7}
\pdfglyphtounicode{coverkgsquare}{33C6}
\pdfglyphtounicode{cparen}{249E}
\pdfglyphtounicode{cruzeiro}{20A2}
\pdfglyphtounicode{cstretched}{0297}
\pdfglyphtounicode{ct}{0063 0074}
\pdfglyphtounicode{curlyand}{22CF}
\pdfglyphtounicode{curlyleft}{21AB}
\pdfglyphtounicode{curlyor}{22CE}
\pdfglyphtounicode{curlyright}{21AC}
\pdfglyphtounicode{currency}{00A4}
\pdfglyphtounicode{cwm}{200C}
\pdfglyphtounicode{cyrBreve}{02D8}
\pdfglyphtounicode{cyrFlex}{00A0 0311}
\pdfglyphtounicode{cyrbreve}{02D8}
\pdfglyphtounicode{cyrflex}{00A0 0311}
\pdfglyphtounicode{d}{0064}
\pdfglyphtounicode{daarmenian}{0564}
\pdfglyphtounicode{dabengali}{09A6}
\pdfglyphtounicode{dadarabic}{0636}
\pdfglyphtounicode{dadeva}{0926}
\pdfglyphtounicode{dadfinalarabic}{FEBE}
\pdfglyphtounicode{dadinitialarabic}{FEBF}
\pdfglyphtounicode{dadmedialarabic}{FEC0}
\pdfglyphtounicode{dagesh}{05BC}
\pdfglyphtounicode{dageshhebrew}{05BC}
\pdfglyphtounicode{dagger}{2020}
\pdfglyphtounicode{daggerdbl}{2021}
\pdfglyphtounicode{dagujarati}{0AA6}
\pdfglyphtounicode{dagurmukhi}{0A26}
\pdfglyphtounicode{dahiragana}{3060}
\pdfglyphtounicode{dakatakana}{30C0}
\pdfglyphtounicode{dalarabic}{062F}
\pdfglyphtounicode{dalet}{05D3}
\pdfglyphtounicode{daletdagesh}{FB33}
\pdfglyphtounicode{daletdageshhebrew}{FB33}
\pdfglyphtounicode{daleth}{2138}
\pdfglyphtounicode{dalethatafpatah}{05D3 05B2}
\pdfglyphtounicode{dalethatafpatahhebrew}{05D3 05B2}
\pdfglyphtounicode{dalethatafsegol}{05D3 05B1}
\pdfglyphtounicode{dalethatafsegolhebrew}{05D3 05B1}
\pdfglyphtounicode{dalethebrew}{05D3}
\pdfglyphtounicode{dalethiriq}{05D3 05B4}
\pdfglyphtounicode{dalethiriqhebrew}{05D3 05B4}
\pdfglyphtounicode{daletholam}{05D3 05B9}
\pdfglyphtounicode{daletholamhebrew}{05D3 05B9}
\pdfglyphtounicode{daletpatah}{05D3 05B7}
\pdfglyphtounicode{daletpatahhebrew}{05D3 05B7}
\pdfglyphtounicode{daletqamats}{05D3 05B8}
\pdfglyphtounicode{daletqamatshebrew}{05D3 05B8}
\pdfglyphtounicode{daletqubuts}{05D3 05BB}
\pdfglyphtounicode{daletqubutshebrew}{05D3 05BB}
\pdfglyphtounicode{daletsegol}{05D3 05B6}
\pdfglyphtounicode{daletsegolhebrew}{05D3 05B6}
\pdfglyphtounicode{daletsheva}{05D3 05B0}
\pdfglyphtounicode{daletshevahebrew}{05D3 05B0}
\pdfglyphtounicode{dalettsere}{05D3 05B5}
\pdfglyphtounicode{dalettserehebrew}{05D3 05B5}
\pdfglyphtounicode{dalfinalarabic}{FEAA}
\pdfglyphtounicode{dammaarabic}{064F}
\pdfglyphtounicode{dammalowarabic}{064F}
\pdfglyphtounicode{dammatanaltonearabic}{064C}
\pdfglyphtounicode{dammatanarabic}{064C}
\pdfglyphtounicode{danda}{0964}
\pdfglyphtounicode{dargahebrew}{05A7}
\pdfglyphtounicode{dargalefthebrew}{05A7}
\pdfglyphtounicode{dasiapneumatacyrilliccmb}{0485}
\pdfglyphtounicode{dbar}{0111}
\pdfglyphtounicode{dblGrave}{00A0 030F}
\pdfglyphtounicode{dblanglebracketleft}{300A}
\pdfglyphtounicode{dblanglebracketleftvertical}{FE3D}
\pdfglyphtounicode{dblanglebracketright}{300B}
\pdfglyphtounicode{dblanglebracketrightvertical}{FE3E}
\pdfglyphtounicode{dblarchinvertedbelowcmb}{032B}
\pdfglyphtounicode{dblarrowdwn}{21CA}
\pdfglyphtounicode{dblarrowheadleft}{219E}
\pdfglyphtounicode{dblarrowheadright}{21A0}
\pdfglyphtounicode{dblarrowleft}{21D4}
\pdfglyphtounicode{dblarrowright}{21D2}
\pdfglyphtounicode{dblarrowup}{21C8}
\pdfglyphtounicode{dblbracketleft}{27E6}
\pdfglyphtounicode{dblbracketright}{27E7}
\pdfglyphtounicode{dbldanda}{0965}
\pdfglyphtounicode{dblgrave}{00A0 030F}
\pdfglyphtounicode{dblgravecmb}{030F}
\pdfglyphtounicode{dblintegral}{222C}
\pdfglyphtounicode{dbllowline}{2017}
\pdfglyphtounicode{dbllowlinecmb}{0333}
\pdfglyphtounicode{dbloverlinecmb}{033F}
\pdfglyphtounicode{dblprimemod}{02BA}
\pdfglyphtounicode{dblverticalbar}{2016}
\pdfglyphtounicode{dblverticallineabovecmb}{030E}
\pdfglyphtounicode{dbopomofo}{3109}
\pdfglyphtounicode{dbsquare}{33C8}
\pdfglyphtounicode{dcaron}{010F}
\pdfglyphtounicode{dcedilla}{1E11}
\pdfglyphtounicode{dcircle}{24D3}
\pdfglyphtounicode{dcircumflexbelow}{1E13}
\pdfglyphtounicode{dcroat}{0111}
\pdfglyphtounicode{ddabengali}{09A1}
\pdfglyphtounicode{ddadeva}{0921}
\pdfglyphtounicode{ddagujarati}{0AA1}
\pdfglyphtounicode{ddagurmukhi}{0A21}
\pdfglyphtounicode{ddalarabic}{0688}
\pdfglyphtounicode{ddalfinalarabic}{FB89}
\pdfglyphtounicode{dddhadeva}{095C}
\pdfglyphtounicode{ddhabengali}{09A2}
\pdfglyphtounicode{ddhadeva}{0922}
\pdfglyphtounicode{ddhagujarati}{0AA2}
\pdfglyphtounicode{ddhagurmukhi}{0A22}
\pdfglyphtounicode{ddotaccent}{1E0B}
\pdfglyphtounicode{ddotbelow}{1E0D}
\pdfglyphtounicode{decimalseparatorarabic}{066B}
\pdfglyphtounicode{decimalseparatorpersian}{066B}
\pdfglyphtounicode{decyrillic}{0434}
\pdfglyphtounicode{defines}{225C}
\pdfglyphtounicode{degree}{00B0}
\pdfglyphtounicode{dehihebrew}{05AD}
\pdfglyphtounicode{dehiragana}{3067}
\pdfglyphtounicode{deicoptic}{03EF}
\pdfglyphtounicode{dekatakana}{30C7}
\pdfglyphtounicode{deleteleft}{232B}
\pdfglyphtounicode{deleteright}{2326}
\pdfglyphtounicode{delta}{03B4}
\pdfglyphtounicode{deltaturned}{018D}
\pdfglyphtounicode{denominatorminusonenumeratorbengali}{09F8}
\pdfglyphtounicode{dezh}{02A4}
\pdfglyphtounicode{dhabengali}{09A7}
\pdfglyphtounicode{dhadeva}{0927}
\pdfglyphtounicode{dhagujarati}{0AA7}
\pdfglyphtounicode{dhagurmukhi}{0A27}
\pdfglyphtounicode{dhook}{0257}
\pdfglyphtounicode{dialytikatonos}{0385}
\pdfglyphtounicode{dialytikatonoscmb}{0344}
\pdfglyphtounicode{diamond}{2666}
\pdfglyphtounicode{diamondmath}{22C4}
\pdfglyphtounicode{diamondsolid}{2666}
\pdfglyphtounicode{diamondsuitwhite}{2662}
\pdfglyphtounicode{dieresis}{00A8}
\pdfglyphtounicode{dieresisacute}{00A0 0308 0301}
\pdfglyphtounicode{dieresisbelowcmb}{0324}
\pdfglyphtounicode{dieresiscmb}{0308}
\pdfglyphtounicode{dieresisgrave}{00A0 0308 0300}
\pdfglyphtounicode{dieresistonos}{0385}
\pdfglyphtounicode{difference}{224F}
\pdfglyphtounicode{dihiragana}{3062}
\pdfglyphtounicode{dikatakana}{30C2}
\pdfglyphtounicode{dittomark}{3003}
\pdfglyphtounicode{divide}{00F7}
\pdfglyphtounicode{dividemultiply}{22C7}
\pdfglyphtounicode{divides}{2223}
\pdfglyphtounicode{divisionslash}{2215}
\pdfglyphtounicode{djecyrillic}{0452}
\pdfglyphtounicode{dkshade}{2593}
\pdfglyphtounicode{dlinebelow}{1E0F}
\pdfglyphtounicode{dlsquare}{3397}
\pdfglyphtounicode{dmacron}{0111}
\pdfglyphtounicode{dmonospace}{FF44}
\pdfglyphtounicode{dnblock}{2584}
\pdfglyphtounicode{dochadathai}{0E0E}
\pdfglyphtounicode{dodekthai}{0E14}
\pdfglyphtounicode{dohiragana}{3069}
\pdfglyphtounicode{dokatakana}{30C9}
\pdfglyphtounicode{dollar}{0024}
\pdfglyphtounicode{dollarinferior}{0024}
\pdfglyphtounicode{dollarmonospace}{FF04}
\pdfglyphtounicode{dollaroldstyle}{0024}
\pdfglyphtounicode{dollarsmall}{FE69}
\pdfglyphtounicode{dollarsuperior}{0024}
\pdfglyphtounicode{dong}{20AB}
\pdfglyphtounicode{dorusquare}{3326}
\pdfglyphtounicode{dotaccent}{02D9}
\pdfglyphtounicode{dotaccentcmb}{0307}
\pdfglyphtounicode{dotbelowcmb}{0323}
\pdfglyphtounicode{dotbelowcomb}{0323}
\pdfglyphtounicode{dotkatakana}{30FB}
\pdfglyphtounicode{dotlessi}{0131}
\pdfglyphtounicode{dotlessj}{0237}
\pdfglyphtounicode{dotlessjstrokehook}{0284}
\pdfglyphtounicode{dotmath}{22C5}
\pdfglyphtounicode{dotplus}{2214}
\pdfglyphtounicode{dottedcircle}{25CC}
\pdfglyphtounicode{doubleyodpatah}{FB1F}
\pdfglyphtounicode{doubleyodpatahhebrew}{FB1F}
\pdfglyphtounicode{downfall}{22CE}
\pdfglyphtounicode{downslope}{29F9}
\pdfglyphtounicode{downtackbelowcmb}{031E}
\pdfglyphtounicode{downtackmod}{02D5}
\pdfglyphtounicode{dparen}{249F}
\pdfglyphtounicode{dsuperior}{0064}
\pdfglyphtounicode{dtail}{0256}
\pdfglyphtounicode{dtopbar}{018C}
\pdfglyphtounicode{duhiragana}{3065}
\pdfglyphtounicode{dukatakana}{30C5}
\pdfglyphtounicode{dz}{01F3}
\pdfglyphtounicode{dzaltone}{02A3}
\pdfglyphtounicode{dzcaron}{01C6}
\pdfglyphtounicode{dzcurl}{02A5}
\pdfglyphtounicode{dzeabkhasiancyrillic}{04E1}
\pdfglyphtounicode{dzecyrillic}{0455}
\pdfglyphtounicode{dzhecyrillic}{045F}
\pdfglyphtounicode{e}{0065}
\pdfglyphtounicode{eacute}{00E9}
\pdfglyphtounicode{earth}{2641}
\pdfglyphtounicode{ebengali}{098F}
\pdfglyphtounicode{ebopomofo}{311C}
\pdfglyphtounicode{ebreve}{0115}
\pdfglyphtounicode{ecandradeva}{090D}
\pdfglyphtounicode{ecandragujarati}{0A8D}
\pdfglyphtounicode{ecandravowelsigndeva}{0945}
\pdfglyphtounicode{ecandravowelsigngujarati}{0AC5}
\pdfglyphtounicode{ecaron}{011B}
\pdfglyphtounicode{ecedillabreve}{1E1D}
\pdfglyphtounicode{echarmenian}{0565}
\pdfglyphtounicode{echyiwnarmenian}{0587}
\pdfglyphtounicode{ecircle}{24D4}
\pdfglyphtounicode{ecircumflex}{00EA}
\pdfglyphtounicode{ecircumflexacute}{1EBF}
\pdfglyphtounicode{ecircumflexbelow}{1E19}
\pdfglyphtounicode{ecircumflexdotbelow}{1EC7}
\pdfglyphtounicode{ecircumflexgrave}{1EC1}
\pdfglyphtounicode{ecircumflexhookabove}{1EC3}
\pdfglyphtounicode{ecircumflextilde}{1EC5}
\pdfglyphtounicode{ecyrillic}{0454}
\pdfglyphtounicode{edblgrave}{0205}
\pdfglyphtounicode{edeva}{090F}
\pdfglyphtounicode{edieresis}{00EB}
\pdfglyphtounicode{edot}{0117}
\pdfglyphtounicode{edotaccent}{0117}
\pdfglyphtounicode{edotbelow}{1EB9}
\pdfglyphtounicode{eegurmukhi}{0A0F}
\pdfglyphtounicode{eematragurmukhi}{0A47}
\pdfglyphtounicode{efcyrillic}{0444}
\pdfglyphtounicode{egrave}{00E8}
\pdfglyphtounicode{egujarati}{0A8F}
\pdfglyphtounicode{eharmenian}{0567}
\pdfglyphtounicode{ehbopomofo}{311D}
\pdfglyphtounicode{ehiragana}{3048}
\pdfglyphtounicode{ehookabove}{1EBB}
\pdfglyphtounicode{eibopomofo}{311F}
\pdfglyphtounicode{eight}{0038}
\pdfglyphtounicode{eightarabic}{0668}
\pdfglyphtounicode{eightbengali}{09EE}
\pdfglyphtounicode{eightcircle}{2467}
\pdfglyphtounicode{eightcircleinversesansserif}{2791}
\pdfglyphtounicode{eightdeva}{096E}
\pdfglyphtounicode{eighteencircle}{2471}
\pdfglyphtounicode{eighteenparen}{2485}
\pdfglyphtounicode{eighteenperiod}{2499}
\pdfglyphtounicode{eightgujarati}{0AEE}
\pdfglyphtounicode{eightgurmukhi}{0A6E}
\pdfglyphtounicode{eighthackarabic}{0668}
\pdfglyphtounicode{eighthangzhou}{3028}
\pdfglyphtounicode{eighthnotebeamed}{266B}
\pdfglyphtounicode{eightideographicparen}{3227}
\pdfglyphtounicode{eightinferior}{2088}
\pdfglyphtounicode{eightmonospace}{FF18}
\pdfglyphtounicode{eightoldstyle}{0038}
\pdfglyphtounicode{eightparen}{247B}
\pdfglyphtounicode{eightperiod}{248F}
\pdfglyphtounicode{eightpersian}{06F8}
\pdfglyphtounicode{eightroman}{2177}
\pdfglyphtounicode{eightsuperior}{2078}
\pdfglyphtounicode{eightthai}{0E58}
\pdfglyphtounicode{einvertedbreve}{0207}
\pdfglyphtounicode{eiotifiedcyrillic}{0465}
\pdfglyphtounicode{ekatakana}{30A8}
\pdfglyphtounicode{ekatakanahalfwidth}{FF74}
\pdfglyphtounicode{ekonkargurmukhi}{0A74}
\pdfglyphtounicode{ekorean}{3154}
\pdfglyphtounicode{elcyrillic}{043B}
\pdfglyphtounicode{element}{2208}
\pdfglyphtounicode{elevencircle}{246A}
\pdfglyphtounicode{elevenparen}{247E}
\pdfglyphtounicode{elevenperiod}{2492}
\pdfglyphtounicode{elevenroman}{217A}
\pdfglyphtounicode{ellipsis}{2026}
\pdfglyphtounicode{ellipsisvertical}{22EE}
\pdfglyphtounicode{emacron}{0113}
\pdfglyphtounicode{emacronacute}{1E17}
\pdfglyphtounicode{emacrongrave}{1E15}
\pdfglyphtounicode{emcyrillic}{043C}
\pdfglyphtounicode{emdash}{2014}
\pdfglyphtounicode{emdashvertical}{FE31}
\pdfglyphtounicode{emonospace}{FF45}
\pdfglyphtounicode{emphasismarkarmenian}{055B}
\pdfglyphtounicode{emptyset}{2205}
\pdfglyphtounicode{enbopomofo}{3123}
\pdfglyphtounicode{encyrillic}{043D}
\pdfglyphtounicode{endash}{2013}
\pdfglyphtounicode{endashvertical}{FE32}
\pdfglyphtounicode{endescendercyrillic}{04A3}
\pdfglyphtounicode{eng}{014B}
\pdfglyphtounicode{engbopomofo}{3125}
\pdfglyphtounicode{enghecyrillic}{04A5}
\pdfglyphtounicode{enhookcyrillic}{04C8}
\pdfglyphtounicode{enspace}{2002}
\pdfglyphtounicode{eogonek}{0119}
\pdfglyphtounicode{eokorean}{3153}
\pdfglyphtounicode{eopen}{025B}
\pdfglyphtounicode{eopenclosed}{029A}
\pdfglyphtounicode{eopenreversed}{025C}
\pdfglyphtounicode{eopenreversedclosed}{025E}
\pdfglyphtounicode{eopenreversedhook}{025D}
\pdfglyphtounicode{eparen}{24A0}
\pdfglyphtounicode{epsilon}{03B5}
\pdfglyphtounicode{epsilon1}{03F5}
\pdfglyphtounicode{epsiloninv}{03F6}
\pdfglyphtounicode{epsilontonos}{03AD}
\pdfglyphtounicode{equal}{003D}
\pdfglyphtounicode{equaldotleftright}{2252}
\pdfglyphtounicode{equaldotrightleft}{2253}
\pdfglyphtounicode{equalmonospace}{FF1D}
\pdfglyphtounicode{equalorfollows}{22DF}
\pdfglyphtounicode{equalorgreater}{2A96}
\pdfglyphtounicode{equalorless}{2A95}
\pdfglyphtounicode{equalorprecedes}{22DE}
\pdfglyphtounicode{equalorsimilar}{2242}
\pdfglyphtounicode{equalsdots}{2251}
\pdfglyphtounicode{equalsmall}{FE66}
\pdfglyphtounicode{equalsuperior}{207C}
\pdfglyphtounicode{equivalence}{2261}
\pdfglyphtounicode{equivasymptotic}{224D}
\pdfglyphtounicode{erbopomofo}{3126}
\pdfglyphtounicode{ercyrillic}{0440}
\pdfglyphtounicode{ereversed}{0258}
\pdfglyphtounicode{ereversedcyrillic}{044D}
\pdfglyphtounicode{escyrillic}{0441}
\pdfglyphtounicode{esdescendercyrillic}{04AB}
\pdfglyphtounicode{esh}{0283}
\pdfglyphtounicode{eshcurl}{0286}
\pdfglyphtounicode{eshortdeva}{090E}
\pdfglyphtounicode{eshortvowelsigndeva}{0946}
\pdfglyphtounicode{eshreversedloop}{01AA}
\pdfglyphtounicode{eshsquatreversed}{0285}
\pdfglyphtounicode{esmallhiragana}{3047}
\pdfglyphtounicode{esmallkatakana}{30A7}
\pdfglyphtounicode{esmallkatakanahalfwidth}{FF6A}
\pdfglyphtounicode{estimated}{212E}
\pdfglyphtounicode{esuperior}{0065}
\pdfglyphtounicode{eta}{03B7}
\pdfglyphtounicode{etarmenian}{0568}
\pdfglyphtounicode{etatonos}{03AE}
\pdfglyphtounicode{eth}{00F0}
\pdfglyphtounicode{etilde}{1EBD}
\pdfglyphtounicode{etildebelow}{1E1B}
\pdfglyphtounicode{etnahtafoukhhebrew}{0591}
\pdfglyphtounicode{etnahtafoukhlefthebrew}{0591}
\pdfglyphtounicode{etnahtahebrew}{0591}
\pdfglyphtounicode{etnahtalefthebrew}{0591}
\pdfglyphtounicode{eturned}{01DD}
\pdfglyphtounicode{eukorean}{3161}
\pdfglyphtounicode{euro}{20AC}
\pdfglyphtounicode{evowelsignbengali}{09C7}
\pdfglyphtounicode{evowelsigndeva}{0947}
\pdfglyphtounicode{evowelsigngujarati}{0AC7}
\pdfglyphtounicode{exclam}{0021}
\pdfglyphtounicode{exclamarmenian}{055C}
\pdfglyphtounicode{exclamdbl}{203C}
\pdfglyphtounicode{exclamdown}{00A1}
\pdfglyphtounicode{exclamdownsmall}{00A1}
\pdfglyphtounicode{exclammonospace}{FF01}
\pdfglyphtounicode{exclamsmall}{0021}
\pdfglyphtounicode{existential}{2203}
\pdfglyphtounicode{ezh}{0292}
\pdfglyphtounicode{ezhcaron}{01EF}
\pdfglyphtounicode{ezhcurl}{0293}
\pdfglyphtounicode{ezhreversed}{01B9}
\pdfglyphtounicode{ezhtail}{01BA}
\pdfglyphtounicode{f}{0066}
\pdfglyphtounicode{fadeva}{095E}
\pdfglyphtounicode{fagurmukhi}{0A5E}
\pdfglyphtounicode{fahrenheit}{2109}
\pdfglyphtounicode{fathaarabic}{064E}
\pdfglyphtounicode{fathalowarabic}{064E}
\pdfglyphtounicode{fathatanarabic}{064B}
\pdfglyphtounicode{fbopomofo}{3108}
\pdfglyphtounicode{fcircle}{24D5}
\pdfglyphtounicode{fdotaccent}{1E1F}
\pdfglyphtounicode{feharabic}{0641}
\pdfglyphtounicode{feharmenian}{0586}
\pdfglyphtounicode{fehfinalarabic}{FED2}
\pdfglyphtounicode{fehinitialarabic}{FED3}
\pdfglyphtounicode{fehmedialarabic}{FED4}
\pdfglyphtounicode{feicoptic}{03E5}
\pdfglyphtounicode{female}{2640}
\pdfglyphtounicode{ff}{0066 0066}
\pdfglyphtounicode{ffi}{0066 0066 0069}
\pdfglyphtounicode{ffl}{0066 0066 006C}
\pdfglyphtounicode{fi}{0066 0069}
\pdfglyphtounicode{fifteencircle}{246E}
\pdfglyphtounicode{fifteenparen}{2482}
\pdfglyphtounicode{fifteenperiod}{2496}
\pdfglyphtounicode{figuredash}{2012}
\pdfglyphtounicode{filledbox}{25A0}
\pdfglyphtounicode{filledrect}{25AC}
\pdfglyphtounicode{finalkaf}{05DA}
\pdfglyphtounicode{finalkafdagesh}{FB3A}
\pdfglyphtounicode{finalkafdageshhebrew}{FB3A}
\pdfglyphtounicode{finalkafhebrew}{05DA}
\pdfglyphtounicode{finalkafqamats}{05DA 05B8}
\pdfglyphtounicode{finalkafqamatshebrew}{05DA 05B8}
\pdfglyphtounicode{finalkafsheva}{05DA 05B0}
\pdfglyphtounicode{finalkafshevahebrew}{05DA 05B0}
\pdfglyphtounicode{finalmem}{05DD}
\pdfglyphtounicode{finalmemhebrew}{05DD}
\pdfglyphtounicode{finalnun}{05DF}
\pdfglyphtounicode{finalnunhebrew}{05DF}
\pdfglyphtounicode{finalpe}{05E3}
\pdfglyphtounicode{finalpehebrew}{05E3}
\pdfglyphtounicode{finaltsadi}{05E5}
\pdfglyphtounicode{finaltsadihebrew}{05E5}
\pdfglyphtounicode{firsttonechinese}{02C9}
\pdfglyphtounicode{fisheye}{25C9}
\pdfglyphtounicode{fitacyrillic}{0473}
\pdfglyphtounicode{five}{0035}
\pdfglyphtounicode{fivearabic}{0665}
\pdfglyphtounicode{fivebengali}{09EB}
\pdfglyphtounicode{fivecircle}{2464}
\pdfglyphtounicode{fivecircleinversesansserif}{278E}
\pdfglyphtounicode{fivedeva}{096B}
\pdfglyphtounicode{fiveeighths}{215D}
\pdfglyphtounicode{fivegujarati}{0AEB}
\pdfglyphtounicode{fivegurmukhi}{0A6B}
\pdfglyphtounicode{fivehackarabic}{0665}
\pdfglyphtounicode{fivehangzhou}{3025}
\pdfglyphtounicode{fiveideographicparen}{3224}
\pdfglyphtounicode{fiveinferior}{2085}
\pdfglyphtounicode{fivemonospace}{FF15}
\pdfglyphtounicode{fiveoldstyle}{0035}
\pdfglyphtounicode{fiveparen}{2478}
\pdfglyphtounicode{fiveperiod}{248C}
\pdfglyphtounicode{fivepersian}{06F5}
\pdfglyphtounicode{fiveroman}{2174}
\pdfglyphtounicode{fivesuperior}{2075}
\pdfglyphtounicode{fivethai}{0E55}
\pdfglyphtounicode{fl}{0066 006C}
\pdfglyphtounicode{flat}{266D}
\pdfglyphtounicode{floorleft}{230A}
\pdfglyphtounicode{floorright}{230B}
\pdfglyphtounicode{florin}{0192}
\pdfglyphtounicode{fmonospace}{FF46}
\pdfglyphtounicode{fmsquare}{3399}
\pdfglyphtounicode{fofanthai}{0E1F}
\pdfglyphtounicode{fofathai}{0E1D}
\pdfglyphtounicode{follownotdbleqv}{2ABA}
\pdfglyphtounicode{follownotslnteql}{2AB6}
\pdfglyphtounicode{followornoteqvlnt}{22E9}
\pdfglyphtounicode{follows}{227B}
\pdfglyphtounicode{followsequal}{2AB0}
\pdfglyphtounicode{followsorcurly}{227D}
\pdfglyphtounicode{followsorequal}{227F}
\pdfglyphtounicode{fongmanthai}{0E4F}
\pdfglyphtounicode{forall}{2200}
\pdfglyphtounicode{forces}{22A9}
\pdfglyphtounicode{forcesbar}{22AA}
\pdfglyphtounicode{fork}{22D4}
\pdfglyphtounicode{four}{0034}
\pdfglyphtounicode{fourarabic}{0664}
\pdfglyphtounicode{fourbengali}{09EA}
\pdfglyphtounicode{fourcircle}{2463}
\pdfglyphtounicode{fourcircleinversesansserif}{278D}
\pdfglyphtounicode{fourdeva}{096A}
\pdfglyphtounicode{fourgujarati}{0AEA}
\pdfglyphtounicode{fourgurmukhi}{0A6A}
\pdfglyphtounicode{fourhackarabic}{0664}
\pdfglyphtounicode{fourhangzhou}{3024}
\pdfglyphtounicode{fourideographicparen}{3223}
\pdfglyphtounicode{fourinferior}{2084}
\pdfglyphtounicode{fourmonospace}{FF14}
\pdfglyphtounicode{fournumeratorbengali}{09F7}
\pdfglyphtounicode{fouroldstyle}{0034}
\pdfglyphtounicode{fourparen}{2477}
\pdfglyphtounicode{fourperiod}{248B}
\pdfglyphtounicode{fourpersian}{06F4}
\pdfglyphtounicode{fourroman}{2173}
\pdfglyphtounicode{foursuperior}{2074}
\pdfglyphtounicode{fourteencircle}{246D}
\pdfglyphtounicode{fourteenparen}{2481}
\pdfglyphtounicode{fourteenperiod}{2495}
\pdfglyphtounicode{fourthai}{0E54}
\pdfglyphtounicode{fourthtonechinese}{02CB}
\pdfglyphtounicode{fparen}{24A1}
\pdfglyphtounicode{fraction}{2044}
\pdfglyphtounicode{franc}{20A3}
\pdfglyphtounicode{frown}{2322}
\pdfglyphtounicode{g}{0067}
\pdfglyphtounicode{gabengali}{0997}
\pdfglyphtounicode{gacute}{01F5}
\pdfglyphtounicode{gadeva}{0917}
\pdfglyphtounicode{gafarabic}{06AF}
\pdfglyphtounicode{gaffinalarabic}{FB93}
\pdfglyphtounicode{gafinitialarabic}{FB94}
\pdfglyphtounicode{gafmedialarabic}{FB95}
\pdfglyphtounicode{gagujarati}{0A97}
\pdfglyphtounicode{gagurmukhi}{0A17}
\pdfglyphtounicode{gahiragana}{304C}
\pdfglyphtounicode{gakatakana}{30AC}
\pdfglyphtounicode{gamma}{03B3}
\pdfglyphtounicode{gammalatinsmall}{0263}
\pdfglyphtounicode{gammasuperior}{02E0}
\pdfglyphtounicode{gangiacoptic}{03EB}
\pdfglyphtounicode{gbopomofo}{310D}
\pdfglyphtounicode{gbreve}{011F}
\pdfglyphtounicode{gcaron}{01E7}
\pdfglyphtounicode{gcedilla}{0123}
\pdfglyphtounicode{gcircle}{24D6}
\pdfglyphtounicode{gcircumflex}{011D}
\pdfglyphtounicode{gcommaaccent}{0123}
\pdfglyphtounicode{gdot}{0121}
\pdfglyphtounicode{gdotaccent}{0121}
\pdfglyphtounicode{gecyrillic}{0433}
\pdfglyphtounicode{gehiragana}{3052}
\pdfglyphtounicode{gekatakana}{30B2}
\pdfglyphtounicode{geomequivalent}{224E}
\pdfglyphtounicode{geometricallyequal}{2251}
\pdfglyphtounicode{gereshaccenthebrew}{059C}
\pdfglyphtounicode{gereshhebrew}{05F3}
\pdfglyphtounicode{gereshmuqdamhebrew}{059D}
\pdfglyphtounicode{germandbls}{00DF}
\pdfglyphtounicode{gershayimaccenthebrew}{059E}
\pdfglyphtounicode{gershayimhebrew}{05F4}
\pdfglyphtounicode{getamark}{3013}
\pdfglyphtounicode{ghabengali}{0998}
\pdfglyphtounicode{ghadarmenian}{0572}
\pdfglyphtounicode{ghadeva}{0918}
\pdfglyphtounicode{ghagujarati}{0A98}
\pdfglyphtounicode{ghagurmukhi}{0A18}
\pdfglyphtounicode{ghainarabic}{063A}
\pdfglyphtounicode{ghainfinalarabic}{FECE}
\pdfglyphtounicode{ghaininitialarabic}{FECF}
\pdfglyphtounicode{ghainmedialarabic}{FED0}
\pdfglyphtounicode{ghemiddlehookcyrillic}{0495}
\pdfglyphtounicode{ghestrokecyrillic}{0493}
\pdfglyphtounicode{gheupturncyrillic}{0491}
\pdfglyphtounicode{ghhadeva}{095A}
\pdfglyphtounicode{ghhagurmukhi}{0A5A}
\pdfglyphtounicode{ghook}{0260}
\pdfglyphtounicode{ghzsquare}{3393}
\pdfglyphtounicode{gihiragana}{304E}
\pdfglyphtounicode{gikatakana}{30AE}
\pdfglyphtounicode{gimarmenian}{0563}
\pdfglyphtounicode{gimel}{05D2}
\pdfglyphtounicode{gimeldagesh}{FB32}
\pdfglyphtounicode{gimeldageshhebrew}{FB32}
\pdfglyphtounicode{gimelhebrew}{05D2}
\pdfglyphtounicode{gjecyrillic}{0453}
\pdfglyphtounicode{glottalinvertedstroke}{01BE}
\pdfglyphtounicode{glottalstop}{0294}
\pdfglyphtounicode{glottalstopinverted}{0296}
\pdfglyphtounicode{glottalstopmod}{02C0}
\pdfglyphtounicode{glottalstopreversed}{0295}
\pdfglyphtounicode{glottalstopreversedmod}{02C1}
\pdfglyphtounicode{glottalstopreversedsuperior}{02E4}
\pdfglyphtounicode{glottalstopstroke}{02A1}
\pdfglyphtounicode{glottalstopstrokereversed}{02A2}
\pdfglyphtounicode{gmacron}{1E21}
\pdfglyphtounicode{gmonospace}{FF47}
\pdfglyphtounicode{gohiragana}{3054}
\pdfglyphtounicode{gokatakana}{30B4}
\pdfglyphtounicode{gparen}{24A2}
\pdfglyphtounicode{gpasquare}{33AC}
\pdfglyphtounicode{gradient}{2207}
\pdfglyphtounicode{grave}{0060}
\pdfglyphtounicode{gravebelowcmb}{0316}
\pdfglyphtounicode{gravecmb}{0300}
\pdfglyphtounicode{gravecomb}{0300}
\pdfglyphtounicode{gravedeva}{0953}
\pdfglyphtounicode{gravelowmod}{02CE}
\pdfglyphtounicode{gravemonospace}{FF40}
\pdfglyphtounicode{gravetonecmb}{0340}
\pdfglyphtounicode{greater}{003E}
\pdfglyphtounicode{greaterdbleqlless}{2A8C}
\pdfglyphtounicode{greaterdblequal}{2267}
\pdfglyphtounicode{greaterdot}{22D7}
\pdfglyphtounicode{greaterequal}{2265}
\pdfglyphtounicode{greaterequalorless}{22DB}
\pdfglyphtounicode{greaterlessequal}{22DB}
\pdfglyphtounicode{greatermonospace}{FF1E}
\pdfglyphtounicode{greatermuch}{226B}
\pdfglyphtounicode{greaternotdblequal}{2A8A}
\pdfglyphtounicode{greaternotequal}{2A88}
\pdfglyphtounicode{greaterorapproxeql}{2A86}
\pdfglyphtounicode{greaterorequalslant}{2A7E}
\pdfglyphtounicode{greaterorequivalent}{2273}
\pdfglyphtounicode{greaterorless}{2277}
\pdfglyphtounicode{greaterornotdbleql}{2269}
\pdfglyphtounicode{greaterornotequal}{2269}
\pdfglyphtounicode{greaterorsimilar}{2273}
\pdfglyphtounicode{greateroverequal}{2267}
\pdfglyphtounicode{greatersmall}{FE65}
\pdfglyphtounicode{gscript}{0261}
\pdfglyphtounicode{gstroke}{01E5}
\pdfglyphtounicode{guhiragana}{3050}
\pdfglyphtounicode{guillemotleft}{00AB}
\pdfglyphtounicode{guillemotright}{00BB}
\pdfglyphtounicode{guilsinglleft}{2039}
\pdfglyphtounicode{guilsinglright}{203A}
\pdfglyphtounicode{gukatakana}{30B0}
\pdfglyphtounicode{guramusquare}{3318}
\pdfglyphtounicode{gysquare}{33C9}
\pdfglyphtounicode{h}{0068}
\pdfglyphtounicode{haabkhasiancyrillic}{04A9}
\pdfglyphtounicode{haaltonearabic}{06C1}
\pdfglyphtounicode{habengali}{09B9}
\pdfglyphtounicode{hadescendercyrillic}{04B3}
\pdfglyphtounicode{hadeva}{0939}
\pdfglyphtounicode{hagujarati}{0AB9}
\pdfglyphtounicode{hagurmukhi}{0A39}
\pdfglyphtounicode{haharabic}{062D}
\pdfglyphtounicode{hahfinalarabic}{FEA2}
\pdfglyphtounicode{hahinitialarabic}{FEA3}
\pdfglyphtounicode{hahiragana}{306F}
\pdfglyphtounicode{hahmedialarabic}{FEA4}
\pdfglyphtounicode{haitusquare}{332A}
\pdfglyphtounicode{hakatakana}{30CF}
\pdfglyphtounicode{hakatakanahalfwidth}{FF8A}
\pdfglyphtounicode{halantgurmukhi}{0A4D}
\pdfglyphtounicode{hamzaarabic}{0621}
\pdfglyphtounicode{hamzadammaarabic}{0621 064F}
\pdfglyphtounicode{hamzadammatanarabic}{0621 064C}
\pdfglyphtounicode{hamzafathaarabic}{0621 064E}
\pdfglyphtounicode{hamzafathatanarabic}{0621 064B}
\pdfglyphtounicode{hamzalowarabic}{0621}
\pdfglyphtounicode{hamzalowkasraarabic}{0621 0650}
\pdfglyphtounicode{hamzalowkasratanarabic}{0621 064D}
\pdfglyphtounicode{hamzasukunarabic}{0621 0652}
\pdfglyphtounicode{hangulfiller}{3164}
\pdfglyphtounicode{hardsigncyrillic}{044A}
\pdfglyphtounicode{harpoondownleft}{21C3}
\pdfglyphtounicode{harpoondownright}{21C2}
\pdfglyphtounicode{harpoonleftbarbup}{21BC}
\pdfglyphtounicode{harpoonleftright}{21CC}
\pdfglyphtounicode{harpoonrightbarbup}{21C0}
\pdfglyphtounicode{harpoonrightleft}{21CB}
\pdfglyphtounicode{harpoonupleft}{21BF}
\pdfglyphtounicode{harpoonupright}{21BE}
\pdfglyphtounicode{hasquare}{33CA}
\pdfglyphtounicode{hatafpatah}{05B2}
\pdfglyphtounicode{hatafpatah16}{05B2}
\pdfglyphtounicode{hatafpatah23}{05B2}
\pdfglyphtounicode{hatafpatah2f}{05B2}
\pdfglyphtounicode{hatafpatahhebrew}{05B2}
\pdfglyphtounicode{hatafpatahnarrowhebrew}{05B2}
\pdfglyphtounicode{hatafpatahquarterhebrew}{05B2}
\pdfglyphtounicode{hatafpatahwidehebrew}{05B2}
\pdfglyphtounicode{hatafqamats}{05B3}
\pdfglyphtounicode{hatafqamats1b}{05B3}
\pdfglyphtounicode{hatafqamats28}{05B3}
\pdfglyphtounicode{hatafqamats34}{05B3}
\pdfglyphtounicode{hatafqamatshebrew}{05B3}
\pdfglyphtounicode{hatafqamatsnarrowhebrew}{05B3}
\pdfglyphtounicode{hatafqamatsquarterhebrew}{05B3}
\pdfglyphtounicode{hatafqamatswidehebrew}{05B3}
\pdfglyphtounicode{hatafsegol}{05B1}
\pdfglyphtounicode{hatafsegol17}{05B1}
\pdfglyphtounicode{hatafsegol24}{05B1}
\pdfglyphtounicode{hatafsegol30}{05B1}
\pdfglyphtounicode{hatafsegolhebrew}{05B1}
\pdfglyphtounicode{hatafsegolnarrowhebrew}{05B1}
\pdfglyphtounicode{hatafsegolquarterhebrew}{05B1}
\pdfglyphtounicode{hatafsegolwidehebrew}{05B1}
\pdfglyphtounicode{hbar}{0127}
\pdfglyphtounicode{hbopomofo}{310F}
\pdfglyphtounicode{hbrevebelow}{1E2B}
\pdfglyphtounicode{hcedilla}{1E29}
\pdfglyphtounicode{hcircle}{24D7}
\pdfglyphtounicode{hcircumflex}{0125}
\pdfglyphtounicode{hdieresis}{1E27}
\pdfglyphtounicode{hdotaccent}{1E23}
\pdfglyphtounicode{hdotbelow}{1E25}
\pdfglyphtounicode{he}{05D4}
\pdfglyphtounicode{heart}{2665}
\pdfglyphtounicode{heartsuitblack}{2665}
\pdfglyphtounicode{heartsuitwhite}{2661}
\pdfglyphtounicode{hedagesh}{FB34}
\pdfglyphtounicode{hedageshhebrew}{FB34}
\pdfglyphtounicode{hehaltonearabic}{06C1}
\pdfglyphtounicode{heharabic}{0647}
\pdfglyphtounicode{hehebrew}{05D4}
\pdfglyphtounicode{hehfinalaltonearabic}{FBA7}
\pdfglyphtounicode{hehfinalalttwoarabic}{FEEA}
\pdfglyphtounicode{hehfinalarabic}{FEEA}
\pdfglyphtounicode{hehhamzaabovefinalarabic}{FBA5}
\pdfglyphtounicode{hehhamzaaboveisolatedarabic}{FBA4}
\pdfglyphtounicode{hehinitialaltonearabic}{FBA8}
\pdfglyphtounicode{hehinitialarabic}{FEEB}
\pdfglyphtounicode{hehiragana}{3078}
\pdfglyphtounicode{hehmedialaltonearabic}{FBA9}
\pdfglyphtounicode{hehmedialarabic}{FEEC}
\pdfglyphtounicode{heiseierasquare}{337B}
\pdfglyphtounicode{hekatakana}{30D8}
\pdfglyphtounicode{hekatakanahalfwidth}{FF8D}
\pdfglyphtounicode{hekutaarusquare}{3336}
\pdfglyphtounicode{henghook}{0267}
\pdfglyphtounicode{herutusquare}{3339}
\pdfglyphtounicode{het}{05D7}
\pdfglyphtounicode{hethebrew}{05D7}
\pdfglyphtounicode{hhook}{0266}
\pdfglyphtounicode{hhooksuperior}{02B1}
\pdfglyphtounicode{hieuhacirclekorean}{327B}
\pdfglyphtounicode{hieuhaparenkorean}{321B}
\pdfglyphtounicode{hieuhcirclekorean}{326D}
\pdfglyphtounicode{hieuhkorean}{314E}
\pdfglyphtounicode{hieuhparenkorean}{320D}
\pdfglyphtounicode{hihiragana}{3072}
\pdfglyphtounicode{hikatakana}{30D2}
\pdfglyphtounicode{hikatakanahalfwidth}{FF8B}
\pdfglyphtounicode{hiriq}{05B4}
\pdfglyphtounicode{hiriq14}{05B4}
\pdfglyphtounicode{hiriq21}{05B4}
\pdfglyphtounicode{hiriq2d}{05B4}
\pdfglyphtounicode{hiriqhebrew}{05B4}
\pdfglyphtounicode{hiriqnarrowhebrew}{05B4}
\pdfglyphtounicode{hiriqquarterhebrew}{05B4}
\pdfglyphtounicode{hiriqwidehebrew}{05B4}
\pdfglyphtounicode{hlinebelow}{1E96}
\pdfglyphtounicode{hmonospace}{FF48}
\pdfglyphtounicode{hoarmenian}{0570}
\pdfglyphtounicode{hohipthai}{0E2B}
\pdfglyphtounicode{hohiragana}{307B}
\pdfglyphtounicode{hokatakana}{30DB}
\pdfglyphtounicode{hokatakanahalfwidth}{FF8E}
\pdfglyphtounicode{holam}{05B9}
\pdfglyphtounicode{holam19}{05B9}
\pdfglyphtounicode{holam26}{05B9}
\pdfglyphtounicode{holam32}{05B9}
\pdfglyphtounicode{holamhebrew}{05B9}
\pdfglyphtounicode{holamnarrowhebrew}{05B9}
\pdfglyphtounicode{holamquarterhebrew}{05B9}
\pdfglyphtounicode{holamwidehebrew}{05B9}
\pdfglyphtounicode{honokhukthai}{0E2E}
\pdfglyphtounicode{hookabovecomb}{0309}
\pdfglyphtounicode{hookcmb}{0309}
\pdfglyphtounicode{hookpalatalizedbelowcmb}{0321}
\pdfglyphtounicode{hookretroflexbelowcmb}{0322}
\pdfglyphtounicode{hoonsquare}{3342}
\pdfglyphtounicode{horicoptic}{03E9}
\pdfglyphtounicode{horizontalbar}{2015}
\pdfglyphtounicode{horncmb}{031B}
\pdfglyphtounicode{hotsprings}{2668}
\pdfglyphtounicode{house}{2302}
\pdfglyphtounicode{hparen}{24A3}
\pdfglyphtounicode{hsuperior}{02B0}
\pdfglyphtounicode{hturned}{0265}
\pdfglyphtounicode{huhiragana}{3075}
\pdfglyphtounicode{huiitosquare}{3333}
\pdfglyphtounicode{hukatakana}{30D5}
\pdfglyphtounicode{hukatakanahalfwidth}{FF8C}
\pdfglyphtounicode{hungarumlaut}{02DD}
\pdfglyphtounicode{hungarumlautcmb}{030B}
\pdfglyphtounicode{hv}{0195}
\pdfglyphtounicode{hyphen}{002D}
\pdfglyphtounicode{hyphenchar}{002D}
\pdfglyphtounicode{hypheninferior}{002D}
\pdfglyphtounicode{hyphenmonospace}{FF0D}
\pdfglyphtounicode{hyphensmall}{FE63}
\pdfglyphtounicode{hyphensuperior}{002D}
\pdfglyphtounicode{hyphentwo}{2010}
\pdfglyphtounicode{i}{0069}
\pdfglyphtounicode{iacute}{00ED}
\pdfglyphtounicode{iacyrillic}{044F}
\pdfglyphtounicode{ibengali}{0987}
\pdfglyphtounicode{ibopomofo}{3127}
\pdfglyphtounicode{ibreve}{012D}
\pdfglyphtounicode{icaron}{01D0}
\pdfglyphtounicode{icircle}{24D8}
\pdfglyphtounicode{icircumflex}{00EE}
\pdfglyphtounicode{icyrillic}{0456}
\pdfglyphtounicode{idblgrave}{0209}
\pdfglyphtounicode{ideographearthcircle}{328F}
\pdfglyphtounicode{ideographfirecircle}{328B}
\pdfglyphtounicode{ideographicallianceparen}{323F}
\pdfglyphtounicode{ideographiccallparen}{323A}
\pdfglyphtounicode{ideographiccentrecircle}{32A5}
\pdfglyphtounicode{ideographicclose}{3006}
\pdfglyphtounicode{ideographiccomma}{3001}
\pdfglyphtounicode{ideographiccommaleft}{FF64}
\pdfglyphtounicode{ideographiccongratulationparen}{3237}
\pdfglyphtounicode{ideographiccorrectcircle}{32A3}
\pdfglyphtounicode{ideographicearthparen}{322F}
\pdfglyphtounicode{ideographicenterpriseparen}{323D}
\pdfglyphtounicode{ideographicexcellentcircle}{329D}
\pdfglyphtounicode{ideographicfestivalparen}{3240}
\pdfglyphtounicode{ideographicfinancialcircle}{3296}
\pdfglyphtounicode{ideographicfinancialparen}{3236}
\pdfglyphtounicode{ideographicfireparen}{322B}
\pdfglyphtounicode{ideographichaveparen}{3232}
\pdfglyphtounicode{ideographichighcircle}{32A4}
\pdfglyphtounicode{ideographiciterationmark}{3005}
\pdfglyphtounicode{ideographiclaborcircle}{3298}
\pdfglyphtounicode{ideographiclaborparen}{3238}
\pdfglyphtounicode{ideographicleftcircle}{32A7}
\pdfglyphtounicode{ideographiclowcircle}{32A6}
\pdfglyphtounicode{ideographicmedicinecircle}{32A9}
\pdfglyphtounicode{ideographicmetalparen}{322E}
\pdfglyphtounicode{ideographicmoonparen}{322A}
\pdfglyphtounicode{ideographicnameparen}{3234}
\pdfglyphtounicode{ideographicperiod}{3002}
\pdfglyphtounicode{ideographicprintcircle}{329E}
\pdfglyphtounicode{ideographicreachparen}{3243}
\pdfglyphtounicode{ideographicrepresentparen}{3239}
\pdfglyphtounicode{ideographicresourceparen}{323E}
\pdfglyphtounicode{ideographicrightcircle}{32A8}
\pdfglyphtounicode{ideographicsecretcircle}{3299}
\pdfglyphtounicode{ideographicselfparen}{3242}
\pdfglyphtounicode{ideographicsocietyparen}{3233}
\pdfglyphtounicode{ideographicspace}{3000}
\pdfglyphtounicode{ideographicspecialparen}{3235}
\pdfglyphtounicode{ideographicstockparen}{3231}
\pdfglyphtounicode{ideographicstudyparen}{323B}
\pdfglyphtounicode{ideographicsunparen}{3230}
\pdfglyphtounicode{ideographicsuperviseparen}{323C}
\pdfglyphtounicode{ideographicwaterparen}{322C}
\pdfglyphtounicode{ideographicwoodparen}{322D}
\pdfglyphtounicode{ideographiczero}{3007}
\pdfglyphtounicode{ideographmetalcircle}{328E}
\pdfglyphtounicode{ideographmooncircle}{328A}
\pdfglyphtounicode{ideographnamecircle}{3294}
\pdfglyphtounicode{ideographsuncircle}{3290}
\pdfglyphtounicode{ideographwatercircle}{328C}
\pdfglyphtounicode{ideographwoodcircle}{328D}
\pdfglyphtounicode{ideva}{0907}
\pdfglyphtounicode{idieresis}{00EF}
\pdfglyphtounicode{idieresisacute}{1E2F}
\pdfglyphtounicode{idieresiscyrillic}{04E5}
\pdfglyphtounicode{idotbelow}{1ECB}
\pdfglyphtounicode{iebrevecyrillic}{04D7}
\pdfglyphtounicode{iecyrillic}{0435}
\pdfglyphtounicode{ieungacirclekorean}{3275}
\pdfglyphtounicode{ieungaparenkorean}{3215}
\pdfglyphtounicode{ieungcirclekorean}{3267}
\pdfglyphtounicode{ieungkorean}{3147}
\pdfglyphtounicode{ieungparenkorean}{3207}
\pdfglyphtounicode{igrave}{00EC}
\pdfglyphtounicode{igujarati}{0A87}
\pdfglyphtounicode{igurmukhi}{0A07}
\pdfglyphtounicode{ihiragana}{3044}
\pdfglyphtounicode{ihookabove}{1EC9}
\pdfglyphtounicode{iibengali}{0988}
\pdfglyphtounicode{iicyrillic}{0438}
\pdfglyphtounicode{iideva}{0908}
\pdfglyphtounicode{iigujarati}{0A88}
\pdfglyphtounicode{iigurmukhi}{0A08}
\pdfglyphtounicode{iimatragurmukhi}{0A40}
\pdfglyphtounicode{iinvertedbreve}{020B}
\pdfglyphtounicode{iishortcyrillic}{0439}
\pdfglyphtounicode{iivowelsignbengali}{09C0}
\pdfglyphtounicode{iivowelsigndeva}{0940}
\pdfglyphtounicode{iivowelsigngujarati}{0AC0}
\pdfglyphtounicode{ij}{0133}
\pdfglyphtounicode{ikatakana}{30A4}
\pdfglyphtounicode{ikatakanahalfwidth}{FF72}
\pdfglyphtounicode{ikorean}{3163}
\pdfglyphtounicode{ilde}{02DC}
\pdfglyphtounicode{iluyhebrew}{05AC}
\pdfglyphtounicode{imacron}{012B}
\pdfglyphtounicode{imacroncyrillic}{04E3}
\pdfglyphtounicode{imageorapproximatelyequal}{2253}
\pdfglyphtounicode{imatragurmukhi}{0A3F}
\pdfglyphtounicode{imonospace}{FF49}
\pdfglyphtounicode{increment}{2206}
\pdfglyphtounicode{infinity}{221E}
\pdfglyphtounicode{iniarmenian}{056B}
\pdfglyphtounicode{integerdivide}{2216}
\pdfglyphtounicode{integral}{222B}
\pdfglyphtounicode{integralbottom}{2321}
\pdfglyphtounicode{integralbt}{2321}
\pdfglyphtounicode{integralex}{F8F5}
\pdfglyphtounicode{integraltop}{2320}
\pdfglyphtounicode{integraltp}{2320}
\pdfglyphtounicode{intercal}{22BA}
\pdfglyphtounicode{interrobang}{203D}
\pdfglyphtounicode{interrobangdown}{2E18}
\pdfglyphtounicode{intersection}{2229}
\pdfglyphtounicode{intersectiondbl}{22D2}
\pdfglyphtounicode{intersectionsq}{2293}
\pdfglyphtounicode{intisquare}{3305}
\pdfglyphtounicode{invbullet}{25D8}
\pdfglyphtounicode{invcircle}{25D9}
\pdfglyphtounicode{invsmileface}{263B}
\pdfglyphtounicode{iocyrillic}{0451}
\pdfglyphtounicode{iogonek}{012F}
\pdfglyphtounicode{iota}{03B9}
\pdfglyphtounicode{iotadieresis}{03CA}
\pdfglyphtounicode{iotadieresistonos}{0390}
\pdfglyphtounicode{iotalatin}{0269}
\pdfglyphtounicode{iotatonos}{03AF}
\pdfglyphtounicode{iparen}{24A4}
\pdfglyphtounicode{irigurmukhi}{0A72}
\pdfglyphtounicode{ismallhiragana}{3043}
\pdfglyphtounicode{ismallkatakana}{30A3}
\pdfglyphtounicode{ismallkatakanahalfwidth}{FF68}
\pdfglyphtounicode{issharbengali}{09FA}
\pdfglyphtounicode{istroke}{0268}
\pdfglyphtounicode{isuperior}{0069}
\pdfglyphtounicode{iterationhiragana}{309D}
\pdfglyphtounicode{iterationkatakana}{30FD}
\pdfglyphtounicode{itilde}{0129}
\pdfglyphtounicode{itildebelow}{1E2D}
\pdfglyphtounicode{iubopomofo}{3129}
\pdfglyphtounicode{iucyrillic}{044E}
\pdfglyphtounicode{ivowelsignbengali}{09BF}
\pdfglyphtounicode{ivowelsigndeva}{093F}
\pdfglyphtounicode{ivowelsigngujarati}{0ABF}
\pdfglyphtounicode{izhitsacyrillic}{0475}
\pdfglyphtounicode{izhitsadblgravecyrillic}{0477}
\pdfglyphtounicode{j}{006A}
\pdfglyphtounicode{jaarmenian}{0571}
\pdfglyphtounicode{jabengali}{099C}
\pdfglyphtounicode{jadeva}{091C}
\pdfglyphtounicode{jagujarati}{0A9C}
\pdfglyphtounicode{jagurmukhi}{0A1C}
\pdfglyphtounicode{jbopomofo}{3110}
\pdfglyphtounicode{jcaron}{01F0}
\pdfglyphtounicode{jcircle}{24D9}
\pdfglyphtounicode{jcircumflex}{0135}
\pdfglyphtounicode{jcrossedtail}{029D}
\pdfglyphtounicode{jdotlessstroke}{025F}
\pdfglyphtounicode{jecyrillic}{0458}
\pdfglyphtounicode{jeemarabic}{062C}
\pdfglyphtounicode{jeemfinalarabic}{FE9E}
\pdfglyphtounicode{jeeminitialarabic}{FE9F}
\pdfglyphtounicode{jeemmedialarabic}{FEA0}
\pdfglyphtounicode{jeharabic}{0698}
\pdfglyphtounicode{jehfinalarabic}{FB8B}
\pdfglyphtounicode{jhabengali}{099D}
\pdfglyphtounicode{jhadeva}{091D}
\pdfglyphtounicode{jhagujarati}{0A9D}
\pdfglyphtounicode{jhagurmukhi}{0A1D}
\pdfglyphtounicode{jheharmenian}{057B}
\pdfglyphtounicode{jis}{3004}
\pdfglyphtounicode{jmonospace}{FF4A}
\pdfglyphtounicode{jparen}{24A5}
\pdfglyphtounicode{jsuperior}{02B2}
\pdfglyphtounicode{k}{006B}
\pdfglyphtounicode{kabashkircyrillic}{04A1}
\pdfglyphtounicode{kabengali}{0995}
\pdfglyphtounicode{kacute}{1E31}
\pdfglyphtounicode{kacyrillic}{043A}
\pdfglyphtounicode{kadescendercyrillic}{049B}
\pdfglyphtounicode{kadeva}{0915}
\pdfglyphtounicode{kaf}{05DB}
\pdfglyphtounicode{kafarabic}{0643}
\pdfglyphtounicode{kafdagesh}{FB3B}
\pdfglyphtounicode{kafdageshhebrew}{FB3B}
\pdfglyphtounicode{kaffinalarabic}{FEDA}
\pdfglyphtounicode{kafhebrew}{05DB}
\pdfglyphtounicode{kafinitialarabic}{FEDB}
\pdfglyphtounicode{kafmedialarabic}{FEDC}
\pdfglyphtounicode{kafrafehebrew}{FB4D}
\pdfglyphtounicode{kagujarati}{0A95}
\pdfglyphtounicode{kagurmukhi}{0A15}
\pdfglyphtounicode{kahiragana}{304B}
\pdfglyphtounicode{kahookcyrillic}{04C4}
\pdfglyphtounicode{kakatakana}{30AB}
\pdfglyphtounicode{kakatakanahalfwidth}{FF76}
\pdfglyphtounicode{kappa}{03BA}
\pdfglyphtounicode{kappasymbolgreek}{03F0}
\pdfglyphtounicode{kapyeounmieumkorean}{3171}
\pdfglyphtounicode{kapyeounphieuphkorean}{3184}
\pdfglyphtounicode{kapyeounpieupkorean}{3178}
\pdfglyphtounicode{kapyeounssangpieupkorean}{3179}
\pdfglyphtounicode{karoriisquare}{330D}
\pdfglyphtounicode{kashidaautoarabic}{0640}
\pdfglyphtounicode{kashidaautonosidebearingarabic}{0640}
\pdfglyphtounicode{kasmallkatakana}{30F5}
\pdfglyphtounicode{kasquare}{3384}
\pdfglyphtounicode{kasraarabic}{0650}
\pdfglyphtounicode{kasratanarabic}{064D}
\pdfglyphtounicode{kastrokecyrillic}{049F}
\pdfglyphtounicode{katahiraprolongmarkhalfwidth}{FF70}
\pdfglyphtounicode{kaverticalstrokecyrillic}{049D}
\pdfglyphtounicode{kbopomofo}{310E}
\pdfglyphtounicode{kcalsquare}{3389}
\pdfglyphtounicode{kcaron}{01E9}
\pdfglyphtounicode{kcedilla}{0137}
\pdfglyphtounicode{kcircle}{24DA}
\pdfglyphtounicode{kcommaaccent}{0137}
\pdfglyphtounicode{kdotbelow}{1E33}
\pdfglyphtounicode{keharmenian}{0584}
\pdfglyphtounicode{kehiragana}{3051}
\pdfglyphtounicode{kekatakana}{30B1}
\pdfglyphtounicode{kekatakanahalfwidth}{FF79}
\pdfglyphtounicode{kenarmenian}{056F}
\pdfglyphtounicode{kesmallkatakana}{30F6}
\pdfglyphtounicode{kgreenlandic}{0138}
\pdfglyphtounicode{khabengali}{0996}
\pdfglyphtounicode{khacyrillic}{0445}
\pdfglyphtounicode{khadeva}{0916}
\pdfglyphtounicode{khagujarati}{0A96}
\pdfglyphtounicode{khagurmukhi}{0A16}
\pdfglyphtounicode{khaharabic}{062E}
\pdfglyphtounicode{khahfinalarabic}{FEA6}
\pdfglyphtounicode{khahinitialarabic}{FEA7}
\pdfglyphtounicode{khahmedialarabic}{FEA8}
\pdfglyphtounicode{kheicoptic}{03E7}
\pdfglyphtounicode{khhadeva}{0959}
\pdfglyphtounicode{khhagurmukhi}{0A59}
\pdfglyphtounicode{khieukhacirclekorean}{3278}
\pdfglyphtounicode{khieukhaparenkorean}{3218}
\pdfglyphtounicode{khieukhcirclekorean}{326A}
\pdfglyphtounicode{khieukhkorean}{314B}
\pdfglyphtounicode{khieukhparenkorean}{320A}
\pdfglyphtounicode{khokhaithai}{0E02}
\pdfglyphtounicode{khokhonthai}{0E05}
\pdfglyphtounicode{khokhuatthai}{0E03}
\pdfglyphtounicode{khokhwaithai}{0E04}
\pdfglyphtounicode{khomutthai}{0E5B}
\pdfglyphtounicode{khook}{0199}
\pdfglyphtounicode{khorakhangthai}{0E06}
\pdfglyphtounicode{khzsquare}{3391}
\pdfglyphtounicode{kihiragana}{304D}
\pdfglyphtounicode{kikatakana}{30AD}
\pdfglyphtounicode{kikatakanahalfwidth}{FF77}
\pdfglyphtounicode{kiroguramusquare}{3315}
\pdfglyphtounicode{kiromeetorusquare}{3316}
\pdfglyphtounicode{kirosquare}{3314}
\pdfglyphtounicode{kiyeokacirclekorean}{326E}
\pdfglyphtounicode{kiyeokaparenkorean}{320E}
\pdfglyphtounicode{kiyeokcirclekorean}{3260}
\pdfglyphtounicode{kiyeokkorean}{3131}
\pdfglyphtounicode{kiyeokparenkorean}{3200}
\pdfglyphtounicode{kiyeoksioskorean}{3133}
\pdfglyphtounicode{kjecyrillic}{045C}
\pdfglyphtounicode{klinebelow}{1E35}
\pdfglyphtounicode{klsquare}{3398}
\pdfglyphtounicode{kmcubedsquare}{33A6}
\pdfglyphtounicode{kmonospace}{FF4B}
\pdfglyphtounicode{kmsquaredsquare}{33A2}
\pdfglyphtounicode{kohiragana}{3053}
\pdfglyphtounicode{kohmsquare}{33C0}
\pdfglyphtounicode{kokaithai}{0E01}
\pdfglyphtounicode{kokatakana}{30B3}
\pdfglyphtounicode{kokatakanahalfwidth}{FF7A}
\pdfglyphtounicode{kooposquare}{331E}
\pdfglyphtounicode{koppacyrillic}{0481}
\pdfglyphtounicode{koreanstandardsymbol}{327F}
\pdfglyphtounicode{koroniscmb}{0343}
\pdfglyphtounicode{kparen}{24A6}
\pdfglyphtounicode{kpasquare}{33AA}
\pdfglyphtounicode{ksicyrillic}{046F}
\pdfglyphtounicode{ktsquare}{33CF}
\pdfglyphtounicode{kturned}{029E}
\pdfglyphtounicode{kuhiragana}{304F}
\pdfglyphtounicode{kukatakana}{30AF}
\pdfglyphtounicode{kukatakanahalfwidth}{FF78}
\pdfglyphtounicode{kvsquare}{33B8}
\pdfglyphtounicode{kwsquare}{33BE}
\pdfglyphtounicode{l}{006C}
\pdfglyphtounicode{labengali}{09B2}
\pdfglyphtounicode{lacute}{013A}
\pdfglyphtounicode{ladeva}{0932}
\pdfglyphtounicode{lagujarati}{0AB2}
\pdfglyphtounicode{lagurmukhi}{0A32}
\pdfglyphtounicode{lakkhangyaothai}{0E45}
\pdfglyphtounicode{lamaleffinalarabic}{FEFC}
\pdfglyphtounicode{lamalefhamzaabovefinalarabic}{FEF8}
\pdfglyphtounicode{lamalefhamzaaboveisolatedarabic}{FEF7}
\pdfglyphtounicode{lamalefhamzabelowfinalarabic}{FEFA}
\pdfglyphtounicode{lamalefhamzabelowisolatedarabic}{FEF9}
\pdfglyphtounicode{lamalefisolatedarabic}{FEFB}
\pdfglyphtounicode{lamalefmaddaabovefinalarabic}{FEF6}
\pdfglyphtounicode{lamalefmaddaaboveisolatedarabic}{FEF5}
\pdfglyphtounicode{lamarabic}{0644}
\pdfglyphtounicode{lambda}{03BB}
\pdfglyphtounicode{lambdastroke}{019B}
\pdfglyphtounicode{lamed}{05DC}
\pdfglyphtounicode{lameddagesh}{FB3C}
\pdfglyphtounicode{lameddageshhebrew}{FB3C}
\pdfglyphtounicode{lamedhebrew}{05DC}
\pdfglyphtounicode{lamedholam}{05DC 05B9}
\pdfglyphtounicode{lamedholamdagesh}{05DC 05B9 05BC}
\pdfglyphtounicode{lamedholamdageshhebrew}{05DC 05B9 05BC}
\pdfglyphtounicode{lamedholamhebrew}{05DC 05B9}
\pdfglyphtounicode{lamfinalarabic}{FEDE}
\pdfglyphtounicode{lamhahinitialarabic}{FCCA}
\pdfglyphtounicode{laminitialarabic}{FEDF}
\pdfglyphtounicode{lamjeeminitialarabic}{FCC9}
\pdfglyphtounicode{lamkhahinitialarabic}{FCCB}
\pdfglyphtounicode{lamlamhehisolatedarabic}{FDF2}
\pdfglyphtounicode{lammedialarabic}{FEE0}
\pdfglyphtounicode{lammeemhahinitialarabic}{FD88}
\pdfglyphtounicode{lammeeminitialarabic}{FCCC}
\pdfglyphtounicode{lammeemjeeminitialarabic}{FEDF FEE4 FEA0}
\pdfglyphtounicode{lammeemkhahinitialarabic}{FEDF FEE4 FEA8}
\pdfglyphtounicode{largecircle}{25EF}
\pdfglyphtounicode{latticetop}{22A4}
\pdfglyphtounicode{lbar}{019A}
\pdfglyphtounicode{lbelt}{026C}
\pdfglyphtounicode{lbopomofo}{310C}
\pdfglyphtounicode{lcaron}{013E}
\pdfglyphtounicode{lcedilla}{013C}
\pdfglyphtounicode{lcircle}{24DB}
\pdfglyphtounicode{lcircumflexbelow}{1E3D}
\pdfglyphtounicode{lcommaaccent}{013C}
\pdfglyphtounicode{ldot}{0140}
\pdfglyphtounicode{ldotaccent}{0140}
\pdfglyphtounicode{ldotbelow}{1E37}
\pdfglyphtounicode{ldotbelowmacron}{1E39}
\pdfglyphtounicode{leftangleabovecmb}{031A}
\pdfglyphtounicode{lefttackbelowcmb}{0318}
\pdfglyphtounicode{less}{003C}
\pdfglyphtounicode{lessdbleqlgreater}{2A8B}
\pdfglyphtounicode{lessdblequal}{2266}
\pdfglyphtounicode{lessdot}{22D6}
\pdfglyphtounicode{lessequal}{2264}
\pdfglyphtounicode{lessequalgreater}{22DA}
\pdfglyphtounicode{lessequalorgreater}{22DA}
\pdfglyphtounicode{lessmonospace}{FF1C}
\pdfglyphtounicode{lessmuch}{226A}
\pdfglyphtounicode{lessnotdblequal}{2A89}
\pdfglyphtounicode{lessnotequal}{2A87}
\pdfglyphtounicode{lessorapproxeql}{2A85}
\pdfglyphtounicode{lessorequalslant}{2A7D}
\pdfglyphtounicode{lessorequivalent}{2272}
\pdfglyphtounicode{lessorgreater}{2276}
\pdfglyphtounicode{lessornotdbleql}{2268}
\pdfglyphtounicode{lessornotequal}{2268}
\pdfglyphtounicode{lessorsimilar}{2272}
\pdfglyphtounicode{lessoverequal}{2266}
\pdfglyphtounicode{lesssmall}{FE64}
\pdfglyphtounicode{lezh}{026E}
\pdfglyphtounicode{lfblock}{258C}
\pdfglyphtounicode{lhookretroflex}{026D}
\pdfglyphtounicode{lira}{20A4}
\pdfglyphtounicode{liwnarmenian}{056C}
\pdfglyphtounicode{lj}{01C9}
\pdfglyphtounicode{ljecyrillic}{0459}
\pdfglyphtounicode{ll}{006C 006C}
\pdfglyphtounicode{lladeva}{0933}
\pdfglyphtounicode{llagujarati}{0AB3}
\pdfglyphtounicode{llinebelow}{1E3B}
\pdfglyphtounicode{llladeva}{0934}
\pdfglyphtounicode{llvocalicbengali}{09E1}
\pdfglyphtounicode{llvocalicdeva}{0961}
\pdfglyphtounicode{llvocalicvowelsignbengali}{09E3}
\pdfglyphtounicode{llvocalicvowelsigndeva}{0963}
\pdfglyphtounicode{lmiddletilde}{026B}
\pdfglyphtounicode{lmonospace}{FF4C}
\pdfglyphtounicode{lmsquare}{33D0}
\pdfglyphtounicode{lochulathai}{0E2C}
\pdfglyphtounicode{logicaland}{2227}
\pdfglyphtounicode{logicalnot}{00AC}
\pdfglyphtounicode{logicalnotreversed}{2310}
\pdfglyphtounicode{logicalor}{2228}
\pdfglyphtounicode{lolingthai}{0E25}
\pdfglyphtounicode{longdbls}{017F 017F}
\pdfglyphtounicode{longs}{017F}
\pdfglyphtounicode{longsh}{017F 0068}
\pdfglyphtounicode{longsi}{017F 0069}
\pdfglyphtounicode{longsl}{017F 006C}
\pdfglyphtounicode{longst}{017F 0074}
\pdfglyphtounicode{lowlinecenterline}{FE4E}
\pdfglyphtounicode{lowlinecmb}{0332}
\pdfglyphtounicode{lowlinedashed}{FE4D}
\pdfglyphtounicode{lozenge}{25CA}
\pdfglyphtounicode{lparen}{24A7}
\pdfglyphtounicode{lscript}{2113}
\pdfglyphtounicode{lslash}{0142}
\pdfglyphtounicode{lsquare}{2113}
\pdfglyphtounicode{lsuperior}{006C}
\pdfglyphtounicode{ltshade}{2591}
\pdfglyphtounicode{luthai}{0E26}
\pdfglyphtounicode{lvocalicbengali}{098C}
\pdfglyphtounicode{lvocalicdeva}{090C}
\pdfglyphtounicode{lvocalicvowelsignbengali}{09E2}
\pdfglyphtounicode{lvocalicvowelsigndeva}{0962}
\pdfglyphtounicode{lxsquare}{33D3}
\pdfglyphtounicode{m}{006D}
\pdfglyphtounicode{mabengali}{09AE}
\pdfglyphtounicode{macron}{00AF}
\pdfglyphtounicode{macronbelowcmb}{0331}
\pdfglyphtounicode{macroncmb}{0304}
\pdfglyphtounicode{macronlowmod}{02CD}
\pdfglyphtounicode{macronmonospace}{FFE3}
\pdfglyphtounicode{macute}{1E3F}
\pdfglyphtounicode{madeva}{092E}
\pdfglyphtounicode{magujarati}{0AAE}
\pdfglyphtounicode{magurmukhi}{0A2E}
\pdfglyphtounicode{mahapakhhebrew}{05A4}
\pdfglyphtounicode{mahapakhlefthebrew}{05A4}
\pdfglyphtounicode{mahiragana}{307E}
\pdfglyphtounicode{maichattawalowleftthai}{F895}
\pdfglyphtounicode{maichattawalowrightthai}{F894}
\pdfglyphtounicode{maichattawathai}{0E4B}
\pdfglyphtounicode{maichattawaupperleftthai}{F893}
\pdfglyphtounicode{maieklowleftthai}{F88C}
\pdfglyphtounicode{maieklowrightthai}{F88B}
\pdfglyphtounicode{maiekthai}{0E48}
\pdfglyphtounicode{maiekupperleftthai}{F88A}
\pdfglyphtounicode{maihanakatleftthai}{F884}
\pdfglyphtounicode{maihanakatthai}{0E31}
\pdfglyphtounicode{maitaikhuleftthai}{F889}
\pdfglyphtounicode{maitaikhuthai}{0E47}
\pdfglyphtounicode{maitholowleftthai}{F88F}
\pdfglyphtounicode{maitholowrightthai}{F88E}
\pdfglyphtounicode{maithothai}{0E49}
\pdfglyphtounicode{maithoupperleftthai}{F88D}
\pdfglyphtounicode{maitrilowleftthai}{F892}
\pdfglyphtounicode{maitrilowrightthai}{F891}
\pdfglyphtounicode{maitrithai}{0E4A}
\pdfglyphtounicode{maitriupperleftthai}{F890}
\pdfglyphtounicode{maiyamokthai}{0E46}
\pdfglyphtounicode{makatakana}{30DE}
\pdfglyphtounicode{makatakanahalfwidth}{FF8F}
\pdfglyphtounicode{male}{2642}
\pdfglyphtounicode{maltesecross}{2720}
\pdfglyphtounicode{mansyonsquare}{3347}
\pdfglyphtounicode{maqafhebrew}{05BE}
\pdfglyphtounicode{mars}{2642}
\pdfglyphtounicode{masoracirclehebrew}{05AF}
\pdfglyphtounicode{masquare}{3383}
\pdfglyphtounicode{mbopomofo}{3107}
\pdfglyphtounicode{mbsquare}{33D4}
\pdfglyphtounicode{mcircle}{24DC}
\pdfglyphtounicode{mcubedsquare}{33A5}
\pdfglyphtounicode{mdotaccent}{1E41}
\pdfglyphtounicode{mdotbelow}{1E43}
\pdfglyphtounicode{measuredangle}{2221}
\pdfglyphtounicode{meemarabic}{0645}
\pdfglyphtounicode{meemfinalarabic}{FEE2}
\pdfglyphtounicode{meeminitialarabic}{FEE3}
\pdfglyphtounicode{meemmedialarabic}{FEE4}
\pdfglyphtounicode{meemmeeminitialarabic}{FCD1}
\pdfglyphtounicode{meemmeemisolatedarabic}{FC48}
\pdfglyphtounicode{meetorusquare}{334D}
\pdfglyphtounicode{mehiragana}{3081}
\pdfglyphtounicode{meizierasquare}{337E}
\pdfglyphtounicode{mekatakana}{30E1}
\pdfglyphtounicode{mekatakanahalfwidth}{FF92}
\pdfglyphtounicode{mem}{05DE}
\pdfglyphtounicode{memdagesh}{FB3E}
\pdfglyphtounicode{memdageshhebrew}{FB3E}
\pdfglyphtounicode{memhebrew}{05DE}
\pdfglyphtounicode{menarmenian}{0574}
\pdfglyphtounicode{merkhahebrew}{05A5}
\pdfglyphtounicode{merkhakefulahebrew}{05A6}
\pdfglyphtounicode{merkhakefulalefthebrew}{05A6}
\pdfglyphtounicode{merkhalefthebrew}{05A5}
\pdfglyphtounicode{mhook}{0271}
\pdfglyphtounicode{mhzsquare}{3392}
\pdfglyphtounicode{middledotkatakanahalfwidth}{FF65}
\pdfglyphtounicode{middot}{00B7}
\pdfglyphtounicode{mieumacirclekorean}{3272}
\pdfglyphtounicode{mieumaparenkorean}{3212}
\pdfglyphtounicode{mieumcirclekorean}{3264}
\pdfglyphtounicode{mieumkorean}{3141}
\pdfglyphtounicode{mieumpansioskorean}{3170}
\pdfglyphtounicode{mieumparenkorean}{3204}
\pdfglyphtounicode{mieumpieupkorean}{316E}
\pdfglyphtounicode{mieumsioskorean}{316F}
\pdfglyphtounicode{mihiragana}{307F}
\pdfglyphtounicode{mikatakana}{30DF}
\pdfglyphtounicode{mikatakanahalfwidth}{FF90}
\pdfglyphtounicode{minus}{2212}
\pdfglyphtounicode{minusbelowcmb}{0320}
\pdfglyphtounicode{minuscircle}{2296}
\pdfglyphtounicode{minusmod}{02D7}
\pdfglyphtounicode{minusplus}{2213}
\pdfglyphtounicode{minute}{2032}
\pdfglyphtounicode{miribaarusquare}{334A}
\pdfglyphtounicode{mirisquare}{3349}
\pdfglyphtounicode{mlonglegturned}{0270}
\pdfglyphtounicode{mlsquare}{3396}
\pdfglyphtounicode{mmcubedsquare}{33A3}
\pdfglyphtounicode{mmonospace}{FF4D}
\pdfglyphtounicode{mmsquaredsquare}{339F}
\pdfglyphtounicode{mohiragana}{3082}
\pdfglyphtounicode{mohmsquare}{33C1}
\pdfglyphtounicode{mokatakana}{30E2}
\pdfglyphtounicode{mokatakanahalfwidth}{FF93}
\pdfglyphtounicode{molsquare}{33D6}
\pdfglyphtounicode{momathai}{0E21}
\pdfglyphtounicode{moverssquare}{33A7}
\pdfglyphtounicode{moverssquaredsquare}{33A8}
\pdfglyphtounicode{mparen}{24A8}
\pdfglyphtounicode{mpasquare}{33AB}
\pdfglyphtounicode{mssquare}{33B3}
\pdfglyphtounicode{msuperior}{006D}
\pdfglyphtounicode{mturned}{026F}
\pdfglyphtounicode{mu}{00B5}
\pdfglyphtounicode{mu1}{00B5}
\pdfglyphtounicode{muasquare}{3382}
\pdfglyphtounicode{muchgreater}{226B}
\pdfglyphtounicode{muchless}{226A}
\pdfglyphtounicode{mufsquare}{338C}
\pdfglyphtounicode{mugreek}{03BC}
\pdfglyphtounicode{mugsquare}{338D}
\pdfglyphtounicode{muhiragana}{3080}
\pdfglyphtounicode{mukatakana}{30E0}
\pdfglyphtounicode{mukatakanahalfwidth}{FF91}
\pdfglyphtounicode{mulsquare}{3395}
\pdfglyphtounicode{multicloseleft}{22C9}
\pdfglyphtounicode{multicloseright}{22CA}
\pdfglyphtounicode{multimap}{22B8}
\pdfglyphtounicode{multiopenleft}{22CB}
\pdfglyphtounicode{multiopenright}{22CC}
\pdfglyphtounicode{multiply}{00D7}
\pdfglyphtounicode{mumsquare}{339B}
\pdfglyphtounicode{munahhebrew}{05A3}
\pdfglyphtounicode{munahlefthebrew}{05A3}
\pdfglyphtounicode{musicalnote}{266A}
\pdfglyphtounicode{musicalnotedbl}{266B}
\pdfglyphtounicode{musicflatsign}{266D}
\pdfglyphtounicode{musicsharpsign}{266F}
\pdfglyphtounicode{mussquare}{33B2}
\pdfglyphtounicode{muvsquare}{33B6}
\pdfglyphtounicode{muwsquare}{33BC}
\pdfglyphtounicode{mvmegasquare}{33B9}
\pdfglyphtounicode{mvsquare}{33B7}
\pdfglyphtounicode{mwmegasquare}{33BF}
\pdfglyphtounicode{mwsquare}{33BD}
\pdfglyphtounicode{n}{006E}
\pdfglyphtounicode{nabengali}{09A8}
\pdfglyphtounicode{nabla}{2207}
\pdfglyphtounicode{nacute}{0144}
\pdfglyphtounicode{nadeva}{0928}
\pdfglyphtounicode{nagujarati}{0AA8}
\pdfglyphtounicode{nagurmukhi}{0A28}
\pdfglyphtounicode{nahiragana}{306A}
\pdfglyphtounicode{nakatakana}{30CA}
\pdfglyphtounicode{nakatakanahalfwidth}{FF85}
\pdfglyphtounicode{nand}{22BC}
\pdfglyphtounicode{napostrophe}{0149}
\pdfglyphtounicode{nasquare}{3381}
\pdfglyphtounicode{natural}{266E}
\pdfglyphtounicode{nbopomofo}{310B}
\pdfglyphtounicode{nbspace}{00A0}
\pdfglyphtounicode{ncaron}{0148}
\pdfglyphtounicode{ncedilla}{0146}
\pdfglyphtounicode{ncircle}{24DD}
\pdfglyphtounicode{ncircumflexbelow}{1E4B}
\pdfglyphtounicode{ncommaaccent}{0146}
\pdfglyphtounicode{ndotaccent}{1E45}
\pdfglyphtounicode{ndotbelow}{1E47}
\pdfglyphtounicode{negationslash}{0338}
\pdfglyphtounicode{nehiragana}{306D}
\pdfglyphtounicode{nekatakana}{30CD}
\pdfglyphtounicode{nekatakanahalfwidth}{FF88}
\pdfglyphtounicode{newsheqelsign}{20AA}
\pdfglyphtounicode{nfsquare}{338B}
\pdfglyphtounicode{ng}{014B}
\pdfglyphtounicode{ngabengali}{0999}
\pdfglyphtounicode{ngadeva}{0919}
\pdfglyphtounicode{ngagujarati}{0A99}
\pdfglyphtounicode{ngagurmukhi}{0A19}
\pdfglyphtounicode{ngonguthai}{0E07}
\pdfglyphtounicode{nhiragana}{3093}
\pdfglyphtounicode{nhookleft}{0272}
\pdfglyphtounicode{nhookretroflex}{0273}
\pdfglyphtounicode{nieunacirclekorean}{326F}
\pdfglyphtounicode{nieunaparenkorean}{320F}
\pdfglyphtounicode{nieuncieuckorean}{3135}
\pdfglyphtounicode{nieuncirclekorean}{3261}
\pdfglyphtounicode{nieunhieuhkorean}{3136}
\pdfglyphtounicode{nieunkorean}{3134}
\pdfglyphtounicode{nieunpansioskorean}{3168}
\pdfglyphtounicode{nieunparenkorean}{3201}
\pdfglyphtounicode{nieunsioskorean}{3167}
\pdfglyphtounicode{nieuntikeutkorean}{3166}
\pdfglyphtounicode{nihiragana}{306B}
\pdfglyphtounicode{nikatakana}{30CB}
\pdfglyphtounicode{nikatakanahalfwidth}{FF86}
\pdfglyphtounicode{nikhahitleftthai}{F899}
\pdfglyphtounicode{nikhahitthai}{0E4D}
\pdfglyphtounicode{nine}{0039}
\pdfglyphtounicode{ninearabic}{0669}
\pdfglyphtounicode{ninebengali}{09EF}
\pdfglyphtounicode{ninecircle}{2468}
\pdfglyphtounicode{ninecircleinversesansserif}{2792}
\pdfglyphtounicode{ninedeva}{096F}
\pdfglyphtounicode{ninegujarati}{0AEF}
\pdfglyphtounicode{ninegurmukhi}{0A6F}
\pdfglyphtounicode{ninehackarabic}{0669}
\pdfglyphtounicode{ninehangzhou}{3029}
\pdfglyphtounicode{nineideographicparen}{3228}
\pdfglyphtounicode{nineinferior}{2089}
\pdfglyphtounicode{ninemonospace}{FF19}
\pdfglyphtounicode{nineoldstyle}{0039}
\pdfglyphtounicode{nineparen}{247C}
\pdfglyphtounicode{nineperiod}{2490}
\pdfglyphtounicode{ninepersian}{06F9}
\pdfglyphtounicode{nineroman}{2178}
\pdfglyphtounicode{ninesuperior}{2079}
\pdfglyphtounicode{nineteencircle}{2472}
\pdfglyphtounicode{nineteenparen}{2486}
\pdfglyphtounicode{nineteenperiod}{249A}
\pdfglyphtounicode{ninethai}{0E59}
\pdfglyphtounicode{nj}{01CC}
\pdfglyphtounicode{njecyrillic}{045A}
\pdfglyphtounicode{nkatakana}{30F3}
\pdfglyphtounicode{nkatakanahalfwidth}{FF9D}
\pdfglyphtounicode{nlegrightlong}{019E}
\pdfglyphtounicode{nlinebelow}{1E49}
\pdfglyphtounicode{nmonospace}{FF4E}
\pdfglyphtounicode{nmsquare}{339A}
\pdfglyphtounicode{nnabengali}{09A3}
\pdfglyphtounicode{nnadeva}{0923}
\pdfglyphtounicode{nnagujarati}{0AA3}
\pdfglyphtounicode{nnagurmukhi}{0A23}
\pdfglyphtounicode{nnnadeva}{0929}
\pdfglyphtounicode{nohiragana}{306E}
\pdfglyphtounicode{nokatakana}{30CE}
\pdfglyphtounicode{nokatakanahalfwidth}{FF89}
\pdfglyphtounicode{nonbreakingspace}{00A0}
\pdfglyphtounicode{nonenthai}{0E13}
\pdfglyphtounicode{nonuthai}{0E19}
\pdfglyphtounicode{noonarabic}{0646}
\pdfglyphtounicode{noonfinalarabic}{FEE6}
\pdfglyphtounicode{noonghunnaarabic}{06BA}
\pdfglyphtounicode{noonghunnafinalarabic}{FB9F}
\pdfglyphtounicode{noonhehinitialarabic}{FEE7 FEEC}
\pdfglyphtounicode{nooninitialarabic}{FEE7}
\pdfglyphtounicode{noonjeeminitialarabic}{FCD2}
\pdfglyphtounicode{noonjeemisolatedarabic}{FC4B}
\pdfglyphtounicode{noonmedialarabic}{FEE8}
\pdfglyphtounicode{noonmeeminitialarabic}{FCD5}
\pdfglyphtounicode{noonmeemisolatedarabic}{FC4E}
\pdfglyphtounicode{noonnoonfinalarabic}{FC8D}
\pdfglyphtounicode{notapproxequal}{2247}
\pdfglyphtounicode{notarrowboth}{21AE}
\pdfglyphtounicode{notarrowleft}{219A}
\pdfglyphtounicode{notarrowright}{219B}
\pdfglyphtounicode{notbar}{2224}
\pdfglyphtounicode{notcontains}{220C}
\pdfglyphtounicode{notdblarrowboth}{21CE}
\pdfglyphtounicode{notdblarrowleft}{21CD}
\pdfglyphtounicode{notdblarrowright}{21CF}
\pdfglyphtounicode{notelement}{2209}
\pdfglyphtounicode{notelementof}{2209}
\pdfglyphtounicode{notequal}{2260}
\pdfglyphtounicode{notexistential}{2204}
\pdfglyphtounicode{notfollows}{2281}
\pdfglyphtounicode{notfollowsoreql}{2AB0 0338}
\pdfglyphtounicode{notforces}{22AE}
\pdfglyphtounicode{notforcesextra}{22AF}
\pdfglyphtounicode{notgreater}{226F}
\pdfglyphtounicode{notgreaterdblequal}{2267 0338}
\pdfglyphtounicode{notgreaterequal}{2271}
\pdfglyphtounicode{notgreaternorequal}{2271}
\pdfglyphtounicode{notgreaternorless}{2279}
\pdfglyphtounicode{notgreaterorslnteql}{2A7E 0338}
\pdfglyphtounicode{notidentical}{2262}
\pdfglyphtounicode{notless}{226E}
\pdfglyphtounicode{notlessdblequal}{2266 0338}
\pdfglyphtounicode{notlessequal}{2270}
\pdfglyphtounicode{notlessnorequal}{2270}
\pdfglyphtounicode{notlessorslnteql}{2A7D 0338}
\pdfglyphtounicode{notparallel}{2226}
\pdfglyphtounicode{notprecedes}{2280}
\pdfglyphtounicode{notprecedesoreql}{2AAF 0338}
\pdfglyphtounicode{notsatisfies}{22AD}
\pdfglyphtounicode{notsimilar}{2241}
\pdfglyphtounicode{notsubset}{2284}
\pdfglyphtounicode{notsubseteql}{2288}
\pdfglyphtounicode{notsubsetordbleql}{2AC5 0338}
\pdfglyphtounicode{notsubsetoreql}{228A}
\pdfglyphtounicode{notsucceeds}{2281}
\pdfglyphtounicode{notsuperset}{2285}
\pdfglyphtounicode{notsuperseteql}{2289}
\pdfglyphtounicode{notsupersetordbleql}{2AC6 0338}
\pdfglyphtounicode{notsupersetoreql}{228B}
\pdfglyphtounicode{nottriangeqlleft}{22EC}
\pdfglyphtounicode{nottriangeqlright}{22ED}
\pdfglyphtounicode{nottriangleleft}{22EA}
\pdfglyphtounicode{nottriangleright}{22EB}
\pdfglyphtounicode{notturnstile}{22AC}
\pdfglyphtounicode{nowarmenian}{0576}
\pdfglyphtounicode{nparen}{24A9}
\pdfglyphtounicode{nssquare}{33B1}
\pdfglyphtounicode{nsuperior}{207F}
\pdfglyphtounicode{ntilde}{00F1}
\pdfglyphtounicode{nu}{03BD}
\pdfglyphtounicode{nuhiragana}{306C}
\pdfglyphtounicode{nukatakana}{30CC}
\pdfglyphtounicode{nukatakanahalfwidth}{FF87}
\pdfglyphtounicode{nuktabengali}{09BC}
\pdfglyphtounicode{nuktadeva}{093C}
\pdfglyphtounicode{nuktagujarati}{0ABC}
\pdfglyphtounicode{nuktagurmukhi}{0A3C}
\pdfglyphtounicode{numbersign}{0023}
\pdfglyphtounicode{numbersignmonospace}{FF03}
\pdfglyphtounicode{numbersignsmall}{FE5F}
\pdfglyphtounicode{numeralsigngreek}{0374}
\pdfglyphtounicode{numeralsignlowergreek}{0375}
\pdfglyphtounicode{numero}{2116}
\pdfglyphtounicode{nun}{05E0}
\pdfglyphtounicode{nundagesh}{FB40}
\pdfglyphtounicode{nundageshhebrew}{FB40}
\pdfglyphtounicode{nunhebrew}{05E0}
\pdfglyphtounicode{nvsquare}{33B5}
\pdfglyphtounicode{nwsquare}{33BB}
\pdfglyphtounicode{nyabengali}{099E}
\pdfglyphtounicode{nyadeva}{091E}
\pdfglyphtounicode{nyagujarati}{0A9E}
\pdfglyphtounicode{nyagurmukhi}{0A1E}
\pdfglyphtounicode{o}{006F}
\pdfglyphtounicode{oacute}{00F3}
\pdfglyphtounicode{oangthai}{0E2D}
\pdfglyphtounicode{obarred}{0275}
\pdfglyphtounicode{obarredcyrillic}{04E9}
\pdfglyphtounicode{obarreddieresiscyrillic}{04EB}
\pdfglyphtounicode{obengali}{0993}
\pdfglyphtounicode{obopomofo}{311B}
\pdfglyphtounicode{obreve}{014F}
\pdfglyphtounicode{ocandradeva}{0911}
\pdfglyphtounicode{ocandragujarati}{0A91}
\pdfglyphtounicode{ocandravowelsigndeva}{0949}
\pdfglyphtounicode{ocandravowelsigngujarati}{0AC9}
\pdfglyphtounicode{ocaron}{01D2}
\pdfglyphtounicode{ocircle}{24DE}
\pdfglyphtounicode{ocircumflex}{00F4}
\pdfglyphtounicode{ocircumflexacute}{1ED1}
\pdfglyphtounicode{ocircumflexdotbelow}{1ED9}
\pdfglyphtounicode{ocircumflexgrave}{1ED3}
\pdfglyphtounicode{ocircumflexhookabove}{1ED5}
\pdfglyphtounicode{ocircumflextilde}{1ED7}
\pdfglyphtounicode{ocyrillic}{043E}
\pdfglyphtounicode{odblacute}{0151}
\pdfglyphtounicode{odblgrave}{020D}
\pdfglyphtounicode{odeva}{0913}
\pdfglyphtounicode{odieresis}{00F6}
\pdfglyphtounicode{odieresiscyrillic}{04E7}
\pdfglyphtounicode{odotbelow}{1ECD}
\pdfglyphtounicode{oe}{0153}
\pdfglyphtounicode{oekorean}{315A}
\pdfglyphtounicode{ogonek}{02DB}
\pdfglyphtounicode{ogonekcmb}{0328}
\pdfglyphtounicode{ograve}{00F2}
\pdfglyphtounicode{ogujarati}{0A93}
\pdfglyphtounicode{oharmenian}{0585}
\pdfglyphtounicode{ohiragana}{304A}
\pdfglyphtounicode{ohookabove}{1ECF}
\pdfglyphtounicode{ohorn}{01A1}
\pdfglyphtounicode{ohornacute}{1EDB}
\pdfglyphtounicode{ohorndotbelow}{1EE3}
\pdfglyphtounicode{ohorngrave}{1EDD}
\pdfglyphtounicode{ohornhookabove}{1EDF}
\pdfglyphtounicode{ohorntilde}{1EE1}
\pdfglyphtounicode{ohungarumlaut}{0151}
\pdfglyphtounicode{oi}{01A3}
\pdfglyphtounicode{oinvertedbreve}{020F}
\pdfglyphtounicode{okatakana}{30AA}
\pdfglyphtounicode{okatakanahalfwidth}{FF75}
\pdfglyphtounicode{okorean}{3157}
\pdfglyphtounicode{olehebrew}{05AB}
\pdfglyphtounicode{omacron}{014D}
\pdfglyphtounicode{omacronacute}{1E53}
\pdfglyphtounicode{omacrongrave}{1E51}
\pdfglyphtounicode{omdeva}{0950}
\pdfglyphtounicode{omega}{03C9}
\pdfglyphtounicode{omega1}{03D6}
\pdfglyphtounicode{omegacyrillic}{0461}
\pdfglyphtounicode{omegalatinclosed}{0277}
\pdfglyphtounicode{omegaroundcyrillic}{047B}
\pdfglyphtounicode{omegatitlocyrillic}{047D}
\pdfglyphtounicode{omegatonos}{03CE}
\pdfglyphtounicode{omgujarati}{0AD0}
\pdfglyphtounicode{omicron}{03BF}
\pdfglyphtounicode{omicrontonos}{03CC}
\pdfglyphtounicode{omonospace}{FF4F}
\pdfglyphtounicode{one}{0031}
\pdfglyphtounicode{onearabic}{0661}
\pdfglyphtounicode{onebengali}{09E7}
\pdfglyphtounicode{onecircle}{2460}
\pdfglyphtounicode{onecircleinversesansserif}{278A}
\pdfglyphtounicode{onedeva}{0967}
\pdfglyphtounicode{onedotenleader}{2024}
\pdfglyphtounicode{oneeighth}{215B}
\pdfglyphtounicode{onefitted}{0031}
\pdfglyphtounicode{onegujarati}{0AE7}
\pdfglyphtounicode{onegurmukhi}{0A67}
\pdfglyphtounicode{onehackarabic}{0661}
\pdfglyphtounicode{onehalf}{00BD}
\pdfglyphtounicode{onehangzhou}{3021}
\pdfglyphtounicode{oneideographicparen}{3220}
\pdfglyphtounicode{oneinferior}{2081}
\pdfglyphtounicode{onemonospace}{FF11}
\pdfglyphtounicode{onenumeratorbengali}{09F4}
\pdfglyphtounicode{oneoldstyle}{0031}
\pdfglyphtounicode{oneparen}{2474}
\pdfglyphtounicode{oneperiod}{2488}
\pdfglyphtounicode{onepersian}{06F1}
\pdfglyphtounicode{onequarter}{00BC}
\pdfglyphtounicode{oneroman}{2170}
\pdfglyphtounicode{onesuperior}{00B9}
\pdfglyphtounicode{onethai}{0E51}
\pdfglyphtounicode{onethird}{2153}
\pdfglyphtounicode{oogonek}{01EB}
\pdfglyphtounicode{oogonekmacron}{01ED}
\pdfglyphtounicode{oogurmukhi}{0A13}
\pdfglyphtounicode{oomatragurmukhi}{0A4B}
\pdfglyphtounicode{oopen}{0254}
\pdfglyphtounicode{oparen}{24AA}
\pdfglyphtounicode{openbullet}{25E6}
\pdfglyphtounicode{option}{2325}
\pdfglyphtounicode{ordfeminine}{00AA}
\pdfglyphtounicode{ordmasculine}{00BA}
\pdfglyphtounicode{orthogonal}{221F}
\pdfglyphtounicode{orunderscore}{22BB}
\pdfglyphtounicode{oshortdeva}{0912}
\pdfglyphtounicode{oshortvowelsigndeva}{094A}
\pdfglyphtounicode{oslash}{00F8}
\pdfglyphtounicode{oslashacute}{01FF}
\pdfglyphtounicode{osmallhiragana}{3049}
\pdfglyphtounicode{osmallkatakana}{30A9}
\pdfglyphtounicode{osmallkatakanahalfwidth}{FF6B}
\pdfglyphtounicode{ostrokeacute}{01FF}
\pdfglyphtounicode{osuperior}{006F}
\pdfglyphtounicode{otcyrillic}{047F}
\pdfglyphtounicode{otilde}{00F5}
\pdfglyphtounicode{otildeacute}{1E4D}
\pdfglyphtounicode{otildedieresis}{1E4F}
\pdfglyphtounicode{oubopomofo}{3121}
\pdfglyphtounicode{overline}{203E}
\pdfglyphtounicode{overlinecenterline}{FE4A}
\pdfglyphtounicode{overlinecmb}{0305}
\pdfglyphtounicode{overlinedashed}{FE49}
\pdfglyphtounicode{overlinedblwavy}{FE4C}
\pdfglyphtounicode{overlinewavy}{FE4B}
\pdfglyphtounicode{overscore}{00AF}
\pdfglyphtounicode{ovowelsignbengali}{09CB}
\pdfglyphtounicode{ovowelsigndeva}{094B}
\pdfglyphtounicode{ovowelsigngujarati}{0ACB}
\pdfglyphtounicode{owner}{220B}
\pdfglyphtounicode{p}{0070}
\pdfglyphtounicode{paampssquare}{3380}
\pdfglyphtounicode{paasentosquare}{332B}
\pdfglyphtounicode{pabengali}{09AA}
\pdfglyphtounicode{pacute}{1E55}
\pdfglyphtounicode{padeva}{092A}
\pdfglyphtounicode{pagedown}{21DF}
\pdfglyphtounicode{pageup}{21DE}
\pdfglyphtounicode{pagujarati}{0AAA}
\pdfglyphtounicode{pagurmukhi}{0A2A}
\pdfglyphtounicode{pahiragana}{3071}
\pdfglyphtounicode{paiyannoithai}{0E2F}
\pdfglyphtounicode{pakatakana}{30D1}
\pdfglyphtounicode{palatalizationcyrilliccmb}{0484}
\pdfglyphtounicode{palochkacyrillic}{04C0}
\pdfglyphtounicode{pansioskorean}{317F}
\pdfglyphtounicode{paragraph}{00B6}
\pdfglyphtounicode{parallel}{2225}
\pdfglyphtounicode{parenleft}{0028}
\pdfglyphtounicode{parenleftaltonearabic}{FD3E}
\pdfglyphtounicode{parenleftbt}{F8ED}
\pdfglyphtounicode{parenleftex}{F8EC}
\pdfglyphtounicode{parenleftinferior}{208D}
\pdfglyphtounicode{parenleftmonospace}{FF08}
\pdfglyphtounicode{parenleftsmall}{FE59}
\pdfglyphtounicode{parenleftsuperior}{207D}
\pdfglyphtounicode{parenlefttp}{F8EB}
\pdfglyphtounicode{parenleftvertical}{FE35}
\pdfglyphtounicode{parenright}{0029}
\pdfglyphtounicode{parenrightaltonearabic}{FD3F}
\pdfglyphtounicode{parenrightbt}{F8F8}
\pdfglyphtounicode{parenrightex}{F8F7}
\pdfglyphtounicode{parenrightinferior}{208E}
\pdfglyphtounicode{parenrightmonospace}{FF09}
\pdfglyphtounicode{parenrightsmall}{FE5A}
\pdfglyphtounicode{parenrightsuperior}{207E}
\pdfglyphtounicode{parenrighttp}{F8F6}
\pdfglyphtounicode{parenrightvertical}{FE36}
\pdfglyphtounicode{partialdiff}{2202}
\pdfglyphtounicode{paseqhebrew}{05C0}
\pdfglyphtounicode{pashtahebrew}{0599}
\pdfglyphtounicode{pasquare}{33A9}
\pdfglyphtounicode{patah}{05B7}
\pdfglyphtounicode{patah11}{05B7}
\pdfglyphtounicode{patah1d}{05B7}
\pdfglyphtounicode{patah2a}{05B7}
\pdfglyphtounicode{patahhebrew}{05B7}
\pdfglyphtounicode{patahnarrowhebrew}{05B7}
\pdfglyphtounicode{patahquarterhebrew}{05B7}
\pdfglyphtounicode{patahwidehebrew}{05B7}
\pdfglyphtounicode{pazerhebrew}{05A1}
\pdfglyphtounicode{pbopomofo}{3106}
\pdfglyphtounicode{pcircle}{24DF}
\pdfglyphtounicode{pdotaccent}{1E57}
\pdfglyphtounicode{pe}{05E4}
\pdfglyphtounicode{pecyrillic}{043F}
\pdfglyphtounicode{pedagesh}{FB44}
\pdfglyphtounicode{pedageshhebrew}{FB44}
\pdfglyphtounicode{peezisquare}{333B}
\pdfglyphtounicode{pefinaldageshhebrew}{FB43}
\pdfglyphtounicode{peharabic}{067E}
\pdfglyphtounicode{peharmenian}{057A}
\pdfglyphtounicode{pehebrew}{05E4}
\pdfglyphtounicode{pehfinalarabic}{FB57}
\pdfglyphtounicode{pehinitialarabic}{FB58}
\pdfglyphtounicode{pehiragana}{307A}
\pdfglyphtounicode{pehmedialarabic}{FB59}
\pdfglyphtounicode{pekatakana}{30DA}
\pdfglyphtounicode{pemiddlehookcyrillic}{04A7}
\pdfglyphtounicode{perafehebrew}{FB4E}
\pdfglyphtounicode{percent}{0025}
\pdfglyphtounicode{percentarabic}{066A}
\pdfglyphtounicode{percentmonospace}{FF05}
\pdfglyphtounicode{percentsmall}{FE6A}
\pdfglyphtounicode{period}{002E}
\pdfglyphtounicode{periodarmenian}{0589}
\pdfglyphtounicode{periodcentered}{00B7}
\pdfglyphtounicode{periodhalfwidth}{FF61}
\pdfglyphtounicode{periodinferior}{002E}
\pdfglyphtounicode{periodmonospace}{FF0E}
\pdfglyphtounicode{periodsmall}{FE52}
\pdfglyphtounicode{periodsuperior}{002E}
\pdfglyphtounicode{perispomenigreekcmb}{0342}
\pdfglyphtounicode{perpcorrespond}{2A5E}
\pdfglyphtounicode{perpendicular}{22A5}
\pdfglyphtounicode{pertenthousand}{2031}
\pdfglyphtounicode{perthousand}{2030}
\pdfglyphtounicode{peseta}{20A7}
\pdfglyphtounicode{pfsquare}{338A}
\pdfglyphtounicode{phabengali}{09AB}
\pdfglyphtounicode{phadeva}{092B}
\pdfglyphtounicode{phagujarati}{0AAB}
\pdfglyphtounicode{phagurmukhi}{0A2B}
\pdfglyphtounicode{phi}{03C6}
\pdfglyphtounicode{phi1}{03D5}
\pdfglyphtounicode{phieuphacirclekorean}{327A}
\pdfglyphtounicode{phieuphaparenkorean}{321A}
\pdfglyphtounicode{phieuphcirclekorean}{326C}
\pdfglyphtounicode{phieuphkorean}{314D}
\pdfglyphtounicode{phieuphparenkorean}{320C}
\pdfglyphtounicode{philatin}{0278}
\pdfglyphtounicode{phinthuthai}{0E3A}
\pdfglyphtounicode{phisymbolgreek}{03D5}
\pdfglyphtounicode{phook}{01A5}
\pdfglyphtounicode{phophanthai}{0E1E}
\pdfglyphtounicode{phophungthai}{0E1C}
\pdfglyphtounicode{phosamphaothai}{0E20}
\pdfglyphtounicode{pi}{03C0}
\pdfglyphtounicode{pi1}{03D6}
\pdfglyphtounicode{pieupacirclekorean}{3273}
\pdfglyphtounicode{pieupaparenkorean}{3213}
\pdfglyphtounicode{pieupcieuckorean}{3176}
\pdfglyphtounicode{pieupcirclekorean}{3265}
\pdfglyphtounicode{pieupkiyeokkorean}{3172}
\pdfglyphtounicode{pieupkorean}{3142}
\pdfglyphtounicode{pieupparenkorean}{3205}
\pdfglyphtounicode{pieupsioskiyeokkorean}{3174}
\pdfglyphtounicode{pieupsioskorean}{3144}
\pdfglyphtounicode{pieupsiostikeutkorean}{3175}
\pdfglyphtounicode{pieupthieuthkorean}{3177}
\pdfglyphtounicode{pieuptikeutkorean}{3173}
\pdfglyphtounicode{pihiragana}{3074}
\pdfglyphtounicode{pikatakana}{30D4}
\pdfglyphtounicode{pisymbolgreek}{03D6}
\pdfglyphtounicode{piwrarmenian}{0583}
\pdfglyphtounicode{planckover2pi}{210F}
\pdfglyphtounicode{planckover2pi1}{210F}
\pdfglyphtounicode{plus}{002B}
\pdfglyphtounicode{plusbelowcmb}{031F}
\pdfglyphtounicode{pluscircle}{2295}
\pdfglyphtounicode{plusminus}{00B1}
\pdfglyphtounicode{plusmod}{02D6}
\pdfglyphtounicode{plusmonospace}{FF0B}
\pdfglyphtounicode{plussmall}{FE62}
\pdfglyphtounicode{plussuperior}{207A}
\pdfglyphtounicode{pmonospace}{FF50}
\pdfglyphtounicode{pmsquare}{33D8}
\pdfglyphtounicode{pohiragana}{307D}
\pdfglyphtounicode{pointingindexdownwhite}{261F}
\pdfglyphtounicode{pointingindexleftwhite}{261C}
\pdfglyphtounicode{pointingindexrightwhite}{261E}
\pdfglyphtounicode{pointingindexupwhite}{261D}
\pdfglyphtounicode{pokatakana}{30DD}
\pdfglyphtounicode{poplathai}{0E1B}
\pdfglyphtounicode{postalmark}{3012}
\pdfglyphtounicode{postalmarkface}{3020}
\pdfglyphtounicode{pparen}{24AB}
\pdfglyphtounicode{precedenotdbleqv}{2AB9}
\pdfglyphtounicode{precedenotslnteql}{2AB5}
\pdfglyphtounicode{precedeornoteqvlnt}{22E8}
\pdfglyphtounicode{precedes}{227A}
\pdfglyphtounicode{precedesequal}{2AAF}
\pdfglyphtounicode{precedesorcurly}{227C}
\pdfglyphtounicode{precedesorequal}{227E}
\pdfglyphtounicode{prescription}{211E}
\pdfglyphtounicode{prime}{2032}
\pdfglyphtounicode{primemod}{02B9}
\pdfglyphtounicode{primereverse}{2035}
\pdfglyphtounicode{primereversed}{2035}
\pdfglyphtounicode{product}{220F}
\pdfglyphtounicode{projective}{2305}
\pdfglyphtounicode{prolongedkana}{30FC}
\pdfglyphtounicode{propellor}{2318}
\pdfglyphtounicode{propersubset}{2282}
\pdfglyphtounicode{propersuperset}{2283}
\pdfglyphtounicode{proportion}{2237}
\pdfglyphtounicode{proportional}{221D}
\pdfglyphtounicode{psi}{03C8}
\pdfglyphtounicode{psicyrillic}{0471}
\pdfglyphtounicode{psilipneumatacyrilliccmb}{0486}
\pdfglyphtounicode{pssquare}{33B0}
\pdfglyphtounicode{puhiragana}{3077}
\pdfglyphtounicode{pukatakana}{30D7}
\pdfglyphtounicode{punctdash}{2014}
\pdfglyphtounicode{pvsquare}{33B4}
\pdfglyphtounicode{pwsquare}{33BA}
\pdfglyphtounicode{q}{0071}
\pdfglyphtounicode{qadeva}{0958}
\pdfglyphtounicode{qadmahebrew}{05A8}
\pdfglyphtounicode{qafarabic}{0642}
\pdfglyphtounicode{qaffinalarabic}{FED6}
\pdfglyphtounicode{qafinitialarabic}{FED7}
\pdfglyphtounicode{qafmedialarabic}{FED8}
\pdfglyphtounicode{qamats}{05B8}
\pdfglyphtounicode{qamats10}{05B8}
\pdfglyphtounicode{qamats1a}{05B8}
\pdfglyphtounicode{qamats1c}{05B8}
\pdfglyphtounicode{qamats27}{05B8}
\pdfglyphtounicode{qamats29}{05B8}
\pdfglyphtounicode{qamats33}{05B8}
\pdfglyphtounicode{qamatsde}{05B8}
\pdfglyphtounicode{qamatshebrew}{05B8}
\pdfglyphtounicode{qamatsnarrowhebrew}{05B8}
\pdfglyphtounicode{qamatsqatanhebrew}{05B8}
\pdfglyphtounicode{qamatsqatannarrowhebrew}{05B8}
\pdfglyphtounicode{qamatsqatanquarterhebrew}{05B8}
\pdfglyphtounicode{qamatsqatanwidehebrew}{05B8}
\pdfglyphtounicode{qamatsquarterhebrew}{05B8}
\pdfglyphtounicode{qamatswidehebrew}{05B8}
\pdfglyphtounicode{qarneyparahebrew}{059F}
\pdfglyphtounicode{qbopomofo}{3111}
\pdfglyphtounicode{qcircle}{24E0}
\pdfglyphtounicode{qhook}{02A0}
\pdfglyphtounicode{qmonospace}{FF51}
\pdfglyphtounicode{qof}{05E7}
\pdfglyphtounicode{qofdagesh}{FB47}
\pdfglyphtounicode{qofdageshhebrew}{FB47}
\pdfglyphtounicode{qofhatafpatah}{05E7 05B2}
\pdfglyphtounicode{qofhatafpatahhebrew}{05E7 05B2}
\pdfglyphtounicode{qofhatafsegol}{05E7 05B1}
\pdfglyphtounicode{qofhatafsegolhebrew}{05E7 05B1}
\pdfglyphtounicode{qofhebrew}{05E7}
\pdfglyphtounicode{qofhiriq}{05E7 05B4}
\pdfglyphtounicode{qofhiriqhebrew}{05E7 05B4}
\pdfglyphtounicode{qofholam}{05E7 05B9}
\pdfglyphtounicode{qofholamhebrew}{05E7 05B9}
\pdfglyphtounicode{qofpatah}{05E7 05B7}
\pdfglyphtounicode{qofpatahhebrew}{05E7 05B7}
\pdfglyphtounicode{qofqamats}{05E7 05B8}
\pdfglyphtounicode{qofqamatshebrew}{05E7 05B8}
\pdfglyphtounicode{qofqubuts}{05E7 05BB}
\pdfglyphtounicode{qofqubutshebrew}{05E7 05BB}
\pdfglyphtounicode{qofsegol}{05E7 05B6}
\pdfglyphtounicode{qofsegolhebrew}{05E7 05B6}
\pdfglyphtounicode{qofsheva}{05E7 05B0}
\pdfglyphtounicode{qofshevahebrew}{05E7 05B0}
\pdfglyphtounicode{qoftsere}{05E7 05B5}
\pdfglyphtounicode{qoftserehebrew}{05E7 05B5}
\pdfglyphtounicode{qparen}{24AC}
\pdfglyphtounicode{quarternote}{2669}
\pdfglyphtounicode{qubuts}{05BB}
\pdfglyphtounicode{qubuts18}{05BB}
\pdfglyphtounicode{qubuts25}{05BB}
\pdfglyphtounicode{qubuts31}{05BB}
\pdfglyphtounicode{qubutshebrew}{05BB}
\pdfglyphtounicode{qubutsnarrowhebrew}{05BB}
\pdfglyphtounicode{qubutsquarterhebrew}{05BB}
\pdfglyphtounicode{qubutswidehebrew}{05BB}
\pdfglyphtounicode{question}{003F}
\pdfglyphtounicode{questionarabic}{061F}
\pdfglyphtounicode{questionarmenian}{055E}
\pdfglyphtounicode{questiondown}{00BF}
\pdfglyphtounicode{questiondownsmall}{00BF}
\pdfglyphtounicode{questiongreek}{037E}
\pdfglyphtounicode{questionmonospace}{FF1F}
\pdfglyphtounicode{questionsmall}{003F}
\pdfglyphtounicode{quotedbl}{0022}
\pdfglyphtounicode{quotedblbase}{201E}
\pdfglyphtounicode{quotedblleft}{201C}
\pdfglyphtounicode{quotedblmonospace}{FF02}
\pdfglyphtounicode{quotedblprime}{301E}
\pdfglyphtounicode{quotedblprimereversed}{301D}
\pdfglyphtounicode{quotedblright}{201D}
\pdfglyphtounicode{quoteleft}{2018}
\pdfglyphtounicode{quoteleftreversed}{201B}
\pdfglyphtounicode{quotereversed}{201B}
\pdfglyphtounicode{quoteright}{2019}
\pdfglyphtounicode{quoterightn}{0149}
\pdfglyphtounicode{quotesinglbase}{201A}
\pdfglyphtounicode{quotesingle}{0027}
\pdfglyphtounicode{quotesinglemonospace}{FF07}
\pdfglyphtounicode{r}{0072}
\pdfglyphtounicode{raarmenian}{057C}
\pdfglyphtounicode{rabengali}{09B0}
\pdfglyphtounicode{racute}{0155}
\pdfglyphtounicode{radeva}{0930}
\pdfglyphtounicode{radical}{221A}
\pdfglyphtounicode{radicalex}{F8E5}
\pdfglyphtounicode{radoverssquare}{33AE}
\pdfglyphtounicode{radoverssquaredsquare}{33AF}
\pdfglyphtounicode{radsquare}{33AD}
\pdfglyphtounicode{rafe}{05BF}
\pdfglyphtounicode{rafehebrew}{05BF}
\pdfglyphtounicode{ragujarati}{0AB0}
\pdfglyphtounicode{ragurmukhi}{0A30}
\pdfglyphtounicode{rahiragana}{3089}
\pdfglyphtounicode{rakatakana}{30E9}
\pdfglyphtounicode{rakatakanahalfwidth}{FF97}
\pdfglyphtounicode{ralowerdiagonalbengali}{09F1}
\pdfglyphtounicode{ramiddlediagonalbengali}{09F0}
\pdfglyphtounicode{ramshorn}{0264}
\pdfglyphtounicode{rangedash}{2013}
\pdfglyphtounicode{ratio}{2236}
\pdfglyphtounicode{rbopomofo}{3116}
\pdfglyphtounicode{rcaron}{0159}
\pdfglyphtounicode{rcedilla}{0157}
\pdfglyphtounicode{rcircle}{24E1}
\pdfglyphtounicode{rcommaaccent}{0157}
\pdfglyphtounicode{rdblgrave}{0211}
\pdfglyphtounicode{rdotaccent}{1E59}
\pdfglyphtounicode{rdotbelow}{1E5B}
\pdfglyphtounicode{rdotbelowmacron}{1E5D}
\pdfglyphtounicode{referencemark}{203B}
\pdfglyphtounicode{reflexsubset}{2286}
\pdfglyphtounicode{reflexsuperset}{2287}
\pdfglyphtounicode{registered}{00AE}
\pdfglyphtounicode{registersans}{00AE}
\pdfglyphtounicode{registerserif}{00AE}
\pdfglyphtounicode{reharabic}{0631}
\pdfglyphtounicode{reharmenian}{0580}
\pdfglyphtounicode{rehfinalarabic}{FEAE}
\pdfglyphtounicode{rehiragana}{308C}
\pdfglyphtounicode{rehyehaleflamarabic}{0631 FEF3 FE8E 0644}
\pdfglyphtounicode{rekatakana}{30EC}
\pdfglyphtounicode{rekatakanahalfwidth}{FF9A}
\pdfglyphtounicode{resh}{05E8}
\pdfglyphtounicode{reshdageshhebrew}{FB48}
\pdfglyphtounicode{reshhatafpatah}{05E8 05B2}
\pdfglyphtounicode{reshhatafpatahhebrew}{05E8 05B2}
\pdfglyphtounicode{reshhatafsegol}{05E8 05B1}
\pdfglyphtounicode{reshhatafsegolhebrew}{05E8 05B1}
\pdfglyphtounicode{reshhebrew}{05E8}
\pdfglyphtounicode{reshhiriq}{05E8 05B4}
\pdfglyphtounicode{reshhiriqhebrew}{05E8 05B4}
\pdfglyphtounicode{reshholam}{05E8 05B9}
\pdfglyphtounicode{reshholamhebrew}{05E8 05B9}
\pdfglyphtounicode{reshpatah}{05E8 05B7}
\pdfglyphtounicode{reshpatahhebrew}{05E8 05B7}
\pdfglyphtounicode{reshqamats}{05E8 05B8}
\pdfglyphtounicode{reshqamatshebrew}{05E8 05B8}
\pdfglyphtounicode{reshqubuts}{05E8 05BB}
\pdfglyphtounicode{reshqubutshebrew}{05E8 05BB}
\pdfglyphtounicode{reshsegol}{05E8 05B6}
\pdfglyphtounicode{reshsegolhebrew}{05E8 05B6}
\pdfglyphtounicode{reshsheva}{05E8 05B0}
\pdfglyphtounicode{reshshevahebrew}{05E8 05B0}
\pdfglyphtounicode{reshtsere}{05E8 05B5}
\pdfglyphtounicode{reshtserehebrew}{05E8 05B5}
\pdfglyphtounicode{revasymptequal}{22CD}
\pdfglyphtounicode{reversedtilde}{223D}
\pdfglyphtounicode{reviahebrew}{0597}
\pdfglyphtounicode{reviamugrashhebrew}{0597}
\pdfglyphtounicode{revlogicalnot}{2310}
\pdfglyphtounicode{revsimilar}{223D}
\pdfglyphtounicode{rfishhook}{027E}
\pdfglyphtounicode{rfishhookreversed}{027F}
\pdfglyphtounicode{rhabengali}{09DD}
\pdfglyphtounicode{rhadeva}{095D}
\pdfglyphtounicode{rho}{03C1}
\pdfglyphtounicode{rho1}{03F1}
\pdfglyphtounicode{rhook}{027D}
\pdfglyphtounicode{rhookturned}{027B}
\pdfglyphtounicode{rhookturnedsuperior}{02B5}
\pdfglyphtounicode{rhosymbolgreek}{03F1}
\pdfglyphtounicode{rhotichookmod}{02DE}
\pdfglyphtounicode{rieulacirclekorean}{3271}
\pdfglyphtounicode{rieulaparenkorean}{3211}
\pdfglyphtounicode{rieulcirclekorean}{3263}
\pdfglyphtounicode{rieulhieuhkorean}{3140}
\pdfglyphtounicode{rieulkiyeokkorean}{313A}
\pdfglyphtounicode{rieulkiyeoksioskorean}{3169}
\pdfglyphtounicode{rieulkorean}{3139}
\pdfglyphtounicode{rieulmieumkorean}{313B}
\pdfglyphtounicode{rieulpansioskorean}{316C}
\pdfglyphtounicode{rieulparenkorean}{3203}
\pdfglyphtounicode{rieulphieuphkorean}{313F}
\pdfglyphtounicode{rieulpieupkorean}{313C}
\pdfglyphtounicode{rieulpieupsioskorean}{316B}
\pdfglyphtounicode{rieulsioskorean}{313D}
\pdfglyphtounicode{rieulthieuthkorean}{313E}
\pdfglyphtounicode{rieultikeutkorean}{316A}
\pdfglyphtounicode{rieulyeorinhieuhkorean}{316D}
\pdfglyphtounicode{rightangle}{221F}
\pdfglyphtounicode{rightanglene}{231D}
\pdfglyphtounicode{rightanglenw}{231C}
\pdfglyphtounicode{rightanglese}{231F}
\pdfglyphtounicode{rightanglesw}{231E}
\pdfglyphtounicode{righttackbelowcmb}{0319}
\pdfglyphtounicode{righttriangle}{22BF}
\pdfglyphtounicode{rihiragana}{308A}
\pdfglyphtounicode{rikatakana}{30EA}
\pdfglyphtounicode{rikatakanahalfwidth}{FF98}
\pdfglyphtounicode{ring}{02DA}
\pdfglyphtounicode{ringbelowcmb}{0325}
\pdfglyphtounicode{ringcmb}{030A}
\pdfglyphtounicode{ringhalfleft}{02BF}
\pdfglyphtounicode{ringhalfleftarmenian}{0559}
\pdfglyphtounicode{ringhalfleftbelowcmb}{031C}
\pdfglyphtounicode{ringhalfleftcentered}{02D3}
\pdfglyphtounicode{ringhalfright}{02BE}
\pdfglyphtounicode{ringhalfrightbelowcmb}{0339}
\pdfglyphtounicode{ringhalfrightcentered}{02D2}
\pdfglyphtounicode{ringinequal}{2256}
\pdfglyphtounicode{rinvertedbreve}{0213}
\pdfglyphtounicode{rittorusquare}{3351}
\pdfglyphtounicode{rlinebelow}{1E5F}
\pdfglyphtounicode{rlongleg}{027C}
\pdfglyphtounicode{rlonglegturned}{027A}
\pdfglyphtounicode{rmonospace}{FF52}
\pdfglyphtounicode{rohiragana}{308D}
\pdfglyphtounicode{rokatakana}{30ED}
\pdfglyphtounicode{rokatakanahalfwidth}{FF9B}
\pdfglyphtounicode{roruathai}{0E23}
\pdfglyphtounicode{rparen}{24AD}
\pdfglyphtounicode{rrabengali}{09DC}
\pdfglyphtounicode{rradeva}{0931}
\pdfglyphtounicode{rragurmukhi}{0A5C}
\pdfglyphtounicode{rreharabic}{0691}
\pdfglyphtounicode{rrehfinalarabic}{FB8D}
\pdfglyphtounicode{rrvocalicbengali}{09E0}
\pdfglyphtounicode{rrvocalicdeva}{0960}
\pdfglyphtounicode{rrvocalicgujarati}{0AE0}
\pdfglyphtounicode{rrvocalicvowelsignbengali}{09C4}
\pdfglyphtounicode{rrvocalicvowelsigndeva}{0944}
\pdfglyphtounicode{rrvocalicvowelsigngujarati}{0AC4}
\pdfglyphtounicode{rsuperior}{0072}
\pdfglyphtounicode{rtblock}{2590}
\pdfglyphtounicode{rturned}{0279}
\pdfglyphtounicode{rturnedsuperior}{02B4}
\pdfglyphtounicode{ruhiragana}{308B}
\pdfglyphtounicode{rukatakana}{30EB}
\pdfglyphtounicode{rukatakanahalfwidth}{FF99}
\pdfglyphtounicode{rupeemarkbengali}{09F2}
\pdfglyphtounicode{rupeesignbengali}{09F3}
\pdfglyphtounicode{rupiah}{20A8}
\pdfglyphtounicode{ruthai}{0E24}
\pdfglyphtounicode{rvocalicbengali}{098B}
\pdfglyphtounicode{rvocalicdeva}{090B}
\pdfglyphtounicode{rvocalicgujarati}{0A8B}
\pdfglyphtounicode{rvocalicvowelsignbengali}{09C3}
\pdfglyphtounicode{rvocalicvowelsigndeva}{0943}
\pdfglyphtounicode{rvocalicvowelsigngujarati}{0AC3}
\pdfglyphtounicode{s}{0073}
\pdfglyphtounicode{sabengali}{09B8}
\pdfglyphtounicode{sacute}{015B}
\pdfglyphtounicode{sacutedotaccent}{1E65}
\pdfglyphtounicode{sadarabic}{0635}
\pdfglyphtounicode{sadeva}{0938}
\pdfglyphtounicode{sadfinalarabic}{FEBA}
\pdfglyphtounicode{sadinitialarabic}{FEBB}
\pdfglyphtounicode{sadmedialarabic}{FEBC}
\pdfglyphtounicode{sagujarati}{0AB8}
\pdfglyphtounicode{sagurmukhi}{0A38}
\pdfglyphtounicode{sahiragana}{3055}
\pdfglyphtounicode{sakatakana}{30B5}
\pdfglyphtounicode{sakatakanahalfwidth}{FF7B}
\pdfglyphtounicode{sallallahoualayhewasallamarabic}{FDFA}
\pdfglyphtounicode{samekh}{05E1}
\pdfglyphtounicode{samekhdagesh}{FB41}
\pdfglyphtounicode{samekhdageshhebrew}{FB41}
\pdfglyphtounicode{samekhhebrew}{05E1}
\pdfglyphtounicode{saraaathai}{0E32}
\pdfglyphtounicode{saraaethai}{0E41}
\pdfglyphtounicode{saraaimaimalaithai}{0E44}
\pdfglyphtounicode{saraaimaimuanthai}{0E43}
\pdfglyphtounicode{saraamthai}{0E33}
\pdfglyphtounicode{saraathai}{0E30}
\pdfglyphtounicode{saraethai}{0E40}
\pdfglyphtounicode{saraiileftthai}{F886}
\pdfglyphtounicode{saraiithai}{0E35}
\pdfglyphtounicode{saraileftthai}{F885}
\pdfglyphtounicode{saraithai}{0E34}
\pdfglyphtounicode{saraothai}{0E42}
\pdfglyphtounicode{saraueeleftthai}{F888}
\pdfglyphtounicode{saraueethai}{0E37}
\pdfglyphtounicode{saraueleftthai}{F887}
\pdfglyphtounicode{sarauethai}{0E36}
\pdfglyphtounicode{sarauthai}{0E38}
\pdfglyphtounicode{sarauuthai}{0E39}
\pdfglyphtounicode{satisfies}{22A8}
\pdfglyphtounicode{sbopomofo}{3119}
\pdfglyphtounicode{scaron}{0161}
\pdfglyphtounicode{scarondotaccent}{1E67}
\pdfglyphtounicode{scedilla}{015F}
\pdfglyphtounicode{schwa}{0259}
\pdfglyphtounicode{schwacyrillic}{04D9}
\pdfglyphtounicode{schwadieresiscyrillic}{04DB}
\pdfglyphtounicode{schwahook}{025A}
\pdfglyphtounicode{scircle}{24E2}
\pdfglyphtounicode{scircumflex}{015D}
\pdfglyphtounicode{scommaaccent}{0219}
\pdfglyphtounicode{sdotaccent}{1E61}
\pdfglyphtounicode{sdotbelow}{1E63}
\pdfglyphtounicode{sdotbelowdotaccent}{1E69}
\pdfglyphtounicode{seagullbelowcmb}{033C}
\pdfglyphtounicode{second}{2033}
\pdfglyphtounicode{secondtonechinese}{02CA}
\pdfglyphtounicode{section}{00A7}
\pdfglyphtounicode{seenarabic}{0633}
\pdfglyphtounicode{seenfinalarabic}{FEB2}
\pdfglyphtounicode{seeninitialarabic}{FEB3}
\pdfglyphtounicode{seenmedialarabic}{FEB4}
\pdfglyphtounicode{segol}{05B6}
\pdfglyphtounicode{segol13}{05B6}
\pdfglyphtounicode{segol1f}{05B6}
\pdfglyphtounicode{segol2c}{05B6}
\pdfglyphtounicode{segolhebrew}{05B6}
\pdfglyphtounicode{segolnarrowhebrew}{05B6}
\pdfglyphtounicode{segolquarterhebrew}{05B6}
\pdfglyphtounicode{segoltahebrew}{0592}
\pdfglyphtounicode{segolwidehebrew}{05B6}
\pdfglyphtounicode{seharmenian}{057D}
\pdfglyphtounicode{sehiragana}{305B}
\pdfglyphtounicode{sekatakana}{30BB}
\pdfglyphtounicode{sekatakanahalfwidth}{FF7E}
\pdfglyphtounicode{semicolon}{003B}
\pdfglyphtounicode{semicolonarabic}{061B}
\pdfglyphtounicode{semicolonmonospace}{FF1B}
\pdfglyphtounicode{semicolonsmall}{FE54}
\pdfglyphtounicode{semivoicedmarkkana}{309C}
\pdfglyphtounicode{semivoicedmarkkanahalfwidth}{FF9F}
\pdfglyphtounicode{sentisquare}{3322}
\pdfglyphtounicode{sentosquare}{3323}
\pdfglyphtounicode{seven}{0037}
\pdfglyphtounicode{sevenarabic}{0667}
\pdfglyphtounicode{sevenbengali}{09ED}
\pdfglyphtounicode{sevencircle}{2466}
\pdfglyphtounicode{sevencircleinversesansserif}{2790}
\pdfglyphtounicode{sevendeva}{096D}
\pdfglyphtounicode{seveneighths}{215E}
\pdfglyphtounicode{sevengujarati}{0AED}
\pdfglyphtounicode{sevengurmukhi}{0A6D}
\pdfglyphtounicode{sevenhackarabic}{0667}
\pdfglyphtounicode{sevenhangzhou}{3027}
\pdfglyphtounicode{sevenideographicparen}{3226}
\pdfglyphtounicode{seveninferior}{2087}
\pdfglyphtounicode{sevenmonospace}{FF17}
\pdfglyphtounicode{sevenoldstyle}{0037}
\pdfglyphtounicode{sevenparen}{247A}
\pdfglyphtounicode{sevenperiod}{248E}
\pdfglyphtounicode{sevenpersian}{06F7}
\pdfglyphtounicode{sevenroman}{2176}
\pdfglyphtounicode{sevensuperior}{2077}
\pdfglyphtounicode{seventeencircle}{2470}
\pdfglyphtounicode{seventeenparen}{2484}
\pdfglyphtounicode{seventeenperiod}{2498}
\pdfglyphtounicode{seventhai}{0E57}
\pdfglyphtounicode{sfthyphen}{00AD}
\pdfglyphtounicode{shaarmenian}{0577}
\pdfglyphtounicode{shabengali}{09B6}
\pdfglyphtounicode{shacyrillic}{0448}
\pdfglyphtounicode{shaddaarabic}{0651}
\pdfglyphtounicode{shaddadammaarabic}{FC61}
\pdfglyphtounicode{shaddadammatanarabic}{FC5E}
\pdfglyphtounicode{shaddafathaarabic}{FC60}
\pdfglyphtounicode{shaddafathatanarabic}{0651 064B}
\pdfglyphtounicode{shaddakasraarabic}{FC62}
\pdfglyphtounicode{shaddakasratanarabic}{FC5F}
\pdfglyphtounicode{shade}{2592}
\pdfglyphtounicode{shadedark}{2593}
\pdfglyphtounicode{shadelight}{2591}
\pdfglyphtounicode{shademedium}{2592}
\pdfglyphtounicode{shadeva}{0936}
\pdfglyphtounicode{shagujarati}{0AB6}
\pdfglyphtounicode{shagurmukhi}{0A36}
\pdfglyphtounicode{shalshelethebrew}{0593}
\pdfglyphtounicode{sharp}{266F}
\pdfglyphtounicode{shbopomofo}{3115}
\pdfglyphtounicode{shchacyrillic}{0449}
\pdfglyphtounicode{sheenarabic}{0634}
\pdfglyphtounicode{sheenfinalarabic}{FEB6}
\pdfglyphtounicode{sheeninitialarabic}{FEB7}
\pdfglyphtounicode{sheenmedialarabic}{FEB8}
\pdfglyphtounicode{sheicoptic}{03E3}
\pdfglyphtounicode{sheqel}{20AA}
\pdfglyphtounicode{sheqelhebrew}{20AA}
\pdfglyphtounicode{sheva}{05B0}
\pdfglyphtounicode{sheva115}{05B0}
\pdfglyphtounicode{sheva15}{05B0}
\pdfglyphtounicode{sheva22}{05B0}
\pdfglyphtounicode{sheva2e}{05B0}
\pdfglyphtounicode{shevahebrew}{05B0}
\pdfglyphtounicode{shevanarrowhebrew}{05B0}
\pdfglyphtounicode{shevaquarterhebrew}{05B0}
\pdfglyphtounicode{shevawidehebrew}{05B0}
\pdfglyphtounicode{shhacyrillic}{04BB}
\pdfglyphtounicode{shiftleft}{21B0}
\pdfglyphtounicode{shiftright}{21B1}
\pdfglyphtounicode{shimacoptic}{03ED}
\pdfglyphtounicode{shin}{05E9}
\pdfglyphtounicode{shindagesh}{FB49}
\pdfglyphtounicode{shindageshhebrew}{FB49}
\pdfglyphtounicode{shindageshshindot}{FB2C}
\pdfglyphtounicode{shindageshshindothebrew}{FB2C}
\pdfglyphtounicode{shindageshsindot}{FB2D}
\pdfglyphtounicode{shindageshsindothebrew}{FB2D}
\pdfglyphtounicode{shindothebrew}{05C1}
\pdfglyphtounicode{shinhebrew}{05E9}
\pdfglyphtounicode{shinshindot}{FB2A}
\pdfglyphtounicode{shinshindothebrew}{FB2A}
\pdfglyphtounicode{shinsindot}{FB2B}
\pdfglyphtounicode{shinsindothebrew}{FB2B}
\pdfglyphtounicode{shook}{0282}
\pdfglyphtounicode{sigma}{03C3}
\pdfglyphtounicode{sigma1}{03C2}
\pdfglyphtounicode{sigmafinal}{03C2}
\pdfglyphtounicode{sigmalunatesymbolgreek}{03F2}
\pdfglyphtounicode{sihiragana}{3057}
\pdfglyphtounicode{sikatakana}{30B7}
\pdfglyphtounicode{sikatakanahalfwidth}{FF7C}
\pdfglyphtounicode{siluqhebrew}{05BD}
\pdfglyphtounicode{siluqlefthebrew}{05BD}
\pdfglyphtounicode{similar}{223C}
\pdfglyphtounicode{similarequal}{2243}
\pdfglyphtounicode{sindothebrew}{05C2}
\pdfglyphtounicode{siosacirclekorean}{3274}
\pdfglyphtounicode{siosaparenkorean}{3214}
\pdfglyphtounicode{sioscieuckorean}{317E}
\pdfglyphtounicode{sioscirclekorean}{3266}
\pdfglyphtounicode{sioskiyeokkorean}{317A}
\pdfglyphtounicode{sioskorean}{3145}
\pdfglyphtounicode{siosnieunkorean}{317B}
\pdfglyphtounicode{siosparenkorean}{3206}
\pdfglyphtounicode{siospieupkorean}{317D}
\pdfglyphtounicode{siostikeutkorean}{317C}
\pdfglyphtounicode{six}{0036}
\pdfglyphtounicode{sixarabic}{0666}
\pdfglyphtounicode{sixbengali}{09EC}
\pdfglyphtounicode{sixcircle}{2465}
\pdfglyphtounicode{sixcircleinversesansserif}{278F}
\pdfglyphtounicode{sixdeva}{096C}
\pdfglyphtounicode{sixgujarati}{0AEC}
\pdfglyphtounicode{sixgurmukhi}{0A6C}
\pdfglyphtounicode{sixhackarabic}{0666}
\pdfglyphtounicode{sixhangzhou}{3026}
\pdfglyphtounicode{sixideographicparen}{3225}
\pdfglyphtounicode{sixinferior}{2086}
\pdfglyphtounicode{sixmonospace}{FF16}
\pdfglyphtounicode{sixoldstyle}{0036}
\pdfglyphtounicode{sixparen}{2479}
\pdfglyphtounicode{sixperiod}{248D}
\pdfglyphtounicode{sixpersian}{06F6}
\pdfglyphtounicode{sixroman}{2175}
\pdfglyphtounicode{sixsuperior}{2076}
\pdfglyphtounicode{sixteencircle}{246F}
\pdfglyphtounicode{sixteencurrencydenominatorbengali}{09F9}
\pdfglyphtounicode{sixteenparen}{2483}
\pdfglyphtounicode{sixteenperiod}{2497}
\pdfglyphtounicode{sixthai}{0E56}
\pdfglyphtounicode{slash}{002F}
\pdfglyphtounicode{slashmonospace}{FF0F}
\pdfglyphtounicode{slong}{017F}
\pdfglyphtounicode{slongdotaccent}{1E9B}
\pdfglyphtounicode{slurabove}{2322}
\pdfglyphtounicode{slurbelow}{2323}
\pdfglyphtounicode{smile}{2323}
\pdfglyphtounicode{smileface}{263A}
\pdfglyphtounicode{smonospace}{FF53}
\pdfglyphtounicode{sofpasuqhebrew}{05C3}
\pdfglyphtounicode{softhyphen}{00AD}
\pdfglyphtounicode{softsigncyrillic}{044C}
\pdfglyphtounicode{sohiragana}{305D}
\pdfglyphtounicode{sokatakana}{30BD}
\pdfglyphtounicode{sokatakanahalfwidth}{FF7F}
\pdfglyphtounicode{soliduslongoverlaycmb}{0338}
\pdfglyphtounicode{solidusshortoverlaycmb}{0337}
\pdfglyphtounicode{sorusithai}{0E29}
\pdfglyphtounicode{sosalathai}{0E28}
\pdfglyphtounicode{sosothai}{0E0B}
\pdfglyphtounicode{sosuathai}{0E2A}
\pdfglyphtounicode{space}{0020}
\pdfglyphtounicode{spacehackarabic}{0020}
\pdfglyphtounicode{spade}{2660}
\pdfglyphtounicode{spadesuitblack}{2660}
\pdfglyphtounicode{spadesuitwhite}{2664}
\pdfglyphtounicode{sparen}{24AE}
\pdfglyphtounicode{sphericalangle}{2222}
\pdfglyphtounicode{square}{25A1}
\pdfglyphtounicode{squarebelowcmb}{033B}
\pdfglyphtounicode{squarecc}{33C4}
\pdfglyphtounicode{squarecm}{339D}
\pdfglyphtounicode{squarediagonalcrosshatchfill}{25A9}
\pdfglyphtounicode{squaredot}{22A1}
\pdfglyphtounicode{squarehorizontalfill}{25A4}
\pdfglyphtounicode{squareimage}{228F}
\pdfglyphtounicode{squarekg}{338F}
\pdfglyphtounicode{squarekm}{339E}
\pdfglyphtounicode{squarekmcapital}{33CE}
\pdfglyphtounicode{squareln}{33D1}
\pdfglyphtounicode{squarelog}{33D2}
\pdfglyphtounicode{squaremg}{338E}
\pdfglyphtounicode{squaremil}{33D5}
\pdfglyphtounicode{squareminus}{229F}
\pdfglyphtounicode{squaremm}{339C}
\pdfglyphtounicode{squaremsquared}{33A1}
\pdfglyphtounicode{squaremultiply}{22A0}
\pdfglyphtounicode{squareoriginal}{2290}
\pdfglyphtounicode{squareorthogonalcrosshatchfill}{25A6}
\pdfglyphtounicode{squareplus}{229E}
\pdfglyphtounicode{squaresolid}{25A0}
\pdfglyphtounicode{squareupperlefttolowerrightfill}{25A7}
\pdfglyphtounicode{squareupperrighttolowerleftfill}{25A8}
\pdfglyphtounicode{squareverticalfill}{25A5}
\pdfglyphtounicode{squarewhitewithsmallblack}{25A3}
\pdfglyphtounicode{squiggleleftright}{21AD}
\pdfglyphtounicode{squiggleright}{21DD}
\pdfglyphtounicode{srsquare}{33DB}
\pdfglyphtounicode{ssabengali}{09B7}
\pdfglyphtounicode{ssadeva}{0937}
\pdfglyphtounicode{ssagujarati}{0AB7}
\pdfglyphtounicode{ssangcieuckorean}{3149}
\pdfglyphtounicode{ssanghieuhkorean}{3185}
\pdfglyphtounicode{ssangieungkorean}{3180}
\pdfglyphtounicode{ssangkiyeokkorean}{3132}
\pdfglyphtounicode{ssangnieunkorean}{3165}
\pdfglyphtounicode{ssangpieupkorean}{3143}
\pdfglyphtounicode{ssangsioskorean}{3146}
\pdfglyphtounicode{ssangtikeutkorean}{3138}
\pdfglyphtounicode{ssuperior}{0073}
\pdfglyphtounicode{st}{0073 0074}
\pdfglyphtounicode{star}{22C6}
\pdfglyphtounicode{sterling}{00A3}
\pdfglyphtounicode{sterlingmonospace}{FFE1}
\pdfglyphtounicode{strokelongoverlaycmb}{0336}
\pdfglyphtounicode{strokeshortoverlaycmb}{0335}
\pdfglyphtounicode{subset}{2282}
\pdfglyphtounicode{subsetdbl}{22D0}
\pdfglyphtounicode{subsetdblequal}{2AC5}
\pdfglyphtounicode{subsetnoteql}{228A}
\pdfglyphtounicode{subsetnotequal}{228A}
\pdfglyphtounicode{subsetorequal}{2286}
\pdfglyphtounicode{subsetornotdbleql}{2ACB}
\pdfglyphtounicode{subsetsqequal}{2291}
\pdfglyphtounicode{succeeds}{227B}
\pdfglyphtounicode{suchthat}{220B}
\pdfglyphtounicode{suhiragana}{3059}
\pdfglyphtounicode{sukatakana}{30B9}
\pdfglyphtounicode{sukatakanahalfwidth}{FF7D}
\pdfglyphtounicode{sukunarabic}{0652}
\pdfglyphtounicode{summation}{2211}
\pdfglyphtounicode{sun}{263C}
\pdfglyphtounicode{superset}{2283}
\pdfglyphtounicode{supersetdbl}{22D1}
\pdfglyphtounicode{supersetdblequal}{2AC6}
\pdfglyphtounicode{supersetnoteql}{228B}
\pdfglyphtounicode{supersetnotequal}{228B}
\pdfglyphtounicode{supersetorequal}{2287}
\pdfglyphtounicode{supersetornotdbleql}{2ACC}
\pdfglyphtounicode{supersetsqequal}{2292}
\pdfglyphtounicode{svsquare}{33DC}
\pdfglyphtounicode{syouwaerasquare}{337C}
\pdfglyphtounicode{t}{0074}
\pdfglyphtounicode{tabengali}{09A4}
\pdfglyphtounicode{tackdown}{22A4}
\pdfglyphtounicode{tackleft}{22A3}
\pdfglyphtounicode{tadeva}{0924}
\pdfglyphtounicode{tagujarati}{0AA4}
\pdfglyphtounicode{tagurmukhi}{0A24}
\pdfglyphtounicode{taharabic}{0637}
\pdfglyphtounicode{tahfinalarabic}{FEC2}
\pdfglyphtounicode{tahinitialarabic}{FEC3}
\pdfglyphtounicode{tahiragana}{305F}
\pdfglyphtounicode{tahmedialarabic}{FEC4}
\pdfglyphtounicode{taisyouerasquare}{337D}
\pdfglyphtounicode{takatakana}{30BF}
\pdfglyphtounicode{takatakanahalfwidth}{FF80}
\pdfglyphtounicode{tatweelarabic}{0640}
\pdfglyphtounicode{tau}{03C4}
\pdfglyphtounicode{tav}{05EA}
\pdfglyphtounicode{tavdages}{FB4A}
\pdfglyphtounicode{tavdagesh}{FB4A}
\pdfglyphtounicode{tavdageshhebrew}{FB4A}
\pdfglyphtounicode{tavhebrew}{05EA}
\pdfglyphtounicode{tbar}{0167}
\pdfglyphtounicode{tbopomofo}{310A}
\pdfglyphtounicode{tcaron}{0165}
\pdfglyphtounicode{tccurl}{02A8}
\pdfglyphtounicode{tcedilla}{0163}
\pdfglyphtounicode{tcheharabic}{0686}
\pdfglyphtounicode{tchehfinalarabic}{FB7B}
\pdfglyphtounicode{tchehinitialarabic}{FB7C}
\pdfglyphtounicode{tchehmedialarabic}{FB7D}
\pdfglyphtounicode{tchehmeeminitialarabic}{FB7C FEE4}
\pdfglyphtounicode{tcircle}{24E3}
\pdfglyphtounicode{tcircumflexbelow}{1E71}
\pdfglyphtounicode{tcommaaccent}{0163}
\pdfglyphtounicode{tdieresis}{1E97}
\pdfglyphtounicode{tdotaccent}{1E6B}
\pdfglyphtounicode{tdotbelow}{1E6D}
\pdfglyphtounicode{tecyrillic}{0442}
\pdfglyphtounicode{tedescendercyrillic}{04AD}
\pdfglyphtounicode{teharabic}{062A}
\pdfglyphtounicode{tehfinalarabic}{FE96}
\pdfglyphtounicode{tehhahinitialarabic}{FCA2}
\pdfglyphtounicode{tehhahisolatedarabic}{FC0C}
\pdfglyphtounicode{tehinitialarabic}{FE97}
\pdfglyphtounicode{tehiragana}{3066}
\pdfglyphtounicode{tehjeeminitialarabic}{FCA1}
\pdfglyphtounicode{tehjeemisolatedarabic}{FC0B}
\pdfglyphtounicode{tehmarbutaarabic}{0629}
\pdfglyphtounicode{tehmarbutafinalarabic}{FE94}
\pdfglyphtounicode{tehmedialarabic}{FE98}
\pdfglyphtounicode{tehmeeminitialarabic}{FCA4}
\pdfglyphtounicode{tehmeemisolatedarabic}{FC0E}
\pdfglyphtounicode{tehnoonfinalarabic}{FC73}
\pdfglyphtounicode{tekatakana}{30C6}
\pdfglyphtounicode{tekatakanahalfwidth}{FF83}
\pdfglyphtounicode{telephone}{2121}
\pdfglyphtounicode{telephoneblack}{260E}
\pdfglyphtounicode{telishagedolahebrew}{05A0}
\pdfglyphtounicode{telishaqetanahebrew}{05A9}
\pdfglyphtounicode{tencircle}{2469}
\pdfglyphtounicode{tenideographicparen}{3229}
\pdfglyphtounicode{tenparen}{247D}
\pdfglyphtounicode{tenperiod}{2491}
\pdfglyphtounicode{tenroman}{2179}
\pdfglyphtounicode{tesh}{02A7}
\pdfglyphtounicode{tet}{05D8}
\pdfglyphtounicode{tetdagesh}{FB38}
\pdfglyphtounicode{tetdageshhebrew}{FB38}
\pdfglyphtounicode{tethebrew}{05D8}
\pdfglyphtounicode{tetsecyrillic}{04B5}
\pdfglyphtounicode{tevirhebrew}{059B}
\pdfglyphtounicode{tevirlefthebrew}{059B}
\pdfglyphtounicode{tfm:cmbsy10/diamond}{2662}
\pdfglyphtounicode{tfm:cmbsy10/heart}{2661}
\pdfglyphtounicode{tfm:cmbsy5/diamond}{2662}
\pdfglyphtounicode{tfm:cmbsy5/heart}{2661}
\pdfglyphtounicode{tfm:cmbsy6/diamond}{2662}
\pdfglyphtounicode{tfm:cmbsy6/heart}{2661}
\pdfglyphtounicode{tfm:cmbsy7/diamond}{2662}
\pdfglyphtounicode{tfm:cmbsy7/heart}{2661}
\pdfglyphtounicode{tfm:cmbsy8/diamond}{2662}
\pdfglyphtounicode{tfm:cmbsy8/heart}{2661}
\pdfglyphtounicode{tfm:cmbsy9/diamond}{2662}
\pdfglyphtounicode{tfm:cmbsy9/heart}{2661}
\pdfglyphtounicode{tfm:cmmi10/phi}{03D5}
\pdfglyphtounicode{tfm:cmmi10/phi1}{03C6}
\pdfglyphtounicode{tfm:cmmi12/phi}{03D5}
\pdfglyphtounicode{tfm:cmmi12/phi1}{03C6}
\pdfglyphtounicode{tfm:cmmi5/phi}{03D5}
\pdfglyphtounicode{tfm:cmmi5/phi1}{03C6}
\pdfglyphtounicode{tfm:cmmi6/phi}{03D5}
\pdfglyphtounicode{tfm:cmmi6/phi1}{03C6}
\pdfglyphtounicode{tfm:cmmi7/phi}{03D5}
\pdfglyphtounicode{tfm:cmmi7/phi1}{03C6}
\pdfglyphtounicode{tfm:cmmi8/phi}{03D5}
\pdfglyphtounicode{tfm:cmmi8/phi1}{03C6}
\pdfglyphtounicode{tfm:cmmi9/phi}{03D5}
\pdfglyphtounicode{tfm:cmmi9/phi1}{03C6}
\pdfglyphtounicode{tfm:cmmib10/phi}{03D5}
\pdfglyphtounicode{tfm:cmmib10/phi1}{03C6}
\pdfglyphtounicode{tfm:cmmib5/phi}{03D5}
\pdfglyphtounicode{tfm:cmmib5/phi1}{03C6}
\pdfglyphtounicode{tfm:cmmib6/phi}{03D5}
\pdfglyphtounicode{tfm:cmmib6/phi1}{03C6}
\pdfglyphtounicode{tfm:cmmib7/phi}{03D5}
\pdfglyphtounicode{tfm:cmmib7/phi1}{03C6}
\pdfglyphtounicode{tfm:cmmib8/phi}{03D5}
\pdfglyphtounicode{tfm:cmmib8/phi1}{03C6}
\pdfglyphtounicode{tfm:cmmib9/phi}{03D5}
\pdfglyphtounicode{tfm:cmmib9/phi1}{03C6}
\pdfglyphtounicode{tfm:cmsy10/diamond}{2662}
\pdfglyphtounicode{tfm:cmsy10/heart}{2661}
\pdfglyphtounicode{tfm:cmsy5/heart}{2661}
\pdfglyphtounicode{tfm:cmsy6/diamond}{2662}
\pdfglyphtounicode{tfm:cmsy6/heart}{2661}
\pdfglyphtounicode{tfm:cmsy7/diamond}{2662}
\pdfglyphtounicode{tfm:cmsy7/heart}{2661}
\pdfglyphtounicode{tfm:cmsy8/diamond}{2662}
\pdfglyphtounicode{tfm:cmsy8/heart}{2661}
\pdfglyphtounicode{tfm:cmsy9/diamond}{2662}
\pdfglyphtounicode{tfm:cmsy9/heart}{2661}
\pdfglyphtounicode{tfm:eurb10/phi}{03D5}
\pdfglyphtounicode{tfm:eurb10/phi1}{03C6}
\pdfglyphtounicode{tfm:eurb5/phi}{03D5}
\pdfglyphtounicode{tfm:eurb5/phi1}{03C6}
\pdfglyphtounicode{tfm:eurb6/phi}{03D5}
\pdfglyphtounicode{tfm:eurb6/phi1}{03C6}
\pdfglyphtounicode{tfm:eurb7/phi}{03D5}
\pdfglyphtounicode{tfm:eurb7/phi1}{03C6}
\pdfglyphtounicode{tfm:eurb8/phi}{03D5}
\pdfglyphtounicode{tfm:eurb8/phi1}{03C6}
\pdfglyphtounicode{tfm:eurb9/phi}{03D5}
\pdfglyphtounicode{tfm:eurb9/phi1}{03C6}
\pdfglyphtounicode{tfm:eurm10/phi}{03D5}
\pdfglyphtounicode{tfm:eurm10/phi1}{03C6}
\pdfglyphtounicode{tfm:eurm5/phi}{03D5}
\pdfglyphtounicode{tfm:eurm5/phi1}{03C6}
\pdfglyphtounicode{tfm:eurm6/phi}{03D5}
\pdfglyphtounicode{tfm:eurm6/phi1}{03C6}
\pdfglyphtounicode{tfm:eurm7/phi}{03D5}
\pdfglyphtounicode{tfm:eurm7/phi1}{03C6}
\pdfglyphtounicode{tfm:eurm8/phi}{03D5}
\pdfglyphtounicode{tfm:eurm8/phi1}{03C6}
\pdfglyphtounicode{tfm:eurm9/phi}{03D5}
\pdfglyphtounicode{tfm:eurm9/phi1}{03C6}
\pdfglyphtounicode{tfm:fplmbi/phi}{03D5}
\pdfglyphtounicode{tfm:fplmbi/phi1}{03C6}
\pdfglyphtounicode{tfm:fplmri/phi}{03D5}
\pdfglyphtounicode{tfm:fplmri/phi1}{03C6}
\pdfglyphtounicode{tfm:lmbsy10/diamond}{2662}
\pdfglyphtounicode{tfm:lmbsy10/heart}{2661}
\pdfglyphtounicode{tfm:lmbsy5/diamond}{2662}
\pdfglyphtounicode{tfm:lmbsy5/heart}{2661}
\pdfglyphtounicode{tfm:lmbsy7/diamond}{2662}
\pdfglyphtounicode{tfm:lmbsy7/heart}{2661}
\pdfglyphtounicode{tfm:lmmi10/phi}{03D5}
\pdfglyphtounicode{tfm:lmmi10/phi1}{03C6}
\pdfglyphtounicode{tfm:lmmi12/phi}{03D5}
\pdfglyphtounicode{tfm:lmmi12/phi1}{03C6}
\pdfglyphtounicode{tfm:lmmi5/phi}{03D5}
\pdfglyphtounicode{tfm:lmmi5/phi1}{03C6}
\pdfglyphtounicode{tfm:lmmi6/phi}{03D5}
\pdfglyphtounicode{tfm:lmmi6/phi1}{03C6}
\pdfglyphtounicode{tfm:lmmi7/phi}{03D5}
\pdfglyphtounicode{tfm:lmmi7/phi1}{03C6}
\pdfglyphtounicode{tfm:lmmi8/phi}{03D5}
\pdfglyphtounicode{tfm:lmmi8/phi1}{03C6}
\pdfglyphtounicode{tfm:lmmi9/phi}{03D5}
\pdfglyphtounicode{tfm:lmmi9/phi1}{03C6}
\pdfglyphtounicode{tfm:lmmib10/phi}{03D5}
\pdfglyphtounicode{tfm:lmmib10/phi1}{03C6}
\pdfglyphtounicode{tfm:lmmib5/phi}{03D5}
\pdfglyphtounicode{tfm:lmmib5/phi1}{03C6}
\pdfglyphtounicode{tfm:lmmib7/phi}{03D5}
\pdfglyphtounicode{tfm:lmmib7/phi1}{03C6}
\pdfglyphtounicode{tfm:lmsy10/diamond}{2662}
\pdfglyphtounicode{tfm:lmsy10/heart}{2661}
\pdfglyphtounicode{tfm:lmsy5/diamond}{2662}
\pdfglyphtounicode{tfm:lmsy5/heart}{2661}
\pdfglyphtounicode{tfm:lmsy6/diamond}{2662}
\pdfglyphtounicode{tfm:lmsy6/heart}{2661}
\pdfglyphtounicode{tfm:lmsy7/diamond}{2662}
\pdfglyphtounicode{tfm:lmsy7/heart}{2661}
\pdfglyphtounicode{tfm:lmsy8/diamond}{2662}
\pdfglyphtounicode{tfm:lmsy8/heart}{2661}
\pdfglyphtounicode{tfm:lmsy9/diamond}{2662}
\pdfglyphtounicode{tfm:lmsy9/heart}{2661}
\pdfglyphtounicode{tfm:msam10/diamond}{2662}
\pdfglyphtounicode{tfm:msam5/diamond}{2662}
\pdfglyphtounicode{tfm:msam6/diamond}{2662}
\pdfglyphtounicode{tfm:msam7/diamond}{2662}
\pdfglyphtounicode{tfm:msam8/diamond}{2662}
\pdfglyphtounicode{tfm:msam9/diamond}{2662}
\pdfglyphtounicode{tfm:pxbmia/phi}{03D5}
\pdfglyphtounicode{tfm:pxbmia/phi1}{03C6}
\pdfglyphtounicode{tfm:pxbsy/diamond}{2662}
\pdfglyphtounicode{tfm:pxbsy/heart}{2661}
\pdfglyphtounicode{tfm:pxbsya/diamond}{2662}
\pdfglyphtounicode{tfm:pxmia/phi}{03D5}
\pdfglyphtounicode{tfm:pxmia/phi1}{03C6}
\pdfglyphtounicode{tfm:pxsy/diamond}{2662}
\pdfglyphtounicode{tfm:pxsy/heart}{2661}
\pdfglyphtounicode{tfm:pxsya/diamond}{2662}
\pdfglyphtounicode{tfm:pzdr/a1}{2701}
\pdfglyphtounicode{tfm:pzdr/a10}{2721}
\pdfglyphtounicode{tfm:pzdr/a100}{275E}
\pdfglyphtounicode{tfm:pzdr/a101}{2761}
\pdfglyphtounicode{tfm:pzdr/a102}{2762}
\pdfglyphtounicode{tfm:pzdr/a103}{2763}
\pdfglyphtounicode{tfm:pzdr/a104}{2764}
\pdfglyphtounicode{tfm:pzdr/a105}{2710}
\pdfglyphtounicode{tfm:pzdr/a106}{2765}
\pdfglyphtounicode{tfm:pzdr/a107}{2766}
\pdfglyphtounicode{tfm:pzdr/a108}{2767}
\pdfglyphtounicode{tfm:pzdr/a109}{2660}
\pdfglyphtounicode{tfm:pzdr/a11}{261B}
\pdfglyphtounicode{tfm:pzdr/a110}{2665}
\pdfglyphtounicode{tfm:pzdr/a111}{2666}
\pdfglyphtounicode{tfm:pzdr/a112}{2663}
\pdfglyphtounicode{tfm:pzdr/a117}{2709}
\pdfglyphtounicode{tfm:pzdr/a118}{2708}
\pdfglyphtounicode{tfm:pzdr/a119}{2707}
\pdfglyphtounicode{tfm:pzdr/a12}{261E}
\pdfglyphtounicode{tfm:pzdr/a120}{2460}
\pdfglyphtounicode{tfm:pzdr/a121}{2461}
\pdfglyphtounicode{tfm:pzdr/a122}{2462}
\pdfglyphtounicode{tfm:pzdr/a123}{2463}
\pdfglyphtounicode{tfm:pzdr/a124}{2464}
\pdfglyphtounicode{tfm:pzdr/a125}{2465}
\pdfglyphtounicode{tfm:pzdr/a126}{2466}
\pdfglyphtounicode{tfm:pzdr/a127}{2467}
\pdfglyphtounicode{tfm:pzdr/a128}{2468}
\pdfglyphtounicode{tfm:pzdr/a129}{2469}
\pdfglyphtounicode{tfm:pzdr/a13}{270C}
\pdfglyphtounicode{tfm:pzdr/a130}{2776}
\pdfglyphtounicode{tfm:pzdr/a131}{2777}
\pdfglyphtounicode{tfm:pzdr/a132}{2778}
\pdfglyphtounicode{tfm:pzdr/a133}{2779}
\pdfglyphtounicode{tfm:pzdr/a134}{277A}
\pdfglyphtounicode{tfm:pzdr/a135}{277B}
\pdfglyphtounicode{tfm:pzdr/a136}{277C}
\pdfglyphtounicode{tfm:pzdr/a137}{277D}
\pdfglyphtounicode{tfm:pzdr/a138}{277E}
\pdfglyphtounicode{tfm:pzdr/a139}{277F}
\pdfglyphtounicode{tfm:pzdr/a14}{270D}
\pdfglyphtounicode{tfm:pzdr/a140}{2780}
\pdfglyphtounicode{tfm:pzdr/a141}{2781}
\pdfglyphtounicode{tfm:pzdr/a142}{2782}
\pdfglyphtounicode{tfm:pzdr/a143}{2783}
\pdfglyphtounicode{tfm:pzdr/a144}{2784}
\pdfglyphtounicode{tfm:pzdr/a145}{2785}
\pdfglyphtounicode{tfm:pzdr/a146}{2786}
\pdfglyphtounicode{tfm:pzdr/a147}{2787}
\pdfglyphtounicode{tfm:pzdr/a148}{2788}
\pdfglyphtounicode{tfm:pzdr/a149}{2789}
\pdfglyphtounicode{tfm:pzdr/a15}{270E}
\pdfglyphtounicode{tfm:pzdr/a150}{278A}
\pdfglyphtounicode{tfm:pzdr/a151}{278B}
\pdfglyphtounicode{tfm:pzdr/a152}{278C}
\pdfglyphtounicode{tfm:pzdr/a153}{278D}
\pdfglyphtounicode{tfm:pzdr/a154}{278E}
\pdfglyphtounicode{tfm:pzdr/a155}{278F}
\pdfglyphtounicode{tfm:pzdr/a156}{2790}
\pdfglyphtounicode{tfm:pzdr/a157}{2791}
\pdfglyphtounicode{tfm:pzdr/a158}{2792}
\pdfglyphtounicode{tfm:pzdr/a159}{2793}
\pdfglyphtounicode{tfm:pzdr/a16}{270F}
\pdfglyphtounicode{tfm:pzdr/a160}{2794}
\pdfglyphtounicode{tfm:pzdr/a161}{2192}
\pdfglyphtounicode{tfm:pzdr/a162}{27A3}
\pdfglyphtounicode{tfm:pzdr/a163}{2194}
\pdfglyphtounicode{tfm:pzdr/a164}{2195}
\pdfglyphtounicode{tfm:pzdr/a165}{2799}
\pdfglyphtounicode{tfm:pzdr/a166}{279B}
\pdfglyphtounicode{tfm:pzdr/a167}{279C}
\pdfglyphtounicode{tfm:pzdr/a168}{279D}
\pdfglyphtounicode{tfm:pzdr/a169}{279E}
\pdfglyphtounicode{tfm:pzdr/a17}{2711}
\pdfglyphtounicode{tfm:pzdr/a170}{279F}
\pdfglyphtounicode{tfm:pzdr/a171}{27A0}
\pdfglyphtounicode{tfm:pzdr/a172}{27A1}
\pdfglyphtounicode{tfm:pzdr/a173}{27A2}
\pdfglyphtounicode{tfm:pzdr/a174}{27A4}
\pdfglyphtounicode{tfm:pzdr/a175}{27A5}
\pdfglyphtounicode{tfm:pzdr/a176}{27A6}
\pdfglyphtounicode{tfm:pzdr/a177}{27A7}
\pdfglyphtounicode{tfm:pzdr/a178}{27A8}
\pdfglyphtounicode{tfm:pzdr/a179}{27A9}
\pdfglyphtounicode{tfm:pzdr/a18}{2712}
\pdfglyphtounicode{tfm:pzdr/a180}{27AB}
\pdfglyphtounicode{tfm:pzdr/a181}{27AD}
\pdfglyphtounicode{tfm:pzdr/a182}{27AF}
\pdfglyphtounicode{tfm:pzdr/a183}{27B2}
\pdfglyphtounicode{tfm:pzdr/a184}{27B3}
\pdfglyphtounicode{tfm:pzdr/a185}{27B5}
\pdfglyphtounicode{tfm:pzdr/a186}{27B8}
\pdfglyphtounicode{tfm:pzdr/a187}{27BA}
\pdfglyphtounicode{tfm:pzdr/a188}{27BB}
\pdfglyphtounicode{tfm:pzdr/a189}{27BC}
\pdfglyphtounicode{tfm:pzdr/a19}{2713}
\pdfglyphtounicode{tfm:pzdr/a190}{27BD}
\pdfglyphtounicode{tfm:pzdr/a191}{27BE}
\pdfglyphtounicode{tfm:pzdr/a192}{279A}
\pdfglyphtounicode{tfm:pzdr/a193}{27AA}
\pdfglyphtounicode{tfm:pzdr/a194}{27B6}
\pdfglyphtounicode{tfm:pzdr/a195}{27B9}
\pdfglyphtounicode{tfm:pzdr/a196}{2798}
\pdfglyphtounicode{tfm:pzdr/a197}{27B4}
\pdfglyphtounicode{tfm:pzdr/a198}{27B7}
\pdfglyphtounicode{tfm:pzdr/a199}{27AC}
\pdfglyphtounicode{tfm:pzdr/a2}{2702}
\pdfglyphtounicode{tfm:pzdr/a20}{2714}
\pdfglyphtounicode{tfm:pzdr/a200}{27AE}
\pdfglyphtounicode{tfm:pzdr/a201}{27B1}
\pdfglyphtounicode{tfm:pzdr/a202}{2703}
\pdfglyphtounicode{tfm:pzdr/a203}{2750}
\pdfglyphtounicode{tfm:pzdr/a204}{2752}
\pdfglyphtounicode{tfm:pzdr/a205}{276E}
\pdfglyphtounicode{tfm:pzdr/a206}{2770}
\pdfglyphtounicode{tfm:pzdr/a21}{2715}
\pdfglyphtounicode{tfm:pzdr/a22}{2716}
\pdfglyphtounicode{tfm:pzdr/a23}{2717}
\pdfglyphtounicode{tfm:pzdr/a24}{2718}
\pdfglyphtounicode{tfm:pzdr/a25}{2719}
\pdfglyphtounicode{tfm:pzdr/a26}{271A}
\pdfglyphtounicode{tfm:pzdr/a27}{271B}
\pdfglyphtounicode{tfm:pzdr/a28}{271C}
\pdfglyphtounicode{tfm:pzdr/a29}{2722}
\pdfglyphtounicode{tfm:pzdr/a3}{2704}
\pdfglyphtounicode{tfm:pzdr/a30}{2723}
\pdfglyphtounicode{tfm:pzdr/a31}{2724}
\pdfglyphtounicode{tfm:pzdr/a32}{2725}
\pdfglyphtounicode{tfm:pzdr/a33}{2726}
\pdfglyphtounicode{tfm:pzdr/a34}{2727}
\pdfglyphtounicode{tfm:pzdr/a35}{2605}
\pdfglyphtounicode{tfm:pzdr/a36}{2729}
\pdfglyphtounicode{tfm:pzdr/a37}{272A}
\pdfglyphtounicode{tfm:pzdr/a38}{272B}
\pdfglyphtounicode{tfm:pzdr/a39}{272C}
\pdfglyphtounicode{tfm:pzdr/a4}{260E}
\pdfglyphtounicode{tfm:pzdr/a40}{272D}
\pdfglyphtounicode{tfm:pzdr/a41}{272E}
\pdfglyphtounicode{tfm:pzdr/a42}{272F}
\pdfglyphtounicode{tfm:pzdr/a43}{2730}
\pdfglyphtounicode{tfm:pzdr/a44}{2731}
\pdfglyphtounicode{tfm:pzdr/a45}{2732}
\pdfglyphtounicode{tfm:pzdr/a46}{2733}
\pdfglyphtounicode{tfm:pzdr/a47}{2734}
\pdfglyphtounicode{tfm:pzdr/a48}{2735}
\pdfglyphtounicode{tfm:pzdr/a49}{2736}
\pdfglyphtounicode{tfm:pzdr/a5}{2706}
\pdfglyphtounicode{tfm:pzdr/a50}{2737}
\pdfglyphtounicode{tfm:pzdr/a51}{2738}
\pdfglyphtounicode{tfm:pzdr/a52}{2739}
\pdfglyphtounicode{tfm:pzdr/a53}{273A}
\pdfglyphtounicode{tfm:pzdr/a54}{273B}
\pdfglyphtounicode{tfm:pzdr/a55}{273C}
\pdfglyphtounicode{tfm:pzdr/a56}{273D}
\pdfglyphtounicode{tfm:pzdr/a57}{273E}
\pdfglyphtounicode{tfm:pzdr/a58}{273F}
\pdfglyphtounicode{tfm:pzdr/a59}{2740}
\pdfglyphtounicode{tfm:pzdr/a6}{271D}
\pdfglyphtounicode{tfm:pzdr/a60}{2741}
\pdfglyphtounicode{tfm:pzdr/a61}{2742}
\pdfglyphtounicode{tfm:pzdr/a62}{2743}
\pdfglyphtounicode{tfm:pzdr/a63}{2744}
\pdfglyphtounicode{tfm:pzdr/a64}{2745}
\pdfglyphtounicode{tfm:pzdr/a65}{2746}
\pdfglyphtounicode{tfm:pzdr/a66}{2747}
\pdfglyphtounicode{tfm:pzdr/a67}{2748}
\pdfglyphtounicode{tfm:pzdr/a68}{2749}
\pdfglyphtounicode{tfm:pzdr/a69}{274A}
\pdfglyphtounicode{tfm:pzdr/a7}{271E}
\pdfglyphtounicode{tfm:pzdr/a70}{274B}
\pdfglyphtounicode{tfm:pzdr/a71}{25CF}
\pdfglyphtounicode{tfm:pzdr/a72}{274D}
\pdfglyphtounicode{tfm:pzdr/a73}{25A0}
\pdfglyphtounicode{tfm:pzdr/a74}{274F}
\pdfglyphtounicode{tfm:pzdr/a75}{2751}
\pdfglyphtounicode{tfm:pzdr/a76}{25B2}
\pdfglyphtounicode{tfm:pzdr/a77}{25BC}
\pdfglyphtounicode{tfm:pzdr/a78}{25C6}
\pdfglyphtounicode{tfm:pzdr/a79}{2756}
\pdfglyphtounicode{tfm:pzdr/a8}{271F}
\pdfglyphtounicode{tfm:pzdr/a81}{25D7}
\pdfglyphtounicode{tfm:pzdr/a82}{2758}
\pdfglyphtounicode{tfm:pzdr/a83}{2759}
\pdfglyphtounicode{tfm:pzdr/a84}{275A}
\pdfglyphtounicode{tfm:pzdr/a85}{276F}
\pdfglyphtounicode{tfm:pzdr/a86}{2771}
\pdfglyphtounicode{tfm:pzdr/a87}{2772}
\pdfglyphtounicode{tfm:pzdr/a88}{2773}
\pdfglyphtounicode{tfm:pzdr/a89}{2768}
\pdfglyphtounicode{tfm:pzdr/a9}{2720}
\pdfglyphtounicode{tfm:pzdr/a90}{2769}
\pdfglyphtounicode{tfm:pzdr/a91}{276C}
\pdfglyphtounicode{tfm:pzdr/a92}{276D}
\pdfglyphtounicode{tfm:pzdr/a93}{276A}
\pdfglyphtounicode{tfm:pzdr/a94}{276B}
\pdfglyphtounicode{tfm:pzdr/a95}{2774}
\pdfglyphtounicode{tfm:pzdr/a96}{2775}
\pdfglyphtounicode{tfm:pzdr/a97}{275B}
\pdfglyphtounicode{tfm:pzdr/a98}{275C}
\pdfglyphtounicode{tfm:pzdr/a99}{275D}
\pdfglyphtounicode{tfm:rpxbmi/phi}{03D5}
\pdfglyphtounicode{tfm:rpxbmi/phi1}{03C6}
\pdfglyphtounicode{tfm:rpxmi/phi}{03D5}
\pdfglyphtounicode{tfm:rpxmi/phi1}{03C6}
\pdfglyphtounicode{tfm:rpzdr/a1}{2701}
\pdfglyphtounicode{tfm:rpzdr/a10}{2721}
\pdfglyphtounicode{tfm:rpzdr/a100}{275E}
\pdfglyphtounicode{tfm:rpzdr/a101}{2761}
\pdfglyphtounicode{tfm:rpzdr/a102}{2762}
\pdfglyphtounicode{tfm:rpzdr/a103}{2763}
\pdfglyphtounicode{tfm:rpzdr/a104}{2764}
\pdfglyphtounicode{tfm:rpzdr/a105}{2710}
\pdfglyphtounicode{tfm:rpzdr/a106}{2765}
\pdfglyphtounicode{tfm:rpzdr/a107}{2766}
\pdfglyphtounicode{tfm:rpzdr/a108}{2767}
\pdfglyphtounicode{tfm:rpzdr/a109}{2660}
\pdfglyphtounicode{tfm:rpzdr/a11}{261B}
\pdfglyphtounicode{tfm:rpzdr/a110}{2665}
\pdfglyphtounicode{tfm:rpzdr/a111}{2666}
\pdfglyphtounicode{tfm:rpzdr/a112}{2663}
\pdfglyphtounicode{tfm:rpzdr/a117}{2709}
\pdfglyphtounicode{tfm:rpzdr/a118}{2708}
\pdfglyphtounicode{tfm:rpzdr/a119}{2707}
\pdfglyphtounicode{tfm:rpzdr/a12}{261E}
\pdfglyphtounicode{tfm:rpzdr/a120}{2460}
\pdfglyphtounicode{tfm:rpzdr/a121}{2461}
\pdfglyphtounicode{tfm:rpzdr/a122}{2462}
\pdfglyphtounicode{tfm:rpzdr/a123}{2463}
\pdfglyphtounicode{tfm:rpzdr/a124}{2464}
\pdfglyphtounicode{tfm:rpzdr/a125}{2465}
\pdfglyphtounicode{tfm:rpzdr/a126}{2466}
\pdfglyphtounicode{tfm:rpzdr/a127}{2467}
\pdfglyphtounicode{tfm:rpzdr/a128}{2468}
\pdfglyphtounicode{tfm:rpzdr/a129}{2469}
\pdfglyphtounicode{tfm:rpzdr/a13}{270C}
\pdfglyphtounicode{tfm:rpzdr/a130}{2776}
\pdfglyphtounicode{tfm:rpzdr/a131}{2777}
\pdfglyphtounicode{tfm:rpzdr/a132}{2778}
\pdfglyphtounicode{tfm:rpzdr/a133}{2779}
\pdfglyphtounicode{tfm:rpzdr/a134}{277A}
\pdfglyphtounicode{tfm:rpzdr/a135}{277B}
\pdfglyphtounicode{tfm:rpzdr/a136}{277C}
\pdfglyphtounicode{tfm:rpzdr/a137}{277D}
\pdfglyphtounicode{tfm:rpzdr/a138}{277E}
\pdfglyphtounicode{tfm:rpzdr/a139}{277F}
\pdfglyphtounicode{tfm:rpzdr/a14}{270D}
\pdfglyphtounicode{tfm:rpzdr/a140}{2780}
\pdfglyphtounicode{tfm:rpzdr/a141}{2781}
\pdfglyphtounicode{tfm:rpzdr/a142}{2782}
\pdfglyphtounicode{tfm:rpzdr/a143}{2783}
\pdfglyphtounicode{tfm:rpzdr/a144}{2784}
\pdfglyphtounicode{tfm:rpzdr/a145}{2785}
\pdfglyphtounicode{tfm:rpzdr/a146}{2786}
\pdfglyphtounicode{tfm:rpzdr/a147}{2787}
\pdfglyphtounicode{tfm:rpzdr/a148}{2788}
\pdfglyphtounicode{tfm:rpzdr/a149}{2789}
\pdfglyphtounicode{tfm:rpzdr/a15}{270E}
\pdfglyphtounicode{tfm:rpzdr/a150}{278A}
\pdfglyphtounicode{tfm:rpzdr/a151}{278B}
\pdfglyphtounicode{tfm:rpzdr/a152}{278C}
\pdfglyphtounicode{tfm:rpzdr/a153}{278D}
\pdfglyphtounicode{tfm:rpzdr/a154}{278E}
\pdfglyphtounicode{tfm:rpzdr/a155}{278F}
\pdfglyphtounicode{tfm:rpzdr/a156}{2790}
\pdfglyphtounicode{tfm:rpzdr/a157}{2791}
\pdfglyphtounicode{tfm:rpzdr/a158}{2792}
\pdfglyphtounicode{tfm:rpzdr/a159}{2793}
\pdfglyphtounicode{tfm:rpzdr/a16}{270F}
\pdfglyphtounicode{tfm:rpzdr/a160}{2794}
\pdfglyphtounicode{tfm:rpzdr/a161}{2192}
\pdfglyphtounicode{tfm:rpzdr/a162}{27A3}
\pdfglyphtounicode{tfm:rpzdr/a163}{2194}
\pdfglyphtounicode{tfm:rpzdr/a164}{2195}
\pdfglyphtounicode{tfm:rpzdr/a165}{2799}
\pdfglyphtounicode{tfm:rpzdr/a166}{279B}
\pdfglyphtounicode{tfm:rpzdr/a167}{279C}
\pdfglyphtounicode{tfm:rpzdr/a168}{279D}
\pdfglyphtounicode{tfm:rpzdr/a169}{279E}
\pdfglyphtounicode{tfm:rpzdr/a17}{2711}
\pdfglyphtounicode{tfm:rpzdr/a170}{279F}
\pdfglyphtounicode{tfm:rpzdr/a171}{27A0}
\pdfglyphtounicode{tfm:rpzdr/a172}{27A1}
\pdfglyphtounicode{tfm:rpzdr/a173}{27A2}
\pdfglyphtounicode{tfm:rpzdr/a174}{27A4}
\pdfglyphtounicode{tfm:rpzdr/a175}{27A5}
\pdfglyphtounicode{tfm:rpzdr/a176}{27A6}
\pdfglyphtounicode{tfm:rpzdr/a177}{27A7}
\pdfglyphtounicode{tfm:rpzdr/a178}{27A8}
\pdfglyphtounicode{tfm:rpzdr/a179}{27A9}
\pdfglyphtounicode{tfm:rpzdr/a18}{2712}
\pdfglyphtounicode{tfm:rpzdr/a180}{27AB}
\pdfglyphtounicode{tfm:rpzdr/a181}{27AD}
\pdfglyphtounicode{tfm:rpzdr/a182}{27AF}
\pdfglyphtounicode{tfm:rpzdr/a183}{27B2}
\pdfglyphtounicode{tfm:rpzdr/a184}{27B3}
\pdfglyphtounicode{tfm:rpzdr/a185}{27B5}
\pdfglyphtounicode{tfm:rpzdr/a186}{27B8}
\pdfglyphtounicode{tfm:rpzdr/a187}{27BA}
\pdfglyphtounicode{tfm:rpzdr/a188}{27BB}
\pdfglyphtounicode{tfm:rpzdr/a189}{27BC}
\pdfglyphtounicode{tfm:rpzdr/a19}{2713}
\pdfglyphtounicode{tfm:rpzdr/a190}{27BD}
\pdfglyphtounicode{tfm:rpzdr/a191}{27BE}
\pdfglyphtounicode{tfm:rpzdr/a192}{279A}
\pdfglyphtounicode{tfm:rpzdr/a193}{27AA}
\pdfglyphtounicode{tfm:rpzdr/a194}{27B6}
\pdfglyphtounicode{tfm:rpzdr/a195}{27B9}
\pdfglyphtounicode{tfm:rpzdr/a196}{2798}
\pdfglyphtounicode{tfm:rpzdr/a197}{27B4}
\pdfglyphtounicode{tfm:rpzdr/a198}{27B7}
\pdfglyphtounicode{tfm:rpzdr/a199}{27AC}
\pdfglyphtounicode{tfm:rpzdr/a2}{2702}
\pdfglyphtounicode{tfm:rpzdr/a20}{2714}
\pdfglyphtounicode{tfm:rpzdr/a200}{27AE}
\pdfglyphtounicode{tfm:rpzdr/a201}{27B1}
\pdfglyphtounicode{tfm:rpzdr/a202}{2703}
\pdfglyphtounicode{tfm:rpzdr/a203}{2750}
\pdfglyphtounicode{tfm:rpzdr/a204}{2752}
\pdfglyphtounicode{tfm:rpzdr/a205}{276E}
\pdfglyphtounicode{tfm:rpzdr/a206}{2770}
\pdfglyphtounicode{tfm:rpzdr/a21}{2715}
\pdfglyphtounicode{tfm:rpzdr/a22}{2716}
\pdfglyphtounicode{tfm:rpzdr/a23}{2717}
\pdfglyphtounicode{tfm:rpzdr/a24}{2718}
\pdfglyphtounicode{tfm:rpzdr/a25}{2719}
\pdfglyphtounicode{tfm:rpzdr/a26}{271A}
\pdfglyphtounicode{tfm:rpzdr/a27}{271B}
\pdfglyphtounicode{tfm:rpzdr/a28}{271C}
\pdfglyphtounicode{tfm:rpzdr/a29}{2722}
\pdfglyphtounicode{tfm:rpzdr/a3}{2704}
\pdfglyphtounicode{tfm:rpzdr/a30}{2723}
\pdfglyphtounicode{tfm:rpzdr/a31}{2724}
\pdfglyphtounicode{tfm:rpzdr/a32}{2725}
\pdfglyphtounicode{tfm:rpzdr/a33}{2726}
\pdfglyphtounicode{tfm:rpzdr/a34}{2727}
\pdfglyphtounicode{tfm:rpzdr/a35}{2605}
\pdfglyphtounicode{tfm:rpzdr/a36}{2729}
\pdfglyphtounicode{tfm:rpzdr/a37}{272A}
\pdfglyphtounicode{tfm:rpzdr/a38}{272B}
\pdfglyphtounicode{tfm:rpzdr/a39}{272C}
\pdfglyphtounicode{tfm:rpzdr/a4}{260E}
\pdfglyphtounicode{tfm:rpzdr/a40}{272D}
\pdfglyphtounicode{tfm:rpzdr/a41}{272E}
\pdfglyphtounicode{tfm:rpzdr/a42}{272F}
\pdfglyphtounicode{tfm:rpzdr/a43}{2730}
\pdfglyphtounicode{tfm:rpzdr/a44}{2731}
\pdfglyphtounicode{tfm:rpzdr/a45}{2732}
\pdfglyphtounicode{tfm:rpzdr/a46}{2733}
\pdfglyphtounicode{tfm:rpzdr/a47}{2734}
\pdfglyphtounicode{tfm:rpzdr/a48}{2735}
\pdfglyphtounicode{tfm:rpzdr/a49}{2736}
\pdfglyphtounicode{tfm:rpzdr/a5}{2706}
\pdfglyphtounicode{tfm:rpzdr/a50}{2737}
\pdfglyphtounicode{tfm:rpzdr/a51}{2738}
\pdfglyphtounicode{tfm:rpzdr/a52}{2739}
\pdfglyphtounicode{tfm:rpzdr/a53}{273A}
\pdfglyphtounicode{tfm:rpzdr/a54}{273B}
\pdfglyphtounicode{tfm:rpzdr/a55}{273C}
\pdfglyphtounicode{tfm:rpzdr/a56}{273D}
\pdfglyphtounicode{tfm:rpzdr/a57}{273E}
\pdfglyphtounicode{tfm:rpzdr/a58}{273F}
\pdfglyphtounicode{tfm:rpzdr/a59}{2740}
\pdfglyphtounicode{tfm:rpzdr/a6}{271D}
\pdfglyphtounicode{tfm:rpzdr/a60}{2741}
\pdfglyphtounicode{tfm:rpzdr/a61}{2742}
\pdfglyphtounicode{tfm:rpzdr/a62}{2743}
\pdfglyphtounicode{tfm:rpzdr/a63}{2744}
\pdfglyphtounicode{tfm:rpzdr/a64}{2745}
\pdfglyphtounicode{tfm:rpzdr/a65}{2746}
\pdfglyphtounicode{tfm:rpzdr/a66}{2747}
\pdfglyphtounicode{tfm:rpzdr/a67}{2748}
\pdfglyphtounicode{tfm:rpzdr/a68}{2749}
\pdfglyphtounicode{tfm:rpzdr/a69}{274A}
\pdfglyphtounicode{tfm:rpzdr/a7}{271E}
\pdfglyphtounicode{tfm:rpzdr/a70}{274B}
\pdfglyphtounicode{tfm:rpzdr/a71}{25CF}
\pdfglyphtounicode{tfm:rpzdr/a72}{274D}
\pdfglyphtounicode{tfm:rpzdr/a73}{25A0}
\pdfglyphtounicode{tfm:rpzdr/a74}{274F}
\pdfglyphtounicode{tfm:rpzdr/a75}{2751}
\pdfglyphtounicode{tfm:rpzdr/a76}{25B2}
\pdfglyphtounicode{tfm:rpzdr/a77}{25BC}
\pdfglyphtounicode{tfm:rpzdr/a78}{25C6}
\pdfglyphtounicode{tfm:rpzdr/a79}{2756}
\pdfglyphtounicode{tfm:rpzdr/a8}{271F}
\pdfglyphtounicode{tfm:rpzdr/a81}{25D7}
\pdfglyphtounicode{tfm:rpzdr/a82}{2758}
\pdfglyphtounicode{tfm:rpzdr/a83}{2759}
\pdfglyphtounicode{tfm:rpzdr/a84}{275A}
\pdfglyphtounicode{tfm:rpzdr/a85}{276F}
\pdfglyphtounicode{tfm:rpzdr/a86}{2771}
\pdfglyphtounicode{tfm:rpzdr/a87}{2772}
\pdfglyphtounicode{tfm:rpzdr/a88}{2773}
\pdfglyphtounicode{tfm:rpzdr/a89}{2768}
\pdfglyphtounicode{tfm:rpzdr/a9}{2720}
\pdfglyphtounicode{tfm:rpzdr/a90}{2769}
\pdfglyphtounicode{tfm:rpzdr/a91}{276C}
\pdfglyphtounicode{tfm:rpzdr/a92}{276D}
\pdfglyphtounicode{tfm:rpzdr/a93}{276A}
\pdfglyphtounicode{tfm:rpzdr/a94}{276B}
\pdfglyphtounicode{tfm:rpzdr/a95}{2774}
\pdfglyphtounicode{tfm:rpzdr/a96}{2775}
\pdfglyphtounicode{tfm:rpzdr/a97}{275B}
\pdfglyphtounicode{tfm:rpzdr/a98}{275C}
\pdfglyphtounicode{tfm:rpzdr/a99}{275D}
\pdfglyphtounicode{tfm:rtxbmi/phi}{03D5}
\pdfglyphtounicode{tfm:rtxbmi/phi1}{03C6}
\pdfglyphtounicode{tfm:rtxmi/phi}{03D5}
\pdfglyphtounicode{tfm:rtxmi/phi1}{03C6}
\pdfglyphtounicode{tfm:txbmia/phi}{03D5}
\pdfglyphtounicode{tfm:txbmia/phi1}{03C6}
\pdfglyphtounicode{tfm:txbsy/diamond}{2662}
\pdfglyphtounicode{tfm:txbsy/heart}{2661}
\pdfglyphtounicode{tfm:txbsya/diamond}{2662}
\pdfglyphtounicode{tfm:txmia/phi}{03D5}
\pdfglyphtounicode{tfm:txmia/phi1}{03C6}
\pdfglyphtounicode{tfm:txsy/diamond}{2662}
\pdfglyphtounicode{tfm:txsy/heart}{2661}
\pdfglyphtounicode{tfm:txsya/diamond}{2662}
\pdfglyphtounicode{tfm:zd/a1}{2701}
\pdfglyphtounicode{tfm:zd/a10}{2721}
\pdfglyphtounicode{tfm:zd/a100}{275E}
\pdfglyphtounicode{tfm:zd/a101}{2761}
\pdfglyphtounicode{tfm:zd/a102}{2762}
\pdfglyphtounicode{tfm:zd/a103}{2763}
\pdfglyphtounicode{tfm:zd/a104}{2764}
\pdfglyphtounicode{tfm:zd/a105}{2710}
\pdfglyphtounicode{tfm:zd/a106}{2765}
\pdfglyphtounicode{tfm:zd/a107}{2766}
\pdfglyphtounicode{tfm:zd/a108}{2767}
\pdfglyphtounicode{tfm:zd/a109}{2660}
\pdfglyphtounicode{tfm:zd/a11}{261B}
\pdfglyphtounicode{tfm:zd/a110}{2665}
\pdfglyphtounicode{tfm:zd/a111}{2666}
\pdfglyphtounicode{tfm:zd/a112}{2663}
\pdfglyphtounicode{tfm:zd/a117}{2709}
\pdfglyphtounicode{tfm:zd/a118}{2708}
\pdfglyphtounicode{tfm:zd/a119}{2707}
\pdfglyphtounicode{tfm:zd/a12}{261E}
\pdfglyphtounicode{tfm:zd/a120}{2460}
\pdfglyphtounicode{tfm:zd/a121}{2461}
\pdfglyphtounicode{tfm:zd/a122}{2462}
\pdfglyphtounicode{tfm:zd/a123}{2463}
\pdfglyphtounicode{tfm:zd/a124}{2464}
\pdfglyphtounicode{tfm:zd/a125}{2465}
\pdfglyphtounicode{tfm:zd/a126}{2466}
\pdfglyphtounicode{tfm:zd/a127}{2467}
\pdfglyphtounicode{tfm:zd/a128}{2468}
\pdfglyphtounicode{tfm:zd/a129}{2469}
\pdfglyphtounicode{tfm:zd/a13}{270C}
\pdfglyphtounicode{tfm:zd/a130}{2776}
\pdfglyphtounicode{tfm:zd/a131}{2777}
\pdfglyphtounicode{tfm:zd/a132}{2778}
\pdfglyphtounicode{tfm:zd/a133}{2779}
\pdfglyphtounicode{tfm:zd/a134}{277A}
\pdfglyphtounicode{tfm:zd/a135}{277B}
\pdfglyphtounicode{tfm:zd/a136}{277C}
\pdfglyphtounicode{tfm:zd/a137}{277D}
\pdfglyphtounicode{tfm:zd/a138}{277E}
\pdfglyphtounicode{tfm:zd/a139}{277F}
\pdfglyphtounicode{tfm:zd/a14}{270D}
\pdfglyphtounicode{tfm:zd/a140}{2780}
\pdfglyphtounicode{tfm:zd/a141}{2781}
\pdfglyphtounicode{tfm:zd/a142}{2782}
\pdfglyphtounicode{tfm:zd/a143}{2783}
\pdfglyphtounicode{tfm:zd/a144}{2784}
\pdfglyphtounicode{tfm:zd/a145}{2785}
\pdfglyphtounicode{tfm:zd/a146}{2786}
\pdfglyphtounicode{tfm:zd/a147}{2787}
\pdfglyphtounicode{tfm:zd/a148}{2788}
\pdfglyphtounicode{tfm:zd/a149}{2789}
\pdfglyphtounicode{tfm:zd/a15}{270E}
\pdfglyphtounicode{tfm:zd/a150}{278A}
\pdfglyphtounicode{tfm:zd/a151}{278B}
\pdfglyphtounicode{tfm:zd/a152}{278C}
\pdfglyphtounicode{tfm:zd/a153}{278D}
\pdfglyphtounicode{tfm:zd/a154}{278E}
\pdfglyphtounicode{tfm:zd/a155}{278F}
\pdfglyphtounicode{tfm:zd/a156}{2790}
\pdfglyphtounicode{tfm:zd/a157}{2791}
\pdfglyphtounicode{tfm:zd/a158}{2792}
\pdfglyphtounicode{tfm:zd/a159}{2793}
\pdfglyphtounicode{tfm:zd/a16}{270F}
\pdfglyphtounicode{tfm:zd/a160}{2794}
\pdfglyphtounicode{tfm:zd/a161}{2192}
\pdfglyphtounicode{tfm:zd/a162}{27A3}
\pdfglyphtounicode{tfm:zd/a163}{2194}
\pdfglyphtounicode{tfm:zd/a164}{2195}
\pdfglyphtounicode{tfm:zd/a165}{2799}
\pdfglyphtounicode{tfm:zd/a166}{279B}
\pdfglyphtounicode{tfm:zd/a167}{279C}
\pdfglyphtounicode{tfm:zd/a168}{279D}
\pdfglyphtounicode{tfm:zd/a169}{279E}
\pdfglyphtounicode{tfm:zd/a17}{2711}
\pdfglyphtounicode{tfm:zd/a170}{279F}
\pdfglyphtounicode{tfm:zd/a171}{27A0}
\pdfglyphtounicode{tfm:zd/a172}{27A1}
\pdfglyphtounicode{tfm:zd/a173}{27A2}
\pdfglyphtounicode{tfm:zd/a174}{27A4}
\pdfglyphtounicode{tfm:zd/a175}{27A5}
\pdfglyphtounicode{tfm:zd/a176}{27A6}
\pdfglyphtounicode{tfm:zd/a177}{27A7}
\pdfglyphtounicode{tfm:zd/a178}{27A8}
\pdfglyphtounicode{tfm:zd/a179}{27A9}
\pdfglyphtounicode{tfm:zd/a18}{2712}
\pdfglyphtounicode{tfm:zd/a180}{27AB}
\pdfglyphtounicode{tfm:zd/a181}{27AD}
\pdfglyphtounicode{tfm:zd/a182}{27AF}
\pdfglyphtounicode{tfm:zd/a183}{27B2}
\pdfglyphtounicode{tfm:zd/a184}{27B3}
\pdfglyphtounicode{tfm:zd/a185}{27B5}
\pdfglyphtounicode{tfm:zd/a186}{27B8}
\pdfglyphtounicode{tfm:zd/a187}{27BA}
\pdfglyphtounicode{tfm:zd/a188}{27BB}
\pdfglyphtounicode{tfm:zd/a189}{27BC}
\pdfglyphtounicode{tfm:zd/a19}{2713}
\pdfglyphtounicode{tfm:zd/a190}{27BD}
\pdfglyphtounicode{tfm:zd/a191}{27BE}
\pdfglyphtounicode{tfm:zd/a192}{279A}
\pdfglyphtounicode{tfm:zd/a193}{27AA}
\pdfglyphtounicode{tfm:zd/a194}{27B6}
\pdfglyphtounicode{tfm:zd/a195}{27B9}
\pdfglyphtounicode{tfm:zd/a196}{2798}
\pdfglyphtounicode{tfm:zd/a197}{27B4}
\pdfglyphtounicode{tfm:zd/a198}{27B7}
\pdfglyphtounicode{tfm:zd/a199}{27AC}
\pdfglyphtounicode{tfm:zd/a2}{2702}
\pdfglyphtounicode{tfm:zd/a20}{2714}
\pdfglyphtounicode{tfm:zd/a200}{27AE}
\pdfglyphtounicode{tfm:zd/a201}{27B1}
\pdfglyphtounicode{tfm:zd/a202}{2703}
\pdfglyphtounicode{tfm:zd/a203}{2750}
\pdfglyphtounicode{tfm:zd/a204}{2752}
\pdfglyphtounicode{tfm:zd/a205}{276E}
\pdfglyphtounicode{tfm:zd/a206}{2770}
\pdfglyphtounicode{tfm:zd/a21}{2715}
\pdfglyphtounicode{tfm:zd/a22}{2716}
\pdfglyphtounicode{tfm:zd/a23}{2717}
\pdfglyphtounicode{tfm:zd/a24}{2718}
\pdfglyphtounicode{tfm:zd/a25}{2719}
\pdfglyphtounicode{tfm:zd/a26}{271A}
\pdfglyphtounicode{tfm:zd/a27}{271B}
\pdfglyphtounicode{tfm:zd/a28}{271C}
\pdfglyphtounicode{tfm:zd/a29}{2722}
\pdfglyphtounicode{tfm:zd/a3}{2704}
\pdfglyphtounicode{tfm:zd/a30}{2723}
\pdfglyphtounicode{tfm:zd/a31}{2724}
\pdfglyphtounicode{tfm:zd/a32}{2725}
\pdfglyphtounicode{tfm:zd/a33}{2726}
\pdfglyphtounicode{tfm:zd/a34}{2727}
\pdfglyphtounicode{tfm:zd/a35}{2605}
\pdfglyphtounicode{tfm:zd/a36}{2729}
\pdfglyphtounicode{tfm:zd/a37}{272A}
\pdfglyphtounicode{tfm:zd/a38}{272B}
\pdfglyphtounicode{tfm:zd/a39}{272C}
\pdfglyphtounicode{tfm:zd/a4}{260E}
\pdfglyphtounicode{tfm:zd/a40}{272D}
\pdfglyphtounicode{tfm:zd/a41}{272E}
\pdfglyphtounicode{tfm:zd/a42}{272F}
\pdfglyphtounicode{tfm:zd/a43}{2730}
\pdfglyphtounicode{tfm:zd/a44}{2731}
\pdfglyphtounicode{tfm:zd/a45}{2732}
\pdfglyphtounicode{tfm:zd/a46}{2733}
\pdfglyphtounicode{tfm:zd/a47}{2734}
\pdfglyphtounicode{tfm:zd/a48}{2735}
\pdfglyphtounicode{tfm:zd/a49}{2736}
\pdfglyphtounicode{tfm:zd/a5}{2706}
\pdfglyphtounicode{tfm:zd/a50}{2737}
\pdfglyphtounicode{tfm:zd/a51}{2738}
\pdfglyphtounicode{tfm:zd/a52}{2739}
\pdfglyphtounicode{tfm:zd/a53}{273A}
\pdfglyphtounicode{tfm:zd/a54}{273B}
\pdfglyphtounicode{tfm:zd/a55}{273C}
\pdfglyphtounicode{tfm:zd/a56}{273D}
\pdfglyphtounicode{tfm:zd/a57}{273E}
\pdfglyphtounicode{tfm:zd/a58}{273F}
\pdfglyphtounicode{tfm:zd/a59}{2740}
\pdfglyphtounicode{tfm:zd/a6}{271D}
\pdfglyphtounicode{tfm:zd/a60}{2741}
\pdfglyphtounicode{tfm:zd/a61}{2742}
\pdfglyphtounicode{tfm:zd/a62}{2743}
\pdfglyphtounicode{tfm:zd/a63}{2744}
\pdfglyphtounicode{tfm:zd/a64}{2745}
\pdfglyphtounicode{tfm:zd/a65}{2746}
\pdfglyphtounicode{tfm:zd/a66}{2747}
\pdfglyphtounicode{tfm:zd/a67}{2748}
\pdfglyphtounicode{tfm:zd/a68}{2749}
\pdfglyphtounicode{tfm:zd/a69}{274A}
\pdfglyphtounicode{tfm:zd/a7}{271E}
\pdfglyphtounicode{tfm:zd/a70}{274B}
\pdfglyphtounicode{tfm:zd/a71}{25CF}
\pdfglyphtounicode{tfm:zd/a72}{274D}
\pdfglyphtounicode{tfm:zd/a73}{25A0}
\pdfglyphtounicode{tfm:zd/a74}{274F}
\pdfglyphtounicode{tfm:zd/a75}{2751}
\pdfglyphtounicode{tfm:zd/a76}{25B2}
\pdfglyphtounicode{tfm:zd/a77}{25BC}
\pdfglyphtounicode{tfm:zd/a78}{25C6}
\pdfglyphtounicode{tfm:zd/a79}{2756}
\pdfglyphtounicode{tfm:zd/a8}{271F}
\pdfglyphtounicode{tfm:zd/a81}{25D7}
\pdfglyphtounicode{tfm:zd/a82}{2758}
\pdfglyphtounicode{tfm:zd/a83}{2759}
\pdfglyphtounicode{tfm:zd/a84}{275A}
\pdfglyphtounicode{tfm:zd/a85}{276F}
\pdfglyphtounicode{tfm:zd/a86}{2771}
\pdfglyphtounicode{tfm:zd/a87}{2772}
\pdfglyphtounicode{tfm:zd/a88}{2773}
\pdfglyphtounicode{tfm:zd/a89}{2768}
\pdfglyphtounicode{tfm:zd/a9}{2720}
\pdfglyphtounicode{tfm:zd/a90}{2769}
\pdfglyphtounicode{tfm:zd/a91}{276C}
\pdfglyphtounicode{tfm:zd/a92}{276D}
\pdfglyphtounicode{tfm:zd/a93}{276A}
\pdfglyphtounicode{tfm:zd/a94}{276B}
\pdfglyphtounicode{tfm:zd/a95}{2774}
\pdfglyphtounicode{tfm:zd/a96}{2775}
\pdfglyphtounicode{tfm:zd/a97}{275B}
\pdfglyphtounicode{tfm:zd/a98}{275C}
\pdfglyphtounicode{tfm:zd/a99}{275D}
\pdfglyphtounicode{tfm:zpzdr-reversed/a1}{2701}
\pdfglyphtounicode{tfm:zpzdr-reversed/a10}{2721}
\pdfglyphtounicode{tfm:zpzdr-reversed/a100}{275E}
\pdfglyphtounicode{tfm:zpzdr-reversed/a101}{2761}
\pdfglyphtounicode{tfm:zpzdr-reversed/a102}{2762}
\pdfglyphtounicode{tfm:zpzdr-reversed/a103}{2763}
\pdfglyphtounicode{tfm:zpzdr-reversed/a104}{2764}
\pdfglyphtounicode{tfm:zpzdr-reversed/a105}{2710}
\pdfglyphtounicode{tfm:zpzdr-reversed/a106}{2765}
\pdfglyphtounicode{tfm:zpzdr-reversed/a107}{2766}
\pdfglyphtounicode{tfm:zpzdr-reversed/a108}{2767}
\pdfglyphtounicode{tfm:zpzdr-reversed/a109}{2660}
\pdfglyphtounicode{tfm:zpzdr-reversed/a11}{261B}
\pdfglyphtounicode{tfm:zpzdr-reversed/a110}{2665}
\pdfglyphtounicode{tfm:zpzdr-reversed/a111}{2666}
\pdfglyphtounicode{tfm:zpzdr-reversed/a112}{2663}
\pdfglyphtounicode{tfm:zpzdr-reversed/a117}{2709}
\pdfglyphtounicode{tfm:zpzdr-reversed/a118}{2708}
\pdfglyphtounicode{tfm:zpzdr-reversed/a119}{2707}
\pdfglyphtounicode{tfm:zpzdr-reversed/a12}{261E}
\pdfglyphtounicode{tfm:zpzdr-reversed/a120}{2460}
\pdfglyphtounicode{tfm:zpzdr-reversed/a121}{2461}
\pdfglyphtounicode{tfm:zpzdr-reversed/a122}{2462}
\pdfglyphtounicode{tfm:zpzdr-reversed/a123}{2463}
\pdfglyphtounicode{tfm:zpzdr-reversed/a124}{2464}
\pdfglyphtounicode{tfm:zpzdr-reversed/a125}{2465}
\pdfglyphtounicode{tfm:zpzdr-reversed/a126}{2466}
\pdfglyphtounicode{tfm:zpzdr-reversed/a127}{2467}
\pdfglyphtounicode{tfm:zpzdr-reversed/a128}{2468}
\pdfglyphtounicode{tfm:zpzdr-reversed/a129}{2469}
\pdfglyphtounicode{tfm:zpzdr-reversed/a13}{270C}
\pdfglyphtounicode{tfm:zpzdr-reversed/a130}{2776}
\pdfglyphtounicode{tfm:zpzdr-reversed/a131}{2777}
\pdfglyphtounicode{tfm:zpzdr-reversed/a132}{2778}
\pdfglyphtounicode{tfm:zpzdr-reversed/a133}{2779}
\pdfglyphtounicode{tfm:zpzdr-reversed/a134}{277A}
\pdfglyphtounicode{tfm:zpzdr-reversed/a135}{277B}
\pdfglyphtounicode{tfm:zpzdr-reversed/a136}{277C}
\pdfglyphtounicode{tfm:zpzdr-reversed/a137}{277D}
\pdfglyphtounicode{tfm:zpzdr-reversed/a138}{277E}
\pdfglyphtounicode{tfm:zpzdr-reversed/a139}{277F}
\pdfglyphtounicode{tfm:zpzdr-reversed/a14}{270D}
\pdfglyphtounicode{tfm:zpzdr-reversed/a140}{2780}
\pdfglyphtounicode{tfm:zpzdr-reversed/a141}{2781}
\pdfglyphtounicode{tfm:zpzdr-reversed/a142}{2782}
\pdfglyphtounicode{tfm:zpzdr-reversed/a143}{2783}
\pdfglyphtounicode{tfm:zpzdr-reversed/a144}{2784}
\pdfglyphtounicode{tfm:zpzdr-reversed/a145}{2785}
\pdfglyphtounicode{tfm:zpzdr-reversed/a146}{2786}
\pdfglyphtounicode{tfm:zpzdr-reversed/a147}{2787}
\pdfglyphtounicode{tfm:zpzdr-reversed/a148}{2788}
\pdfglyphtounicode{tfm:zpzdr-reversed/a149}{2789}
\pdfglyphtounicode{tfm:zpzdr-reversed/a15}{270E}
\pdfglyphtounicode{tfm:zpzdr-reversed/a150}{278A}
\pdfglyphtounicode{tfm:zpzdr-reversed/a151}{278B}
\pdfglyphtounicode{tfm:zpzdr-reversed/a152}{278C}
\pdfglyphtounicode{tfm:zpzdr-reversed/a153}{278D}
\pdfglyphtounicode{tfm:zpzdr-reversed/a154}{278E}
\pdfglyphtounicode{tfm:zpzdr-reversed/a155}{278F}
\pdfglyphtounicode{tfm:zpzdr-reversed/a156}{2790}
\pdfglyphtounicode{tfm:zpzdr-reversed/a157}{2791}
\pdfglyphtounicode{tfm:zpzdr-reversed/a158}{2792}
\pdfglyphtounicode{tfm:zpzdr-reversed/a159}{2793}
\pdfglyphtounicode{tfm:zpzdr-reversed/a16}{270F}
\pdfglyphtounicode{tfm:zpzdr-reversed/a160}{2794}
\pdfglyphtounicode{tfm:zpzdr-reversed/a161}{2192}
\pdfglyphtounicode{tfm:zpzdr-reversed/a162}{27A3}
\pdfglyphtounicode{tfm:zpzdr-reversed/a163}{2194}
\pdfglyphtounicode{tfm:zpzdr-reversed/a164}{2195}
\pdfglyphtounicode{tfm:zpzdr-reversed/a165}{2799}
\pdfglyphtounicode{tfm:zpzdr-reversed/a166}{279B}
\pdfglyphtounicode{tfm:zpzdr-reversed/a167}{279C}
\pdfglyphtounicode{tfm:zpzdr-reversed/a168}{279D}
\pdfglyphtounicode{tfm:zpzdr-reversed/a169}{279E}
\pdfglyphtounicode{tfm:zpzdr-reversed/a17}{2711}
\pdfglyphtounicode{tfm:zpzdr-reversed/a170}{279F}
\pdfglyphtounicode{tfm:zpzdr-reversed/a171}{27A0}
\pdfglyphtounicode{tfm:zpzdr-reversed/a172}{27A1}
\pdfglyphtounicode{tfm:zpzdr-reversed/a173}{27A2}
\pdfglyphtounicode{tfm:zpzdr-reversed/a174}{27A4}
\pdfglyphtounicode{tfm:zpzdr-reversed/a175}{27A5}
\pdfglyphtounicode{tfm:zpzdr-reversed/a176}{27A6}
\pdfglyphtounicode{tfm:zpzdr-reversed/a177}{27A7}
\pdfglyphtounicode{tfm:zpzdr-reversed/a178}{27A8}
\pdfglyphtounicode{tfm:zpzdr-reversed/a179}{27A9}
\pdfglyphtounicode{tfm:zpzdr-reversed/a18}{2712}
\pdfglyphtounicode{tfm:zpzdr-reversed/a180}{27AB}
\pdfglyphtounicode{tfm:zpzdr-reversed/a181}{27AD}
\pdfglyphtounicode{tfm:zpzdr-reversed/a182}{27AF}
\pdfglyphtounicode{tfm:zpzdr-reversed/a183}{27B2}
\pdfglyphtounicode{tfm:zpzdr-reversed/a184}{27B3}
\pdfglyphtounicode{tfm:zpzdr-reversed/a185}{27B5}
\pdfglyphtounicode{tfm:zpzdr-reversed/a186}{27B8}
\pdfglyphtounicode{tfm:zpzdr-reversed/a187}{27BA}
\pdfglyphtounicode{tfm:zpzdr-reversed/a188}{27BB}
\pdfglyphtounicode{tfm:zpzdr-reversed/a189}{27BC}
\pdfglyphtounicode{tfm:zpzdr-reversed/a19}{2713}
\pdfglyphtounicode{tfm:zpzdr-reversed/a190}{27BD}
\pdfglyphtounicode{tfm:zpzdr-reversed/a191}{27BE}
\pdfglyphtounicode{tfm:zpzdr-reversed/a192}{279A}
\pdfglyphtounicode{tfm:zpzdr-reversed/a193}{27AA}
\pdfglyphtounicode{tfm:zpzdr-reversed/a194}{27B6}
\pdfglyphtounicode{tfm:zpzdr-reversed/a195}{27B9}
\pdfglyphtounicode{tfm:zpzdr-reversed/a196}{2798}
\pdfglyphtounicode{tfm:zpzdr-reversed/a197}{27B4}
\pdfglyphtounicode{tfm:zpzdr-reversed/a198}{27B7}
\pdfglyphtounicode{tfm:zpzdr-reversed/a199}{27AC}
\pdfglyphtounicode{tfm:zpzdr-reversed/a2}{2702}
\pdfglyphtounicode{tfm:zpzdr-reversed/a20}{2714}
\pdfglyphtounicode{tfm:zpzdr-reversed/a200}{27AE}
\pdfglyphtounicode{tfm:zpzdr-reversed/a201}{27B1}
\pdfglyphtounicode{tfm:zpzdr-reversed/a202}{2703}
\pdfglyphtounicode{tfm:zpzdr-reversed/a203}{2750}
\pdfglyphtounicode{tfm:zpzdr-reversed/a204}{2752}
\pdfglyphtounicode{tfm:zpzdr-reversed/a205}{276E}
\pdfglyphtounicode{tfm:zpzdr-reversed/a206}{2770}
\pdfglyphtounicode{tfm:zpzdr-reversed/a21}{2715}
\pdfglyphtounicode{tfm:zpzdr-reversed/a22}{2716}
\pdfglyphtounicode{tfm:zpzdr-reversed/a23}{2717}
\pdfglyphtounicode{tfm:zpzdr-reversed/a24}{2718}
\pdfglyphtounicode{tfm:zpzdr-reversed/a25}{2719}
\pdfglyphtounicode{tfm:zpzdr-reversed/a26}{271A}
\pdfglyphtounicode{tfm:zpzdr-reversed/a27}{271B}
\pdfglyphtounicode{tfm:zpzdr-reversed/a28}{271C}
\pdfglyphtounicode{tfm:zpzdr-reversed/a29}{2722}
\pdfglyphtounicode{tfm:zpzdr-reversed/a3}{2704}
\pdfglyphtounicode{tfm:zpzdr-reversed/a30}{2723}
\pdfglyphtounicode{tfm:zpzdr-reversed/a31}{2724}
\pdfglyphtounicode{tfm:zpzdr-reversed/a32}{2725}
\pdfglyphtounicode{tfm:zpzdr-reversed/a33}{2726}
\pdfglyphtounicode{tfm:zpzdr-reversed/a34}{2727}
\pdfglyphtounicode{tfm:zpzdr-reversed/a35}{2605}
\pdfglyphtounicode{tfm:zpzdr-reversed/a36}{2729}
\pdfglyphtounicode{tfm:zpzdr-reversed/a37}{272A}
\pdfglyphtounicode{tfm:zpzdr-reversed/a38}{272B}
\pdfglyphtounicode{tfm:zpzdr-reversed/a39}{272C}
\pdfglyphtounicode{tfm:zpzdr-reversed/a4}{260E}
\pdfglyphtounicode{tfm:zpzdr-reversed/a40}{272D}
\pdfglyphtounicode{tfm:zpzdr-reversed/a41}{272E}
\pdfglyphtounicode{tfm:zpzdr-reversed/a42}{272F}
\pdfglyphtounicode{tfm:zpzdr-reversed/a43}{2730}
\pdfglyphtounicode{tfm:zpzdr-reversed/a44}{2731}
\pdfglyphtounicode{tfm:zpzdr-reversed/a45}{2732}
\pdfglyphtounicode{tfm:zpzdr-reversed/a46}{2733}
\pdfglyphtounicode{tfm:zpzdr-reversed/a47}{2734}
\pdfglyphtounicode{tfm:zpzdr-reversed/a48}{2735}
\pdfglyphtounicode{tfm:zpzdr-reversed/a49}{2736}
\pdfglyphtounicode{tfm:zpzdr-reversed/a5}{2706}
\pdfglyphtounicode{tfm:zpzdr-reversed/a50}{2737}
\pdfglyphtounicode{tfm:zpzdr-reversed/a51}{2738}
\pdfglyphtounicode{tfm:zpzdr-reversed/a52}{2739}
\pdfglyphtounicode{tfm:zpzdr-reversed/a53}{273A}
\pdfglyphtounicode{tfm:zpzdr-reversed/a54}{273B}
\pdfglyphtounicode{tfm:zpzdr-reversed/a55}{273C}
\pdfglyphtounicode{tfm:zpzdr-reversed/a56}{273D}
\pdfglyphtounicode{tfm:zpzdr-reversed/a57}{273E}
\pdfglyphtounicode{tfm:zpzdr-reversed/a58}{273F}
\pdfglyphtounicode{tfm:zpzdr-reversed/a59}{2740}
\pdfglyphtounicode{tfm:zpzdr-reversed/a6}{271D}
\pdfglyphtounicode{tfm:zpzdr-reversed/a60}{2741}
\pdfglyphtounicode{tfm:zpzdr-reversed/a61}{2742}
\pdfglyphtounicode{tfm:zpzdr-reversed/a62}{2743}
\pdfglyphtounicode{tfm:zpzdr-reversed/a63}{2744}
\pdfglyphtounicode{tfm:zpzdr-reversed/a64}{2745}
\pdfglyphtounicode{tfm:zpzdr-reversed/a65}{2746}
\pdfglyphtounicode{tfm:zpzdr-reversed/a66}{2747}
\pdfglyphtounicode{tfm:zpzdr-reversed/a67}{2748}
\pdfglyphtounicode{tfm:zpzdr-reversed/a68}{2749}
\pdfglyphtounicode{tfm:zpzdr-reversed/a69}{274A}
\pdfglyphtounicode{tfm:zpzdr-reversed/a7}{271E}
\pdfglyphtounicode{tfm:zpzdr-reversed/a70}{274B}
\pdfglyphtounicode{tfm:zpzdr-reversed/a71}{25CF}
\pdfglyphtounicode{tfm:zpzdr-reversed/a72}{274D}
\pdfglyphtounicode{tfm:zpzdr-reversed/a73}{25A0}
\pdfglyphtounicode{tfm:zpzdr-reversed/a74}{274F}
\pdfglyphtounicode{tfm:zpzdr-reversed/a75}{2751}
\pdfglyphtounicode{tfm:zpzdr-reversed/a76}{25B2}
\pdfglyphtounicode{tfm:zpzdr-reversed/a77}{25BC}
\pdfglyphtounicode{tfm:zpzdr-reversed/a78}{25C6}
\pdfglyphtounicode{tfm:zpzdr-reversed/a79}{2756}
\pdfglyphtounicode{tfm:zpzdr-reversed/a8}{271F}
\pdfglyphtounicode{tfm:zpzdr-reversed/a81}{25D7}
\pdfglyphtounicode{tfm:zpzdr-reversed/a82}{2758}
\pdfglyphtounicode{tfm:zpzdr-reversed/a83}{2759}
\pdfglyphtounicode{tfm:zpzdr-reversed/a84}{275A}
\pdfglyphtounicode{tfm:zpzdr-reversed/a85}{276F}
\pdfglyphtounicode{tfm:zpzdr-reversed/a86}{2771}
\pdfglyphtounicode{tfm:zpzdr-reversed/a87}{2772}
\pdfglyphtounicode{tfm:zpzdr-reversed/a88}{2773}
\pdfglyphtounicode{tfm:zpzdr-reversed/a89}{2768}
\pdfglyphtounicode{tfm:zpzdr-reversed/a9}{2720}
\pdfglyphtounicode{tfm:zpzdr-reversed/a90}{2769}
\pdfglyphtounicode{tfm:zpzdr-reversed/a91}{276C}
\pdfglyphtounicode{tfm:zpzdr-reversed/a92}{276D}
\pdfglyphtounicode{tfm:zpzdr-reversed/a93}{276A}
\pdfglyphtounicode{tfm:zpzdr-reversed/a94}{276B}
\pdfglyphtounicode{tfm:zpzdr-reversed/a95}{2774}
\pdfglyphtounicode{tfm:zpzdr-reversed/a96}{2775}
\pdfglyphtounicode{tfm:zpzdr-reversed/a97}{275B}
\pdfglyphtounicode{tfm:zpzdr-reversed/a98}{275C}
\pdfglyphtounicode{tfm:zpzdr-reversed/a99}{275D}
\pdfglyphtounicode{thabengali}{09A5}
\pdfglyphtounicode{thadeva}{0925}
\pdfglyphtounicode{thagujarati}{0AA5}
\pdfglyphtounicode{thagurmukhi}{0A25}
\pdfglyphtounicode{thalarabic}{0630}
\pdfglyphtounicode{thalfinalarabic}{FEAC}
\pdfglyphtounicode{thanthakhatlowleftthai}{F898}
\pdfglyphtounicode{thanthakhatlowrightthai}{F897}
\pdfglyphtounicode{thanthakhatthai}{0E4C}
\pdfglyphtounicode{thanthakhatupperleftthai}{F896}
\pdfglyphtounicode{theharabic}{062B}
\pdfglyphtounicode{thehfinalarabic}{FE9A}
\pdfglyphtounicode{thehinitialarabic}{FE9B}
\pdfglyphtounicode{thehmedialarabic}{FE9C}
\pdfglyphtounicode{thereexists}{2203}
\pdfglyphtounicode{therefore}{2234}
\pdfglyphtounicode{theta}{03B8}
\pdfglyphtounicode{theta1}{03D1}
\pdfglyphtounicode{thetasymbolgreek}{03D1}
\pdfglyphtounicode{thieuthacirclekorean}{3279}
\pdfglyphtounicode{thieuthaparenkorean}{3219}
\pdfglyphtounicode{thieuthcirclekorean}{326B}
\pdfglyphtounicode{thieuthkorean}{314C}
\pdfglyphtounicode{thieuthparenkorean}{320B}
\pdfglyphtounicode{thirteencircle}{246C}
\pdfglyphtounicode{thirteenparen}{2480}
\pdfglyphtounicode{thirteenperiod}{2494}
\pdfglyphtounicode{thonangmonthothai}{0E11}
\pdfglyphtounicode{thook}{01AD}
\pdfglyphtounicode{thophuthaothai}{0E12}
\pdfglyphtounicode{thorn}{00FE}
\pdfglyphtounicode{thothahanthai}{0E17}
\pdfglyphtounicode{thothanthai}{0E10}
\pdfglyphtounicode{thothongthai}{0E18}
\pdfglyphtounicode{thothungthai}{0E16}
\pdfglyphtounicode{thousandcyrillic}{0482}
\pdfglyphtounicode{thousandsseparatorarabic}{066C}
\pdfglyphtounicode{thousandsseparatorpersian}{066C}
\pdfglyphtounicode{three}{0033}
\pdfglyphtounicode{threearabic}{0663}
\pdfglyphtounicode{threebengali}{09E9}
\pdfglyphtounicode{threecircle}{2462}
\pdfglyphtounicode{threecircleinversesansserif}{278C}
\pdfglyphtounicode{threedeva}{0969}
\pdfglyphtounicode{threeeighths}{215C}
\pdfglyphtounicode{threegujarati}{0AE9}
\pdfglyphtounicode{threegurmukhi}{0A69}
\pdfglyphtounicode{threehackarabic}{0663}
\pdfglyphtounicode{threehangzhou}{3023}
\pdfglyphtounicode{threeideographicparen}{3222}
\pdfglyphtounicode{threeinferior}{2083}
\pdfglyphtounicode{threemonospace}{FF13}
\pdfglyphtounicode{threenumeratorbengali}{09F6}
\pdfglyphtounicode{threeoldstyle}{0033}
\pdfglyphtounicode{threeparen}{2476}
\pdfglyphtounicode{threeperiod}{248A}
\pdfglyphtounicode{threepersian}{06F3}
\pdfglyphtounicode{threequarters}{00BE}
\pdfglyphtounicode{threequartersemdash}{F6DE}
\pdfglyphtounicode{threeroman}{2172}
\pdfglyphtounicode{threesuperior}{00B3}
\pdfglyphtounicode{threethai}{0E53}
\pdfglyphtounicode{thzsquare}{3394}
\pdfglyphtounicode{tihiragana}{3061}
\pdfglyphtounicode{tikatakana}{30C1}
\pdfglyphtounicode{tikatakanahalfwidth}{FF81}
\pdfglyphtounicode{tikeutacirclekorean}{3270}
\pdfglyphtounicode{tikeutaparenkorean}{3210}
\pdfglyphtounicode{tikeutcirclekorean}{3262}
\pdfglyphtounicode{tikeutkorean}{3137}
\pdfglyphtounicode{tikeutparenkorean}{3202}
\pdfglyphtounicode{tilde}{02DC}
\pdfglyphtounicode{tildebelowcmb}{0330}
\pdfglyphtounicode{tildecmb}{0303}
\pdfglyphtounicode{tildecomb}{0303}
\pdfglyphtounicode{tildedoublecmb}{0360}
\pdfglyphtounicode{tildeoperator}{223C}
\pdfglyphtounicode{tildeoverlaycmb}{0334}
\pdfglyphtounicode{tildeverticalcmb}{033E}
\pdfglyphtounicode{timescircle}{2297}
\pdfglyphtounicode{tipehahebrew}{0596}
\pdfglyphtounicode{tipehalefthebrew}{0596}
\pdfglyphtounicode{tippigurmukhi}{0A70}
\pdfglyphtounicode{titlocyrilliccmb}{0483}
\pdfglyphtounicode{tiwnarmenian}{057F}
\pdfglyphtounicode{tlinebelow}{1E6F}
\pdfglyphtounicode{tmonospace}{FF54}
\pdfglyphtounicode{toarmenian}{0569}
\pdfglyphtounicode{tohiragana}{3068}
\pdfglyphtounicode{tokatakana}{30C8}
\pdfglyphtounicode{tokatakanahalfwidth}{FF84}
\pdfglyphtounicode{tonebarextrahighmod}{02E5}
\pdfglyphtounicode{tonebarextralowmod}{02E9}
\pdfglyphtounicode{tonebarhighmod}{02E6}
\pdfglyphtounicode{tonebarlowmod}{02E8}
\pdfglyphtounicode{tonebarmidmod}{02E7}
\pdfglyphtounicode{tonefive}{01BD}
\pdfglyphtounicode{tonesix}{0185}
\pdfglyphtounicode{tonetwo}{01A8}
\pdfglyphtounicode{tonos}{0384}
\pdfglyphtounicode{tonsquare}{3327}
\pdfglyphtounicode{topatakthai}{0E0F}
\pdfglyphtounicode{tortoiseshellbracketleft}{3014}
\pdfglyphtounicode{tortoiseshellbracketleftsmall}{FE5D}
\pdfglyphtounicode{tortoiseshellbracketleftvertical}{FE39}
\pdfglyphtounicode{tortoiseshellbracketright}{3015}
\pdfglyphtounicode{tortoiseshellbracketrightsmall}{FE5E}
\pdfglyphtounicode{tortoiseshellbracketrightvertical}{FE3A}
\pdfglyphtounicode{totaothai}{0E15}
\pdfglyphtounicode{tpalatalhook}{01AB}
\pdfglyphtounicode{tparen}{24AF}
\pdfglyphtounicode{trademark}{2122}
\pdfglyphtounicode{trademarksans}{2122}
\pdfglyphtounicode{trademarkserif}{2122}
\pdfglyphtounicode{tretroflexhook}{0288}
\pdfglyphtounicode{triagdn}{25BC}
\pdfglyphtounicode{triaglf}{25C4}
\pdfglyphtounicode{triagrt}{25BA}
\pdfglyphtounicode{triagup}{25B2}
\pdfglyphtounicode{triangle}{25B3}
\pdfglyphtounicode{triangledownsld}{25BC}
\pdfglyphtounicode{triangleinv}{25BD}
\pdfglyphtounicode{triangleleft}{25C1}
\pdfglyphtounicode{triangleleftequal}{22B4}
\pdfglyphtounicode{triangleleftsld}{25C0}
\pdfglyphtounicode{triangleright}{25B7}
\pdfglyphtounicode{trianglerightequal}{22B5}
\pdfglyphtounicode{trianglerightsld}{25B6}
\pdfglyphtounicode{trianglesolid}{25B2}
\pdfglyphtounicode{ts}{02A6}
\pdfglyphtounicode{tsadi}{05E6}
\pdfglyphtounicode{tsadidagesh}{FB46}
\pdfglyphtounicode{tsadidageshhebrew}{FB46}
\pdfglyphtounicode{tsadihebrew}{05E6}
\pdfglyphtounicode{tsecyrillic}{0446}
\pdfglyphtounicode{tsere}{05B5}
\pdfglyphtounicode{tsere12}{05B5}
\pdfglyphtounicode{tsere1e}{05B5}
\pdfglyphtounicode{tsere2b}{05B5}
\pdfglyphtounicode{tserehebrew}{05B5}
\pdfglyphtounicode{tserenarrowhebrew}{05B5}
\pdfglyphtounicode{tserequarterhebrew}{05B5}
\pdfglyphtounicode{tserewidehebrew}{05B5}
\pdfglyphtounicode{tshecyrillic}{045B}
\pdfglyphtounicode{tsuperior}{0074}
\pdfglyphtounicode{ttabengali}{099F}
\pdfglyphtounicode{ttadeva}{091F}
\pdfglyphtounicode{ttagujarati}{0A9F}
\pdfglyphtounicode{ttagurmukhi}{0A1F}
\pdfglyphtounicode{tteharabic}{0679}
\pdfglyphtounicode{ttehfinalarabic}{FB67}
\pdfglyphtounicode{ttehinitialarabic}{FB68}
\pdfglyphtounicode{ttehmedialarabic}{FB69}
\pdfglyphtounicode{tthabengali}{09A0}
\pdfglyphtounicode{tthadeva}{0920}
\pdfglyphtounicode{tthagujarati}{0AA0}
\pdfglyphtounicode{tthagurmukhi}{0A20}
\pdfglyphtounicode{tturned}{0287}
\pdfglyphtounicode{tuhiragana}{3064}
\pdfglyphtounicode{tukatakana}{30C4}
\pdfglyphtounicode{tukatakanahalfwidth}{FF82}
\pdfglyphtounicode{turnstileleft}{22A2}
\pdfglyphtounicode{turnstileright}{22A3}
\pdfglyphtounicode{tusmallhiragana}{3063}
\pdfglyphtounicode{tusmallkatakana}{30C3}
\pdfglyphtounicode{tusmallkatakanahalfwidth}{FF6F}
\pdfglyphtounicode{twelvecircle}{246B}
\pdfglyphtounicode{twelveparen}{247F}
\pdfglyphtounicode{twelveperiod}{2493}
\pdfglyphtounicode{twelveroman}{217B}
\pdfglyphtounicode{twentycircle}{2473}
\pdfglyphtounicode{twentyhangzhou}{5344}
\pdfglyphtounicode{twentyparen}{2487}
\pdfglyphtounicode{twentyperiod}{249B}
\pdfglyphtounicode{two}{0032}
\pdfglyphtounicode{twoarabic}{0662}
\pdfglyphtounicode{twobengali}{09E8}
\pdfglyphtounicode{twocircle}{2461}
\pdfglyphtounicode{twocircleinversesansserif}{278B}
\pdfglyphtounicode{twodeva}{0968}
\pdfglyphtounicode{twodotenleader}{2025}
\pdfglyphtounicode{twodotleader}{2025}
\pdfglyphtounicode{twodotleadervertical}{FE30}
\pdfglyphtounicode{twogujarati}{0AE8}
\pdfglyphtounicode{twogurmukhi}{0A68}
\pdfglyphtounicode{twohackarabic}{0662}
\pdfglyphtounicode{twohangzhou}{3022}
\pdfglyphtounicode{twoideographicparen}{3221}
\pdfglyphtounicode{twoinferior}{2082}
\pdfglyphtounicode{twomonospace}{FF12}
\pdfglyphtounicode{twonumeratorbengali}{09F5}
\pdfglyphtounicode{twooldstyle}{0032}
\pdfglyphtounicode{twoparen}{2475}
\pdfglyphtounicode{twoperiod}{2489}
\pdfglyphtounicode{twopersian}{06F2}
\pdfglyphtounicode{tworoman}{2171}
\pdfglyphtounicode{twostroke}{01BB}
\pdfglyphtounicode{twosuperior}{00B2}
\pdfglyphtounicode{twothai}{0E52}
\pdfglyphtounicode{twothirds}{2154}
\pdfglyphtounicode{u}{0075}
\pdfglyphtounicode{uacute}{00FA}
\pdfglyphtounicode{ubar}{0289}
\pdfglyphtounicode{ubengali}{0989}
\pdfglyphtounicode{ubopomofo}{3128}
\pdfglyphtounicode{ubreve}{016D}
\pdfglyphtounicode{ucaron}{01D4}
\pdfglyphtounicode{ucircle}{24E4}
\pdfglyphtounicode{ucircumflex}{00FB}
\pdfglyphtounicode{ucircumflexbelow}{1E77}
\pdfglyphtounicode{ucyrillic}{0443}
\pdfglyphtounicode{udattadeva}{0951}
\pdfglyphtounicode{udblacute}{0171}
\pdfglyphtounicode{udblgrave}{0215}
\pdfglyphtounicode{udeva}{0909}
\pdfglyphtounicode{udieresis}{00FC}
\pdfglyphtounicode{udieresisacute}{01D8}
\pdfglyphtounicode{udieresisbelow}{1E73}
\pdfglyphtounicode{udieresiscaron}{01DA}
\pdfglyphtounicode{udieresiscyrillic}{04F1}
\pdfglyphtounicode{udieresisgrave}{01DC}
\pdfglyphtounicode{udieresismacron}{01D6}
\pdfglyphtounicode{udotbelow}{1EE5}
\pdfglyphtounicode{ugrave}{00F9}
\pdfglyphtounicode{ugujarati}{0A89}
\pdfglyphtounicode{ugurmukhi}{0A09}
\pdfglyphtounicode{uhiragana}{3046}
\pdfglyphtounicode{uhookabove}{1EE7}
\pdfglyphtounicode{uhorn}{01B0}
\pdfglyphtounicode{uhornacute}{1EE9}
\pdfglyphtounicode{uhorndotbelow}{1EF1}
\pdfglyphtounicode{uhorngrave}{1EEB}
\pdfglyphtounicode{uhornhookabove}{1EED}
\pdfglyphtounicode{uhorntilde}{1EEF}
\pdfglyphtounicode{uhungarumlaut}{0171}
\pdfglyphtounicode{uhungarumlautcyrillic}{04F3}
\pdfglyphtounicode{uinvertedbreve}{0217}
\pdfglyphtounicode{ukatakana}{30A6}
\pdfglyphtounicode{ukatakanahalfwidth}{FF73}
\pdfglyphtounicode{ukcyrillic}{0479}
\pdfglyphtounicode{ukorean}{315C}
\pdfglyphtounicode{umacron}{016B}
\pdfglyphtounicode{umacroncyrillic}{04EF}
\pdfglyphtounicode{umacrondieresis}{1E7B}
\pdfglyphtounicode{umatragurmukhi}{0A41}
\pdfglyphtounicode{umonospace}{FF55}
\pdfglyphtounicode{underscore}{005F}
\pdfglyphtounicode{underscoredbl}{2017}
\pdfglyphtounicode{underscoremonospace}{FF3F}
\pdfglyphtounicode{underscorevertical}{FE33}
\pdfglyphtounicode{underscorewavy}{FE4F}
\pdfglyphtounicode{union}{222A}
\pdfglyphtounicode{uniondbl}{22D3}
\pdfglyphtounicode{unionmulti}{228E}
\pdfglyphtounicode{unionsq}{2294}
\pdfglyphtounicode{universal}{2200}
\pdfglyphtounicode{uogonek}{0173}
\pdfglyphtounicode{uparen}{24B0}
\pdfglyphtounicode{upblock}{2580}
\pdfglyphtounicode{upperdothebrew}{05C4}
\pdfglyphtounicode{uprise}{22CF}
\pdfglyphtounicode{upsilon}{03C5}
\pdfglyphtounicode{upsilondieresis}{03CB}
\pdfglyphtounicode{upsilondieresistonos}{03B0}
\pdfglyphtounicode{upsilonlatin}{028A}
\pdfglyphtounicode{upsilontonos}{03CD}
\pdfglyphtounicode{upslope}{29F8}
\pdfglyphtounicode{uptackbelowcmb}{031D}
\pdfglyphtounicode{uptackmod}{02D4}
\pdfglyphtounicode{uragurmukhi}{0A73}
\pdfglyphtounicode{uring}{016F}
\pdfglyphtounicode{ushortcyrillic}{045E}
\pdfglyphtounicode{usmallhiragana}{3045}
\pdfglyphtounicode{usmallkatakana}{30A5}
\pdfglyphtounicode{usmallkatakanahalfwidth}{FF69}
\pdfglyphtounicode{ustraightcyrillic}{04AF}
\pdfglyphtounicode{ustraightstrokecyrillic}{04B1}
\pdfglyphtounicode{utilde}{0169}
\pdfglyphtounicode{utildeacute}{1E79}
\pdfglyphtounicode{utildebelow}{1E75}
\pdfglyphtounicode{uubengali}{098A}
\pdfglyphtounicode{uudeva}{090A}
\pdfglyphtounicode{uugujarati}{0A8A}
\pdfglyphtounicode{uugurmukhi}{0A0A}
\pdfglyphtounicode{uumatragurmukhi}{0A42}
\pdfglyphtounicode{uuvowelsignbengali}{09C2}
\pdfglyphtounicode{uuvowelsigndeva}{0942}
\pdfglyphtounicode{uuvowelsigngujarati}{0AC2}
\pdfglyphtounicode{uvowelsignbengali}{09C1}
\pdfglyphtounicode{uvowelsigndeva}{0941}
\pdfglyphtounicode{uvowelsigngujarati}{0AC1}
\pdfglyphtounicode{v}{0076}
\pdfglyphtounicode{vadeva}{0935}
\pdfglyphtounicode{vagujarati}{0AB5}
\pdfglyphtounicode{vagurmukhi}{0A35}
\pdfglyphtounicode{vakatakana}{30F7}
\pdfglyphtounicode{vav}{05D5}
\pdfglyphtounicode{vavdagesh}{FB35}
\pdfglyphtounicode{vavdagesh65}{FB35}
\pdfglyphtounicode{vavdageshhebrew}{FB35}
\pdfglyphtounicode{vavhebrew}{05D5}
\pdfglyphtounicode{vavholam}{FB4B}
\pdfglyphtounicode{vavholamhebrew}{FB4B}
\pdfglyphtounicode{vavvavhebrew}{05F0}
\pdfglyphtounicode{vavyodhebrew}{05F1}
\pdfglyphtounicode{vcircle}{24E5}
\pdfglyphtounicode{vdotbelow}{1E7F}
\pdfglyphtounicode{vector}{20D7}
\pdfglyphtounicode{vecyrillic}{0432}
\pdfglyphtounicode{veharabic}{06A4}
\pdfglyphtounicode{vehfinalarabic}{FB6B}
\pdfglyphtounicode{vehinitialarabic}{FB6C}
\pdfglyphtounicode{vehmedialarabic}{FB6D}
\pdfglyphtounicode{vekatakana}{30F9}
\pdfglyphtounicode{venus}{2640}
\pdfglyphtounicode{verticalbar}{007C}
\pdfglyphtounicode{verticallineabovecmb}{030D}
\pdfglyphtounicode{verticallinebelowcmb}{0329}
\pdfglyphtounicode{verticallinelowmod}{02CC}
\pdfglyphtounicode{verticallinemod}{02C8}
\pdfglyphtounicode{vewarmenian}{057E}
\pdfglyphtounicode{vhook}{028B}
\pdfglyphtounicode{vikatakana}{30F8}
\pdfglyphtounicode{viramabengali}{09CD}
\pdfglyphtounicode{viramadeva}{094D}
\pdfglyphtounicode{viramagujarati}{0ACD}
\pdfglyphtounicode{visargabengali}{0983}
\pdfglyphtounicode{visargadeva}{0903}
\pdfglyphtounicode{visargagujarati}{0A83}
\pdfglyphtounicode{visiblespace}{2423}
\pdfglyphtounicode{visualspace}{2423}
\pdfglyphtounicode{vmonospace}{FF56}
\pdfglyphtounicode{voarmenian}{0578}
\pdfglyphtounicode{voicediterationhiragana}{309E}
\pdfglyphtounicode{voicediterationkatakana}{30FE}
\pdfglyphtounicode{voicedmarkkana}{309B}
\pdfglyphtounicode{voicedmarkkanahalfwidth}{FF9E}
\pdfglyphtounicode{vokatakana}{30FA}
\pdfglyphtounicode{vparen}{24B1}
\pdfglyphtounicode{vtilde}{1E7D}
\pdfglyphtounicode{vturned}{028C}
\pdfglyphtounicode{vuhiragana}{3094}
\pdfglyphtounicode{vukatakana}{30F4}
\pdfglyphtounicode{w}{0077}
\pdfglyphtounicode{wacute}{1E83}
\pdfglyphtounicode{waekorean}{3159}
\pdfglyphtounicode{wahiragana}{308F}
\pdfglyphtounicode{wakatakana}{30EF}
\pdfglyphtounicode{wakatakanahalfwidth}{FF9C}
\pdfglyphtounicode{wakorean}{3158}
\pdfglyphtounicode{wasmallhiragana}{308E}
\pdfglyphtounicode{wasmallkatakana}{30EE}
\pdfglyphtounicode{wattosquare}{3357}
\pdfglyphtounicode{wavedash}{301C}
\pdfglyphtounicode{wavyunderscorevertical}{FE34}
\pdfglyphtounicode{wawarabic}{0648}
\pdfglyphtounicode{wawfinalarabic}{FEEE}
\pdfglyphtounicode{wawhamzaabovearabic}{0624}
\pdfglyphtounicode{wawhamzaabovefinalarabic}{FE86}
\pdfglyphtounicode{wbsquare}{33DD}
\pdfglyphtounicode{wcircle}{24E6}
\pdfglyphtounicode{wcircumflex}{0175}
\pdfglyphtounicode{wdieresis}{1E85}
\pdfglyphtounicode{wdotaccent}{1E87}
\pdfglyphtounicode{wdotbelow}{1E89}
\pdfglyphtounicode{wehiragana}{3091}
\pdfglyphtounicode{weierstrass}{2118}
\pdfglyphtounicode{wekatakana}{30F1}
\pdfglyphtounicode{wekorean}{315E}
\pdfglyphtounicode{weokorean}{315D}
\pdfglyphtounicode{wgrave}{1E81}
\pdfglyphtounicode{whitebullet}{25E6}
\pdfglyphtounicode{whitecircle}{25CB}
\pdfglyphtounicode{whitecircleinverse}{25D9}
\pdfglyphtounicode{whitecornerbracketleft}{300E}
\pdfglyphtounicode{whitecornerbracketleftvertical}{FE43}
\pdfglyphtounicode{whitecornerbracketright}{300F}
\pdfglyphtounicode{whitecornerbracketrightvertical}{FE44}
\pdfglyphtounicode{whitediamond}{25C7}
\pdfglyphtounicode{whitediamondcontainingblacksmalldiamond}{25C8}
\pdfglyphtounicode{whitedownpointingsmalltriangle}{25BF}
\pdfglyphtounicode{whitedownpointingtriangle}{25BD}
\pdfglyphtounicode{whiteleftpointingsmalltriangle}{25C3}
\pdfglyphtounicode{whiteleftpointingtriangle}{25C1}
\pdfglyphtounicode{whitelenticularbracketleft}{3016}
\pdfglyphtounicode{whitelenticularbracketright}{3017}
\pdfglyphtounicode{whiterightpointingsmalltriangle}{25B9}
\pdfglyphtounicode{whiterightpointingtriangle}{25B7}
\pdfglyphtounicode{whitesmallsquare}{25AB}
\pdfglyphtounicode{whitesmilingface}{263A}
\pdfglyphtounicode{whitesquare}{25A1}
\pdfglyphtounicode{whitestar}{2606}
\pdfglyphtounicode{whitetelephone}{260F}
\pdfglyphtounicode{whitetortoiseshellbracketleft}{3018}
\pdfglyphtounicode{whitetortoiseshellbracketright}{3019}
\pdfglyphtounicode{whiteuppointingsmalltriangle}{25B5}
\pdfglyphtounicode{whiteuppointingtriangle}{25B3}
\pdfglyphtounicode{wihiragana}{3090}
\pdfglyphtounicode{wikatakana}{30F0}
\pdfglyphtounicode{wikorean}{315F}
\pdfglyphtounicode{wmonospace}{FF57}
\pdfglyphtounicode{wohiragana}{3092}
\pdfglyphtounicode{wokatakana}{30F2}
\pdfglyphtounicode{wokatakanahalfwidth}{FF66}
\pdfglyphtounicode{won}{20A9}
\pdfglyphtounicode{wonmonospace}{FFE6}
\pdfglyphtounicode{wowaenthai}{0E27}
\pdfglyphtounicode{wparen}{24B2}
\pdfglyphtounicode{wreathproduct}{2240}
\pdfglyphtounicode{wring}{1E98}
\pdfglyphtounicode{wsuperior}{02B7}
\pdfglyphtounicode{wturned}{028D}
\pdfglyphtounicode{wynn}{01BF}
\pdfglyphtounicode{x}{0078}
\pdfglyphtounicode{xabovecmb}{033D}
\pdfglyphtounicode{xbopomofo}{3112}
\pdfglyphtounicode{xcircle}{24E7}
\pdfglyphtounicode{xdieresis}{1E8D}
\pdfglyphtounicode{xdotaccent}{1E8B}
\pdfglyphtounicode{xeharmenian}{056D}
\pdfglyphtounicode{xi}{03BE}
\pdfglyphtounicode{xmonospace}{FF58}
\pdfglyphtounicode{xparen}{24B3}
\pdfglyphtounicode{xsuperior}{02E3}
\pdfglyphtounicode{y}{0079}
\pdfglyphtounicode{yaadosquare}{334E}
\pdfglyphtounicode{yabengali}{09AF}
\pdfglyphtounicode{yacute}{00FD}
\pdfglyphtounicode{yadeva}{092F}
\pdfglyphtounicode{yaekorean}{3152}
\pdfglyphtounicode{yagujarati}{0AAF}
\pdfglyphtounicode{yagurmukhi}{0A2F}
\pdfglyphtounicode{yahiragana}{3084}
\pdfglyphtounicode{yakatakana}{30E4}
\pdfglyphtounicode{yakatakanahalfwidth}{FF94}
\pdfglyphtounicode{yakorean}{3151}
\pdfglyphtounicode{yamakkanthai}{0E4E}
\pdfglyphtounicode{yasmallhiragana}{3083}
\pdfglyphtounicode{yasmallkatakana}{30E3}
\pdfglyphtounicode{yasmallkatakanahalfwidth}{FF6C}
\pdfglyphtounicode{yatcyrillic}{0463}
\pdfglyphtounicode{ycircle}{24E8}
\pdfglyphtounicode{ycircumflex}{0177}
\pdfglyphtounicode{ydieresis}{00FF}
\pdfglyphtounicode{ydotaccent}{1E8F}
\pdfglyphtounicode{ydotbelow}{1EF5}
\pdfglyphtounicode{yeharabic}{064A}
\pdfglyphtounicode{yehbarreearabic}{06D2}
\pdfglyphtounicode{yehbarreefinalarabic}{FBAF}
\pdfglyphtounicode{yehfinalarabic}{FEF2}
\pdfglyphtounicode{yehhamzaabovearabic}{0626}
\pdfglyphtounicode{yehhamzaabovefinalarabic}{FE8A}
\pdfglyphtounicode{yehhamzaaboveinitialarabic}{FE8B}
\pdfglyphtounicode{yehhamzaabovemedialarabic}{FE8C}
\pdfglyphtounicode{yehinitialarabic}{FEF3}
\pdfglyphtounicode{yehmedialarabic}{FEF4}
\pdfglyphtounicode{yehmeeminitialarabic}{FCDD}
\pdfglyphtounicode{yehmeemisolatedarabic}{FC58}
\pdfglyphtounicode{yehnoonfinalarabic}{FC94}
\pdfglyphtounicode{yehthreedotsbelowarabic}{06D1}
\pdfglyphtounicode{yekorean}{3156}
\pdfglyphtounicode{yen}{00A5}
\pdfglyphtounicode{yenmonospace}{FFE5}
\pdfglyphtounicode{yeokorean}{3155}
\pdfglyphtounicode{yeorinhieuhkorean}{3186}
\pdfglyphtounicode{yerahbenyomohebrew}{05AA}
\pdfglyphtounicode{yerahbenyomolefthebrew}{05AA}
\pdfglyphtounicode{yericyrillic}{044B}
\pdfglyphtounicode{yerudieresiscyrillic}{04F9}
\pdfglyphtounicode{yesieungkorean}{3181}
\pdfglyphtounicode{yesieungpansioskorean}{3183}
\pdfglyphtounicode{yesieungsioskorean}{3182}
\pdfglyphtounicode{yetivhebrew}{059A}
\pdfglyphtounicode{ygrave}{1EF3}
\pdfglyphtounicode{yhook}{01B4}
\pdfglyphtounicode{yhookabove}{1EF7}
\pdfglyphtounicode{yiarmenian}{0575}
\pdfglyphtounicode{yicyrillic}{0457}
\pdfglyphtounicode{yikorean}{3162}
\pdfglyphtounicode{yinyang}{262F}
\pdfglyphtounicode{yiwnarmenian}{0582}
\pdfglyphtounicode{ymonospace}{FF59}
\pdfglyphtounicode{yod}{05D9}
\pdfglyphtounicode{yoddagesh}{FB39}
\pdfglyphtounicode{yoddageshhebrew}{FB39}
\pdfglyphtounicode{yodhebrew}{05D9}
\pdfglyphtounicode{yodyodhebrew}{05F2}
\pdfglyphtounicode{yodyodpatahhebrew}{FB1F}
\pdfglyphtounicode{yohiragana}{3088}
\pdfglyphtounicode{yoikorean}{3189}
\pdfglyphtounicode{yokatakana}{30E8}
\pdfglyphtounicode{yokatakanahalfwidth}{FF96}
\pdfglyphtounicode{yokorean}{315B}
\pdfglyphtounicode{yosmallhiragana}{3087}
\pdfglyphtounicode{yosmallkatakana}{30E7}
\pdfglyphtounicode{yosmallkatakanahalfwidth}{FF6E}
\pdfglyphtounicode{yotgreek}{03F3}
\pdfglyphtounicode{yoyaekorean}{3188}
\pdfglyphtounicode{yoyakorean}{3187}
\pdfglyphtounicode{yoyakthai}{0E22}
\pdfglyphtounicode{yoyingthai}{0E0D}
\pdfglyphtounicode{yparen}{24B4}
\pdfglyphtounicode{ypogegrammeni}{037A}
\pdfglyphtounicode{ypogegrammenigreekcmb}{0345}
\pdfglyphtounicode{yr}{01A6}
\pdfglyphtounicode{yring}{1E99}
\pdfglyphtounicode{ysuperior}{02B8}
\pdfglyphtounicode{ytilde}{1EF9}
\pdfglyphtounicode{yturned}{028E}
\pdfglyphtounicode{yuhiragana}{3086}
\pdfglyphtounicode{yuikorean}{318C}
\pdfglyphtounicode{yukatakana}{30E6}
\pdfglyphtounicode{yukatakanahalfwidth}{FF95}
\pdfglyphtounicode{yukorean}{3160}
\pdfglyphtounicode{yusbigcyrillic}{046B}
\pdfglyphtounicode{yusbigiotifiedcyrillic}{046D}
\pdfglyphtounicode{yuslittlecyrillic}{0467}
\pdfglyphtounicode{yuslittleiotifiedcyrillic}{0469}
\pdfglyphtounicode{yusmallhiragana}{3085}
\pdfglyphtounicode{yusmallkatakana}{30E5}
\pdfglyphtounicode{yusmallkatakanahalfwidth}{FF6D}
\pdfglyphtounicode{yuyekorean}{318B}
\pdfglyphtounicode{yuyeokorean}{318A}
\pdfglyphtounicode{yyabengali}{09DF}
\pdfglyphtounicode{yyadeva}{095F}
\pdfglyphtounicode{z}{007A}
\pdfglyphtounicode{zaarmenian}{0566}
\pdfglyphtounicode{zacute}{017A}
\pdfglyphtounicode{zadeva}{095B}
\pdfglyphtounicode{zagurmukhi}{0A5B}
\pdfglyphtounicode{zaharabic}{0638}
\pdfglyphtounicode{zahfinalarabic}{FEC6}
\pdfglyphtounicode{zahinitialarabic}{FEC7}
\pdfglyphtounicode{zahiragana}{3056}
\pdfglyphtounicode{zahmedialarabic}{FEC8}
\pdfglyphtounicode{zainarabic}{0632}
\pdfglyphtounicode{zainfinalarabic}{FEB0}
\pdfglyphtounicode{zakatakana}{30B6}
\pdfglyphtounicode{zaqefgadolhebrew}{0595}
\pdfglyphtounicode{zaqefqatanhebrew}{0594}
\pdfglyphtounicode{zarqahebrew}{0598}
\pdfglyphtounicode{zayin}{05D6}
\pdfglyphtounicode{zayindagesh}{FB36}
\pdfglyphtounicode{zayindageshhebrew}{FB36}
\pdfglyphtounicode{zayinhebrew}{05D6}
\pdfglyphtounicode{zbopomofo}{3117}
\pdfglyphtounicode{zcaron}{017E}
\pdfglyphtounicode{zcircle}{24E9}
\pdfglyphtounicode{zcircumflex}{1E91}
\pdfglyphtounicode{zcurl}{0291}
\pdfglyphtounicode{zdot}{017C}
\pdfglyphtounicode{zdotaccent}{017C}
\pdfglyphtounicode{zdotbelow}{1E93}
\pdfglyphtounicode{zecyrillic}{0437}
\pdfglyphtounicode{zedescendercyrillic}{0499}
\pdfglyphtounicode{zedieresiscyrillic}{04DF}
\pdfglyphtounicode{zehiragana}{305C}
\pdfglyphtounicode{zekatakana}{30BC}
\pdfglyphtounicode{zero}{0030}
\pdfglyphtounicode{zeroarabic}{0660}
\pdfglyphtounicode{zerobengali}{09E6}
\pdfglyphtounicode{zerodeva}{0966}
\pdfglyphtounicode{zerogujarati}{0AE6}
\pdfglyphtounicode{zerogurmukhi}{0A66}
\pdfglyphtounicode{zerohackarabic}{0660}
\pdfglyphtounicode{zeroinferior}{2080}
\pdfglyphtounicode{zeromonospace}{FF10}
\pdfglyphtounicode{zerooldstyle}{0030}
\pdfglyphtounicode{zeropersian}{06F0}
\pdfglyphtounicode{zerosuperior}{2070}
\pdfglyphtounicode{zerothai}{0E50}
\pdfglyphtounicode{zerowidthjoiner}{FEFF}
\pdfglyphtounicode{zerowidthnonjoiner}{200C}
\pdfglyphtounicode{zerowidthspace}{200B}
\pdfglyphtounicode{zeta}{03B6}
\pdfglyphtounicode{zhbopomofo}{3113}
\pdfglyphtounicode{zhearmenian}{056A}
\pdfglyphtounicode{zhebrevecyrillic}{04C2}
\pdfglyphtounicode{zhecyrillic}{0436}
\pdfglyphtounicode{zhedescendercyrillic}{0497}
\pdfglyphtounicode{zhedieresiscyrillic}{04DD}
\pdfglyphtounicode{zihiragana}{3058}
\pdfglyphtounicode{zikatakana}{30B8}
\pdfglyphtounicode{zinorhebrew}{05AE}
\pdfglyphtounicode{zlinebelow}{1E95}
\pdfglyphtounicode{zmonospace}{FF5A}
\pdfglyphtounicode{zohiragana}{305E}
\pdfglyphtounicode{zokatakana}{30BE}
\pdfglyphtounicode{zparen}{24B5}
\pdfglyphtounicode{zretroflexhook}{0290}
\pdfglyphtounicode{zstroke}{01B6}
\pdfglyphtounicode{zuhiragana}{305A}
\pdfglyphtounicode{zukatakana}{30BA}

\pdfgentounicode=1

\title{Logical Clustering and Learning\\ for Time-Series Data}
\titlerunning{Logical Clustering}


%
\author{Marcell Vazquez-Chanlatte\inst{1}
  \and Jyotirmoy V. Deshmukh\inst{2} \and \\
  Xiaoqing Jin\inst{2}
  \and Sanjit A. Seshia\inst{1}
}

\authorrunning{Vazquez-Chanlatte et al.}

\institute{University of California Berkeley, \email{$\mathtt{\{marcell.vc,sseshia\}@eecs.berkeley.edu}$}\\
\and
Toyota Motors North America R\&D, \email{$\mathtt{first.lastname@toyota.com}$}\\
}
			
\maketitle

\begin{abstract}
In order to effectively analyze and build cyberphysical systems (CPS),
designers today have to combat the data deluge problem, i.e., the
burden of processing intractably large amounts of data produced by
complex models and experiments. In this work, we utilize monotonic
parametric signal temporal logic (PSTL) to design features for
unsupervised classification of time series data. This enables using
off-the-shelf machine learning tools to automatically cluster similar
traces with respect to a given PSTL formula. We demonstrate how this
technique produces interpretable formulas that are amenable
to analysis and understanding using a few representative examples. We
illustrate this with case studies related to automotive
engine testing, highway traffic analysis, and auto-grading massively
open online courses.
\end{abstract}

\renewcommand{\baselinestretch}{0.97}

\blfootnote{Acknowledgments: We thank Dorsa Sadigh, Yasser Shoukry,
  Gil Lederman, Shromona Ghosh, Oded Maler, Eamon Keogh, and our anonymous
  reviewers for their invaluable feedback, and Ken Butts for the
  Diesel data. This work is funded in part by the DARPA BRASS program
  under agreement number FA8750--16--C--0043, the VeHICaL Project (NSF
  grant \#1545126), and the Toyota Motor Corporation under the CHESS
  center. }
\section{Introduction}
\label{sec:intro}

In order to effectively construct and analyze cyber-physical systems
(CPS), designers today have to combat the {\em data deluge\/} problem,
\ie\ the burden of processing intractably large amounts of data
produced by complex models and experiments. For example, consider the
typical design process for an advanced CPS such as a self-driving car.
Checking whether the car meets all its requirements is typically done
by either physically driving the car around for millions of
miles~\cite{ackerman2014google}, or by performing virtual simulations
of the self-driving algorithms.  Either approach can generate several
gigabytes worth of time-series traces of data, such as sensor
readings, variables within the software controllers, actuator actions,
driver inputs, and environmental conditions. Typically, designers are
interested not in the details of these traces, but in discovering
higher-level insight from them; however, given the volume of data, a
high level of automation is needed.

\begin{figure}[h]
  \centering
  \includegraphics[width=\textwidth]{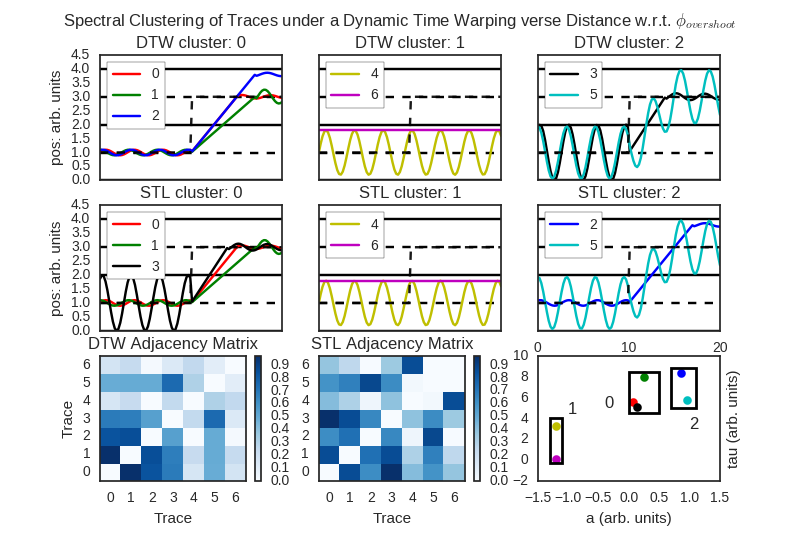}
  \vspace{-15pt}
  \caption{\label{fig:introexample_traces} An example of a pitfall
when using the DTW measure compared to projection using a PSTL
template. We perform spectral clustering~\cite{von2007tutorial} on a
similarity graph representation of 7 traces. Nodes of the graph
represent traces and edges are labeled with the normalized pairwise
distance using (1) the DTW measure and (2) the Euclidean distance
between features extracted using the PSTL template $\overshootSTL$.
Note how under the DTW measure, the black and cyan traces are grouped
together due to their behavior before the lane change, despite the
cyan trace having a much larger overshoot.  Contrast with the STL
labeling in the second row, where both overshooting traces are grouped
together.  The bottom right figure provides the projection of the
traces w.r.t.  $\overshootSTL$ with the associated cluster-labels
shown in the second row.}
\end{figure}

The key challenge then is: ``How do we automatically identify logical
structure or relations within such data?'' One possibility
offered by unsupervised learning algorithms from the machine learning
community is to cluster similar behaviors to identify higher-level
commonalities in the data.  Typical clustering algorithms define
similarity measures on signal spaces, \eg\ the dynamic time warping
distance, or by projecting data to complex feature spaces.
We argue later in this section that these methods can be inadequate 
to learn logical structure in time-series data.

In this paper, we present {\em logical clustering}, 
an {\em unsupervised learning\/} procedure
that utilizes {\em Parametric Signal Temporal Logic\/} (PSTL) templates
to discover logical structure within the given data. Informally,
Signal Temporal Logic (STL) enables specifying temporal relations
between constraints on signal values
\cite{donze2012temporal,akazaki2015time}. PSTL generalizes STL
formulas by replacing, with parameters, time constants in temporal operators and
signal-value constants in atomic predicates in the formula.
With PSTL templates, one can use the template parameters
as {\em features}.  This is done by projecting a trace to parameter
valuations that correspond to a formula that is {\em marginally
satisfied\/} by the trace. As each trace is projected to the
finite-dimensional space of formula-parameters, we can then use
traditional clustering algorithms on this space; thereby grouping
traces that satisfy the same (or similar)  formulas together. Such
{\em logical clustering\/} can reveal heretofore undiscovered structure
in the traces, albeit through the lens of an user-provided template.
We illustrate the basic steps in our technique with an example.

%
Consider the design of a lane-tracking controller for a car and a
scenario where a car has effected a lane-change. A typical control
designer tests design performance by observing the ``overshoot''
behavior of the controller, \ie\ by inspecting the maximum deviation
(say $a$) over a certain duration of time (say, $\tau$) of the vehicle
position trajectory $x(t)$ from a given desired trajectory
$x_{ref}(t)$.  We can use the following PSTL template that captures
such an overshoot:
\begin{equation}
\label{eq:overshoot}
\overshootSTL \eqdef \ev \left (\stepinput \wedge
\ev_{(0,\tau]} \left(x - x_{\mathtt{ref}} > a\right)\right)
\end{equation}
When we project traces appearing in Fig.~\ref{fig:introexample_traces}
through $\overshootSTL$, we find three behavior-clusters as shown in
the second row of the figure: (1) Cluster 0 with traces that track the
desired trajectory with small overshoot, (2) Cluster 1 with traces
that fail to track the desired trajectory altogether, and (3) Cluster
2 with traces that {\em do\/} track the desired trajectory, but have a
large overshoot value.  The three clusters indicate a well-behaved
controller, an incorrect controller, and a controller that needs
tuning respectively.  The key observation here is that though we use a single
overshoot template to analyze the data, qualitatively different
behaviors form separate clusters in the induced parameter space;
furthermore, each cluster has higher-level meaning that the designer
can find valuable.

In contrast to our proposed method, consider the clustering induced by
using the dynamic time warping (DTW) distance measure as shown in
Fig.~\ref{fig:introexample_traces}. Note that DTW is one of the most
popular measures to cluster time-series data~\cite{keogh2000scaling}.
We can see that traces with both high and low overshoots are clustered
together due to similarities in their shape.  Such shape-similarity
based grouping could be quite valuable in certain contexts; however,
it is inadequate when the designer is interested in temporal
properties that may group traces of dissimilar shapes. 

In Section~\ref{sec:temporal_projection}, we show how we can use
feature extraction with PSTL templates to group traces with similar
logical properties together.  An advantage of using PSTL is that the
enhanced feature extraction is computationally efficient for the
fragment of monotonic PSTL formulas
\cite{parametricIDofSTL,jin2015mining}; such a formula has the
property that its satisfaction by a given trace is monotonic in its
parameter values.  The efficiency in feature extraction relies on a
multi-dimensional binary search procedure
\cite{legriel2010approximating} that exploits the monotonicity
property.

A different view of the technique presented here is as a method to
perform temporal logic inference from data, such as the work on
learning STL formulas in a supervised learning
context\cite{dTree,rPSTL,parametricIDofSTL,bartocci2013learning}, in
the unsupervised anomaly detection context\cite{iPSTL}, and in the
context of active learning\cite{juniwal2014cpsgrader}.  Some of these
approaches adapt classical machine learning algorithms such as
decision trees~\cite{dTree} and one-class support vector machines~\cite{iPSTL} to learn (possibly, arbitrarily long) formulas in a
restricted fragment of STL.\@ Formulas exceeding a certain length are
often considered inscrutable by designers.  A key technical
contribution of this paper is to show that using simple shapes such as
specific Boolean combinations of axis-aligned hyperboxes in the
parameter space of monotonic PSTL to represent clusters yields a
formula that may be easier to interpret.  We support this in
Section~\ref{sec:stl_extract}, by showing that such hyperbox-clusters
correspond to STL formulas that have length linear in the number of
parameters in the given PSTL template, and thus of bounded descriptive
complexity.  

Mining parametric temporal logic properties in a model-based design
has also been explored~\cite{jin2015mining,hoxha2017mining}. We note that our
proposed methods does not require such a model, which may not be 
available either due to the complexity of the underlying system or
lack of certainty in the dynamics.  
We also note that there is much work on mining discrete temporal logic
specifications from data (e.g.~\cite{li-dac10}): our work instead focuses
on unsupervised learning of STL properties relevant to CPS.\@

The reader might wonder how much insight is needed by a user to
select the PSTL template to use for classification.
We argue the templates do not pose a burden on the user
and that our technique can have high value in
several ways.  First, we observe that we can combine our technique
with a human-guided (or automated) enumerative learning procedure
that can exploit high-level template pools. We demonstrate such a
procedure in the diesel engine case study.  Second, consider a
scenario where a designer has insight into the data that allows them
to choose the correct PSTL template. Even in this case, our method
automates the task of labeleing trace-clusters with STL labels which
can then be used to automatically classify new data.  Finally, we
argue that many unsupervised learning techniques on time-series data
must ``featurize'' the data to start with, and such features represent
relevant domain knowledge. Our features happen to be PSTL templates.
As the lane controller motivating example illustrates, a common
procedure that doesn't have some domain specific knowledge increases
the risk of wrong classifications. This sentiment is highlighted even
in the data mining literature~\cite{lin2003clustering}.  To illustrate
the value of our technique, in Sec.~\ref{sec:case_studies}, we
demonstrate the use of logic-based templates to analyze time-series
data in case studies from three different application domains.


\section{Preliminaries}

\begin{definition}[Timed Traces]
A timed trace is a finite (or infinite) sequence of pairs
$(t_0,\x_0)$, $\ldots$, $(t_n,\x_n)$, where, $t_0 = 0$, and for all $i
\in [1,n]$, $t_i \in \PosReals$, $t_{i-1} < t_i$, and for $i \in
[0,n]$, $\x_i \in \domain$, where $\domain$ is some compact set. We
refer to the interval $[t_0,t_n]$ as the time domain $\timedomain$.
\end{definition}

Real-time temporal logics are a formalism for reasoning about finite
or infinite timed traces. Logics such as the Timed Propositional
Temporal Logic~\cite{alur1994really}, and Metric Temporal Logic (\MTL)~\cite{KoymanMTL90} were introduced to reason about signals
representing Boolean-predicates varying over dense (or discrete) time.
More recently, Signal Temporal Logic~\cite{MalerN04} was proposed in
the context of analog and mixed-signal circuits as a specification
language for real-valued signals.

\mypara{Signal Temporal Logic} Without loss of generality, atoms in
STL formulas can be be reduced to the form $f(\x) \sim \sPara$, where
$f$ is a function from $\domain$ to $\Reals$, $\sim \in
\setof{\geq,\leq,=}$, and $\sPara \in \Reals$.  Temporal formulas are
formed using temporal operators, ``always'' (denoted as $\G$),
``eventually'' (denoted as $\F$) and ``until'' (denoted as $\U$) that
can each be indexed by an interval $\Intvl$.  An STL formula is
written using the following grammar:
\begin{equation}
\label{eq:stl_syntax}
\begin{array}{l}
\Intvl := (a,b) \mid (a,b] \mid [a,b) \mid [a,b] \\
\f :=      \true 
      \mid f(\x) \sim \sPara 
      \mid \neg\f \mid \f_{1} \wedge \f_{2} 
      \mid \f_{1}\, \U_{\Intvl}\, \f_{2}
\end{array}
\end{equation}
In the above grammar, $a,b \in \timedomain$, and $\sPara \in \Reals$.
The always ($\G$) and eventually ($\F$) operators are defined for
notational convenience, and are just special cases of the until
operator: \mbox{$\F_{\Intvl}\varphi \triangleq \true\, \U_\Intvl\,
\varphi$}, and \mbox{$\G_{\Intvl}\varphi \triangleq \neg \F_{\Intvl}
\neg \varphi$}. We use the notation $(\x,t) \models \f$ to mean that
the suffix of the timed trace $\x$ beginning at time $t$ satisfies the
formula $\f$. The formal semantics of an STL formula are defined
recursively:
\begin{equation*}
\label{eq:stl_semantics}
\begin{array}{lcl}
  (\x,t) \models f(\x) \sim c &\iff& \text{$f(\x(t)) \sim c$ is true} \\
  (\x,t) \models \neg \f &\iff&  (\x,t) \nmodels \f \\
  (\x,t) \models \f_1 \wedge \f_2 &\iff& (\x,t)
  \models \f_1\ \aand\  (\x,t) \models \f_2\\
  (\x,t) \models \f_1\ \U_{\Intvl}\ \f_2 & \iff & \exists t_1\in t \oplus \Intvl: (\x,t_1) \models \f_2\ \aand \\
  &      & \quad \forall t_2\in[t, t_1): (\x,t_2) \models \f_1
\end{array}
\end{equation*}
We write $\x \models \f$ as a shorthand of $(\x,0) \models \f$.

\mypara{Parametric Signal Temporal Logic (PSTL)} PSTL
\cite{parametricIDofSTL} is an extension of STL introduced to define
\emph{template formulas} containing unknown parameters.  Formally, the
set of parameters $\params$ is a set consisting of two disjoint sets
of variables $\valueparams$ and $\timeparams$ of which at least one is
nonempty.  The parameter variables in $\valueparams$ can take values
from their domain, denoted as the set $\valuedomain$.  The parameter
variables in $\timeparams$ are time-parameters that take values from
the time domain $\timedomain$.  We define a valuation function $\val$
that maps a parameter to a value in its domain. We denote a vector of
parameter variables by $\p$, and extend the definition of the
valuation function to map parameter vectors $\p$ into tuples of
respective values over $\valuedomain$ or $\timedomain$. We define the
{\em parameter space\/} $\paramspace$ as a subset of
$\valuedomain^{|\valueparams|} \times \timedomain^{|\timeparams|}$.

A PSTL formula is then defined by modifying the grammar specified in
\eqref{eq:stl_syntax} by allowing $a,b$ to be elements of
$\timeparams$, and $\sPara$ to be an element of $\valueparams$.  An
STL formula is obtained by pairing a PSTL formula with a valuation
function that assigns a value to each parameter variable.  For
example, consider the PSTL formula $\f(\sPara,\tPara)$ =
\mbox{$\G_{[0, \tPara]} x > \sPara$}, with parameters variables
$\sPara$ and $\tPara$.  The STL formula \mbox{$\G_{[0, 10]} x > 1.2$}
is an instance of $\f$ obtained with the valuation $\val =\{\tPara
\mapsto 10,\ \sPara \mapsto 1.2\}$.


\mypara{Monotonic PSTL} 
Monotonic PSTL is a fragment of PSTL introduced as the polarity
fragment in~\cite{parametricIDofSTL}.  A PSTL formula $\f$ is said to
be monotonically increasing in parameter $p_i$  if condition
\eqref{eq:mono_inc} holds for all $\x$, and is said to be
monotonically decreasing in parameter $p_i$ if condition
\eqref{eq:mono_dec} holds for all $\x$. 
\begin{align}
\hspace{-2em}
\val(p_i) \le \val'(p_i)
& \implies & 
\left [\x \models \f(\val(p_i)) \implies 
            \x \models \f(\val'(p_i)) \right]\label{eq:mono_inc} \\
\hspace{-2em}
\val(p_i) \geq \val'(p_i)
& \implies & 
\left[\x \models \f(\val(p_i)) \implies \x \models \f(\val'(p_i))\right] \label{eq:mono_dec}
\end{align}

To indicate the direction of monotonicity, we now introduce the 
polarity of a parameter\cite{parametricIDofSTL}, $\polarity(p_i)$, and
say that $\polarity(p_i) = +$ if the $\f(\p)$ is monotonically
increasing in $p_i$ and $\polarity(p_i) = -$ if it is monotonically
decreasing, and $\polarity(p_i) = \bot$ if it is neither.  A formula
$\f(\p)$ is said to be monotonic in $p_i$ if $\polarity(p_i) \in
\setof{+,-}$, and say that $\f(\p)$ is monotonic if for all $i$, $\f$
is monotonic in $p_i$.

While restrictive, the monotonic fragment of PSTL contains many
formulas of interest, such as those expressing steps and spikes in
trace values, timed-causal relations between traces, and so on.
Moreover, in some instances, for a given non-monotonic PSTL formula,
it may be possible to obtain a related monotonic PSTL formula by
using distinct parameters in place of a repeated parameter, or by
assigning a constant valuation for some parameters (Example in
Sec~\ref{sec:appendix}).

\begin{example} For formula~\eqref{eq:overshoot}, we can see that
$\polarity(a) = -$, because if a trace has a certain overshoot
exceeding the threshold $a^*$, then for a fixed $\tau$, the trace
satisfies any formula where $a < a^{*}$.  Similarly, $\polarity(\tau)$
= $+$, as an overshoot over some interval $(0,\tau^*]$ will be still
considered an overshoot for $\tau > \tau^*$.
\end{example}

\mypara{Orders on Parameter Space} A monotonic parameter induces a
total order $\pOrder_i$ in its domain, and as different parameters for
a given formula are usually independent, valuations for different
parameters induce a partial order:

\begin{definition}[Parameter Space Partial Order]
We define $\pOrder_i$ as a total order on the domain of the parameter
$p_i$ as follows:
\begin{equation}
\val(p_i) \pOrder_i \val'(p_i) \eqdef
\left\{
\begin{array}{ll}
\val(p_i) \le \val'(p_i)  & \text{\ if\ } \polarity(p_i) = + \\
\val(p_i) \geq \val'(p_i) & \text{\ if\ } \polarity(p_i) = - 
\end{array}
\right.
\end{equation}
\end{definition}

Under the order $\pOrder_i$, the parameter space can be viewed as a
partially ordered set $(\paramspace, \pOrder)$, where the ordering
operation $\pOrder$ is defined as follows:
\begin{equation}
\val(\p) \pOrder \val'(\p) \eqdef \forall i: \val(p_i) \pOrder_i \val'(p_i).
\end{equation}

When combined with eqs~\eqref{eq:mono_inc},~\eqref{eq:mono_dec} this
gives us the relation that $\val(\p) \pOrder \val'(\p)$ implies that
$[\f(\val(\p)) \implies \f(\val'(\p))]$.  In order to simplify
notation, we define the subset of $X$ that satisfies $\f(\val(\p))$ as
$\model{\val(\p)}_{\Traces}$. If $\Traces$ and $\p$ are obvious from
context, we simply write: $\model{\val}$. It follows that $(\val
\pOrder \val')$ $\implies$ $(\model{\val} \subseteq \model{\val'})$.
In summary: $\pOrder$ operates in the same direction as implication
and subset. Informally, we say that the ordering is from a stronger to
a weaker formula.

\begin{example}
For formula~\eqref{eq:overshoot}, the order operation $\pOrder$ is
defined as $\val \pOrder \val'$ iff $\val(\tau) < \val'(\tau)$ and
$\val(a) > \val'(a)$.  Consider $\val_1(\p) \eqdef (\tau : 0.1, a :
-1.1)$ and $\val_2(\p) \eqdef (\tau : 3.3, a : -1.3)$.  As
$\polarity(a) = -$, $\polarity(\tau) = +$, $\val_1 \pOrder \val_2$,
and $\f_{\mathrm{overshoot}}(\val_1(\p)) \implies
\f_{\mathrm{overshoot}}(\val_2(\p))$.  Intuitively this means that
if $x(t)$ satisfies a formula specifying a overshoot $> -1.1$ (undershoot
$< 1.1$) over a duration of $0.1$ time units, then $x(t)$ trivially
satisfies the formula specifying an undershoot of $< 1.3$ over a
duration of $3.3$ time units.
\end{example}

Next, we define the downward closure of $\val(\p)$ and relate it to
$\model{\val(\p)}$.
\begin{definition}[Downward closure of a valuation]
For a valuation $\val$, its downward closure (denoted
$\downward(\val)$) is the set $\setof{ \val' \mid \val' \pOrder \val
}$.
\end{definition}

In the following lemma we state that the union of the sets of traces
satisfying formulas corresponding to parameter valuations in the
downward closure of a valuation $\val$ is the same as the set of
traces satisfying the formula corresponding to $\val$. The proof
follows from the definition of downward closure.

\begin{lemma}\label{lemma:stl_to_downward_closure}
  $\bigcup_{\val' \in \downward(\val)} \model{\val'} \equiv \model{\val}$
\end{lemma}

Lastly, we define the validity domain of a set of traces and $\f$.

\begin{definition}[Validity domain] Let $\Traces$ be a (potentially
infinite) collection of timed traces, and let $\f(\p)$ be a PSTL
formula with parameters $\p \in \params$. The validity
domain\footnote{If $\Traces$ is obvious from context (or does not
matter), we write $\validitydomain{\f(\p), \Traces}$ as
$\validitydomain{\f}$.} $\validitydomain{\f(\p),\Traces}$ of $\f(\p)$
is a closed subset of $\paramspace$, such that:
\begin{equation}
\label{eq:validity_domain}
\forall \val(\p) \in \validitydomain{\f(\p), \Traces}: \forall \x \in \Traces: \x
\models \f(\val(\p))
\end{equation}
\end{definition}

\begin{remark}
The validity domain for a given parameter set $\params$ essentially
contains all the parameter valuations s.t.\ for the given set of traces
$\Traces$, each trace satisfies the STL formula obtained by
instantiating the given PSTL formula with the parameter valuation.
\end{remark}

\begin{wrapfigure}{R}{.4\textwidth}
  \includegraphics[trim={13cm 9.6cm 11.2cm 7.2cm},clip]{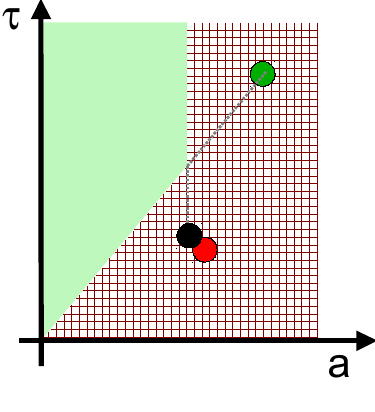}
  \caption{\label{fig:vdomain}Validity domain and projection of traces
    in STL Cluster 0 from Fig.~\ref{fig:introexample_traces}.}
\end{wrapfigure}

\begin{example}

In Fig.~\ref{fig:vdomain}, we show the validity domain of the PSTL
formula~\eqref{eq:overshoot} for the three traces given in the subplot
labeled STL cluster 0 in Fig.~\ref{fig:introexample_traces}.  The
hatched red region contains parameter valuations corresponding to STL
formulas that are not satisfied by any trace, while the shaded-green
region is the validity domain of the formula.  The validity domain
reflects that till the peak value $a^*$ of the black trace is reached
(which is the smallest among the peak values for the three signals),
the curve in $\tau$-$a$ space follows the green trace (which has the
lowest slope among the three traces). For any value of $\tau$, for
which $a > a^*$, the formula is trivially satisfied by all traces.
\end{example}

\section{Trace-Projection and Clustering}
\label{sec:temporal_projection}

In this section, we introduce the projection of a trace to the
parameter space of a given PSTL formula, and discuss mechanisms to
cluster the trace-projections using off-the-shelf clustering
techniques.

\mypara{Trace Projection}
The key idea of this paper is defining a projection operation $\proj$
that maps a given timed trace $\x$ to a suitable parameter
valuation\footnote{For canonicity, $\pi$ need not be a function from
  timed traces to $\paramspace$. For example, it may be expedient to
  project a trace to a subset of $\paramspace$. For simplicity, we defer more
  involved projections to future exposition.}
   $\nu^\star(\p)$ in the validity domain of the
given PSTL formula $\f(\p)$.  We would also like to project the given
timed trace to a valuation that is as close a representative of the
given trace as possible (under the lens of the chosen PSTL
formula).

One way of mapping a given timed trace to a single valuation is by
defining a total order $\tOrder$ on the parameter space by an
appropriate {\em linearization\/} of the partial order on parameter
space. The total order then provides a minimum valuation to which the
given timed trace is mapped.  

\begin{remark}
For technical reasons, we often adjoin two special elements $\top$ and
$\bot$ to $\validitydomain{\f(\p), \Traces}$ such that $\forall
\val(\p) \in \validitydomain{\f(\val(\p)), \Traces}$, $\bot \pOrder
\val(\p) \pOrder \top$ and $\forall x \in \Traces$, $x \models
\f(\top(\p))$ and $\neg(x \models \f(\bot(\p)))$. These special
elements mark whether $\validitydomain{\f(\val(\p))}$ is the whole
parameter space or empty.
\end{remark}

We present the lexicographic order on parameters as one possible
linearization; other linearizations, such as those based on a weighted
sum in the parameter space could also be used (presented in
Sec~\ref{sec:appendix} for brevity).

\mypara{Lexicographic Order}
A lexicographic order (denoted $\lexOrder$) uses the specification of
a total order on parameter indices to linearize the partial order. We
formalize lexicographic ordering as follows.

\begin{definition}[Lexicographic Order]
Suppose we are given a total order on the parameters $j_1 > \cdots >
j_n$.  The total order $\lexOrder$ on the parameter space
$\paramspace$ is defined as:
\begin{equation}
\label{eq:lexorder}
\begin{array}{ll}
\val(\p) \lexOrder \val'(\p)\ & \iff
 \exists j_k \in (j_1,\ldots,j_n)\ \text{\ s.t.\ }  
    \val(p_{j_k}) \pOrder_i \val'(p_{j_k})\ \text{and,} \\
 &  \forall \ell < k,\ \val(p_{j_\ell}) = \val'(p_{j_\ell}).
\end{array}
\end{equation}
\end{definition}

Note that for a given total or partial order, we can define $\inf$ and
$\sup$ under that order in standard fashion. Formally, the projection
function using lexicographic order is defined as follows:
\begin{equation}
\label{eq:lexproj}
\projlex(\x) = \inflex\setof{\val(\p) \in \validitydomain{\f(\p), \set{x}}}
\end{equation}

\mypara{Computing $\projlex$} To approximate $\projlex(\x)$, we recall
Algorithm~\ref{algo:projlex} from~\cite{jin2015mining} that uses a
simple lexicographic binary search\footnote{For simplicity, we have
  omitted a number of optimizations in Algorithm~\ref{algo:projlex}.
  For example, one can replace the iterative loop through parameters
  with a binary search over parameters.}. 

We begin by setting the interval to search for a valuation in
$\validitydomain{\f(\p)}$. We set the initial valuation to
$\top$ since it induces the most permissive STL formula. Next, for
each parameter, (in the order imposed by $\lexOrder$), we perform
bisection search on the interval to find a valuation in
$\validitydomain{\f(\p)}$. Once completed, we return the
lower bound of the search-interval as it is guaranteed to be
satisfiable (if a satisfiable assignment exists).

Crucially, this algorithm exploits the monotonicity of the PSTL
formula to guarantee that there is at most one point during the
bisection search where the satisfaction of $\f$ can change.  The
number of iterations for each parameter index $i$ is bounded above by
$\log\ceil{\frac{\sup(\paramspace_{i})-\inf(\paramspace_{i})}{\epsilon_i}}$, 
and the number of parameters. This gives us an algorithm with
complexity that grows linearly in the number of parameters and
logarithmically in the desired precision.

\begin{algorithm}[t] 
\DontPrintSemicolon{}
\KwIn{$\x(t)$, $\f(\p)$, 
      $\params$, 
      $\paramspace$, 
      $(j_1,\ldots,j_n)$,
      $\vec{\epsilon} > 0$,
      $\lexOrder$}
\KwResult{$\projlex(\x)$}
$\valdn(\p)$ $\assign$ $\inf_{\lexOrder}{\paramspace}$; $\valup(\p)$ $\assign$ $\sup_{\lexOrder}{\paramspace}$\;
\lIf{$\neg(\x \models \f(\valup(\p)))$}{$\Return$ $\top$}
\lIf{$\x \models \f(\valdn(\p))$}{$\Return$ $\bot$}
\For{$i=1$ to $|\params|$}
    {\While{$|\valup(p_{i}) - \valdn(p_{i})| > \epsilon_i$}
      {$\val(p_{i})$ $\assign$
        $\frac{1}{2}\left(\valdn(p_{i}) + \valup(p_{i})\right)$
        \nllabel{algoline:bisect} \;
        \lIf{$\x \models \val(\p)$}
         {$\valup(\p)$ $\assign$ $\val(\p)$ 
            \nllabel{algoline:shrinkdown} 
        } \lElse{$\valdn(\p)$ $\assign$ $\val(\p)$ 
          \nllabel{algoline:shrinkup} 
        }
    }
}
$\Return$ $\projlex(\x)$ $\assign$ $\valup(\p)$
\nllabel{algoline:finalassignment} \;

\caption{\label{algo:projlex}Iterated Binary Search to compute $\projlex(\x)$}
\end{algorithm}

\begin{remark}
Pragmatically, we remark that the projection algorithm is inherently
very parallel at the trace level and as such scales well across
machines.
\end{remark}

\begin{example}
For the running example (PSTL formula~\eqref{eq:overshoot}), we use
the order $a \lexOrder \tau$. As $\polarity(a) = -$, and
$\polarity(\tau) = +$, lexicographic projection has the effect of
first searching for the largest $a$, and then searching for the
smallest $\tau$ such that the resulting valuation is in
$\validitydomain{\overshootSTL}$.  The projections of the
three traces from STL cluster 0 from
Fig.~\ref{fig:introexample_traces} are shown in
Fig.~\ref{fig:vdomain}. We use the same color to denote a trace and
its projection in parameter space.
\end{example}

\begin{wrapfigure}{R}{.4\textwidth}
\includegraphics[width=\textwidth] {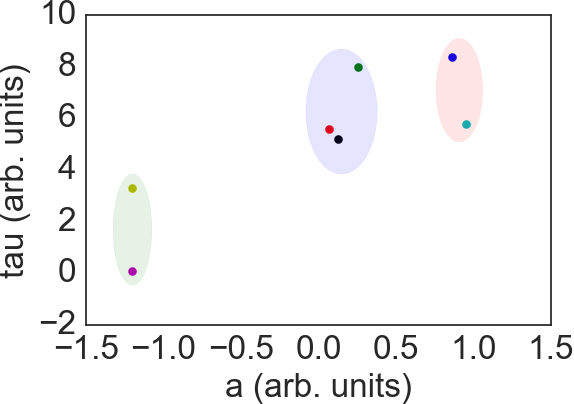}
\caption{\label{fig:gmms} Three
clusters represented using the level sets of Gaussian functions
learned from Gaussian Mixture Models (GMMs) (see Eq~\ref{ex:gmms}). The
user specifies the number of clusters to discover ($3$ in this case),
and specifies that the GMM algorithm use a diagonal covariance matrix
(which restricts cluster shape to axis-aligned ellipsoids).}
\end{wrapfigure}

\mypara{Clustering and Labeling}

What does one gain by defining a projection, $\pi$? We posit that
applying unsupervised learning algorithms for clustering in
$\paramspace$ lets us glean insights about the logical structure of
the trace-space by grouping traces that satisfy similar formulas
together. Let $L$ be a finite, nonempty set of labels.  Let
$\TracesSubset \subset \Traces$ represent a user-provided set of
traces.  In essence, a clustering algorithm identifies a {\em
labeling\/} function $\labelingfn: \TracesSubset \to \powerSet{L}$
assigning to each trace in $\TracesSubset$ zero or more labels. We
elaborate with the help of an example.

\begin{example}\label{ex:gmms} In Fig.~\ref{fig:gmms}, we show a
possible clustering induced by using Gaussian Mixture
Models\footnote{A GMM assumes that the given parameter space can be
  modeled as a random variable distributed according to a set of
  Gaussian distributions with different mean and variance values.  A
  given parameter valuation is labeled $l$ if the probability of the
  valuation belonging to the $l^{th}$ Gaussian distribution exceeds
  the probability of the valuation belonging to other
  distributions. Another way to visualize clusters in the parameter
  space is by level-sets of the probability density functions
  associated with the clusters.  For example, for the $l^{th}$
  cluster, we can represent it using the smallest level-set that
  includes all given points labeled $l$.}
 (GMMs) for the trace-projections for the traces in
Fig.~\ref{fig:introexample_traces}.  The figure shows that the traces
colored green, red and black are grouped in the same cluster; this
matches the observation that all three traces have behaviors
indicating overshoots, but of reasonable magnitudes. On the other
hand, traces colored magenta and yellow have no overshoot and are
grouped into a second cluster. The final cluster contains the blue and
cyan traces, both with a large overshoot.
\end{example}

Supposing the clustering algorithm reasonably groups traces satisfying
similar parameter valuations/logical formulas, one may ask: ``Can we
describe this group of traces in terms of an easily interpretable STL
formula?''  Using an ellipsoid to represent a cluster, unfortunately,
the answer is negative.

\begin{example}
For the cluster labeled $0$ in Fig.~\ref{fig:introexample_traces}, in
\eqref{eq:ellipsoidal_formula}, we show the formula describing the
ellipsoidal cluster. Here the $c_i$s are some constants.
\begin{equation}
\label{eq:ellipsoidal_formula}
\ev \left( \stepinput \wedge \ev_{[0,\tau]} \left( x - x_{ref} > a \right)
\right) \wedge \left({(c_1\tau - c_2)}^2 + {(c_3 a - c_4)}^2 < c_5^2\right)
\end{equation}
\end{example}

It is clear that formula~\eqref{eq:ellipsoidal_formula} is
inscrutable, and actually represents an infinite number of STL
formulas. In case of GMMs, we can at least have an abstract
description of clusters using ellipsoid shapes in the parameter-space.
If we use spectral clustering (as described in
Section~\ref{sec:intro}), the representation of a cluster in the
parameter space is even less obvious. To mitigate this problem, we
observe that the distance between points in $\paramspace$ is a
``good'' proxy for whether they receive the same label. Thus, another
way to define the labeling function $\labelingfn$, is via parameter
ranges.  We argue that the use of axis-aligned hyperboxes enclosing
points with the same labels is a useful approximation of the clusters,
particularly because as we see in the next section, it has a compact
STL encoding.

\begin{remark} For a given set of points, the tightest-enclosing
hyperbox may include points that would not have received the same
label by an off-the-shelf clustering algorithm.  This can lead to a
scenario where hyperbox-clusters intersect (see
Fig.~\ref{fig:experimental_overshoot_clusters} in for an
example). This means that we can now have points in the parameter
space that can have possibly two or more labels. We argue that this
can be addressed in two ways: (1) introduce a new hyperbox cluster for
points in the intersection, (2) indicate that points in the
intersection represent traces for which there is additional guidance
required from the designer.  \end{remark}

\mypara{Hyperbox Clusters}
In the previous section, we showed that we can construct a labeling
function $\labelingfn$ to assign labels to the user-provided set of
traces $\TracesSubset$. We now see how we can extend this labeling
function to all possible traces $\Traces$.

Let $\val_{\top} \eqdef \sup \paramspace$. A valid hyperbox $B$ in the
parameter space is defined in terms of its extreme points
$(\val_s(\p), \val_w(\p))$, (where $\val_s \pOrder \val_w$), where
$\val_s$ and $\val_w$ are the infimum and supremum resp.\ over the box
w.r.t. $\pOrder$.  Formally,

\begin{definition}[Hyperbox]
\label{def:hyperbox}
\begin{equation}
B(\val_s, \val_w)  \eqdef 
\left\{
\begin{array}{lll} 
      \prod_{i} [\val_s(p_i), \val_w(p_i)] & \text{if $\val_w(p_i) \neq \val_\top(p_i)$} \\
      \prod_{i} [\val_s(p_i), \val_w(p_i)] & \text{otherwise.}
      \end{array}
\right.
\end{equation}

\end{definition}

In other words, we assume that a hyperbox is open on all faces not
connected to the infimum of the box, unless the face is connected to
the supremum of $\paramspace$. Let $\aaboxes$ denote the set of all
such hyperboxes.

\begin{definition}[Hyperbox Labeling Function]
Given a trace $x$ and a hyperbox $B$, s.t. $\pi(x) \in B$, we define
$\labelingfnbox$ as the hyperbox labeling function from $\Traces$ to
$\powerSet{L}$ as follows:
\begin{equation}
\label{eq:hyperbox_labeling}
l \in \labelingfnbox(x) \iff \{\pi(x') \mid x' \in \TracesSubset\ \wedge \labelingfn(x') = l\} \subset B
\end{equation}
\end{definition}

In other words, we only consider hyperboxes that contain the
projections of all traces with a specific label (say $l$), and then
any trace that projects to some point in the hyperbox gets all such
labels $l$.  We extend the definition $\labelingfnbox(x)$ to boxes,
such that $\labelingfn_\aaboxes(B) = \setof{l \mid \pi(x) \in B \wedge
l \in \labelingfnbox(x)}$.  We note that $B^{*} \eqdef \inf \{B \mid l
\in \labelingfn_\aaboxes(B)\}$ represents the smallest set containing
all parameter valuations that are labeled $l$. However,  $B^{*}$ does
not satisfy the definition of a hyperbox as per
Definition~\ref{def:hyperbox} as it is a closed set.  Hence, we define
an $\epsilon$ relaxation of this set as the smallest bounding hyperbox
satisfying Definition~\ref{def:hyperbox} at Hausdorff distance
$\epsilon$ from $B^{*}$, and call it the {\em $\epsilon$-bounding
hyperbox}. In the next section, we show how we can translate a cluster
represented as an $\epsilon$-bounding hyperbox to an STL formula. We
will further examine how, in some cases, we can represent a cluster by
a superset $B'$ of the $\epsilon$-bounding hyperbox that satisfies $l
\in \labelingfn_\aaboxes(B')$, but allows a simpler STL
representation.

\begin{example} For the example shown in
Fig.~\ref{fig:introexample_traces}, for each of the red, green and
black traces $x$,  $\labelingfnbox(x) = \setof{0}$, while for the blue
and cyan traces, $\labelingfnbox(x) = \setof{2}$. Any hyperbox $B$
satisfying Definition~\ref{def:hyperbox} that is a superset of the
hyperbox enclosing the red, green and black points shown in the bottom
right figure has $\labelingfn_\aaboxes(B) = \setof{0}$, while the
hyperbox shown in the figure is an $\epsilon$-bounding hyperbox.
\end{example}

\section{Learning STL Formulas from Clusters}\label{sec:stl_extract}

A given $\epsilon$-bounding hyperbox $B$ simply specifies a range of
valuations for the parameters in a PSTL template $\f$.  We now
demonstrate that because $\f$ is monotonic, there exists a simple STL
formula that is satisfied by the set of traces that project to some
valuation in $B$.  Recall that we use $\model{\val}$ to denote the set
of traces that satisfies $\f(\val(\p))$. We define $\Traces_B$ as the
set of traces that have a satisfying valuation in $B$: $\Traces_B
\eqdef \displaystyle\bigcup_{\val(\p) \in B} \model{\val(\p)}$

\begin{theorem}\label{thm:stl_extract}
There is an STL formula $\psi_B$ such that 
$\{x \in \Traces \mid x \models \psi_B\} \equiv \Traces_B$.
\end{theorem}

Before proving this theorem, we introduce some notation:
\begin{definition}[Essential Corners, $E_B$]\label{def:E_B}
Let $\val_w(\p) = (w_1,\ldots,w_n)$, and let $\val_s(\p) =
(s_1,\ldots,s_n)$. A valuation corresponding to an essential corner
has exactly one $i$ such that $\val(p_i) = s_i$, and for all $j \neq
i$, $\val(p_j) = w_j$. 
\end{definition}

\begin{proof}[Theorem~\ref{thm:stl_extract}]
We first introduce the notion of essential corners of a box $B$. 
 
Note that $B$ can be written in terms of downward closures of
valuations: $B = \downward(\val_w) \cap \bigcap_{\val \in E_B}
\overline{\downward(\val)}$.  From
Lemma~\ref{lemma:stl_to_downward_closure}, the set of traces
satisfying a formula in $\f(\downward(\val))$ is equivalent to
$\model{\val}$.  Further, using the equivalence between intersections
($\cap$) of sets of traces and conjunctions ($\wedge$) in STL, and
equivalence of set-complements with negations, we define $\psi_B$
below and note that the set of traces satisfying the formula $\psi_B$
below is $X_B$. $\blacksquare$
\begin{equation}\label{eq:psi(B)}
    \psi_B \eqdef \f({\val_w}) \wedge \bigwedge_{\val \in E_B} \neg \f(\val)
\end{equation}

\end{proof}

\begin{example}\label{ex:overshoot_stl}
Consider the $B \in \aaboxes$ enclosing the projections for the yellow and
magenta traces (Cluster 1). The corner-points of the cluster in
clockwise order from bottom right corner are: $(-1.3,0.1)$,
$(-1.3,3.3)$, $(-1.1,3.3)$, $(-1.1,0.1)$. Observe that as
$\polarity(a) = -$ and $\polarity(\tau) = +$, $\val_s$ = $(a \mapsto
-1.1, \tau \mapsto 0.1)$, $\val_w$ = $(a \mapsto -1.3, \tau \mapsto
3.3)$. Thus, $E_B$ = $\{(-1.3, 0.1), (-1.1, 3.3)\}$. Thus:
\begin{equation}
\label{eq:psiB}
\begin{array}{l}
\overshootSTL(a,\tau)  \equiv 
\ev \left(\stepinput \wedge \ev_{[0,\tau]} \left(x - x_{ref} > a\right)\right) \\
\psi_B  \equiv 
\overshootSTL(-1.3, 3.3) \wedge \neg \overshootSTL(-1.3,.1) \wedge \neg \overshootSTL(-1.1, 3.3) 
\end{array}
\end{equation}
\end{example}

\begin{lemma}$|\psi_B| \leq (|\params| +1|)(|\f| + 2)$\end{lemma}
\begin{proof}
  Recall from Def~\ref{def:E_B} that corners in $E_B$ have exactly 1
  param set to $s_i$. There are $|\params|$ params, thus by pigeon
  hole principle, $|E_B| \leq |\params|$. In $\psi_B$ for each corner
  in $E_B$, the corresponding formula is negated, adding 1
  symbol. Between each $|\params| + 1$ instantiations of $\f$ is a
  $\wedge$. Thus $|\psi_B| \leq (|\params| + 1)(|\f| + 2)$
  $\blacksquare$
\end{proof}

\mypara{Simplifying STL representation}
To motivate this section, let us re-examine
Example~\ref{ex:overshoot_stl}. From
Fig.~\ref{fig:introexample_traces}, we can observe that there is no
hyperbox cluster to the left of or above the chosen hyperbox cluster
$B$, \ie\ the one containing the magenta and yellow
trace-projection. What if we consider supersets of $B$ that are
hyperboxes and have the same infimum point?  For
Example~\ref{ex:overshoot_stl}, we can see that any hyperbox that
extends to the supremum of the parameter space in $\tau$ or $-a$
direction would be acceptable as an enclosure for the yellow and
magenta traces (as there are no other traces in those directions). We
formalize this intuition in terms of relaxing the set of corners that
can appear in $E_B$.

\begin{wrapfigure}{r}{.5\textwidth}
\label{fig:possible_shapes}
{\includegraphics[width=.99\textwidth]{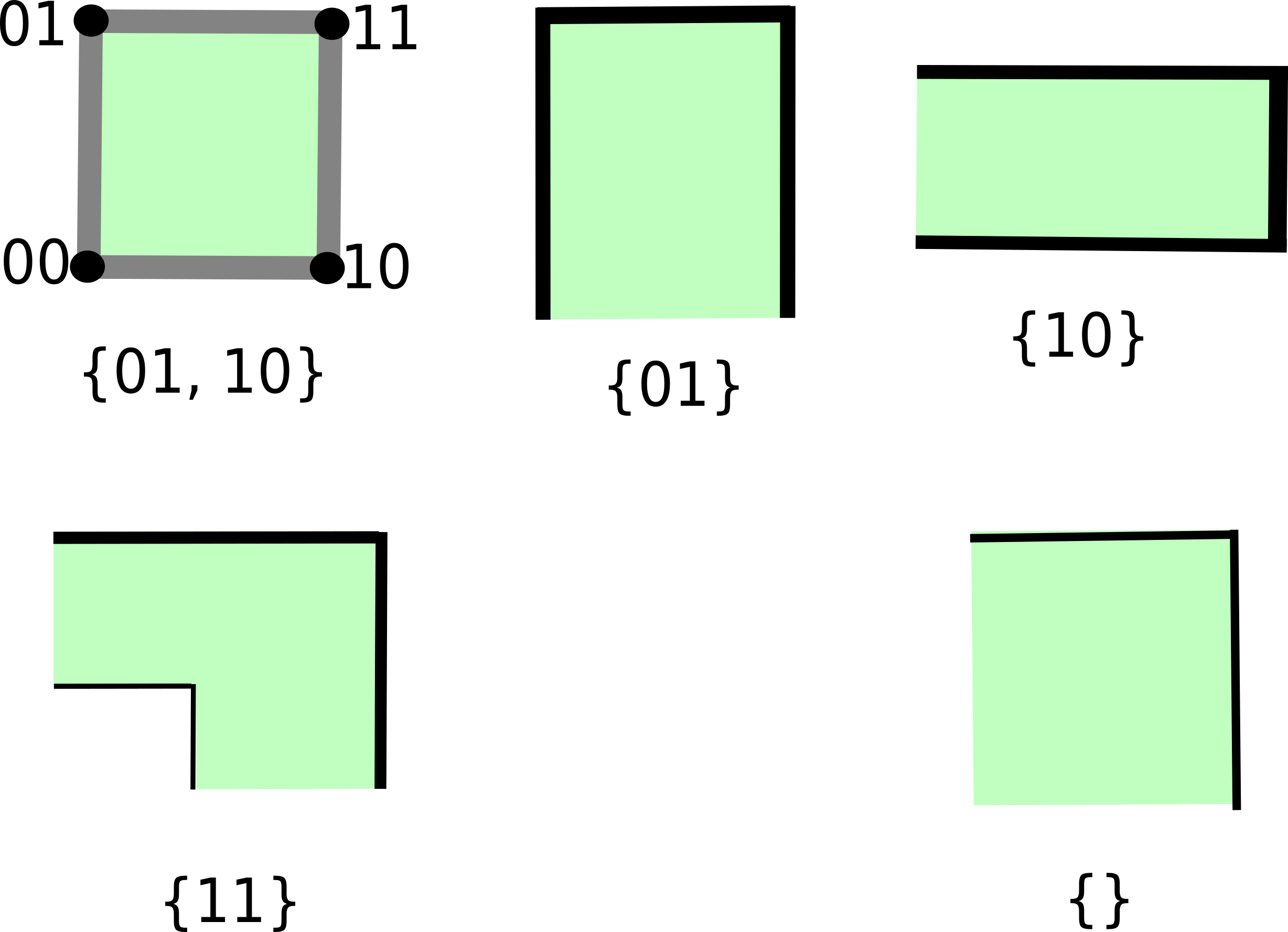}
  \caption{2D shapes generated by different subsets of
    corners.\label{fig:possible_shapes}}
}
\end{wrapfigure}
For instance, suppose that we replace $E_B$ in Eq~\eqref{eq:psi(B)} with
$E'_B$, where $E'_B$ is any subset of the corners of $B$ (excluding
$\val_w$).  We call the collection of shapes induced by this
relaxation as $\aaboxes_2$. For $|\params| = 2$, the possible shapes
of elements in $\aaboxes_2$ are shown in
Fig.~\ref{fig:possible_shapes}. For convenience, we use a bit-vector
encoding for hyperbox corners, where $\val_s$ corresponds to the
bit-vector with all $0$s, $\val_w$ has all $1$s, and essential corners
are bit-vectors with exactly one $0$.  Consider the L shaped region,
$C_L$, created by $E'_B = \{00\}$. The formula corresponding to $C_L$
has obviously less descriptive complexity than $E_B = \{01,
10\}$. Further notice, $\aaboxes2 \setminus \aaboxes$ would have less
descriptive complexity than elements of $\aaboxes$.

One critical feature that $\aaboxes_2$ (and thus $\aaboxes$) has is
{\em comparable convexity\/}:
\begin{definition}[Comparable Convexity]
If $\forall \val, \val' \in B \subset \paramspace$ if $\val \pOrder
\val'$ or $\val' \pOrder \val$ then all {\em convex\/} combinations of $\val$
and $\val'$ are in $B$.
\end{definition}

Comparable convexity allows us to argue that one can gain some insight
into the set of traces by just examining the extremal cases and just
``interpolating'' the cases because of the associated
parameters. We call these extremal cases the ``representatives'' of a
cluster.

\begin{theorem} Each element in $\aaboxes_2$ is
comparably convex. See Fig.~\ref{fig:possible_shapes} for examples.
\end{theorem}
\begin{proof}
  Note that all elements of $\aaboxes \subset \aaboxes_2$ are
trivially comparably convex since hyperboxes are convex. Thus we focus
on elements of $\aaboxes_2 \setminus \aaboxes$. Now observe that any
element $C \in \aaboxes_2$ is the union of a finite set,
$\mathcal{H}$, of boxes in $\aaboxes$. $C$ s.t. $C = \bigcup_{B_i \in
\mathcal{H} \subset \aaboxes} B_i$ where $\val_1 \in B_1$ and $\val_2
\in B_2$.  If $B_1 \subset B_2$ or the other way around or $\val_1\in
B_1 \cap B_2$ or $\val_2 \in B_1 \cap B_2$, then again there trivially
the convex combination of $\val_1$ and $\val_2$ is in $C$ because
hyperboxes are convex (and the intersection of two hyperboxes is a
hyperbox).

This leaves the case where $\val_1 \in B_1 \setminus B_2$ and $\val_2
\in B_2 \setminus B_1$ and neither $B_1 \subset B_2$ nor $B_2 \subset
B_1$. This implies that $\inf(B_1)$ is not comparable to $\inf(B_2)$.
W.L.O.G assume $\val_1 \pOrder \val_2$ and that the convex combination
of $\val_1$ and $\val_2$ is not a subset $C$.  Note that the
definition of downward closure and the fact that $\val_1 \pOrder
\val_2 \implies \val_1 \in B(\sup(C), \val_2) \eqdef B'$. But, $B'$ is
convex and $B' \subset B_2 \subset C$, thus the convex combination of
$\val_1$ and $\val_2$ is in $C$ which is a contradiction.
$\blacksquare$ \end{proof}

\section{Case Studies}\label{sec:case_studies}

\mypara{Implementation Details} We leveraged Breach~\cite{breach} for
performing projections $\projlex$, $\mathtt{scikit}$-$\mathtt{learn}$
toolkit~\cite{scikit-learn} for clustering and custom Python code for
learning STL formulas for clusters. An IPython notebook with
compressed versions of the datasets studied in the case studies (and a
replementation of $\projlex$ in Python) is available for download
at~\cite{ReproducabilityArtifact}.

\mypara{Diesel Engine} In this case study, we are provided with timed
traces for a signal representing the Exhaust Gas Recirculation (EGR)
rate for an early prototype of a Diesel Engine airpath controller. As
the example comes from an automotive setting, we suppress actual
signal values for proprietary reasons.  The controller computes an EGR
reference rate and attempts to track it. Typically, engineers visually
inspect the step-response of the control system and look for patterns
such as unusual overshoots, slow responsiveness, {\em etc}. The ST-Lib
library~\cite{stlib} defines a pool of PSTL formulas designed to
detect violations of such properties.  Using a property from ST-Lib
requires correctly setting the parameters in the PSTL templates
therein.  In this case study, we show how we can use our technique to
determine parameters that characterize undesirable behavior. We focus
on two templates: Rising Step and Overshoot.  Many ST-Lib formulas are
``step-triggered'', $\ie$ they are of the form: $\F(\stepStl \wedge
\phi)$ We first identify parameters for the step template, as it is
used as a primitive in further analysis. For example, in the overshoot
analysis we seek to characterize by what margin traces overshoot the
reference. We use the following templates for \textbf{rising}-step and
overshoot:

\begin{align}
  \stepStl_{(m, w)} & \triangleq \F(\ddot{x} > m \wedge \F_{[0,
      w]}(\ddot{x} < -m))\label{eq:step_PSTL}\\
  \overshootSTL(c, w) & \triangleq \stepStl^* \wedge \F_{[0,w]} (x - x_r)
  > c\label{eq:overshoot_PSTL}
\end{align}

Eq~\eqref{eq:step_PSTL} first reduces step detection (via a discrete
derivative) to spike detection and then applies the ST-Lib spike
detection template (that introduces a second derivative). As the view
of PSTL is signal-centric, such operations can be introduced as new
timed traces of a new signal, and do not require any modification of
the logic. The $\stepStl^*$ that appears in Eq~\eqref{eq:overshoot_PSTL}
is result step primitive we learn during our analysis. Finally, the
lexicographic ordering used in the projections of Eq~\eqref{eq:step_PSTL}
and Eq~\eqref{eq:overshoot_PSTL} are: $m \lexOrder w$ and $c \lexOrder w$
resp. Finally, each parameter is in $\Reals_{>0}$.

\myipara{Experiments} 
We have 33 traces of variable time-length. As a preprocessing step, we
used a sliding window with a size of $1$ second and a sliding offset
of $0.5$ seconds to generate equal length traces.  The sliding window
size and the offset was chosen by observation and experience to
capture the significant local behaviors. In general, such a selection
could be automated based on statistical criteria.  Further, as we did
not exploit the relationship between traces generated by the sliding
window, we effectively analyzed over $2\times 10^6$ traces
(~1GB). Each trace generated is then prepossessed by numerically
computing the second derivative\footnote{As the discrete-time
derivative can introduce considerable noise, we remark that the
discrete-time derivative can often be approximated by a noise-robust
operation (such as the difference from a rolling mean/median.)}. After
projecting to the parameter space for each template, we normalize the
parameters to lie between $0,1$ and fit a Gaussian Mixture Model to
generate labels, and learn the STL formulas for each cluster.

\begin{figure}[t]
  \centering
  \subfloat[$\paramspace$ of $\stepStl_{(m, w)}$\label{fig:step_clusters}]{\includegraphics[width=0.3\textwidth]{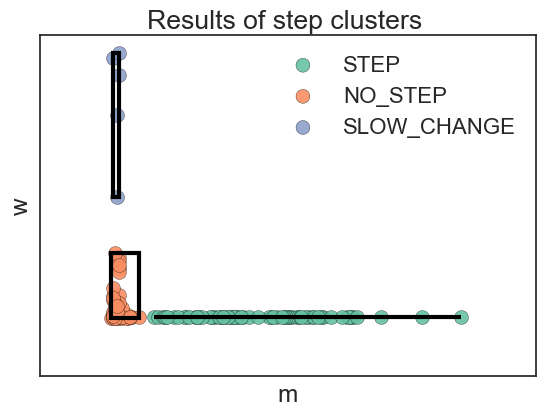}}
  \subfloat[Step Representatives\label{fig:step_representatives}]{\includegraphics[width=0.33\textwidth]{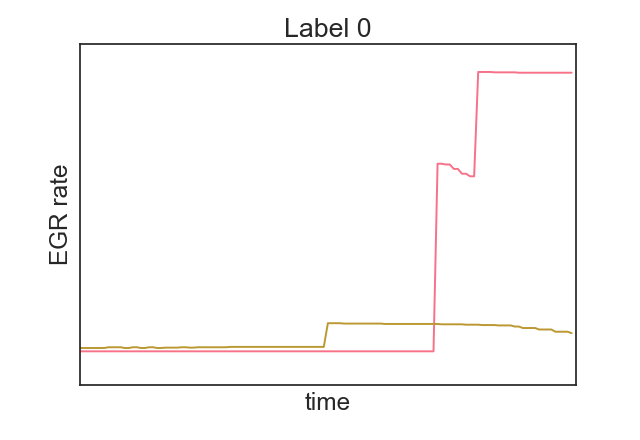}}
  \subfloat[$\paramspace$ of $\overshootSTL$\label{fig:experimental_overshoot_clusters}]{\includegraphics[width=0.3\textwidth]{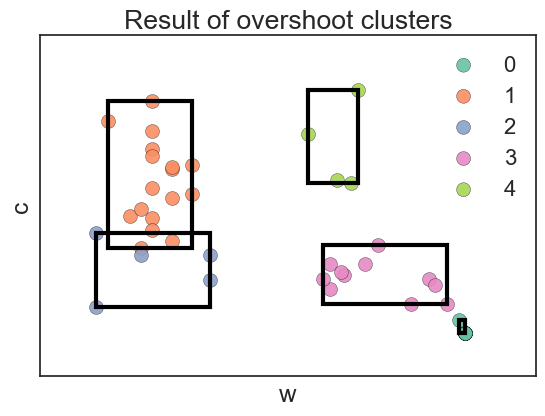}}
  \hfill
  \subfloat[No Overshoot Reps]{\includegraphics[width=0.33\textwidth]{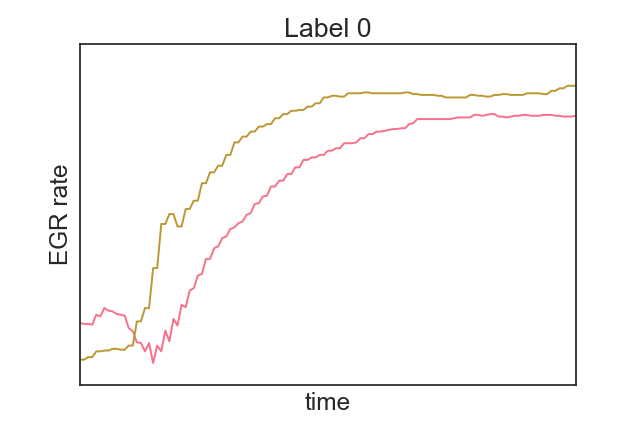}}
  \subfloat[Big Overshoot Reps]{\includegraphics[width=0.33\textwidth]{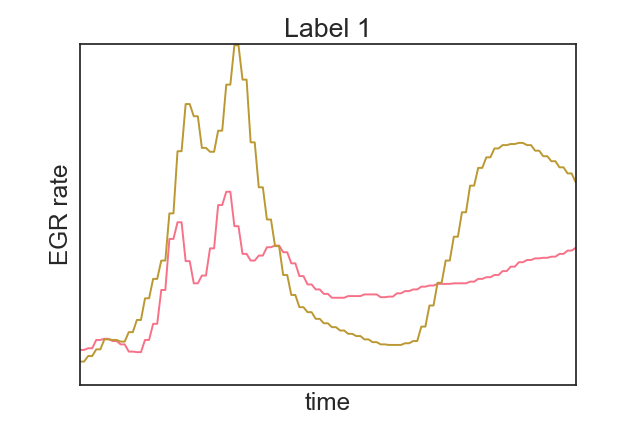}}
  \subfloat[Little Overshoot Reps]{\includegraphics[width=0.33\textwidth]{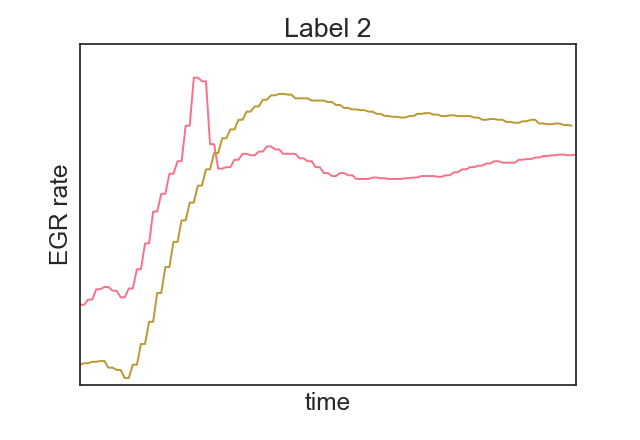}}\label{fig:overshoot_example}
  \vspace{-5pt}
  \caption{$\paramspace$ for the overshoot and step experiments and
    representatives of select clusters.}
\end{figure}

\myipara{Results}
The Step template revealed 3 clusters (Fig~\ref{fig:step_clusters}), of which the cluster labeled Step
(Fig~\ref{fig:step_representatives}), was identified as an admissible
``step'' primitive. In picking the appropriate bounding box in
$\aaboxes$ we noted spikes have no inherit upper limit to their
peaks. Thus, we derived the characterizing STL:\@ $\stepStl^* \eqdef \stepStl(m^*, w^*)$, where $m^*, w^*$ are suppressed
for proprietary reasons. The overshoot analysis revealed 5 clusters.
We note that there are actually 2 distinct clusters which can
reasonably be called overshoots, given by label 1 and 4 in Fig~\ref{fig:experimental_overshoot_clusters}. The interpretation, is that
while the majority of the overshoots occur soon after the step, there
is a cluster that occurred later, either due to slow rise time or
non-linear effects causing the oscillation about the reference to
increasing before dying away. In either case, as with spike, we
declare that any overshoot is still an overshoot as the amplitude $c$
rises. Thus for cluster 1 we again chose to use a box from $\aaboxes$
that does not bound $c$. This lead to:
$\overshootSTL^* \eqdef \overshootSTL(c^*, w_2^*) \wedge \neg
\overshootSTL(c^*, w_1^*)$ again suppressing values.

\mypara{Traffic Behavior on the US-101 highway}
In order to model and predict driver behavior on highways, the Federal
Highway Administration collected detailed traffic data on southbound
US-101 freeway, in Los Angeles
\cite{NGSIM}. The pre-selected segment of the freeway is about 640
meters in length and consists of five main lanes and some auxiliary
lanes. Traffic through the segment was monitored and recorded through
eight synchronized cameras, mounted on top of the buildings next to
the freeway. A total of $45$ minutes of traffic data was recorded
including vehicle trajectory data providing lane positions of each
vehicle within the study area.

\newcommand{\cItem}[1] {\multicolumn{1}{c}{#1}}

\floatsetup{heightadjust=object}
\begin{figure}[t]
\centering

\begin{floatrow}
\subfloat[Lane Clusters\label{tbl:lane_clusters}]{\small
  \begin{tabular*}{0.39\textwidth}{@{\extracolsep{\fill}}crrrcr}
    \toprule C & \cItem{$\tau_2$} & \cItem{$\tau_3$} &
    \cItem{$\tau_4$} & T & |C| \\ 
    \midrule 
    0 & $\top$ & $\top$ & $\top$ & T1 & 626\\ 
    1 & $\top$ & $\top$ & $[\bot, 78]$ & T2 & 115\\ 
    2 & $[\bot, 67]$ & $\top$ & $\top$ & T2 & 44\\ 
    3 & $\top$ & [.26, 30] & $[\bot, 65]$ & T3 & 52\\ 
    4 & $\top$ & $[\bot, 70]$ & $\top$ & T2 & 32\\ 
    5 & [.56, 54] & $[\bot, 40]$ & $\top$ & T3 & 14\\ 
    6 & [.56, 32] & [1.8, 24] & $[\bot, 31]$ & T4 & 12\\ 
    7 & 76 & $\top$ & $\top$ & T1 & 1\\ 
    \bottomrule
  \end{tabular*}
}

\hspace{-3pt}
\subfloat[Examples from lane cluster C$=1$\label{Fig:cluster1}]{\includegraphics[width=0.29\textwidth, valign=c]{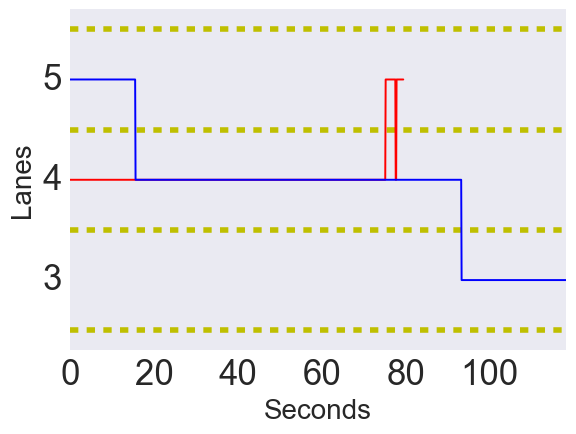}}
\subfloat[Aggressive Behavior\label{Fig:aggressive}]{\includegraphics[width=0.29\textwidth, valign=c]{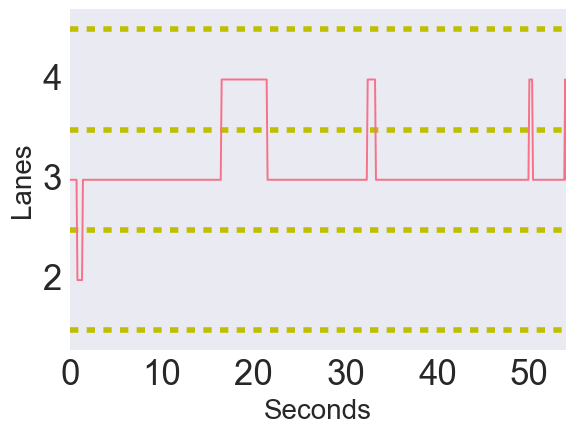}}
\end{floatrow}

\caption{(a) Lane Switching
  Behavior results. Columns: Cluster, parameters, cluster type,
  num$($traces$)$. (b) shows the representatives of cluster 1, which upon
  the inspection are qualitatively very different. The blue car moves
  from lane 5 to lane 4, remains for $\approx$60s and then moves to lane
  3. The red car appears to use lane 5 to pass another car, move's
  back into lane 4 and then to lane 3 shortly after. Inspecting the
  data, most of cluster 1 large $\tau_4$ value. We subdivided the
  behavior further using a one class svm and interpreted the small
  $\tau_4$ values as ``aggressive''. New ``aggressive''
  representatives given in (c).}
\end{figure}

Here, we apply our method to analyze lane switching ``aggressiveness''
characterized by how often a driver switches lanes and the dwell time
in each lane before switching. We focus on lanes 2, 3, and 4, ignoring
the outer lanes lanes 1 and 5 since they are used entering and exiting
the freeway, and thus have qualitatively different behavior.  Each
vehicle trajectory $x(t)$, stores the lane position for the vehicle,
and we use the following STL formula to capture the dwell time in
$\lanei$:
\begin{equation}
\label{eq:lane}
\F \left( x \neq \lanei\wedge (\F_{[0, \epsilon]} x= \lanei
\U_{[\epsilon, \tau_i]} x\neq \lanei) \right)
\end{equation}

\myipara{Results} For this experiment, from 4824 total vehicle
trajectories, we discard trajectories with no lane switching behavior
and group them with the conservative driving behaviors.  We analyze
the remaining 896 targeted trajectories that have at least one
lane-switch behavior, and each trajectory is at most 100 seconds long.
As all parameters are independent, lexicographic ordering has no
impact on $\projlex$.  After normalizing the parameters by centering
and scaling, we apply GMMs to label and generate bounding
hyperboxes/STL formulas.

The resulting clusters are shown in Table~\ref{tbl:lane_clusters}.
Upon examining the representatives,we classified the behaviors of each
cluster into 4 groups: 

\begin{itemize}
\item T1: No Weaving: only switching to adjacent
lanes and never changing back.
\item T2: Normal driving behavior, from
switching to adjacent lanes and coming back to overtake a slow vehicle
in front
\item T3: Slightly aggressive behavior, weaving between $2$ lanes.
\item T4: Aggressive behavior, weaving between all three lanes.
\end{itemize}

The largest cluster, 0, contains behaviors without any weaving
behavior. Cluster 3 and 5 represent the weaving behavior involving 2
lanes. Cluster 6 represents aggressive behavior and one of the
representative is shown in Fig.~\ref{Fig:aggressive}. We consider
Cluster 7 as an anomaly for Cluster 2 as it has only 1 trajectory.

For clusters 1, 2, and 4, we cannot distinguish if drivers were
rapidly weaving or weaving within a short period of time, due to the
scarcity of the data. As seen in Fig.~\ref{Fig:cluster1}, the
representatives for cluster $1$, demonstrated two different behaviors;
one involving rapid lane-switching (red trace), one where the driver
switched lanes more slowly (blue trace). Applying an additional
1-class SVM to the points in cluster $1$ was used to distinguish these
two cases.

\mypara{CPS Grader} Massively Open Online Courses (MOOCs) present
instructors the opportunity to learn from a large amount of collected
data.  For example, the data could be clustered to identify common
correct solutions and mistakes. Juniwal et
al.~\cite{juniwal2014cpsgrader} demonstrated a semi-supervised
procedure for a CPS MOOC;\@ this involved first using DTW and K-Nearest
Neighbors (KNN) to cluster traces of student solutions, and then
picking representatives from clusters to ask the instructor to
label. From the labeled data, they extract a characterizing STL
formula given a PSTL template.  The techniques demonstrated in this
paper offer an alternative approach that can overcome some limitations
of~\cite{juniwal2014cpsgrader}. Firstly, as demonstrated in the
opening example (see Fig.~\ref{fig:introexample_traces}), DTW does not
necessarily group traces in a way consistent with their logical
classification.  Second, the burden of labeling traces can still be
quite large for instructors if the number of clusters is very large.
Instead, 
unsupervised our approach offers a fully unsupervised approach (e.g.,
based on GMMs or K-Means) which still offers some degree of confidence
that elements in the same cluster are similar w.r.t.\ a given PSTL
template.

The tests in~\cite{juniwal2014cpsgrader} involved the simulation of an
IRobot Create and student generated controllers. The controller needed
to navigate the robot up an incline and around static obstacles. To
test this, the authors created a series of parameterized environments
and a set of PSTL formula that characterized failure. In this work, we
attempt to reproduce a somewhat arbitrary subset of the results shown
in~\cite{juniwal2014cpsgrader} that required no additional
preprocessing on our part.

\myipara{Obstacle Avoidance} We focus on 2 tests centered around
obstacle avoidance. The authors used an environment where an obstacle
is placed in front of a moving robot and the robot is expected to
bypass the obstacle and reorient to it's pre-collision orientation
before continuing. The relevant PSTL formulas were ``Failing simple
obstacle avoidance'' and ``Failing re-orienting after obstacle
avoidance'' given below as $\f_{avoid}$ and $\f_{reorient}$ resp.:
\begin{eqnarray}
  &\f_{avoid}(\tau, y_{\min}) = \G_{[0, \tau]}(pos.y < y_{\min})\\
  &\f_{reorient}(y_{\min}, x_{\max}) = \G_{[0, \tau]}(pos.y < y_{\min} \vee pos.x > x_{\max})
\end{eqnarray}

\begin{figure}[t]
  \centering
  \subfloat[$\paramspace$ of $\f_{avoid}$\label{fig:obstacleDist}]{\includegraphics[width=0.42\textwidth]{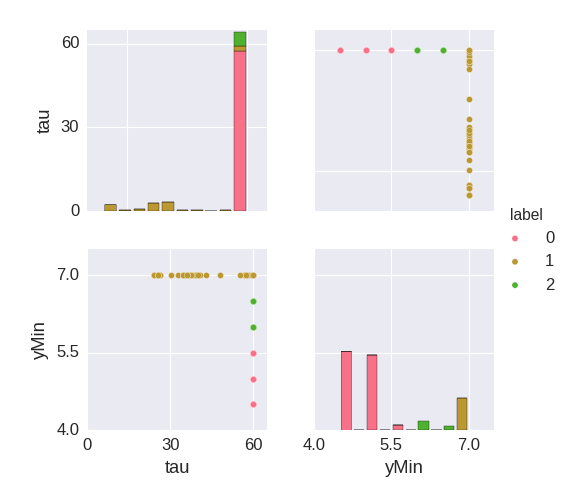}}
  \subfloat[$\paramspace$ of $\f_{reorient}$\label{fig:reorient}]{\includegraphics[width=0.42\textwidth]{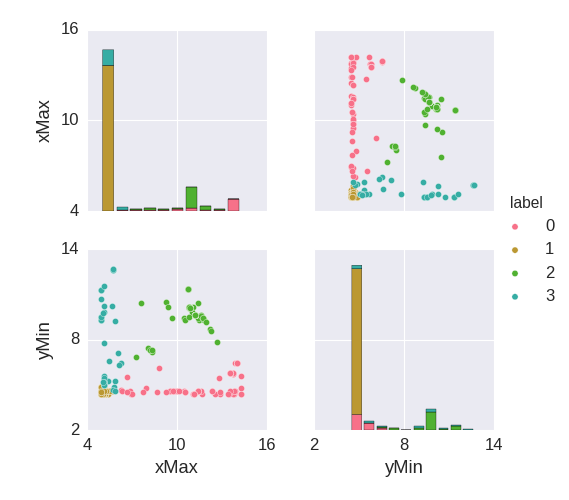}}
  \subfloat[$\validitydomain{\f_{avoid}}$\label{fig:obstacleValidityDomain}]{\includegraphics[width=0.16\textwidth]{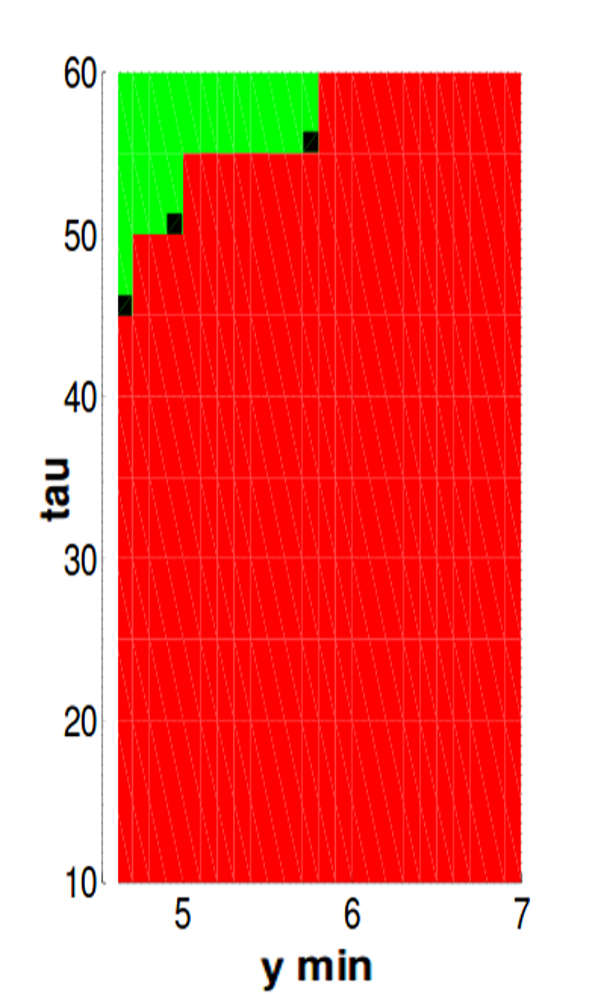}}
  \caption{CPS Grader Study Results, w.
    $\validitydomain{\f_{avoid}}$, (c),
    from~\cite{juniwal2014cpsgrader} included for comparison
    (valuations in the green region
    of (c) correspond to mistakes). We
    note that in (b) we are able to identify 3 modes of
    failure (obstacle not avoided, 2x obstacle avoided but did not
    reorient), an insight not present in~\cite{juniwal2014cpsgrader}.}
\end{figure}

\myipara{Results} A surprising observation for both templates is that
the vast majority of data is captured in a relatively small parameter
range. Upon investigation, it was revealed that the students were able
to submit multiple solutions for grading --- each corresponding to a
trace. This biased the dataset towards incorrect solutions since one
expects the student to produce many incorrect solutions and then a few
final correct solutions. As seen in Fig.~\ref{fig:obstacleDist}, the
results imply that a classifier for label 0, which corresponded to the
robot not passing the obstacle, would have a low misclassification
rate when compared against the STL artifact from
~\cite{juniwal2014cpsgrader}. Moreover, for obstacle avoidance, there
are two other families of correct solutions uncovered. One is the set
of traces that just barely pass the obstacle in time (label 2 in
Fig.~\ref{fig:obstacleDist}), and the other is the spectrum of
traces that pass the minimum threshold with a healthy margin (label 1
in Fig.~\ref{fig:obstacleDist}). For the reorient template, we
discovered 3 general types of behaviors (again with GMMs), see Fig.~\ref{fig:reorient}. The first (label 1) is a failure to move past the
obstacle (echoing the large group under the obstacle avoidance
template). The other 3 groups seem to move passed the obstacle, but
two (labels 0 and 3) of them display failure to reorient to the
original orientation of $45^\circ$. One could leverage this behavior
to craft diagnostic feedback for these common cases.

\\ \\
\noindent{\bf Conclusion.} In this work we explored a technique
to leverage PSTL to extract features from a time series that can be
used to group together qualitatively similar traces under the lens of
a PSTL formula. Our approach produced a simple STL formula for each
cluster, which along with the extremal cases,
enable one to develop insights into a set of traces. We then
illustrated with a number of case studies how this technique could be
used and the kinds of insights it can develop.  For future work, we
will study extensions of this approach to supervised, semi-supervised,
and active learning. A key missing component in this work is a
principled way to select a projection function (perhaps via learning
or posterior methods). Other possible extensions involve integration
with systematic PSTL enumeration, and learning non-monotonic PSTL
formulas.

\bibliographystyle{splncs03}
\bibliography{refs}

\section{Appendix}\label{sec:appendix}

\mypara{Related monotonic formula for a non-monotonic PSTL formula}

\begin{example}
Consider the PSTL template: $\varphi(h,w) \eqdef \F\left((x \leq h)
\wedge \F_{[0,w]}\left(x \geq h\right)\right)$.
We first show that the given formula is not monotonic.
\begin{proof}
  Consider the trace $x(t) = 0$. Keep fixed $w = 1$. Observe that
  $h=0$, $x(t)$ satisfies the formula. If $h=-1$, then $x(t) \nmodels
  \f(h, w)$, since $x(t) \leq h$ is not eventually satisifed. If
  $h=1$, then $x(0) \leq 1$ implying that for satisfaction, within the
  next 1 time units, the signal must becomes greater than 1. The
  signal is always 0, so at $h=1$, the formula is unsatisfied. Thus,
  while increasing $h$ from $-1$ to $0$ to $1$, the satifaction has
  changed signs twice. Thus, $\f(h, w)$ is not monotonic.
\end{proof}

Now consider the following related PSTL formula in which repeated
instances of the parameter $h$ are replaced by distinct parameters
$h_1$ and $h_2$. We observe that this formula is trivially monotonic:
$\varphi((w,h_1,h_2)) \eqdef \F\left((x \leq h_1) \wedge
\F_{[0,w]}\left(x \geq h_2\right)\right)$

\end{example}

\mypara{Linearization based on scalarization}

Borrowing a common trick from multi-objective optimization, we
define a cost function on the space of valuations as follows:
$\cost(\val(\p)) =  \sum_{i=1}^{|\params|} \lambda_i \val(p_i)$.
Here, $\lambda_i \in \Reals$, are weights on each parameter.  The
above cost function implicitly defines an order $\scalarOrder$, where,
$\val(\p) \scalarOrder \val'(\p)$ iff $\cost(\val(\p)) \le
\cost(\val'(\p))$.  Then, the projection operation $\projscalar$ is
defined as:
$\projscalar(\x) = \argmin_{\val(\p) \in
  \validitydomainboundary{\f(\p)}} \cost(\val(\p))$.
We postpone any discussion of how to choose such a scalarization to
future work.

\end{document}